\documentclass{article}

\PassOptionsToPackage{numbers, compress}{natbib}
\usepackage[final]{neurips_2022}

\usepackage[intlimits]{amsmath}
\usepackage{mathtools}
\usepackage{amsmath}
\usepackage{dsfont}
\usepackage{stmaryrd}
\usepackage{comment}
\usepackage{setspace}

\usepackage[unicode,psdextra]{hyperref}
\newcommand\pdfmath[1]{\texorpdfstring{$#1$}{#1}}
\usepackage[capitalize,noabbrev]{cleveref}
\newcommand{\bc}[1]{\left\{{#1}\right\}}
\newcommand{\br}[1]{\left({#1}\right)}
\newcommand{\bs}[1]{\left[{#1}\right]}
\newcommand{\abs}[1]{\left| {#1} \right|}

\newcommand{\inner}[1]{\langle#1\rangle}
\newcommand{\norm}[1]{\left\| {#1} \right\|}
\renewcommand{\P}[1]{\mathbb{P}\bs{{#1}}}

\newcommand{\E}[1]{\mathbb{E}\bs{{#1}}}
\newcommand{\Ee}[2]{\underset{#1}{\mathbb{E}}\bs{{#2}}}

\DeclareMathOperator*{\argmin}{arg\,min}
\DeclareMathOperator*{\argmax}{arg\,max}

\usepackage{algorithm}
\usepackage{algorithmic}
\usepackage{amssymb}
\usepackage{amsfonts}
\usepackage{amsthm}
\usepackage{graphics}
\usepackage{graphicx}
\usepackage{setspace}
\usepackage{subcaption}
\newtheorem{theorem}{\protect\theoremname}
  \newtheorem{lemma}{\protect\lemmaname}
  \newtheorem{corollary}{\protect\corrolaryname}
  \newtheorem{defn}{\protect\definitionname}
  \newtheorem{proposition}{\protect\propositionname}
  
  \newtheorem{remark}{Remark}
  
    \newtheorem{assumption}{\protect\assumptionname}
    
    \newenvironment{proofof}[1]{\begin{proof}[{#1}]}{\end{proof}}
    
\providecommand{\definitionname}{Definition}
\providecommand{\examplename}{Example}
\providecommand{\lemmaname}{Lemma}
\providecommand{\corrolaryname}{Corollary}
\providecommand{\propositionname}{Proposition}
\providecommand{\conditionsname}{Conditions}
\providecommand{\theoremname}{Theorem}
\providecommand{\assumptionname}{Assumption}

\newcommand{\dv}{\mbf{d}}

\newcommand{\trans}{\intercal}
\newcommand{\qval}{\mbf{Q}}
\newcommand{\FEV}[1]{\mbs{\rho}_{\phim}(#1)}
\newcommand{\EFEV}[1]{\mbs{\rho}_{\phim}(\widehat{#1})}
\newcommand{\mcf}{\mathcal}
\newcommand{\mbf}{\mathbf}

\DeclareMathOperator{\Exp}{\mathds{E}} 
\DeclareMathOperator{\Prob}{\mathds{P}} 
 
\renewcommand{\Re}{\mathds{R}} 
\newcommand{\ar}{\mathds{R}} 

\newcommand{\sspace}{\mcf{S}}     
\newcommand{\aspace}{\mcf{A}}     
\newcommand{\wspace}{\mcf{W}}     

\newcommand{\normm}[1]{\lVert#1\rVert}               

\newcommand{\innerprod}[2]{\left\langle{#1},{#2}\right\rangle}

\newcommand{\expert}{{\pi_{\textup{E}}}}
\newcommand{\apprentice}{\pi_{\textup{A}}}
\newcommand{\mbs}{\boldsymbol}
\newcommand{\cost}{\mbf{c}}
\newcommand{\true}{\mbf{c_{\textup{true}}}}
\newcommand{\weight}{\mbf{w}}
\newcommand{\uv}{\mbf{u}}
\newcommand{\mv}{\mbs{\mu}}

\newcommand{\initial}{\mbs{\nu}_0}
\newcommand{\val}{\mbf{V}}
\newcommand{\thv}{\mbs{\theta}}
\newcommand{\lv}{\mbs{\lambda}}

\newcommand{\phim}{\mbs{\Phi}}

\newcommand{\phiv}{\mbs{\phi}}

\newcommand{\efevphi}{\mbs{\rho}_{\phim}(\widehat{\expert})}

\newcommand{\mexp}{\mv_{\expert}}
\newcommand{\xv}{\mbf{x}}
\newcommand{\yv}{\mbf{y}}
\newcommand{\zv}{\mbf{z}}

\newcommand{\wa}{\weight_{\textup{A}}}

\newcommand{\mmat}{\mbf{M}}
\newcommand{\bmat}{\mbf{B}}
\newcommand{\pmat}{\mbf{P}}

 \newcommand{\gv}{\mbf{g}}

\usepackage{xcolor}
\newcommand{\colorNoteBox}[3]{
	\begin{center}
		\vspace{-2ex}\small
		\fcolorbox[rgb]{#1}{#2}{\parbox[t]{\linewidth}{\setlength{\parskip}{1.5ex}%
				#3}}
	\end{center}}

\newcommand{\AK}[1]{\colorNoteBox{0,1,0}{0.9,0.9,0.9}{AK:~#1}}

\usepackage[utf8]{inputenc} 
\usepackage[T1]{fontenc}    
\usepackage{url}            
\usepackage{booktabs}       
\usepackage{amsfonts}       
\usepackage{nicefrac}       
\usepackage{microtype}      
\usepackage{xcolor}         

\title{Proximal Point Imitation Learning}

\author{%
  Luca Viano \\
  LIONS, EPFL\\
  Lausanne, Switzerland\\
  \texttt{luca.viano@epfl.ch} \\
  \And
  Angeliki Kamoutsi \\
  ETH Zurich \\
  Zurich, Switzerland\\
  \texttt{kamoutsa@ethz.ch} \\
  \AND
  Gergely Neu \\
  Universitat Pompeu Fabra \\
  Barcelona, Spain\\
  \texttt{gergely.neu@gmail.com} \\
  \And
  Igor Krawczuk \\
  LIONS, EPFL\\
  Lausanne, Switzerland\\
  \texttt{igor.krawczuk@epfl.ch} \\
  \And
  Volkan Cevher \\
  LIONS, EPFL \\
  Lausanne, Switzerland\\
  \texttt{volkan.cevher@epfl.ch} \\
}

\begin{document}

\maketitle

\begin{abstract}
This work develops new algorithms with rigorous efficiency guarantees for infinite horizon imitation learning (IL) with linear function approximation without restrictive coherence assumptions. We begin with the minimax formulation of the problem and then outline how to leverage classical tools from optimization, in particular, the proximal-point method (PPM) and dual smoothing, for online and offline IL, respectively. Thanks to PPM, we avoid nested policy evaluation and cost updates for online IL appearing in the prior literature. In particular, we do away with the conventional alternating updates
by the optimization of a single convex and smooth objective over both cost and $Q$-functions. When solved inexactly, we relate the optimization errors to the suboptimality of the recovered policy. As an added bonus, by re-interpreting PPM as dual smoothing with the expert policy as a center point, we also obtain an offline IL algorithm enjoying theoretical guarantees in terms of required expert trajectories. Finally, we achieve convincing empirical performance for both linear and neural network function approximation.
\end{abstract}

\section{Introduction}\label{introduction}
This work is concerned with the prototypical setting of imitation learning (IL) where
\begin{enumerate}
    \item An expert provides demonstrations of state-action pairs in an environment. The expert could be optimal or suboptimal with respect to an unknown cost/reward function.
    \item The learner chooses distance measure between its policy to be learned and the expert empirical  distribution estimated from demonstrations. 
    \item The learner employs an algorithm, which additionally may or may not use interactions with the environment, to minimize the chosen distance. 
\end{enumerate}


In IL, the central goal of the learner is to recover a policy competitive with expert with respect to the underlying unknown cost function. IL is important for several real world applications like driving \cite{Knox:2021}, robotics \cite{Osa:2018}, and economics/finance \cite{Charpentier:2020} at the expense of following resources: ({\sc R1}) expert demonstrations, ({\sc R2})  (optional) interactions  with the environment where the expert collected the demonstrations, and ({\sc R3}) computational resources for solving the problem template. 

Interestingly, while there is a vast amount of literature using optimization ideas on the IL problem template, i.e. Lagrangian duality \cite{Ho:2016,Fu:2018, Ke:2020, Kostrikov:2019, Kostrikov:2020}, resource guarantees are still widely missing since the optimization literature focuses on the resource ({\sc R3}) where IL literature mainly focuses on the first two resources ({\sc R1}) and ({\sc R2}). Our work leverages deeper connections between optimization tools and IL by showing how classical optimization tools can be applied in a linear programming formulation of IL problem guaranteeing efficiency in all ({\sc R1}), ({\sc R2}), ({\sc R3}). 
\textbf{Our contributions:}  
This work aims at designing an algorithm enjoying both theoretical guarantees and convincing empirical performance. Our methodology is rooted in classical optimization tools and the LP approach to MDPs. More precisely, the method uses the recently repopularized overparameterization technique to obtain the Q-function as a Lagrangian multiplier~\cite{Mehta:2020,Bas-Serrano:2021} and solves the associated program using a \textsc{PPM} update with appropriately chosen Bregman divergences.
This results to an actor-critic algorithm, with the key feature that the policy evaluation step involves optimization of a single concave and smooth objective over both cost and $Q$-functions.  In this way, we avoid instability or poor convergence due to adversarial training~\cite{Ho:2016,Zhang:2020,Liu:2022,Shani:2021}, and can also recover an explicit cost along with Q-function.  We further  account for potential optimization errors, presenting an error propagation analysis that leads to rigorous guarantees for both online and offline setting. For the context of linear MDPs~\cite{Bas-Serrano:2021, Yang:2019, Jin:2020, Cai:2020, Wang:2020b, Agarwal:2020b, Neu:2020}, we provide explicit convergence rates and error bounds for the suboptimality of the learned policy, under mild assumptions, significantly weaker than those found in the literature until now. 
To our knowledge, such guarantees in this setting are provided for the first time.
Finally, we demonstrate that our approach achieves convincing empirical performance for both linear and neural network function approximation.


\textbf{Related Literature.} 
The first algorithm addressing the imitation learning problem is behavioral cloning \cite{Pomerleau:1991}. Due to the covariate shift problem \cite{Ross:2010,Ross:2011}, it has low efficiency in terms of expert trajectories ({\sc R1}). To address this issue, \cite{Russell:1998, Ng:2000, Abbeel:2004, Ratliff:2006, Syed:2007, Neu:2007, Ziebart:2008, Abbeel:2008, Levine:2010, Levine:2011} proposed to cast the problem as inverse reinforcement learning (IRL). 
IRL improves the efficiency in terms of expert trajectories, at the cost of introducing the need of running reinforcement learning (RL) repetitively, which can be prohibitive in terms of environment samples ({\sc R2}) and computation ({\sc R3}). A successive line of work started with \cite{Syed:2008} highlights that repeated calls to an RL routine can be avoided. 
This work inspired generative adversarial imitation learning (GAIL) \cite{Ho:2016} and other follow-up works \cite{Fu:2018, Ke:2020, Kostrikov:2019, Kostrikov:2020} that leveraged optimization tools like primal-dual algorithms but did not try to deepen the optimization connections to derive efficiency guarantees in terms of all ({\sc R1}),({\sc R2}),({\sc R3}). Finally, a recent line of work~\cite{Garg:2021, Kalweit:2020} in IL bypasses the need of optimizing over cost functions and thus avoids instability due to adversarial training. Although these algorithms achieve impressive empirical performance in challenging high dimensional benchmark tasks, they are hampered by limited theoretical understanding. This is the fundamental difference from our work, which enjoys both favorable practical performance and strong theoretical guarantees.

Existing model-free IL theoretical papers with global convergence guarantees assume either a finite horizon episodic MDP setting~\cite{Liu:2022}, or tabular MDPs~\cite{Shani:2021}, or the infinite horizon case but with restrictive assumptions, such as linear quadratic regulator setting~\cite{Cai:2019}, continuous kernelized nonlinear regulator \cite{Chang:2021, Kakade:2020},  access to a generative model and coherence assumption on the choice of features~\cite{Kamoutsi:2021,Bas-Serrano:2021}, bounded strong concentrability coefficients~\cite{Zhang:2020} or a linear transition law that can be completely specified by a finite-dimensional matrix~\cite{Liu:2022}. On the other hand, we provide convergence guarantees and error bounds for the context of linear MDPs ~\cite{Bas-Serrano:2021, Yang:2019, Jin:2020, Cai:2020,  Wang:2020b, Agarwal:2020b, Neu:2020} under a mild \emph{feature excitation} condition assumption. Despite being linear, the transition law can still have infinite degrees of freedom. To our knowledge, such guarantees in this setting are provided for the first time.

Our work applies the technique known as regularization in the online learning literature \cite{Abernethy:2008, Shalev-Shwartz:2012} and Bregman proximal-point or smoothing in optimization literature \cite{Rockafellar:1976, Nesterov:2005} to the LP formulation for MDPs \cite{Manne:1960, DeGhellinck:1967, Denardo:1970, Borkar:1988, Hernandez-Lerma:1996, Hernandez-Lerma:1999, DeFarias:2003, DeFarias:2004, Schweitzer:1985, Petrik:2009, Petrik:2010, Abbasi-Yadkori:2014, Laksh:2018, Chen:2018, MohajerinEsfahani:2018, Wang:2019, Lee:2019a, Bas-Serrano:2020, Cheng:2020, Jin:2020, Shariff:2020}. From this perspective, we can see Deep Inverse Q-Learning~\cite{Kalweit:2020} and IQ-Learn~\cite{Garg:2021} that consider entropy regularization in the objective as smoothing using uniform distribution as center point. In our case, we instead use as center point the previous iteration of the algorithm (for the online case) or the expert (for the offline case). 

From the technical point of view, the most important related works are the analysis of  REPS/Q-REPS~\cite{Peters:2010, Bas-Serrano:2021, Pacchiano:2021} and O-REPS~\cite{Zimin:2013} that first pointed out the connection between REPS and PPM. We build on their techniques with some important differences. In particular, while in the LP formulation of RL, PPM and mirror descent \cite{Beck:2003, Hazan:2016} are equivalent, recognizing that they are \textit{not equivalent} in IL is critical for stronger empirical performance. As an independent interest, our techniques can be used to improve upon the best rate for REPS in the tabular setting \cite{Pacchiano:2021} and to extend the guarantees to linear MDPs.  In order to discuss in more detail  our research questions and situate them among prior related theoretical  and practical works, we provide in Appendix~\ref{app:related-literature} an extended literature review.
\section{Background}
\label{sec:background}
\subsection{Markov Decision Processes}\label{sec:IL:MDPs}
The RL environment and its underlying dynamics are typically abstracted as an MDP given by a tuple $(\sspace,\aspace,P,\initial,\cost,\gamma)$, where $\sspace$ is the state space, $\aspace$ is the action space, $P:\sspace\times\aspace\rightarrow\Delta_{\sspace}$ is the transition law, 
$\initial\in\Delta_{\sspace}$ is the initial state distribution, $\cost\in[0,1]^{|\sspace||\aspace|}$ is the cost,
and $\gamma\in(0,1)$ is the discount factor. For simplicity, we focus on problems where $\sspace$ and $\aspace$ are finite but too large to be enumerated.
A \emph{stationary Markov policy} $\pi\colon\sspace\to\Delta_{\aspace}$ interacts with the environment iteratively, starting with an initial state $s_0\sim\initial$. At round $t$, if the system is at state $s_t$, an action $a_t\sim\pi(\cdot|s_t)$ is sampled and applied to the environment. Then a cost $c(s,a)$ is incurred, and the system transitions to the next state $s_{t+1}\sim P(\cdot|s,a)$. The goal of RL is to solve the optimal control problem
		$
		\rho_\cost^\star\triangleq\min_{\pi}\rho_\cost(\pi),
		$
		where $\rho_\cost(\pi)\triangleq(1-\gamma)\innerprod{\initial}{\val^\pi_\cost}$ is the \emph{normalized total discounted expected cost} of $\pi$.

The \emph{state value function} $\val_\cost^\pi\in\ar^{|\sspace|}$ of $\pi$, given cost $\cost$, is defined by 
		$
		V_\cost^{\pi}(s) \triangleq\Exp_s^{\pi}\Big[\sum_{t=0}^\infty \gamma^t  c(s_t, a_t)\Big]$,
		where $\Exp^{\pi}_{s}$ denotes the expectation with respect to the trajectories generated by $\pi$ starting from $s_0=s$. The \emph{optimal value function} $\val_\cost^\star\in\ar^{|\sspace|}$ is defined by
		$
		V_\cost^\star(s) \triangleq \min_{\pi}V_\cost^\pi(s). 
		$
		The \emph{optimal state-action value function} $\qval^\star_\cost\in\ar^{|\sspace||\aspace|}$, given by $Q_\cost^\star(s,a)\triangleq c(s,a)+\gamma\sum_{s'}V_\cost^\star(s')P(s'|s,a)$, is known to characterize optimal behaviors. Indeed $\val^\star_\cost$ is the unique solution to the \emph{Bellman optimality equation} $V^\star_\cost(s)=\min_{a}Q^\star_\cost(s,a)$. In addition, any deterministic policy $\pi^\star_\cost(s)=\arg\min_a Q^\star_\cost(s,a)$ is known to be optimal. 
	
	For every policy $\pi$, we define the \emph{normalized state-action occupancy measure} $\mv_\pi\in\Delta_{\sspace\times\aspace}$, by
		$
		\mu_\pi(s,a) \triangleq (1-\gamma) \sum_{t=0}^\infty \gamma^t  \Prob_{\initial}^{\pi}\left[s_t=s,a_t=a\right],
		$
		where $\Prob_{\initial}^{\pi}[\cdot]$ denotes the probability of an event when following $\pi$ starting from $s_0\sim\initial$.
		The occupancy measure can be interpreted as the discounted visitation frequency of state-action pairs. This allows
us to write 
$\rho_{\cost}(\pi)=\innerprod{\mv_\pi}{\cost}$.


\subsection{Imitation Learning}
\looseness=-1
Similarly to RL, the IL problem is posed in the MDP formalism, with the critical difference that the true cost $\true$ is unknown. Instead, we have access to a finite set of truncated trajectories sampled \textrm{i.i.d.} by executing an expert policy $\expert$ in the environment. The goal is to learn a policy that performs better than $\expert$ with respect to the unknown $\true$. To this end, we adopt the \emph{apprenticeship learning} formalism~\cite{Abbeel:2004,Syed:2008,Ho:2016b,Ho:2016,Shani:2021}, which carries the assumption that $\true$ belongs to a class of cost functions $\mathcal{C}$. We then seek an \emph{apprentice policy} $\apprentice$ that outperforms the expert across $\mathcal{C}$ by solving the following optimization problem
\begin{equation}\label{eq:IL}
    \zeta^\star\triangleq\min_{\pi} d_{\mathcal{C}}(\pi,\expert),
\end{equation}
where $d_\mcf{C}(\pi,\expert)\triangleq\max_{\cost\in\mcf{C}} \big(\rho_\cost(\pi)-\rho_\cost(\expert)\big)$ defines the $\mcf{C}$-distance between $\pi$ and $\expert$~\cite{Ho:2016,Chen:2020a,Zhang:2020,Liu:2022}. Then, $\apprentice$ satisfies the goal of IL, since it holds that $\rho_{\true}(\apprentice)-\rho_{\true}(\expert)\le\zeta^\star\le 0$. Intuitively, the cost class $\mathcal{C}$ distinguishes the expert from other policies. The maximization in~(\ref{eq:IL}) assigns high total cost to non-expert policies and low total cost to $\expert$~\cite{Ho:2016}, while the minimization aims to find the policy that matches the expert as close as possible with respect to $d_{\mathcal{C}}$.

By writing $d_\mcf{C}$ in its \emph{dual} form $\bar{d}_{\mathcal{C}}(\mv_\pi,\mv_{\expert})\triangleq\max_{\cost\in\mcf{C}} \big(\innerprod{\mv_\pi}{\cost}-\innerprod{\mv_{\expert}}{\cost}\big)$, it can be interpreted as an \emph{integral probability metric}~\cite{Muller:1997,Kent:2021} 
between the occupancy measures $\mv_\pi$ and $\mv_{\expert}$. Depending on how $\mathcal{C}$ is chosen, $d_{\mathcal{C}}$ turns to a different metric of probability measures like the $1$-Wasserstein distance~\cite{Xiao:2019,Dadashi:2021} for $\mathcal{C}=\textup{Lip}_1(\sspace\times\aspace)$, the total variation for $\mathcal{C}=\{\cost\mid\norm{\cost}_\infty\le 1\}$, or the maximum mean discrepancy for $\mathcal{C}=\{\cost\mid\norm{\cost}_{\mathcal{H}}\le 1\}$, where $\textup{Lip}_1(\sspace\times\aspace)$ denotes the space of $1$-Lipschitz functions on $\sspace\times\aspace$, and $\norm{\cdot}_{\mathcal{H}}$ denotes the norm of a reproducing kernel Hilbert space $\mathcal{H}$~\cite{Shalev-Shwartz:2014}.

\looseness=-1
In our theoretical analysis, we focus on linearly parameterized cost classes~\cite{Syed:2007,Syed:2008,Ho:2016,Liu:2022,Shani:2021} of the form $\mcf{C}\triangleq\{\cost_{\weight}\triangleq\sum_{i=1}^m w_i \phiv_i \mid \weight\in\mathcal{W}\}$, where $\{\phiv_i\}_{i=1}^m\subset\Re_+^{\abs{\sspace}\abs{\aspace}}$ are fixed feature vectors, such that $\norm{\phiv_i}_1 \le 1$ for all $i\in[m]$, and $\mathcal{W}$ is a a convex constraint set for the cost weights $\weight$. This assumption is not necessarily restrictive as
usually in practice the true cost depends on just
a few key properties, but the desirable weighting that specifies how different desiderata should be traded-off is unknown~\cite{Abbeel:2004}. Moreover, the cost features can be complex nonlinear functions that can be
obtained via unsupervised learning from raw state observations~\cite{Brown:2020b,Chen:2020b}. The matrix $\phim\triangleq\begin{bmatrix}\phiv_1&\ldots&\phiv_{m}\end{bmatrix}$ gives rise a \emph{feature expectation vector} (FEV) 
$\FEV{\pi}
		\triangleq (\rho_{\phiv_1}(\expert),\ldots,\rho_{\phiv_{m}}(\expert))^{\mathsf{T}}
		\in\Re^m$ of a policy $\pi$. Then, by choosing $\mathcal{W}$ to be the $\ell_2$ unit ball $B_1^m\triangleq\{\weight\in\ar^m\mid\norm{\weight}_2\le1\}$~\cite{Abbeel:2004}, we get a \emph{feature expectation matching} objective $d_{\mathcal{C}}(\pi,\pi_{\expert})=\norm{\FEV{\pi}-\FEV{\expert}}_2$, while for $\mathcal{W}$ being the probability simplex $\Delta_{[m]}$~\cite{Syed:2007,Syed:2008} we have a worst-case excess cost objective $d_{\mathcal{C}}(\pi,\pi_{\expert})=\max_{i\in[m]}\big(\rho_{\phiv_i}(\pi)-\rho_{\phiv_i}(\expert)\big)$. For clarity, we will replace $\cost$ by $\weight$
		in the notation of the quantities defined in Section~\ref{sec:IL:MDPs}. 

\section{A \pdfmath{Q}-Convex-Analytic Viewpoint}
\label{sec:LP_form}
Our methodology builds upon the convex-analytic approach to AL, first introduced by~\cite{Syed:2008}, with the key difference that we consider a different convex formulation that introduces  $Q$-functions as slack variables. This allows to design a practical scalable model-free algorithm with theoretical guarantees.

Let $\mathfrak{F}\triangleq\{\mv\in\ar^{|\sspace||\aspace|}\mid (\mbf{B}-\gamma\pmat)^\trans\mv=(1-\gamma)\initial,\; \mv\geq\mbf{0}\}$ be the \emph{state-action polytope}, where $\pmat$ is the vector form of $P$, i.e., $P_{(s,a),s'}\triangleq P(s'|s,a)$, and $\bmat$ is a binary matrix defined by $B_{(s,a),s'}\triangleq 1$ if $s=s'$, and $B_{(s,a),s'}\triangleq 0$ otherwise. The linear constraints that define the set $\mathfrak{F}$, also known as \emph{Bellman flow constraints}, precisely characterize the set of state-action occupancy measures.

\begin{proposition}[\citealp{Puterman:1994}]\label{pror:state-action-polytope}
We have that $\mv\in\mathfrak{F}$ if and only if there exists a unique stationary Markov policy $\pi$ such that $\mv=\mv_\pi$. If $\mv\in\mathfrak{F}$ then the policy $
			\pi_{\mv}(a|x) \triangleq \frac{\mv(x,a)}{\sum_{a'\in\aspace}\mv(x,a')}
			$ has occupancy measure $\mv$.
		\end{proposition}
	
	Using Proposition~\ref{pror:state-action-polytope} and the dual form of the $\mathcal{C}$-distance $\bar{d}_{\mathcal{C}}(\mv,\mv_{\expert})=\max_{\weight\in\mathcal{W}}\innerprod{\mv-\mv_{\expert}}{\cost_{\weight}}$, it follows that~(\ref{eq:IL}) is equivalent to the primal convex program $\zeta^\star=\min_{\mv}\{\bar{d}_{\mathcal{C}}(\mv,\mv_{\expert})\mid\mv\in\mathfrak{F}\}$. In particular for $\mathcal{W}=\Delta_{[m]}$ and by using an epigraphic transformation, we end up with an LP program~\cite{Syed:2008}, while for $\mathcal{W}=B_1^m$ we get a quadratic objective with linear constraints~\cite{Abbeel:2004}.
	
	A slight variation of the above reasoning is to introduce a mirror variable $\dv$ and split the Bellman flow constraints in the definition of $\mathfrak{F}$. We then get the primal convex program 
	\begin{equation}\label{eq:primal}
	\zeta^\star=\min_{(\mv,\dv)}\{\bar{d}_{\mathcal{C}}(\mv,\mv_{\expert})\mid(\mv,\dv)\in\mathfrak{M}\}, \tag{\color{blue}Primal}
	\end{equation} where the new polytope is given by 
	$\mathfrak{M}\triangleq\{(\mv,\dv)\mid\mbf{B}^\trans\dv=\gamma\pmat^\trans\mv+(1-\gamma)\initial,\; \mv=\dv,\; \dv\geq\mbs{0}\}$. This overparameterization trick has been first introduced by Mehta and Meyn~\cite{Mehta:2009} and has been recently revisited by~\cite{Bas-Serrano:2021,Neu:2020,Lee:2019a,Neu:2021,Mehta:2020,Lu:2021}. A salient feature of this equivalent formulation is that it introduces a $Q$-function as Lagrange multiplier to the equality constraint $\dv=\mv$, and so lends itself to data-driven algorithms. To motivate further this new formulation, in Appendix~\ref{app:strong-duality}, we shed light to its dual and provide an interpretation of the dual optimizers. In particular, when $\mathcal{W}=B_1^m$, we show that $(\mbf{V}^\star_{\mbf{w_{\textup{true}}}},\mbf{Q}^\star_{\mbf{w_{\textup{true}}}},\mbf{w_{\textup{true}}})$ is a dual optimizer.
	
	
	For our theoretical analysis we focus on the linear MDP setting~\cite{Jin:2020}, i.e., we assume that the transition law is linear in the feature mapping. We denote by $\phiv(s,a)$ the $(s,a)$-th row of $\phim$.
	\begin{assumption}[Linear MDP]\label{ass:linear-MDP}
			There exists a collection of $m$ probability measures $\boldsymbol{\omega}=(\omega_1,\ldots,\omega_m)$ on $\sspace$, such that $P(\cdot|s,a)=\innerprod{\boldsymbol{\omega}(\cdot)}{\phi(s,a)}$, for all $(s,a)$. Moreover $\phiv(s,a)\in\Delta_{[m]}$, for all $(s,a)$.
		\end{assumption}
		\looseness=-1
		Assumption~\ref{ass:linear-MDP} essentialy says that the transition matrix $\pmat$ has rank at most $m$, and $\pmat=\phim\mmat$ for some matrix $\mmat\in\ar^{m\times|\sspace|}$. It is worth noting that in the case of continuous MDPs, despite being linear, the transition law $P(\cdot|s,a)$ can still have infinite degrees of freedom. This is a substantial difference from the recent theoretical works on IL~\cite{Liu:2022,Shani:2021} which consider either a linear quadratic regulator, or a transition law that can be completely specified by a finite-dimensional matrix such that the degrees of freedom are bounded.

\looseness=-1
Assumption~\ref{ass:linear-MDP} enables us to consider a relaxation of~(\ref{eq:primal}). In particular, we aggregate the constraints $\mv=\dv$ by imposing $\phim^\trans\mv=\phim^\trans\dv$ instead, and introduce a variable $\lv = \phim^\trans  \mv$.
It follows that $\lv$ lies in the $m$-dimensional simplex $\Delta_{[m]}$. Then, we get the following convex program 
\begin{equation}\label{eq:primal'}
	\zeta^\star=\min_{(\lv,\dv)}\{\max_{\weight\in\mathcal{W}}\innerprod{\lv}{\weight}-\innerprod{\mv_{\expert}}{\cost_{\weight}}\mid(\lv,\dv)\in\mathfrak{M}_{\phim}\}, \tag{\color{blue}Primal$^\prime$}
	\end{equation} where  
	$\mathfrak{M}_{\phim}\triangleq\{(\lv,\dv)\mid\mbf{B}^\trans\dv=\gamma\mmat^\trans\lv+(1-\gamma)\initial,\; \lv=\phim^\trans\dv,\;\lv\in\Delta_{[m]},\; \dv\in\Delta_{\sspace\times\aspace}\}$. As shown in~\cite{Neu:2020,Bas-Serrano:2021,Neu:2021}, for linear MDPs, the set of occupancy measures $\mathfrak{F}$ can be completely characterized by the set $\mathfrak{M}_{\phim}$ (c.f., Proposition~\ref{prop:q-update}). While the number of constraints and variables in~(\ref{eq:primal'}) is intractable for large scale MDPs, in the next paragraph, we show how this problem can be solved using a proximal point scheme.  

%



\section{Proximal Point Imitation Learning}
\label{sec:PPM}
By using a Lagrangian decomposition, we have that~(\ref{eq:primal'}) is equivalent to the following bilinear saddle-point problem 
\begin{equation}
    \min_{\xv\in \mathcal{X}} \max_{\yv \in \mathcal{Y}} \innerprod{\yv}{\mbf{A}\xv+\mbf{b}}
    \label{eq:SPP}, \tag{\color{blue}SPP}
\end{equation}
where $\mbf{A}\in\ar^{(2m+|\sspace|)\times(m+|\sspace||\aspace|})$, and $\mbf{b}\in\ar^{(m+|\sspace|+|\sspace||\aspace|)}$ are appropriately defined (see Appendix~\ref{app:SPP}),
 $\xv\triangleq[\lv^\trans$, $\dv^\trans ]^\trans $, $\yv\triangleq[\weight^\trans$, $\val^\trans , \thv^\trans ]^\trans$, $\mathcal{X} \triangleq \Delta_{[m]} \times \Delta_{\sspace\times\aspace}$, and $\mathcal{Y}\triangleq\mathcal{W}\times\ar^{|\sspace|}\times\ar^{m}$. 
 
 Since in practice we do not have access to the whole policy $\expert$, but instead can observe a finite set of \textrm{i.i.d.} sample trajectories $\mathcal{D}_{\textup{E}}\triangleq\{(x_0^{(l)},a_0^{(l)},x_1^{(l)},a_1^{(l)},\ldots,x_H^{(l)},a_H^{(l)})\}_{l=1}^{n_{\textup{E}}}\sim\expert$, we define the vector $\widehat{\mbf{b}}$ by replacing $\FEV{\expert}$ with its empirical counterpart $\EFEV{\expert}$ (by taking sample averages) in the definition of ${\mbf{b}}$. We then consider the empirical objective $f(\xv)\triangleq\max_{\yv\in\mathcal{Y}} \innerprod{\yv}{\mbf{A}\xv + \widehat{\mbf{b}}}$ and apply PPM  on the decision variable $\xv$. For the $\lv$-variable we use the relative entropy $D(\lv || \lv^\prime)\triangleq\sum^m_{i=1} \lambda(i)\log\frac{\lambda(i)}{\lambda^\prime(i)}$, while for the occupancy measure $\dv$ we use the conditional relative entropy $H(\dv||\dv^\prime)\triangleq\sum_{s,a} d(s,a)\log\frac{\pi_\dv(a|s)}{\pi_{\dv^\prime}(a|s)}$. With this choice we can rewrite the PPM update as
\begin{equation}\label{eq:q-update}
    (\lv_{k+1},\dv_{k+1})=\argmin_{\lv\in\Delta_{[m]},\dv \in\Delta_{\sspace\times\aspace}}\max_{\yv\in\mathcal{Y}}\innerprod{\yv}{ \mbf{A}\left[ {\begin{array}{ccc}
    \lv \\ \dv
  \end{array} } \right]+  \widehat{\mbf{b}}} + \frac{1}{\eta}D(\lv||\phim^\trans\dv_k) + \frac{1}{\alpha}H(\dv||\dv_k), \\
\end{equation}
where we used primal feasibility to replace $\lv_k$ with $\phim^\trans\dv_k$ as the center point of the relative entropy.
\looseness=-1
PPM is implicit, meaning that it requires the evaluation of the gradient at the next iterate $\xv_{k+1}$. Such a requirement makes it not implementable in general. However, in the following, we describe a procedure to apply proximal point to our specific $f(\xv)$.  
The following Proposition summarizes the result.
\begin{proposition}\label{prop:q-update}
For a parameter $\thv\in\ar^m$, we define the logistic state-action value function $\qval_{\thv}\in\ar^{|\sspace||\aspace|}$ by $\qval_{\thv}\triangleq\phim\thv$, and the $k$-step logistic state value function $\val_{\thv}^k\in\ar^{|\sspace|}$ by
\[
V_{\thv}^k(s)\triangleq-\frac{1}{\alpha}\log\left(\sum_a \pi_{\dv_{k-1}}(a|s)e^{-\alpha Q_{\thv}(s,a)}\right).
\]
Moreover, we define the $k$-step reduced Bellman error function $\boldsymbol{\delta}_{\weight,\thv}^k\in\ar^m$ by
$
\boldsymbol{\delta}_{\weight,\thv}^k\triangleq\weight+\gamma\mmat\val_{\thv}^k-\thv.
$
Then, the PPM update $(\lv_k^\star,\dv_k^\star)$ in~\ref{eq:q-update} is given by 
\begin{align}
\lambda_k^\star(i) &\propto (\phim^\trans \dv_{k-1})(i)\,e^{-\eta\delta_{\weight_k^\star,\thv_{k}^\star}^k(i)},\label{eq:update1}\\
\pi_{\dv_k^\star}(a|s)&\propto\pi_{\dv_{k-1}}(a|s)\,e^{-\alpha Q_{\thv_k^\star}(s,a)},\label{eq:update2}
\end{align}
where $(\weight_k^\star,\thv_k^\star)$ is the maximizer {over $\mathcal{W}\times\ar^m$} of the $k$-step logistic policy evaluation objective 
\begin{equation}
\mathcal{G}_k(\weight,\thv)\triangleq-\frac{1}{\eta}\log\sum_{i=1}^m (\phim^\trans \dv_{k-1})(i)   e^{-\eta\delta^k_{\weight,\thv}(i)}+(1-\gamma)\innerprod{\initial}{\val_{\thv}^k}- \innerprod{\EFEV{\expert}}{\weight}.\label{eq:PEobjective}
\end{equation}
Moreover, it holds that 
$\mathcal{G}_k(\weight_k^\star,\thv_k^\star)=\innerprod{\lv_{k}^\star}{\weight_k^\star} - \innerprod{\EFEV{\expert}}{\weight_k^\star} + \frac{1}{\eta}D(\lv_{k}^\star ||\phim^\trans \lv_{k-1}) + \frac{1}{\alpha}H(\dv_{k}^\star||\dv_{k-1}).$
If in addition Assumption~\ref{ass:linear-MDP} holds, then $\dv_k^\star$ is a valid occupancy measure, i.e., $\dv_k^\star\in\mathfrak{F}$ and so $\dv_k^\star=\mv_{\pi_{\dv_k^\star}}$.
\end{proposition}

The proof of Proposition~\ref{prop:q-update} is broken down into a sequence of lemmas and is presented in Appendix~\ref{app:proof-of-upadates-proposition}. It employs an \texttt{analytical-oracle} $\gv:\mathcal{Y}\rightarrow\mathcal{X}$ given by
\begin{align*}
    \gv(\yv; \xv_k) &\triangleq\argmin_{\lv\in\Delta_{[m]},\dv \in\Delta_{\sspace\times\aspace}}\innerprod{\yv}{ \mbf{A}\left[ {\begin{array}{ccc}
    \lv \\ \dv
  \end{array} } \right]+  \widehat{\mbf{b}}} + \frac{1}{\eta}D(\lv||\phim^\trans\dv_k) + \frac{1}{\alpha}H(\dv||\dv_k),
\end{align*}
and a \texttt{max-oracle} $\mbf{h}:\mathcal{X}\rightarrow\mathcal{Y}$ given by
$
\mbf{h}(\xv) \triangleq \argmax_{\yv\in\mathcal{Y}} \innerprod{\yv}{\mbf{A}\gv(\yv;\xv)} + \frac{1}{\tau}D_{\Omega}(\gv(\yv;\xv)||\xv),
$ where we used $D_\Omega$ to compact the two divergences.
By noting that the PPM update~\Cref{eq:q-update} can be rewritten as
$ \xv_{k+1}
     =     \gv(\mbf{h}(\xv_k); \xv_k),
     \label{eq:proximal_point_compact_update}
$
its analytical computation is reduced to the characterization of the two aforementioned oracles. In particular, the updates~(\ref{eq:update1})--(\ref{eq:update2}) come from the \texttt{analytical-oracle} while~(\ref{eq:PEobjective}) is the objective of the \texttt{max-oracle}.

The choice of conditional entropy as Bregman divergence for the $\lv$ variable living in the probability simplex is standard in the optimization literature and is known to mitigate the effect of dimension. In  particular, as noted in~\cite{Neu:2007}, the classic REPS algorithm~\cite{Peters:2010} can be seen as mirror descent with relative entropy regularization. On the other hand, the choice of conditional entropy as Bregman divergence for the $\dv$ variable is less standard and has been popularized by Q-REPS \cite{Bas-Serrano:2021}. Such particular divergence leads to an actor-critic algorithm that comes with several merits.
 By Proposition~\ref{prop:q-update}, it is apparent that we get analytical softmin updates for the policy $\pi_{\dv}$ rather than the occupancy measure $\dv$. Moreover, these softmin updates are expressed in terms of the logistic $Q$-function and do not involve the unknown transition matrix $\pmat$. Consequently, we avoid the problematic occupancy measure approximation and the restrictive coherence assumption on the choice of features needed in~\cite{Bas-Serrano:2020,Kamoutsi:2021}, as well as the biased policy updates appearing in REPS \cite{Peters:2010, Pacchiano:2021}. In addition,  the newly introduced logistic policy evaluation objective $\mathcal{G}_k(\weight,\thv)$ has several desired properties. It is   concave and smooth in $(\weight,\thv)$ and has bounded gradients. 
 Therefore, it does not suffer from the pathologies of the squared Bellman error~\cite{Mnih:2015} and does not require heuristic gradient clipping techniques. Moreover, unlike~\cite{Kamoutsi:2021} it allows a model-free implementation without the need for a generative model (see Section~\ref{sec:PPM_model_free})
 
\looseness=-1
We stress the fact that the \texttt{max-oracle} of our proximal point scheme performs the cost update and policy evaluation phases jointly. 
This is a rather novel feature of our algorithm that differs from the separate cost update and policy evaluation step used in recent theoretical imitation learning works~\cite{Zhang:2020,Shani:2021,Liu:2022}. Our joint optimization over cost and $Q$-functions avoids instability due to adversarial training and can also recover an explicit cost along with the $Q$-function without requiring knowledge or additional interaction with the environment (see Section~\ref{sec:experiments}). It is worth noting that application of primal-dual mirror descent to~(\ref{eq:SPP}) does not have this favorable property. While in the standard MDP setting, proximal point and mirror descent coincide because of the linear objective, in imitation learning proximal point optimization makes a difference. In Appendix~\ref{sec:mirror-descent}, we include a more detailed discussion and numerical comparison between PPM and mirror descent updates.

 
 \subsection{Practical Implementation}
\label{sec:PPM_model_free}
Exact optimization of the logistic policy evaluation objective is infeasible in practical scenarios, due to unknown dynamics and limited computation power. In this section, we design a practical algorithm that uses only sample transitions by obtaining stochastic (albeit biased) gradient estimators.

Proposition~\ref{prop:q-update} gives rise to Proximal Point Imitation Learning (\texttt{P$^2$IL}), a model-free actor-critic IRL algorithm described in Algorithm~\ref{alg:PPIQL}. The key feature of \texttt{P$^2$IL} is  that  the policy evaluation step involves optimization of a single smooth and concave objective over both cost and state-action value function parameters. In this way, we avoid instability or poor convergence in optimization due to nested policy evaluation and cost updates, as well as the undesirable properties of the widely used squared Bellman error. In particular, the $k$th iteration of \texttt{P$^2$IL} consists of the following two steps : 
(i) (\textbf{Critic Step}) Computation of an approximate maximizer  $(\weight_k,\thv_k)\approx\argmax_{\weight,\thv}{\mathcal{G}}_k(\weight,\thv)$ of the concave logistic policy evaluation objective, by using a biased stochastic gradient ascent subroutine;
(ii) (\textbf{Actor Step}) Soft-min policy update $
				\pi_{k}(a|s)\propto\pi_{k-1}(a|s)\,e^{-\alpha Q_{\thv_k}(s,a)}
				$ expressed in terms of the logistic $Q$-function.

	\begin{algorithm}[!t]
			\caption{Proximal Point Imitation Learning: \texttt{P$^2$IL}$(\phim,\mathcal{D}_{\textup{E}},K,\eta, \alpha)$}
			\label{alg:PPIQL}
			\begin{algorithmic}
				\STATE {\bfseries Input:} Feature matrix $\mbs{\Phi}$, expert demonstrations $\mathcal{D}_{\textup{E}}$, number of iterations $K$, step sizes $\eta$ and $\alpha$,
				number of SGD iterations T, SGD learning rates $\mbs{\beta}=\{\beta_t\}_{t=0}^{T-1}$, number-of-samples function $n:\mathds{N}\rightarrow\mathds{N}$
				\STATE Initialize $\pi_0$ as uniform distribution over $\aspace$
				\STATE Compute the empirical FEV $\EFEV{\expert}$ using expert demonstrations $\mathcal{D}_{\textup{E}}$
				\FOR{$k=1,\ldots K$}
			    \STATE \texttt{// Critic-step (policy evaluation)}
					\STATE Initialize $\thv_{k,0}=\mbf{0}$ and $\weight_{k,0}=\mbf{0}$
				\STATE Run $\pi_{k-1}$ and collect \textrm{i.i.d.} samples $\mathcal{B}_k=\{(s_{{k-1}}^{(n)},a_{{k-1}}^{(n)},s_{{k-1}}^{\prime (n)})\}_{n=1}^{n(T)}$ such that 
				\STATE $(s_{{k-1}}^{(n)},a_{{k-1}}^{(n)})\sim\mv_{\pi_{k-1}}$ and $s_{{k-1}}^{\prime (n)}\sim \mathsf{P}(\cdot|s_{k-1}^{(n)},a_{k-1}^{(n)})$
			\FOR{$t=0,\ldots T-1$}
			    \STATE Compute biased stochastic gradient estimators $$\big(\widehat{\nabla}_{\weight}\mathcal{G}_k(\weight_{k,t},\thv_{k,t}),\widehat{\nabla}_{\thv}\mathcal{G}_k(\weight_{k,t},\thv_{k,t})\big)=\textrm{BSGE}\big(k,\weight_{k,t},\thv_{k,t},n(t)\big)$$\vspace{-0.5cm}
				\STATE $\weight_{k,t+1}=\Pi_{\mathcal{W}}\big(\weight_{k,t}+\beta_t \widehat{\nabla}_{\weight}\mathcal{G}_k(\weight_{k,t},\thv_{k,t})\big)$
				\STATE  $\thv_{k,t+1}=\Pi_{\Theta}\big(\thv_{k,t}+\beta_t \widehat{\nabla}_{\thv}\mathcal{G}_k(\weight_{k,t},\thv_{k,t})\big)$
				\ENDFOR
				\STATE  $(\weight_k,\thv_k)=(\frac{1}{T}\sum_{t=1}^T\weight_{k,t},\frac{1}{T}\sum_{t=1}^T\thv_{k,t})$
				\STATE \texttt{// Actor-step (policy update)}
				\STATE Policy update:
				$
				\pi_{k}(a|s)\propto\pi_{k-1}(a|s)\,e^{-\alpha Q_{\thv_k}(s,a)}
				$
				\ENDFOR
				\STATE {\bfseries Output:} Mixed policy $\widehat{\pi}_K$ of $\{\pi_k\}_{k\in[K]}$
				
			\end{algorithmic}
		\end{algorithm}
 The domain $\Theta$ in Algorithm~\ref{alg:PPIQL} is the $\ell_{\infty}$-ball with appropriately chosen radius $D$ to be specified later (see Proposition~\ref{prop:optimal_theta_bound}). Moreover, $\Pi_{\Theta}(\mbf{x})\triangleq\arg\min_{\mbf{y}\in\Theta}\norm{\mbf{x}-\mbf{y}}_{2}$ (resp.  $\Pi_{\mathcal{W}}(\mbf{\weight})$) denotes the  Euclidean projection of $\mbf{x}$ (resp. $\weight$)  onto $\Theta$ (resp. $\mathcal{W}$).   


In order to estimate the gradients $\nabla_{\thv}\, \mathcal{G}_k(\weight,\thv)$ and $\nabla_{\weight}\, \mathcal{G}_k(\weight,\thv)$ we invoke the Biased Stochastic Gradient Estimator subroutine (\textrm{BSGE}) (Algorithm~\ref{alg:BSGE}) given in Appendix~\ref{app:stochastic gradients}. By using the linear MDP Assumption~\ref{ass:linear-MDP} and leveraging ridge regression and plug-in estimators, the proposed stochastic gradients can be computed via simple linear algebra with computational complexity  $\textup{poly}(m,n(t))$, independent of the size of the state space.
\subsection{Theoretical Analysis}
\label{sec:theorems}
The first step in our theoretical analysis is to study the propagation of optimization errors made by the algorithm on the true policy evaluation objective. In particular at each iteration step $k$, the ideal policy evaluation update $(\weight_k^\star,\thv_k^\star)$ and the ideal policy update $\pi_k^\star$ are given by
		$
		(\weight_k^\star,\thv^\star_k)=\arg\max_{\weight,\thv}\mathcal{G}_k(\weight,\thv)$, and $\pi_k^\star(a|s)=\pi_{k-1}(a|s)e^{-\alpha(Q_{\thv^\star_k}(s,a)-V^k_{\thv^\star_k}(s))}.
		$
		On the other hand, consider the realised policy evaluation update $(\weight_k,\thv_k)$ such that ${\mathcal{G}}_k(\weight_k^\star,\thv_k^\star)-\mathcal{G}_k(\weight_k,\thv_k)=\epsilon_k$, the corresponding policy $\pi_k$ given by  $\pi_k=\pi_{k-1}(a|s)e^{-\alpha(Q_{\thv_k}(s,a)-V^k_{\thv_k}(s))}$, and let $\dv_k\triangleq\mv_{\pi_k}$. We denote by $\widehat{\pi}_K$ the extracted mixed policy of $\{\pi_k\}_{k=1}^K$. We are interested in upper-bounding the suboptimality gap $d_{\mathcal{C}}(\widehat{\pi}_K,\expert)$ of Algorithm~\ref{alg:PPIQL} as a function of $\varepsilon_k$. To this end, we need the following assumption.

\begin{assumption}
\label{ass:eigenvalue}
It holds that
$
\lambda_{\mathrm{min}}(\Exp_{(s,a)\sim \dv_k}{\phiv(s,a)\phiv(s,a)^{\mathsf{T}}}) \geq \beta
$, for all $k\in[K]$.
\end{assumption}
 Assumption~\ref{ass:eigenvalue} states that every occupancy measure $\dv_k$ induces a positive definite feature covariance matrix, and so every policy $\pi_k$ explores uniformly well in the feature space. This assumption is common in the RL theory literature~\cite{Abbasi-Yadkori:2019b, Hao:2021, Duan:2020, Lazic:2020, Abbasi-Yadkori:2019c, Agarwal:2020b}. It is also related to the condition of persistent excitation from the control literature~\cite{Narenda:1987}.

The following proposition ensures that $\max_{\weight,\thv\in\wspace\times\mathbb{R}^m}{\mathcal{G}}_k(\weight,\thv) = \max_{\weight,\thv\in\wspace\times\Theta}{\mathcal{G}}_k(\weight,\thv)$.  Therefore, this constraint does not change the problem optimality, but will considerably accelerate the convergence of the algorithm by considering smaller domains.
\begin{proposition}\label{prop:optimal_theta_bound}
There exists a maximizer $\thv^\star_k$ such that $\norm{\thv^\star_k}_{\infty}\le\frac{1 + \abs{\log\beta}}{1 - \gamma}\triangleq D$. 
\end{proposition}
We can now state our error propagation theorem.
\begin{theorem}
\label{thm:error_propagation}
Let $\widehat{\pi}_K$ be the output of running Algorithm~\ref{alg:PPIQL} for $K$ iterations, with $n_{\textup{E}}\geq\frac{2\log(\frac{2m}{{\boldsymbol{\delta}}})}{\varepsilon^2}$ expert trajectories of length $H\geq\frac{1}{1-\gamma}\log(\frac{1}{\varepsilon})$. Let $C\triangleq \frac{1}{\beta\eta}\big(\sqrt{\frac{2 \alpha}{1 - \gamma}} + \sqrt{8 \eta}\big) + \sqrt{\frac{18 \alpha}{1 - \gamma}}$. Then, with probability at least $1-\delta$, it holds that 
$
    d_{\mathcal{C}}(\widehat{\pi}_K, \expert) 
    \leq \frac{1}{K}\Big( \frac{\log d}{\eta} + \frac{\log \abs{\aspace}}{\alpha}
    + C\sum_k \sqrt{\epsilon_k} + \sum_k \epsilon_k\Big)+\varepsilon.
    \label{eq:lemma_bound_on_saddle_point}
$

\end{theorem}
By Theorem~\ref{thm:error_propagation}, whenever the policy evaluation errors $\varepsilon_k$, as well as the estimation error $\varepsilon$ can be kept small, Algorithm~\ref{alg:PPIQL} ouputs a policy $\widehat{\pi}_K$ with small suboptimality gap $\rho_{\true}(\widehat{\pi}_K)-\rho_{\true}(\expert)$. Notably, there is no direct dependence on the size of the state space or the dimension of the feature space. In the ideal case, where $\varepsilon_k=0$ for all $k$, the convergence rate is $\mathcal{O}(1/K)$. 
The provided error propagation analysis still holds with general function approximation, i.e., in the context of deep RL. Indeed, by choosing $\phim=\mbf{I}$, Assumption~\ref{ass:linear-MDP} is trivially satisfied and the $\thv$ variable in the objective $\mathcal{G}_k$ is replaced by a $Q$-function. 
In practice, the estimation error $\varepsilon$ can be made arbitrary small, by increasing the number of expert demonstrations $n_{\textup{E}}$. Moreover, the next theorem ensures that under Assumptions~\ref{ass:linear-MDP} and~\ref{ass:eigenvalue} the biased stochastic gradient ascent (BSGA) subroutine has sublinear convergence rate. 
\begin{theorem}
\label{thm:biased_sgd}
Let $(\weight_k,\thv_k)$ be the output of the BSGA subroutine in \Cref{alg:PPIQL} for $T$ iterations, with $n(t) \geq \max\br{\mathcal{O}\br{\frac{\gamma^2 m D t }{(\eta+\alpha)^2\beta}\log\frac{Tm}{\delta}}, \mathcal{O}\br{\frac{m t}{(\eta+\alpha)^2\beta}\log\frac{Tm}{\delta}}}$ sample transitions, and learning rates $\beta_t=\mathcal{O}(\frac{1}{\sqrt{t}})$. Then, $\epsilon_k = {\mathcal{G}}_k(\weight_k^\star,\thv_k^\star)-\mathcal{G}_k(\weight_k,\thv_k) \leq \mathcal{O}(\frac{\max\{\eta,1\} m D}{\beta \sqrt{T}})$, with probability $1-\delta$.
\end{theorem}
\begin{corollary}[Resource guarantees]
\label{cor:sample_complexity}
Choose $\eta=\alpha=1$ and let $K=\Omega\br{\epsilon^{-1}}$, $T=\Omega\br{\epsilon^{-4}}$. Then for $\Omega\br{KT} = \Omega\br{\epsilon^{-5}}$ sample transitions, $\Omega\br{\varepsilon^{-2}}$ expert trajectories and approximately solving $\Omega\br{\epsilon^{-1}}$ concave maximization problems, we can ensure $d_{\mathcal{C}}(\widehat{\pi}, \expert)\leq \mathcal{O}(\epsilon + \varepsilon)$, with high probability.
\end{corollary}

\textbf{Offline Setting.} Finally, we notice that using $\phim^\trans\mv_{\expert}$ as the reference distribution for the relative entropy we can obtain an offline algorithm that does not require environment interactions. By reinterpreting smoothing \cite{Nesterov:2005} as one step of proximal point, and using similar arguments as in the proof of \Cref{thm:error_propagation}, we can provide  similar theoretical guarantees for the offline setting. The formal statement of the theoretical result as well as the optimization of the empirical policy evaluation objective are presented in Appendix~\ref{app:offline} (see Theorems~\ref{thm:offline_error_propagation} { and~\ref{thm:Donsker-Varadhan}}).
\section{Experiments}

In this section, we demonstrate that our approach achieves convincing empirical performance in both online and offline IL settings on several environments.\footnote{
The code is available at the following link \url{https://github.com/lviano/P2IL}.} The precise setting is detailed in Appendix~\ref{app:experiments}. 

\looseness=-1
\textbf{Online Setting.} We first present results in various tabular environments where we can implement our algorithm without any practical relaxation outperforming GAIL \cite{Ho:2016}, AIRL \cite{Fu:2018} and IQ-Learn \cite{Garg:2021}. Results are given in \Cref{fig:simple_env_results}. Good performance but inferior to IQ-Learn is observed also for continuous states environments (CartPole and Acrobot) where we used neural networks function approximation.

\begin{figure*}[t] 
\centering
\begin{tabular}{ccccc}
\subfloat[RiverSwim]{%
       \includegraphics[width=0.16\linewidth]{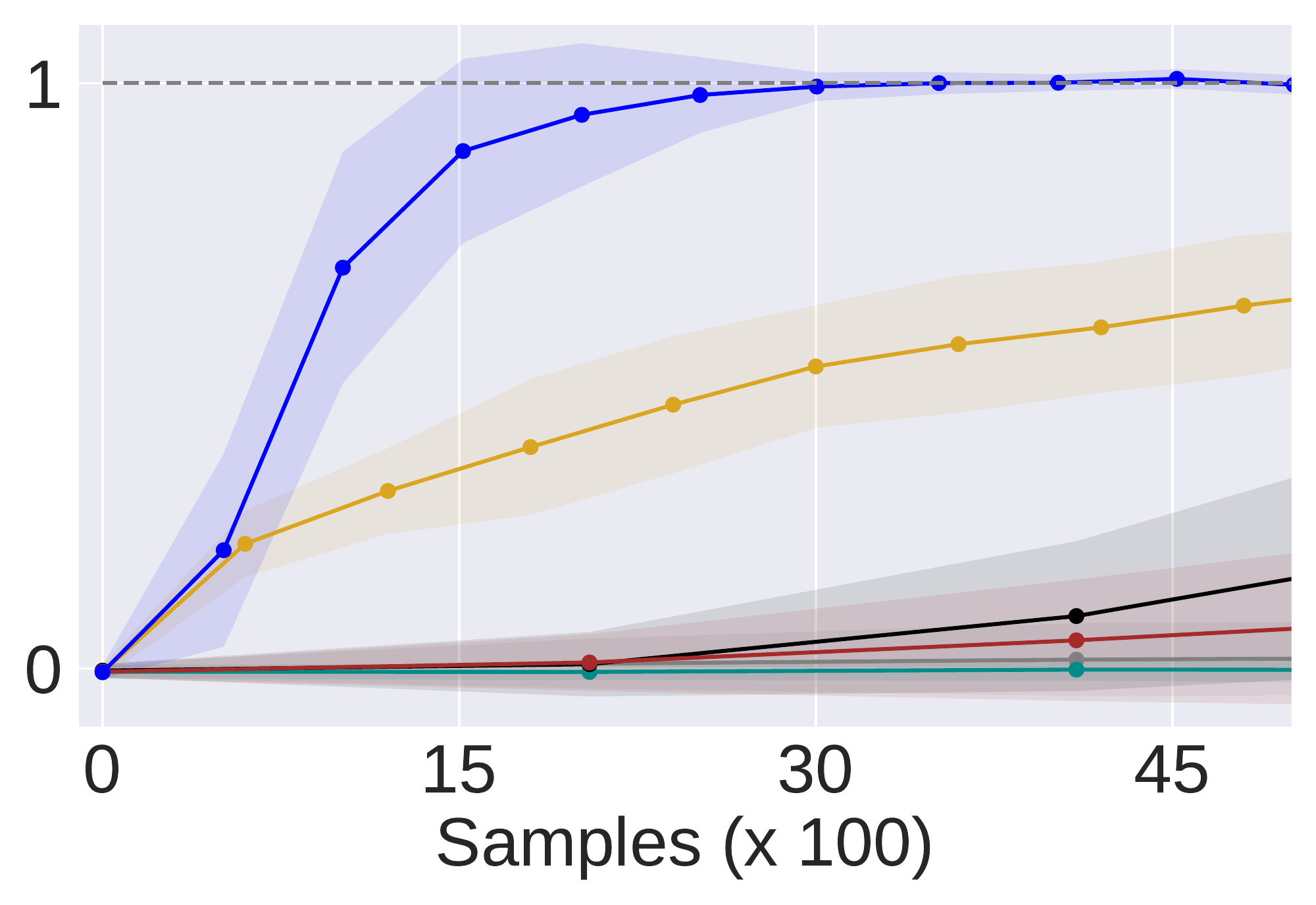}
     } &
     \subfloat[CartPole]{%
       \includegraphics[width=0.16\linewidth]{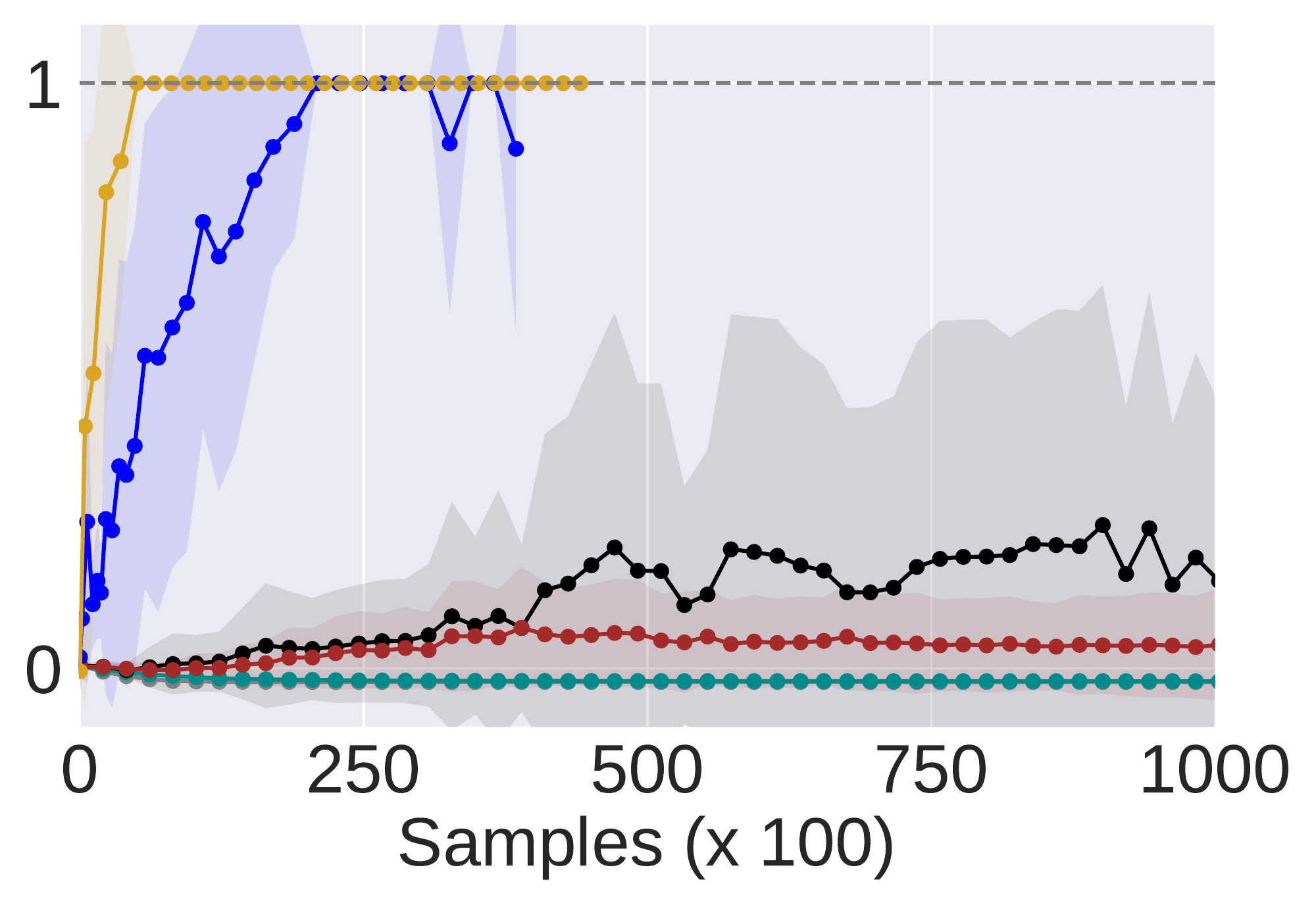}
     } &
\subfloat[DoubleChain]{%
       \includegraphics[width=0.16\linewidth]{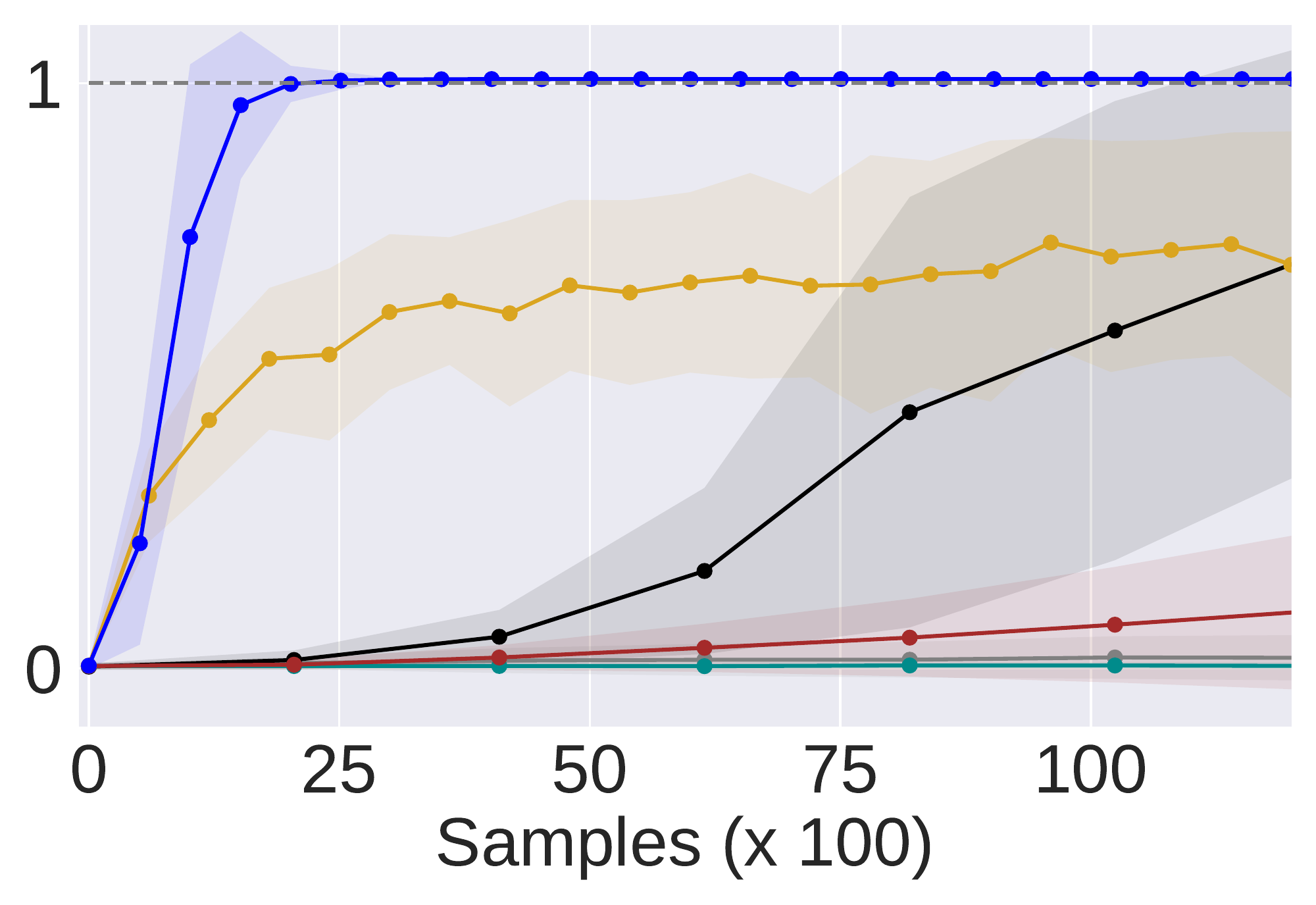}
     } &
\subfloat[Gridworld]{%
       \includegraphics[width=0.16\linewidth]{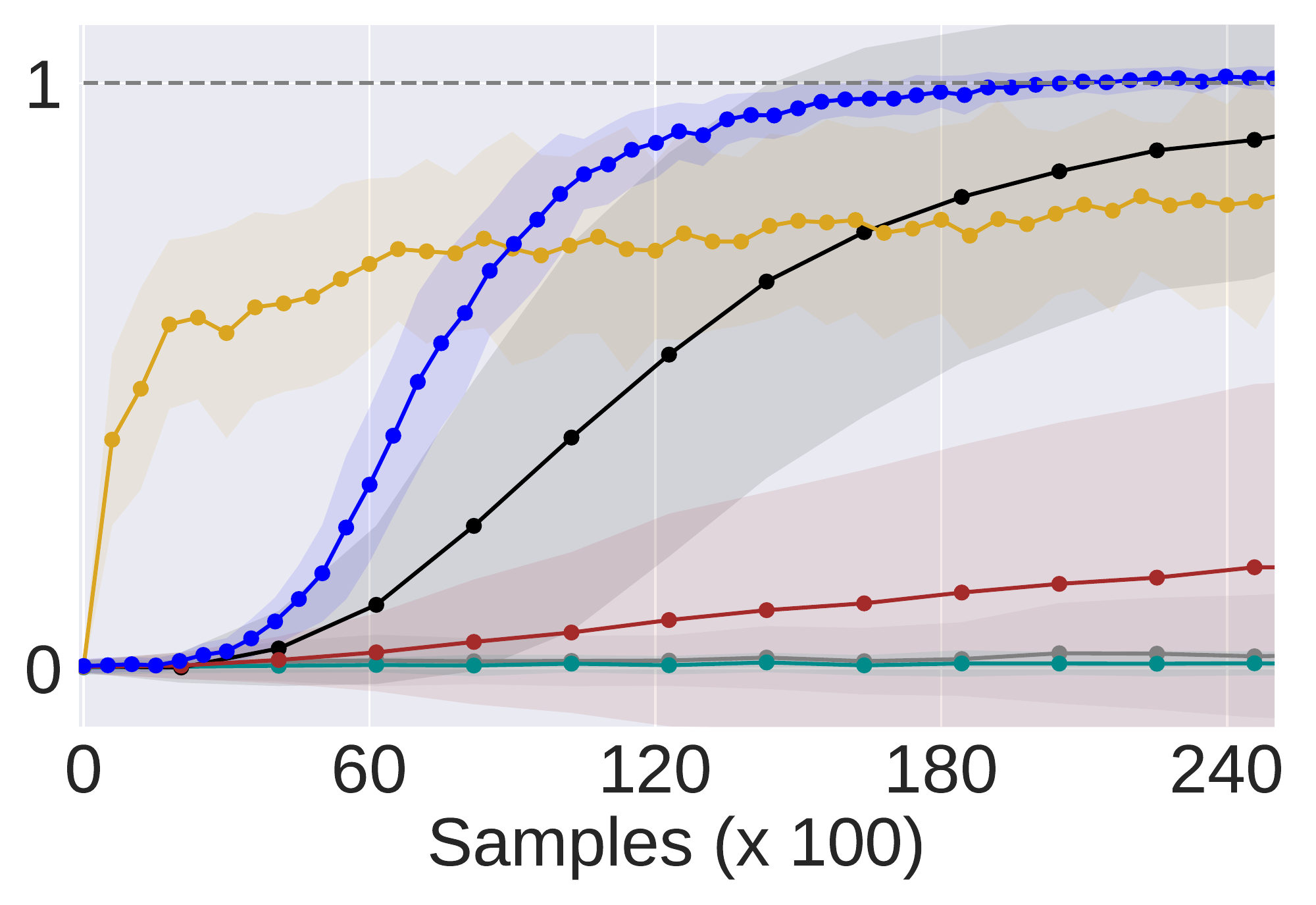}
     } &
     \subfloat[Acrobot]{%
       \includegraphics[width=0.16\linewidth]{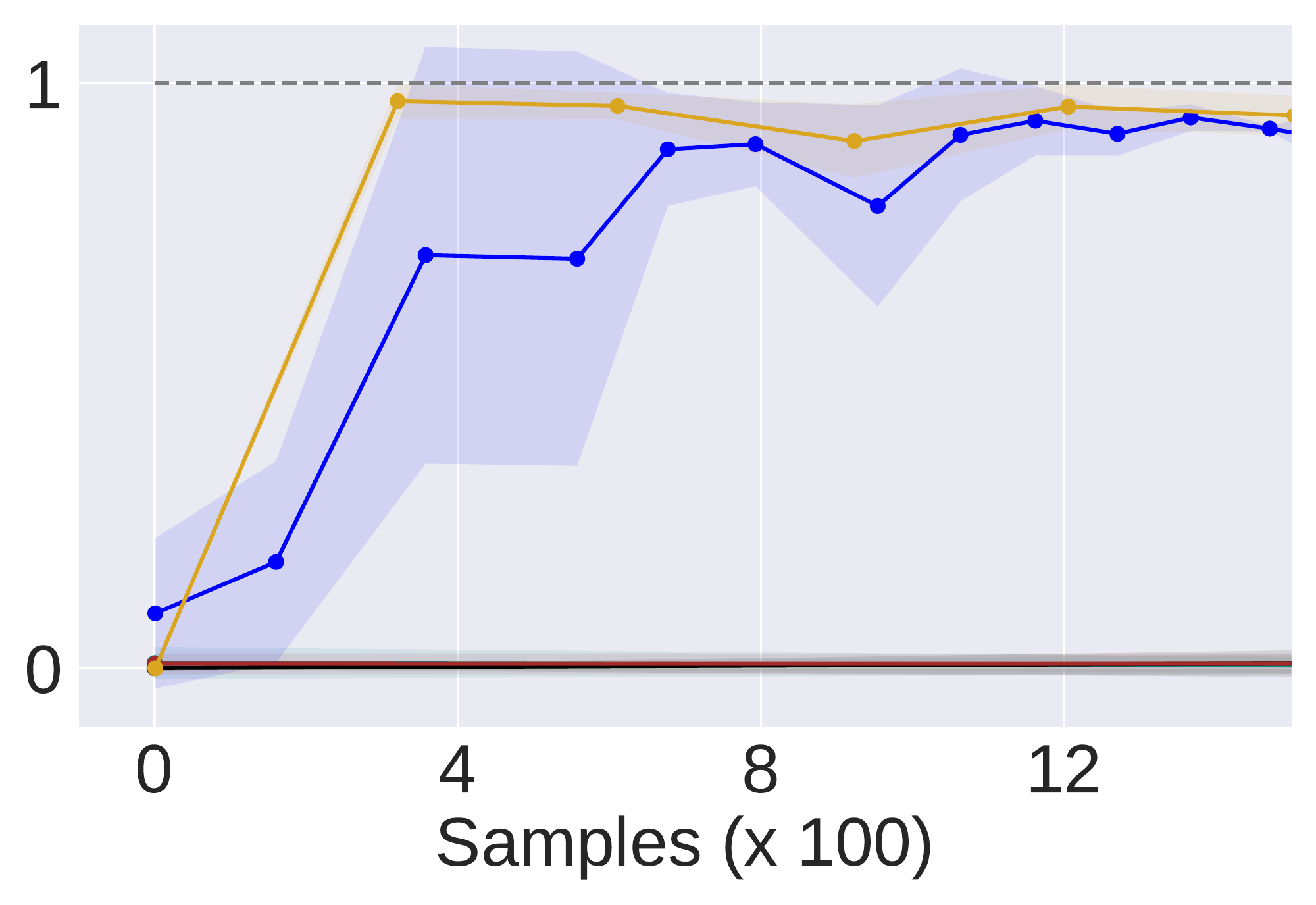}
     } \\ 
    \multicolumn{5}{c}{
       \includegraphics[scale=0.5]{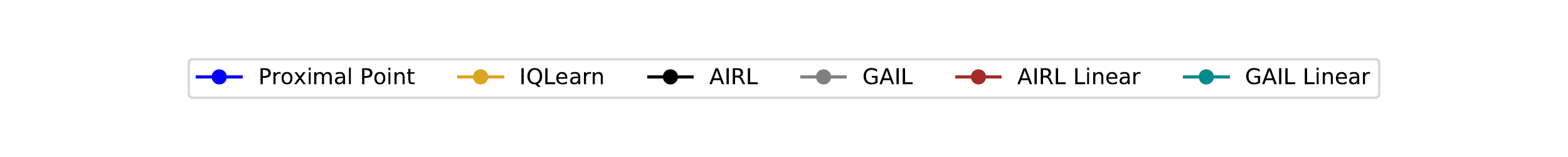}
     }
\end{tabular}
\caption{\textbf{Online IL Experiments}. We show the total returns vs the number of env steps. }
\label{fig:simple_env_results}
\end{figure*}

\textbf{Offline Setting.} \Cref{fig:offline_cartpole,fig:offline_acrobot,fig:offline_lunarlander} shows that our method is competitive with the state-of-the-art offline IL methods IQLearn \cite{Garg:2021} and AVRIL \cite{Chan:2021} that recently showed performances superior to other methods like \cite{Jarrett:2021}\cite{Kostrikov:2020}.
 We also tried our algorithm in the complex image-based \texttt{Pong} task from the Atari suite. 
 \Cref{fig:pong} shows that the algorithm reaches the expert level after observing $2e5$ expert samples. We did not find AVRIL competitive in this setting, and skip it for brevity. 
 In these settings, we verified that the algorithmic performance is convincing even for costs parameterized by neural networks. 

\textbf{Continuous control experiments.}
We attain the expert performance also in $2$ MuJoCo environments: \texttt{Ant}, \texttt{HalfCheetah}, {\texttt{Hopper}, and \texttt{Walker}} (see \Cref{fig:ant,fig:halfcheetah,fig:hopper,fig:walker}).
The additional difficulty in implementing the algorithm in continuous control experiments is that the analytical form of the policy improvement step is no longer computationally tractable because this would require to compute an integral over the continuous action space. Therefore, we approximated this update using the Soft Actor Critic (SAC) \cite{Haarnoja:2018} algorithm. SAC requires environment samples making the algorithm online. The good empirical result opens the question of analyzing policy improvement errors as in \cite{Geist:2019}. 
\begin{figure*}[h]
    \centering
\begin{tabular}{ccccc}
\subfloat[Acrobot\label{fig:offline_acrobot}]{%
       \includegraphics[width=0.16\linewidth]{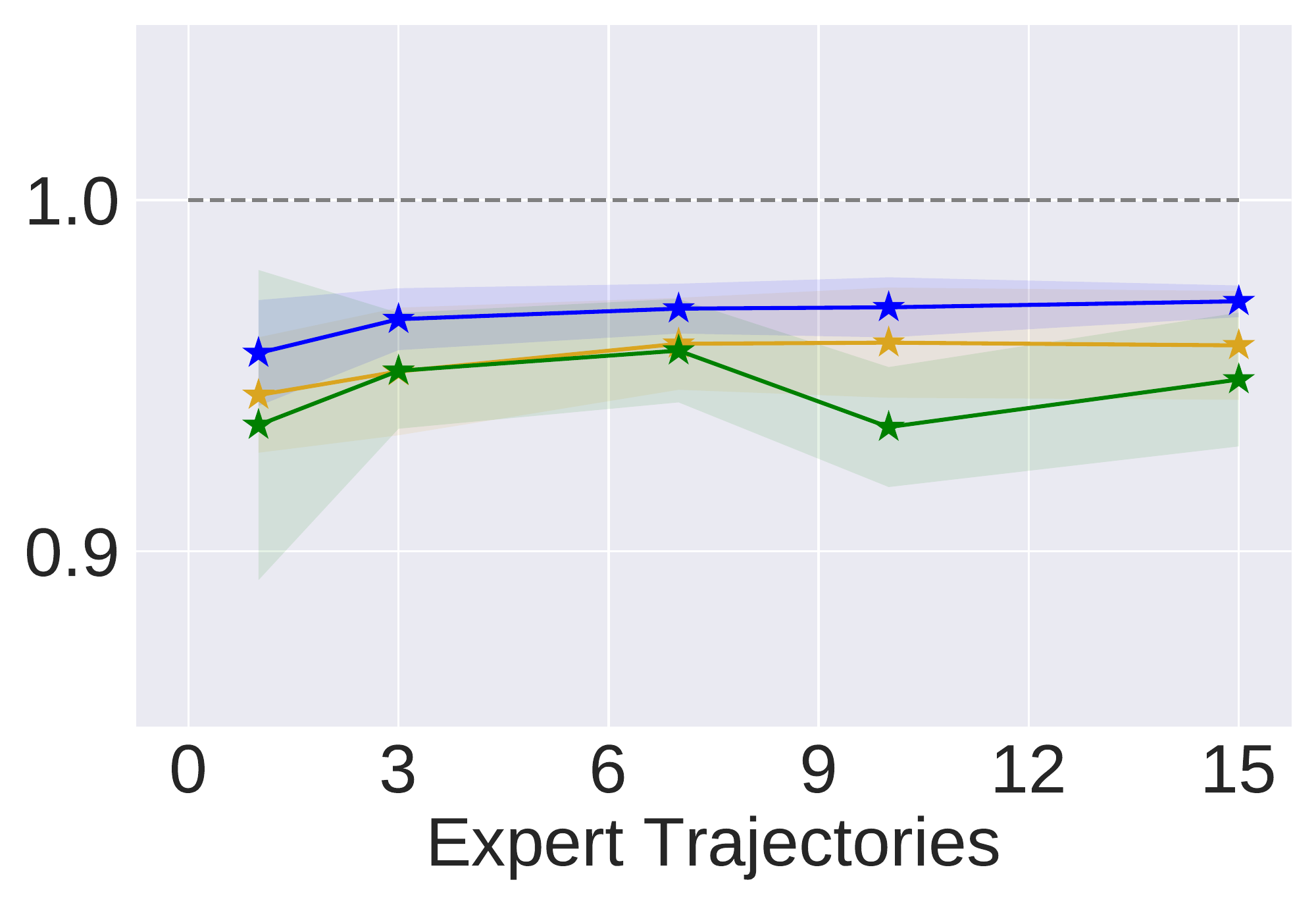}
     } &
     \subfloat[CartPole\label{fig:offline_cartpole}]{%
       \includegraphics[width=0.16\linewidth]{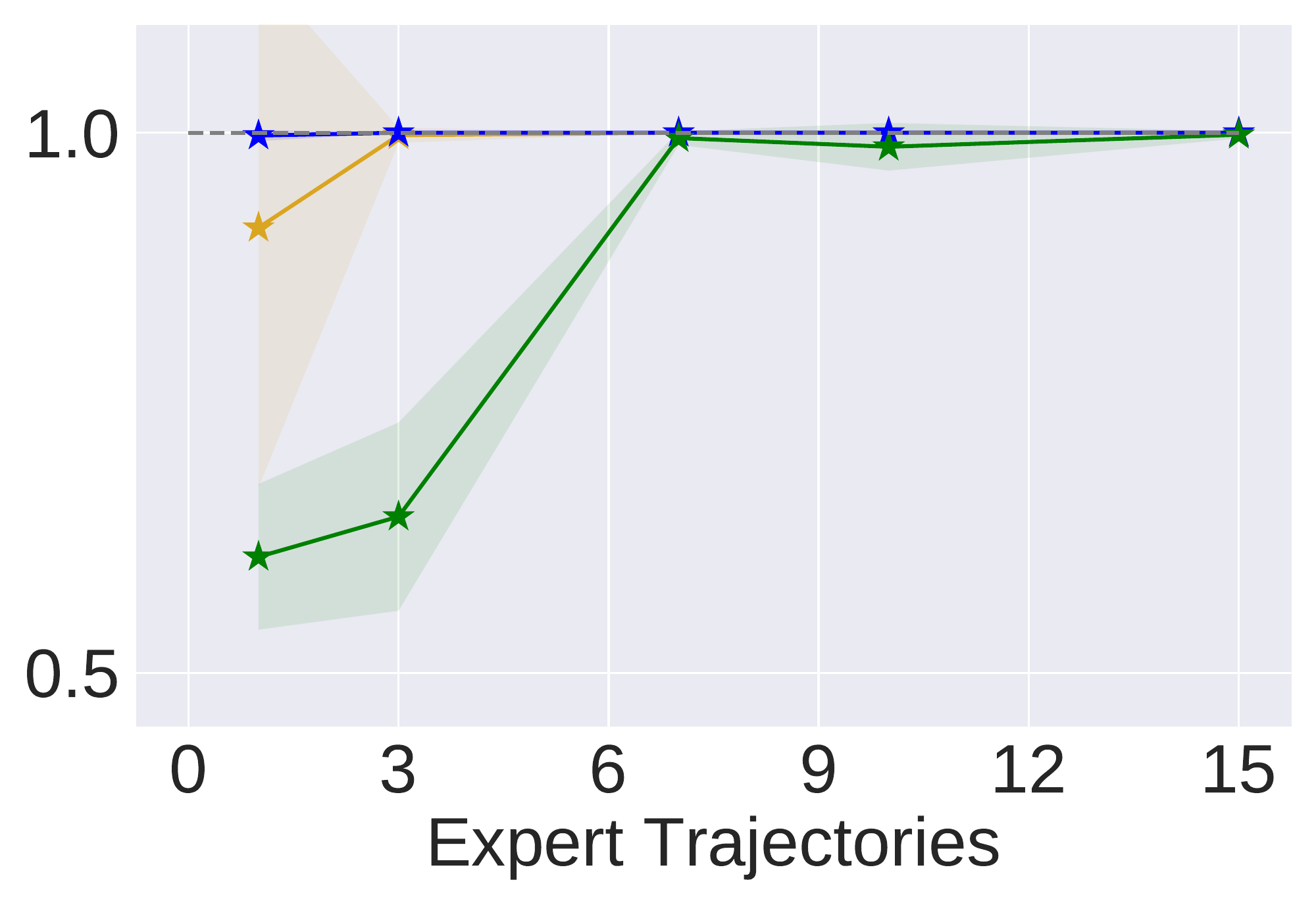}
     } &
     \subfloat[LunarLander\label{fig:offline_lunarlander}]{%
       \includegraphics[width=0.16\linewidth]{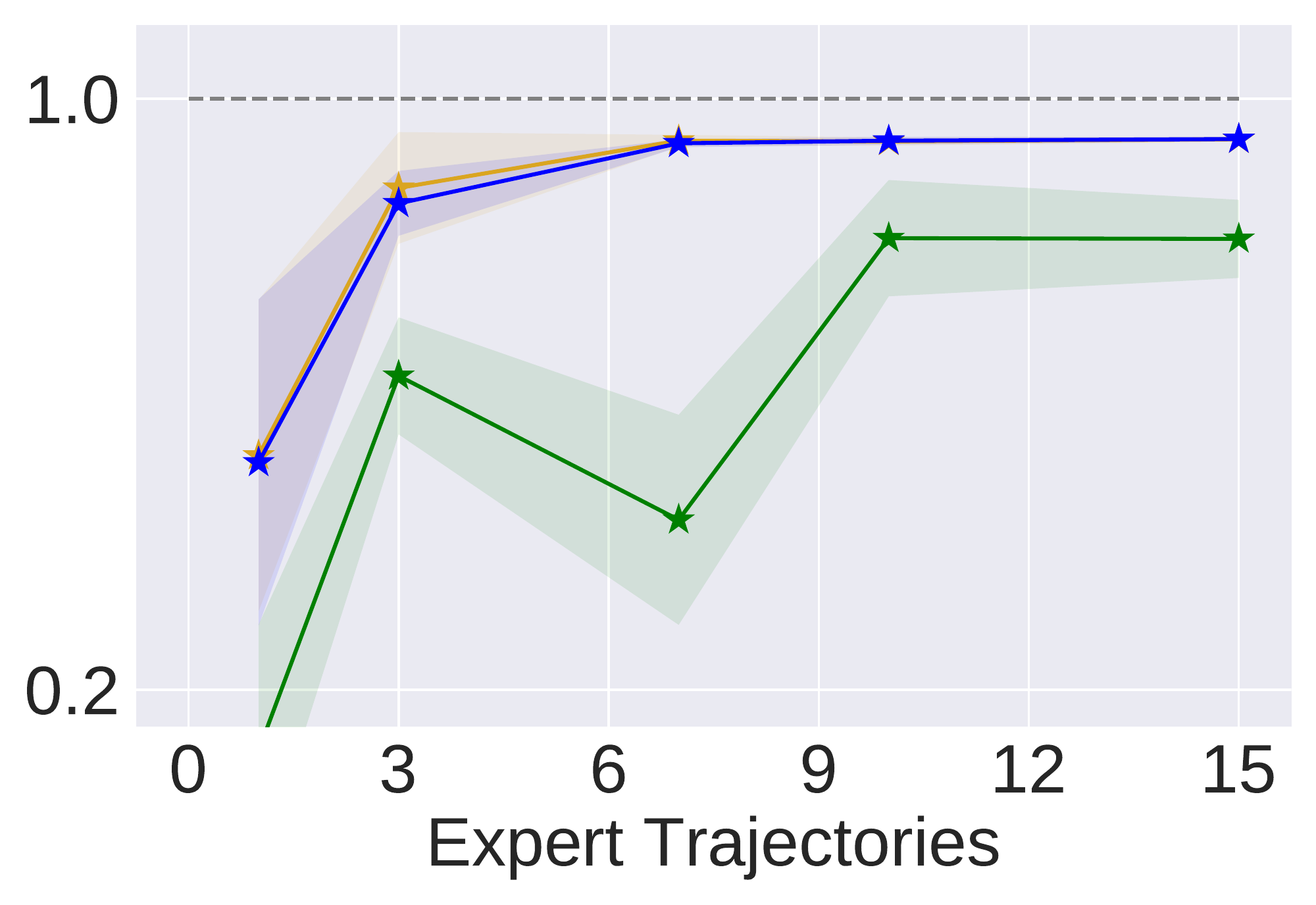}
     } 
     &
     \subfloat[Pong\label{fig:pong}]{%
       \includegraphics[width=0.16\linewidth]{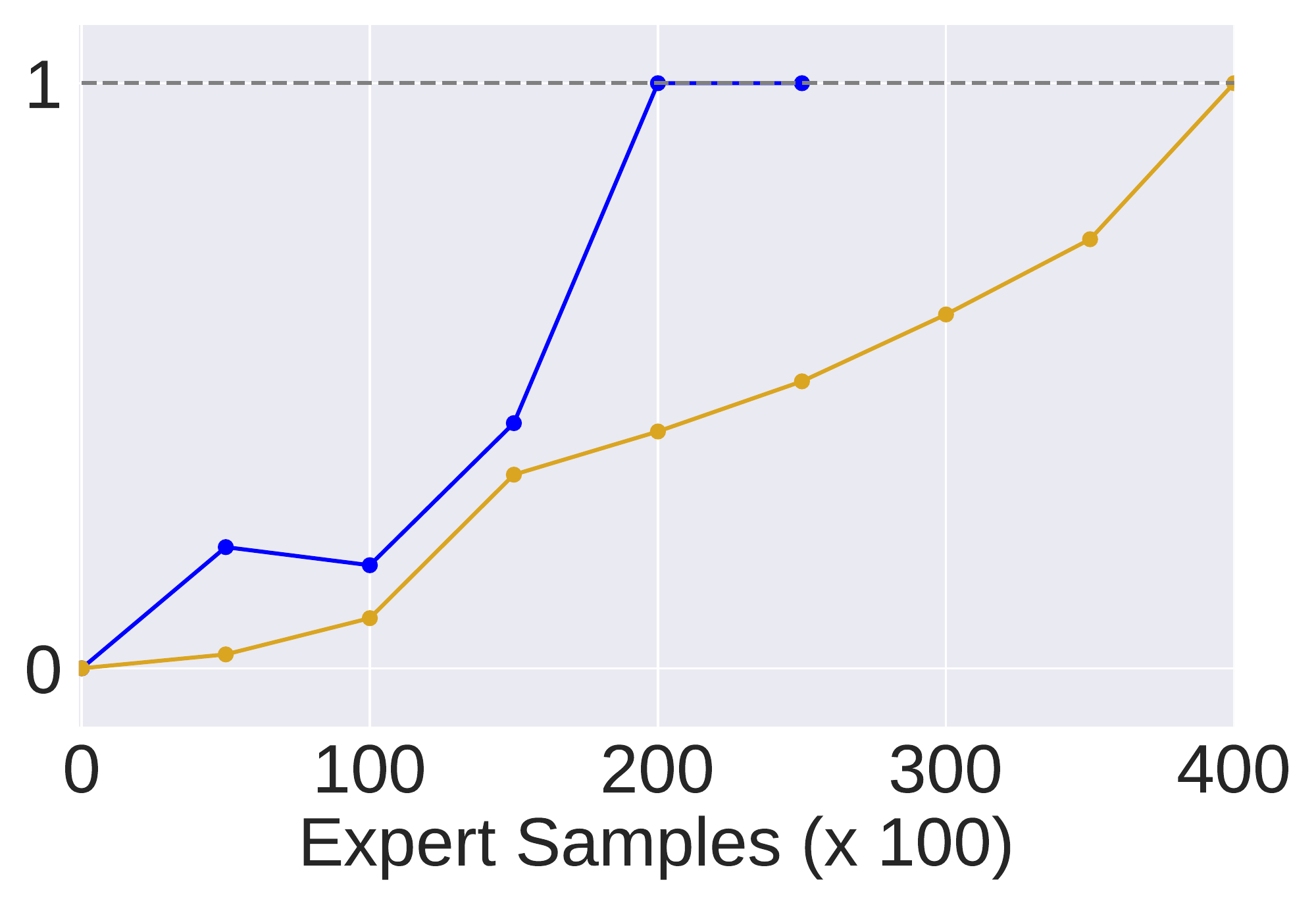}
     } &
       \includegraphics[width=0.16\linewidth]{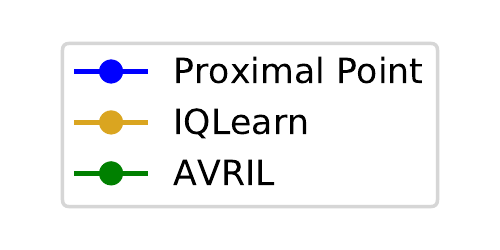}
      \\
     \subfloat[HalfCheetah\label{fig:halfcheetah}]{%
       \includegraphics[width=0.16\linewidth]{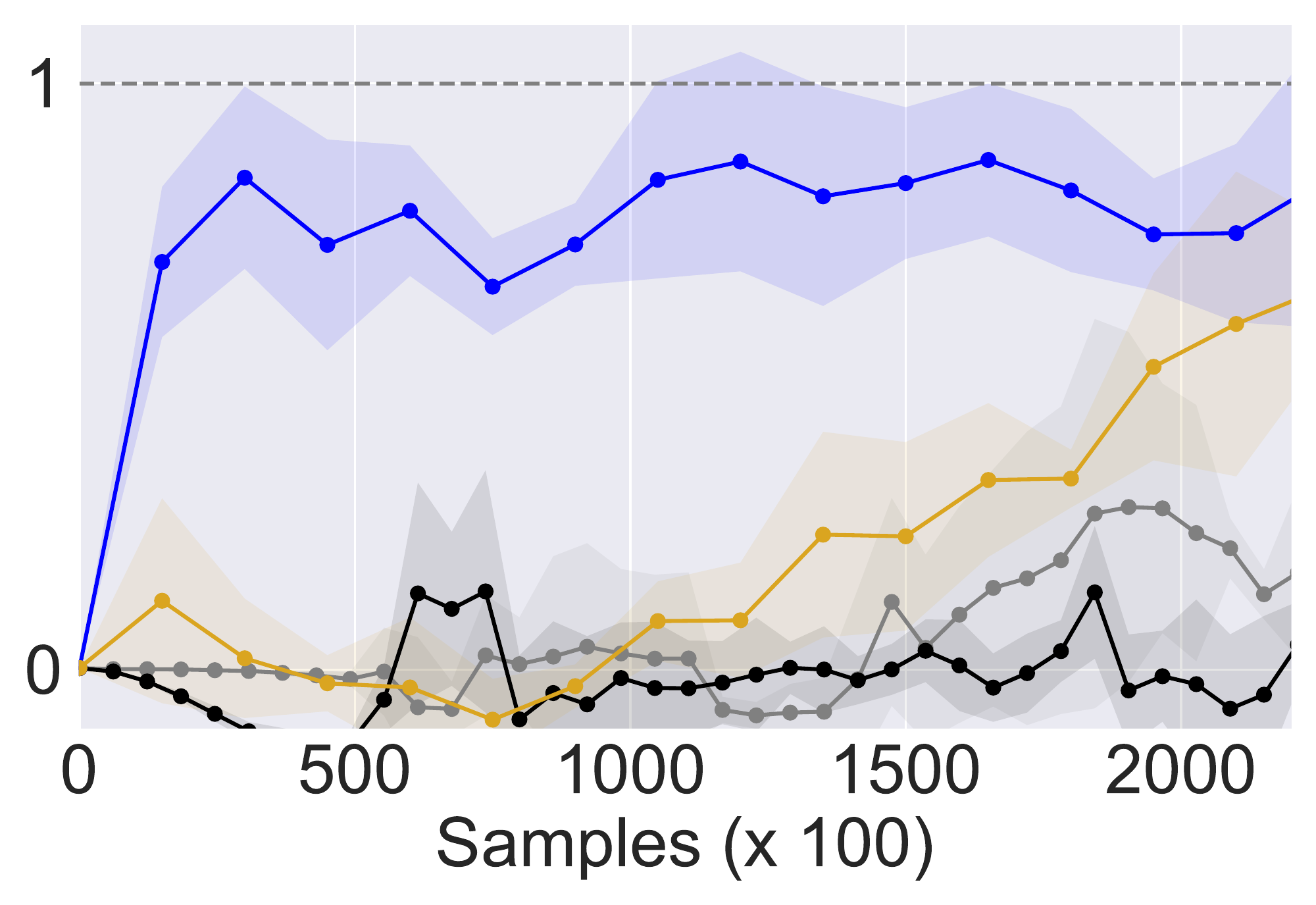}
     } &
     \subfloat[Ant\label{fig:ant}]{%
       \includegraphics[width=0.16\linewidth]{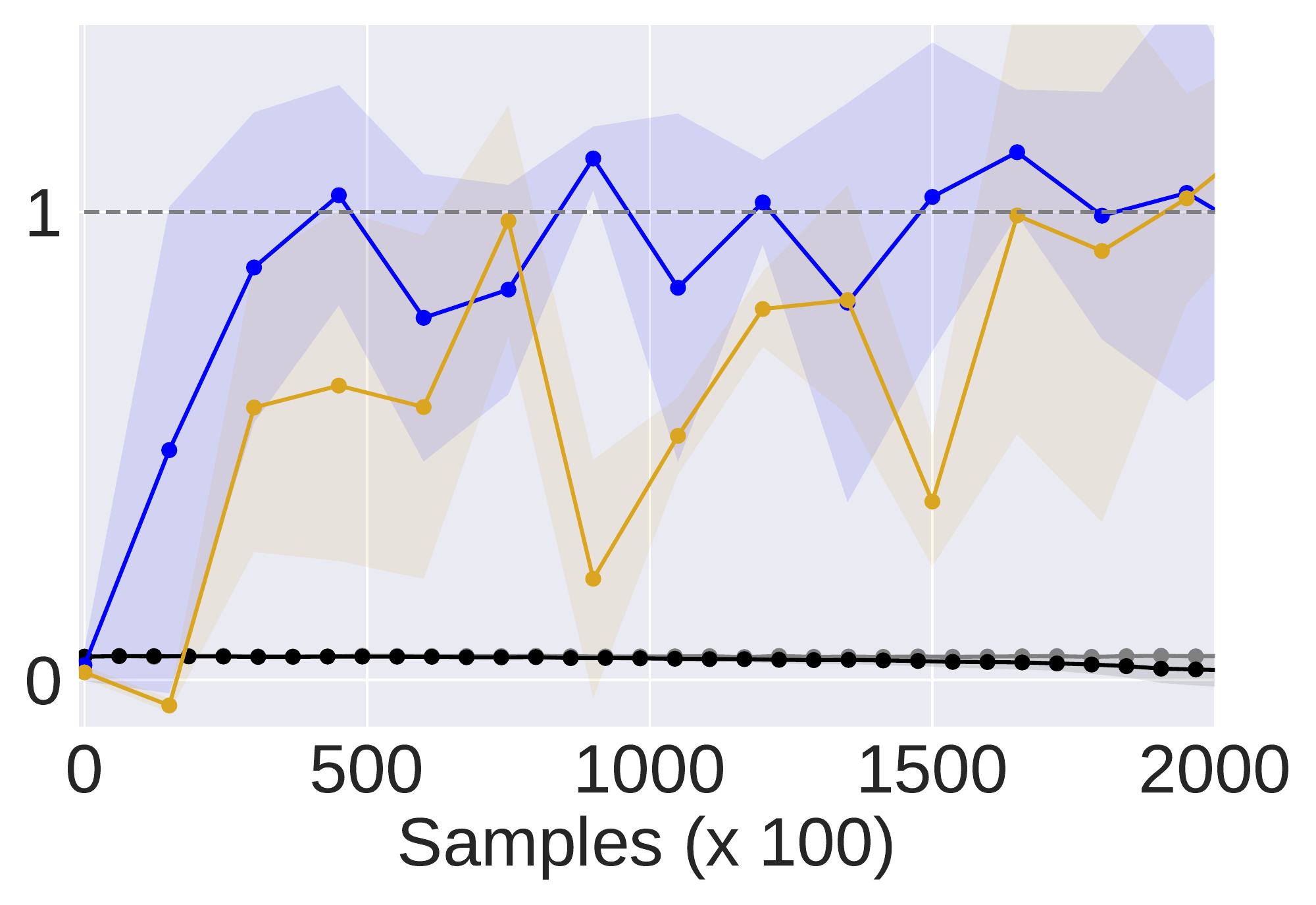}
     } &
    \subfloat[Hopper\label{fig:hopper}]{%
       \includegraphics[width=0.16\linewidth]{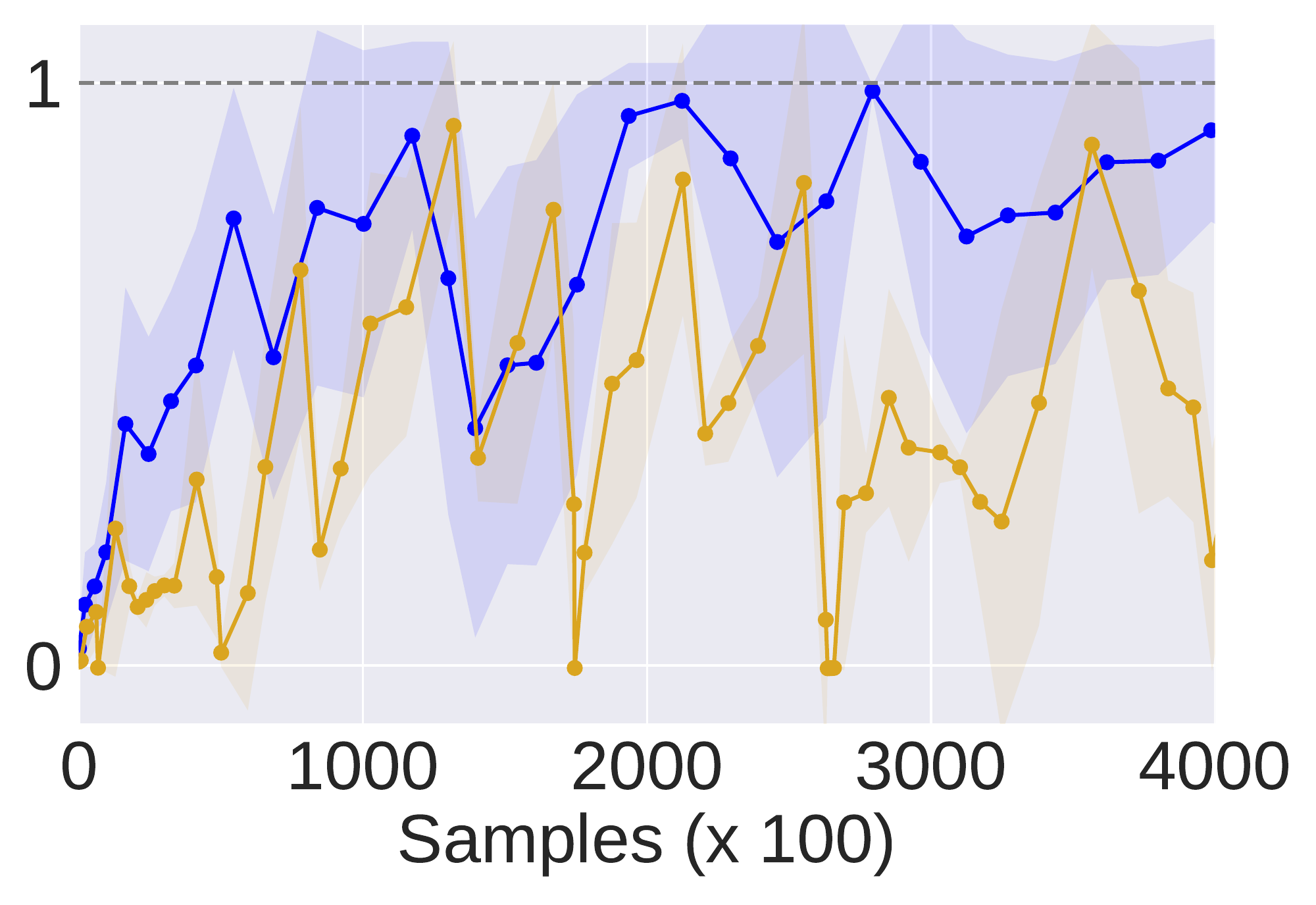}
     } &
     \subfloat[Walker2d\label{fig:walker}]{%
       \includegraphics[width=0.16\linewidth]{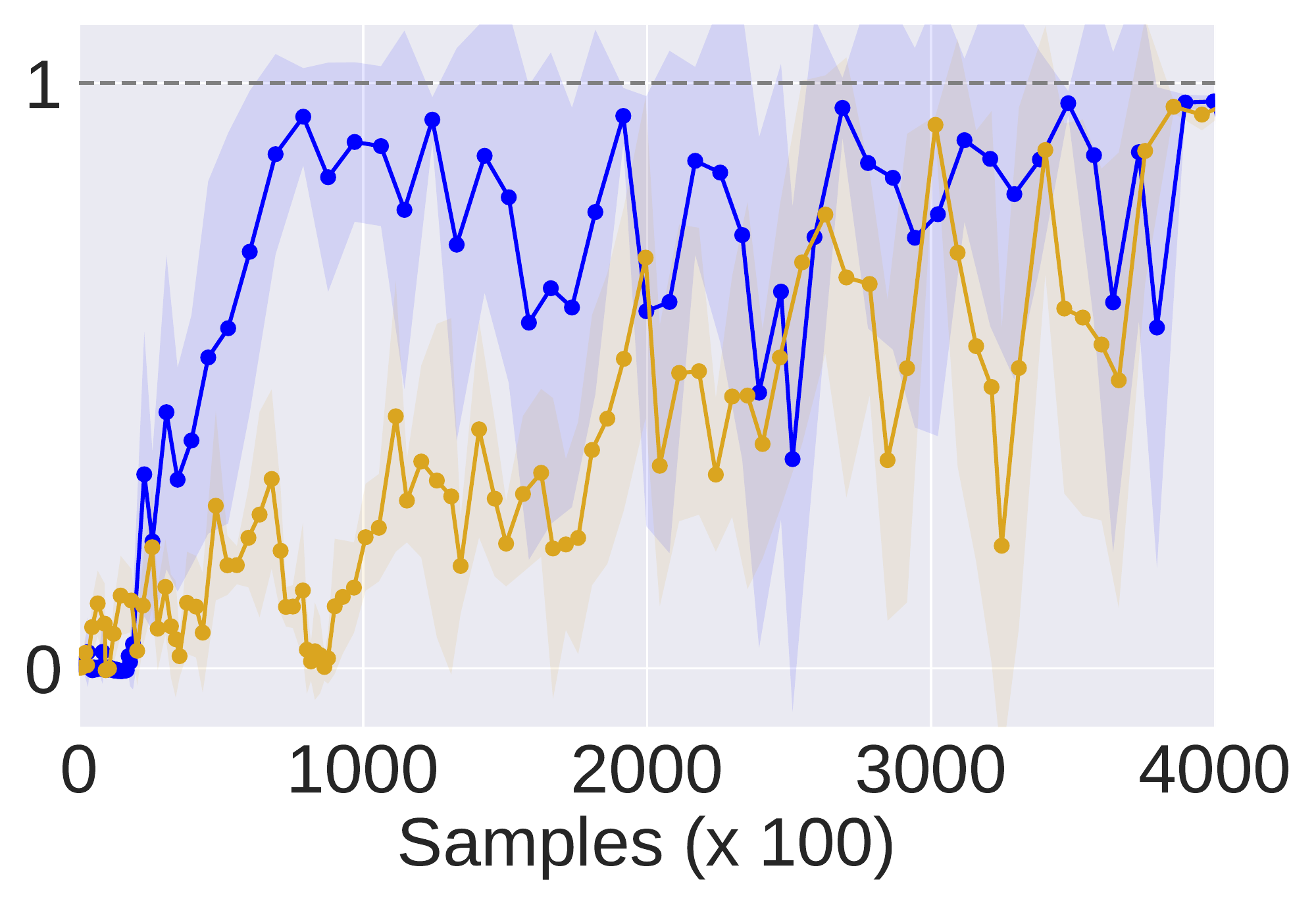}
     } &
       \includegraphics[width=0.16\linewidth]{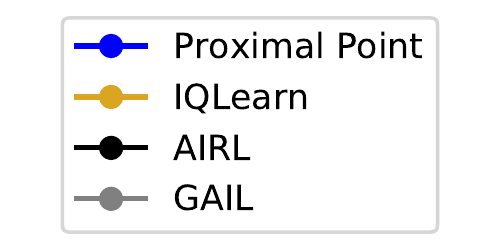}
      \\
\end{tabular}
\caption{\textbf{Neural function approximation experiments.} \Cref{fig:offline_cartpole,fig:offline_acrobot,fig:offline_lunarlander} show the total returns vs the number of expert trajectories. \Cref{fig:ant,fig:halfcheetah,fig:hopper,fig:walker} show the total returns vs the number of env steps. \Cref{fig:pong} shows the total return vs the number of expert state-action pairs.
\label{fig:offline_experiments}}
\end{figure*}
\label{sec:experiments}


\textbf{Recovered Costs.}
 A unique algorithmic feature of the proposed methodology is that we can explicitly recover a cost along with the $\mathsf{Q}$-function without requiring adversarial training. In Figure~\ref{fig:gridworld_cost}, we visualize our recovered costs in a simple 5x5 \texttt{Gridworld}. Most importantly, we verify that the recovered costs induce nearly optimal policies w.r.t. the unknown true cost function.  Compared to I$\mathsf{Q}$-Learn~\citep{Garg:2021}, we do not require knowledge or further interaction with the environment. Therefore, the recovered cost functions show promising transfer capability to new dynamics. 
 \begin{figure}[h]
    \centering
\begin{tabular}{cccc}
\subfloat[$- \mbf{c}_{\mathrm{true}}$]{%
       \includegraphics[height=5cm]{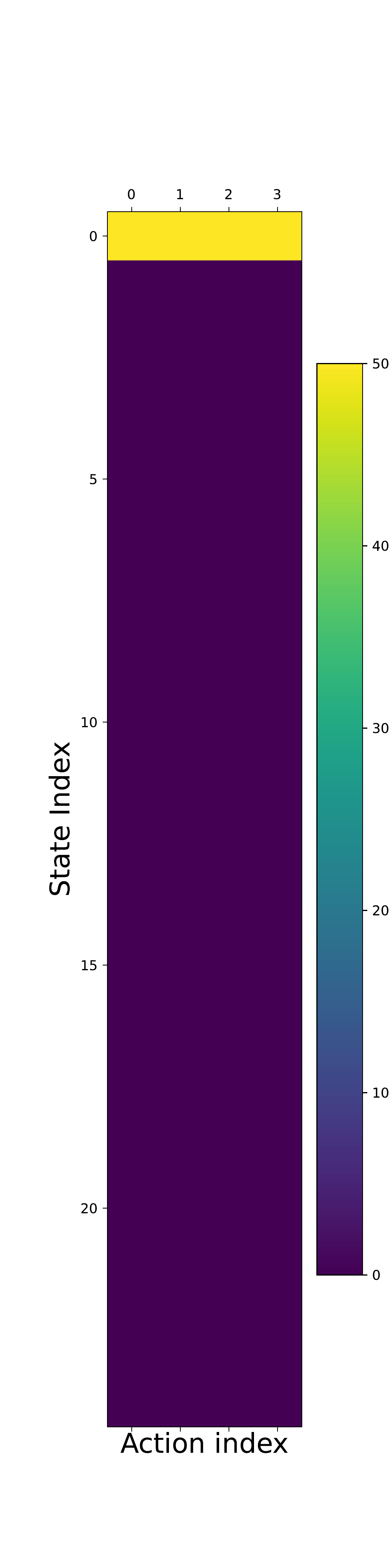}
     } &
     \subfloat[$- \mbf{c}_{K}$]{%
       \includegraphics[height=5cm]{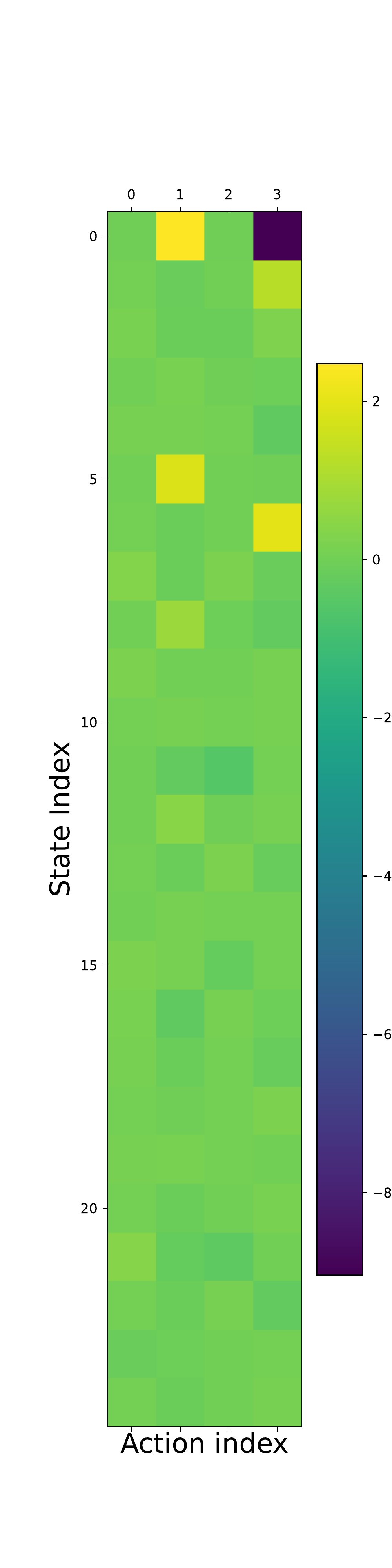}
     } &
\subfloat[-$V^\star_{\mbf{c}_{\mathrm{true}}}$]{%
       \includegraphics[height=3cm]{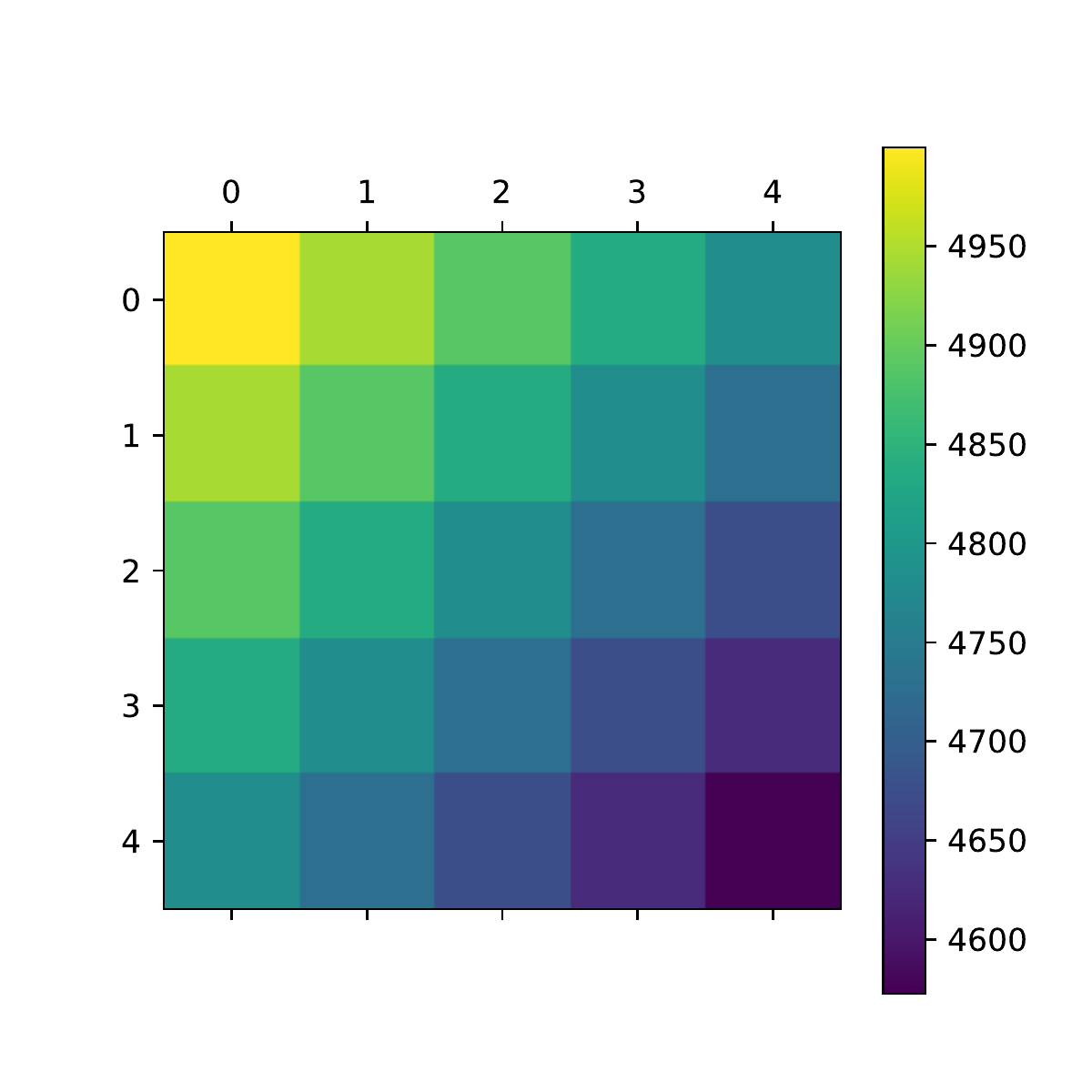}
     } &
     \subfloat[-$V^\star_{\mbf{c}_{K}}$]{%
       \includegraphics[height=3cm]{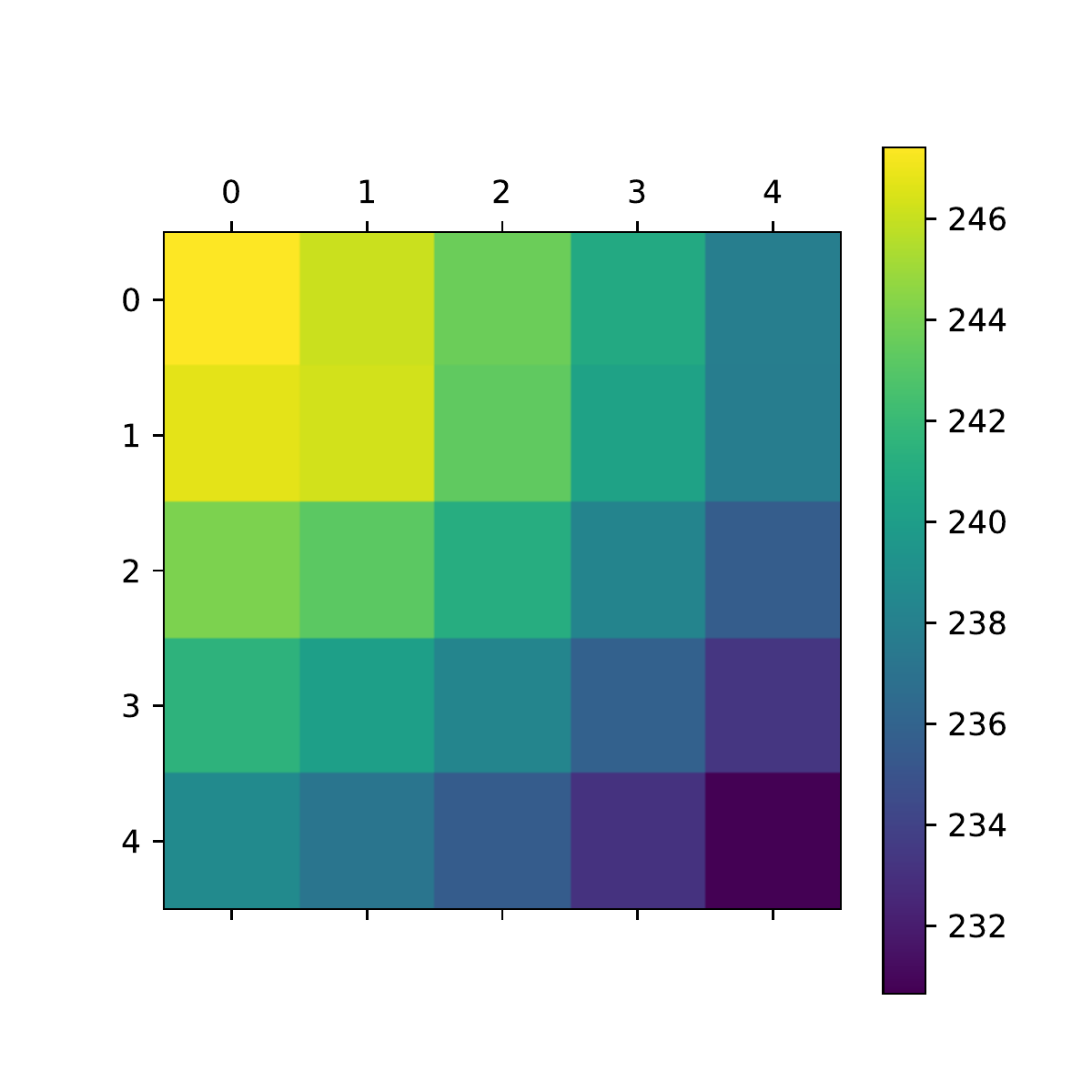}}
\end{tabular}
\caption{\textbf{Recovered Costs in \texttt{Gridworld}.} Comparison between the true cost $\cost_{\mathrm{true}}$ and the cost $\cost_K$ recovered by \texttt{P$^2$IL}. We notice that the optimal value functions $V^\star_{\mbf{c}_{\mathrm{true}}}$ and $V^\star_{\mbf{c}_{K}}$ present the same pattern. Hence, the optimal policy with respect to $\mbf{c}_{K}$ is nearly optimal with respect to $\mbf{c}_{\mathrm{true}}$.
\label{fig:gridworld_cost}}
\end{figure}

\textbf{Cost Transfer Setting.} We experimented with a transfer cost setting on a \texttt{Gridworld} (Figure~\ref{fig:transfer_cost}).
We consider two different Gridworld MDP environments, say $M$ and $\widetilde{M}$, with opposite action effects. This means that action \texttt{Down} in $\widetilde{M}$ corresponds to action \texttt{Left} in $M$ and vice versa. Similarly, the effects of \texttt{Up} and \texttt{Right} are swapped between $\widetilde{M}$ and $M$. We denote by $\val^\pi_{\widetilde{M},\cost_{\mathrm{true}}}$ (resp. $\val^\star_{\widetilde{M},\cost_{\mathrm{true}}})$ the value function of policy $\pi$ (resp. optimal value function) in the MDP environment $\widetilde{M}$ with cost function $\cost_{\mathrm{true}}$. Moreover, we denote by $\pi^\star_{{M},\cost}$ the optimal policy in the MDP environment $M$ under cost function $\cost$. Figure~(a) gives the corresponding optimal value function.
Figure~(b) presents the value function of the expert policy $\expert=\pi^\star_{M,\cost_{\mathrm{true}}}$ used as target by \texttt{P$^2$IL}. Figure~(d) shows the value function of the learned
 imitating policy $\pi_K$ from \texttt{P$^2$IL}. Finally, Figure~(b) depicts the value function of the optimal policy $\pi^\star_{\widetilde{M},\cost_K}$ for the environment $\widetilde{M}$ endowed with the recovered cost function $\cost_K$ by \texttt{P$^2$IL} (with access to samples from $M$).
 We conclude that the policy $\pi^\star_{\widetilde{M},\cost_K}$ is optimal in $\widetilde{M}$ with cost $\cost_{\mathrm{true}}$.
By contrast, the expert policy $\expert=\pi^\star_{M,\cost_{\mathrm{true}}}$ used as target by \texttt{P$^2$IL} performs poorly and as a consequence also the imitating policy $\pi_K$ does so. All in all,
we notice that the recovered cost induces an optimal policy for the new dynamics while the imitating policy fails.  Albeit, cost transfer is successful in this experiment we do not expect this fact to be true in general because we do not tackle the issue of cost shaping \citep{Ng:2000}.

\begin{figure}[h]
    \centering
\begin{tabular}{cccc}
\subfloat[$-\val^\star_{\widetilde{M},\cost_{\mathrm{true}}}$]{%
       \includegraphics[width=0.2\linewidth]{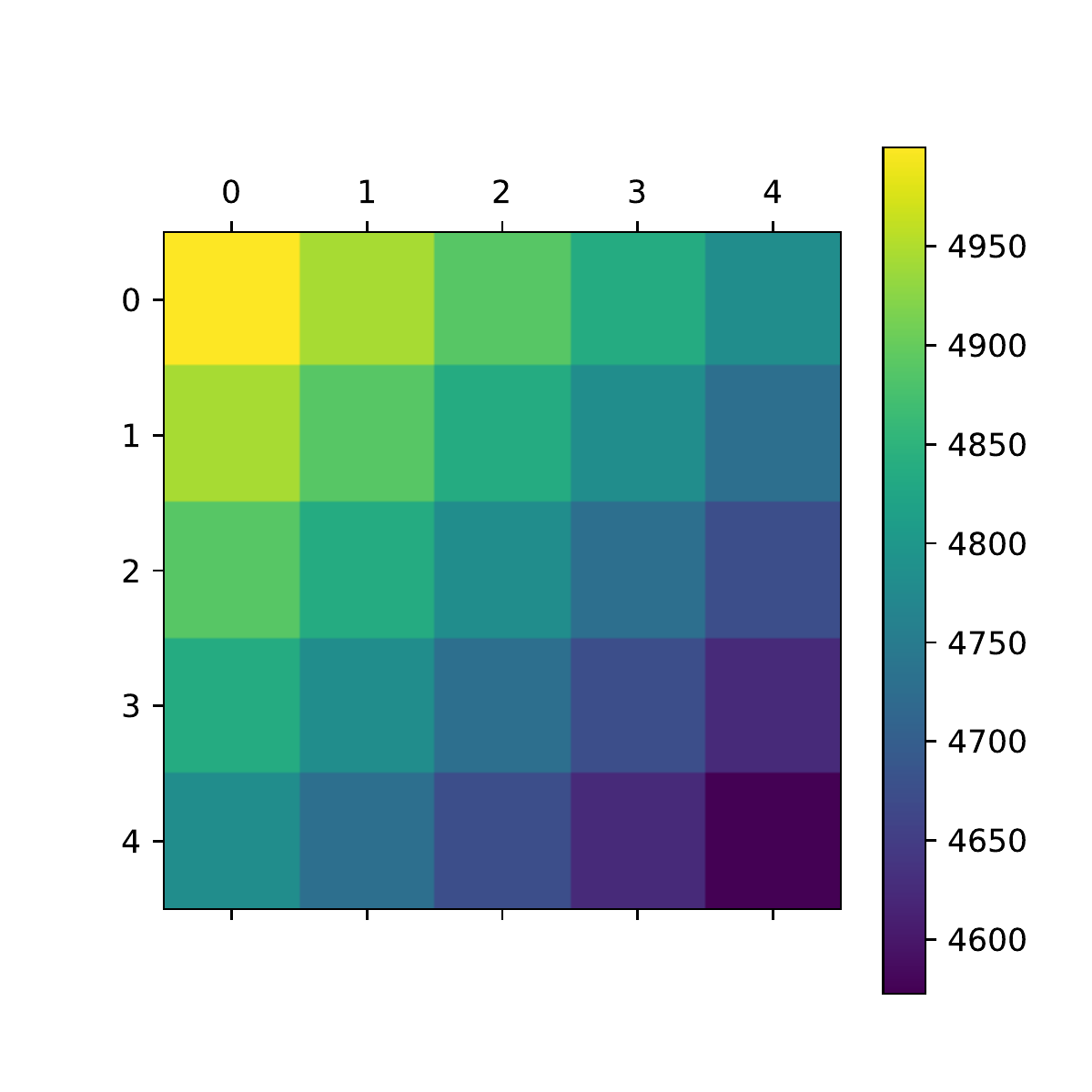}
     } &
     \subfloat[$-\val_{\widetilde{M}, \cost_{\mathrm{true}}}^{\pi^\star_{M,\cost_{\mathrm{true}}}}$]{%
       \includegraphics[width=0.2\linewidth]{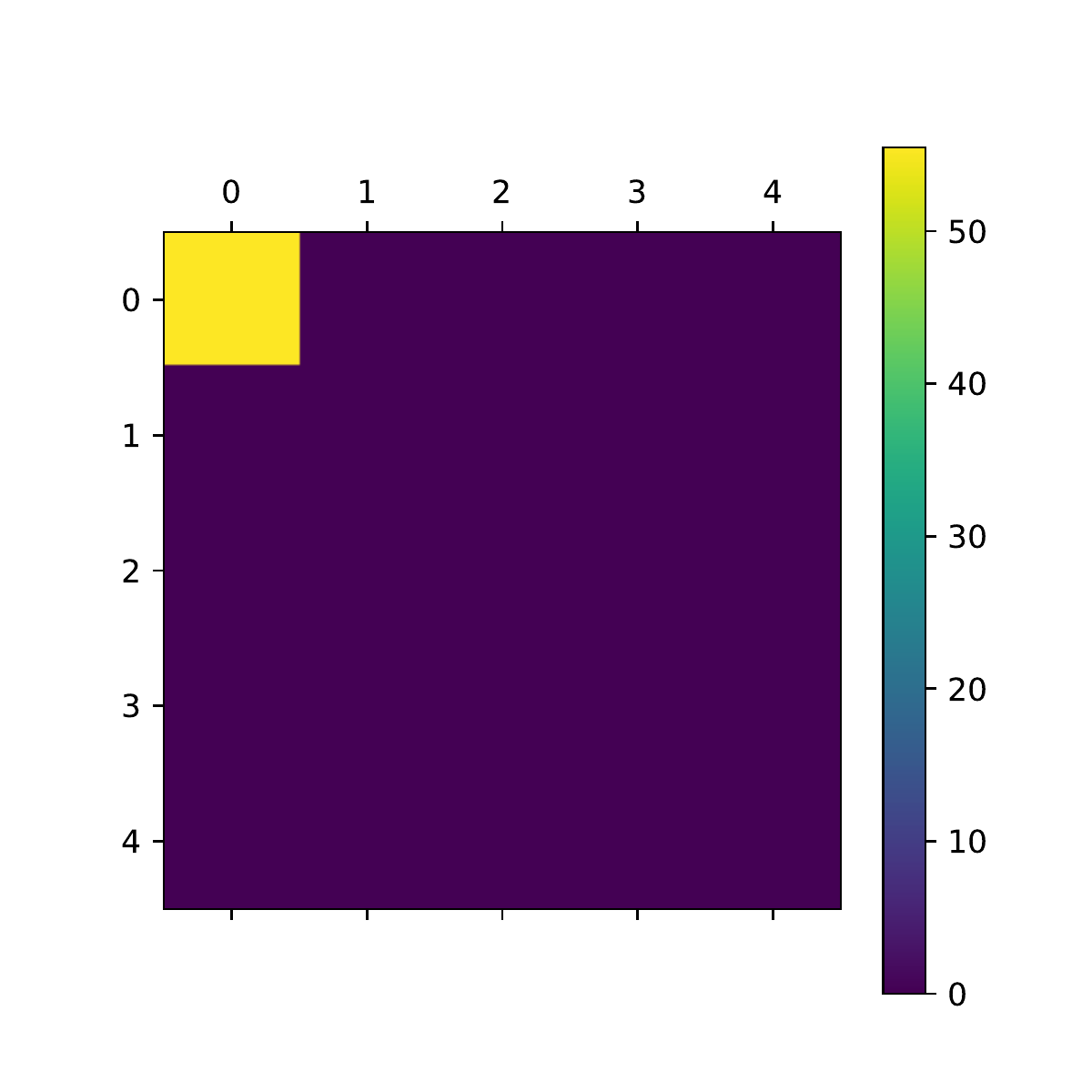}
     } &
\subfloat[$-\val_{\widetilde{M},\cost_{\mathrm{true}}}^{\pi^\star_{\widetilde{M},\cost_K}} $]{%
       \includegraphics[width=0.2\linewidth]{plot/cost/True_ValueWindyGrid-v0201.pdf}
     } &
     \subfloat[$-\val_{\widetilde{M}, \cost_{\mathrm{true}}}^{\pi_K}$]{%
       \includegraphics[width=0.2\linewidth]{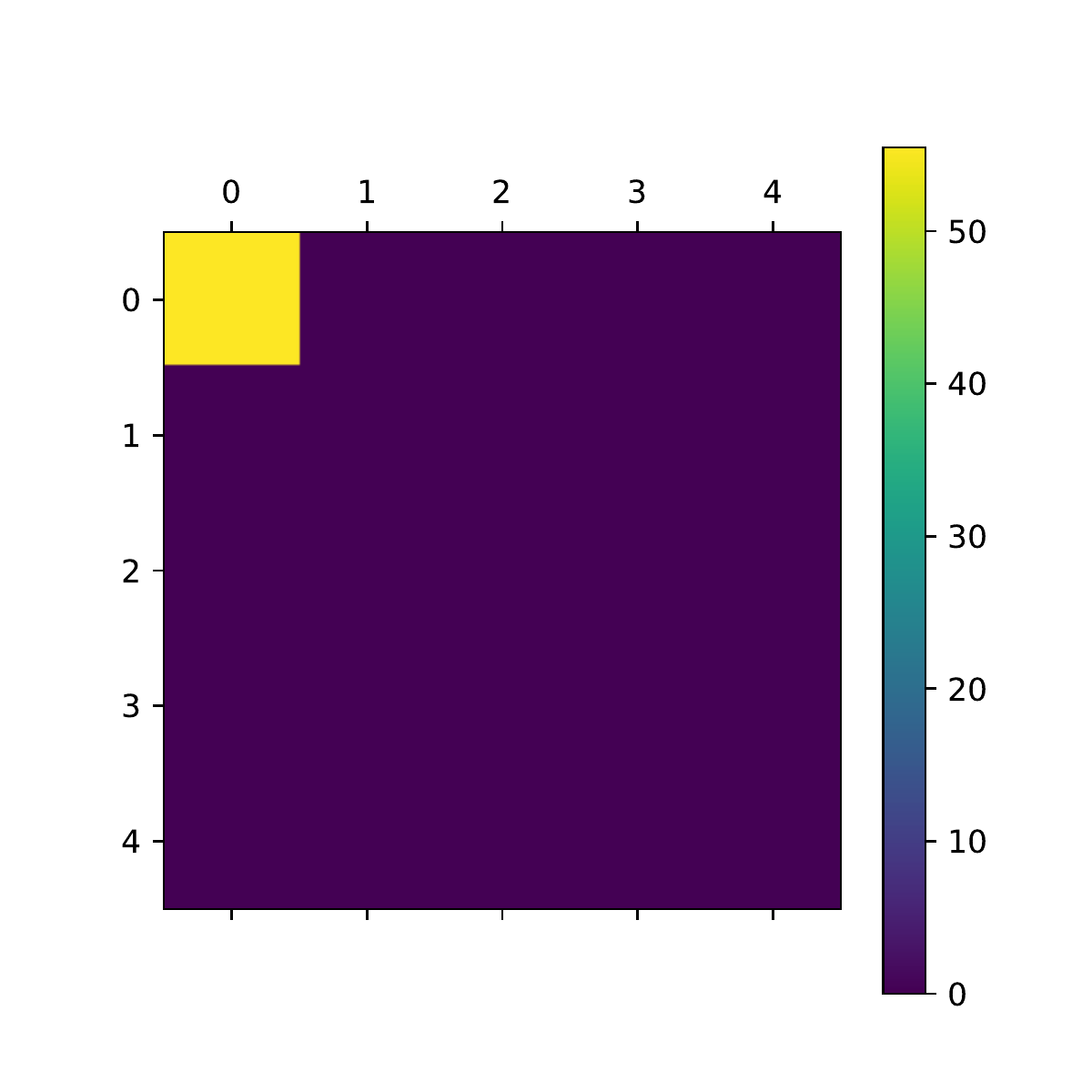}
     }
\end{tabular}
\caption{\textbf{Cost Transfer Experiment in \texttt{Gridworld}.} We compare the performance of several policies in the new MDP environment $\widetilde{M}$ with cost function $\true$. We notice that the recovered cost induces an optimal policy for the new dynamics while the imitating policy fails. 
\label{fig:transfer_cost}}
\end{figure}

\section{Discussion and Outlook}
In this work, we studied a Proximal Point Imitation Learning (\texttt{P$^2$IL}) algorithm with both theoretical guarantees and convincing empirical performance. Our methodology is rooted in classical optimization tools and the LP approach to MDPs.
The most significant merits of \texttt{P$^2$IL} are the following:
(i) It optimizes a convex and smooth logistic Bellman evaluation objective over both cost and Q-functions. In particular, it avoids instability due to adversarial training and can also recover an explicit cost along with Q function;
(ii) In the context of linear MDPs, it comes with efficient resource guarantees and error bounds for the suboptimality of the learned policy (Theorem~\ref{thm:biased_sgd} and Corollary~\ref{cor:sample_complexity}). In particular, given $\mathrm{poly}(1/\varepsilon,\log(1/\delta),m)$ many samples , it recovers an $\varepsilon$-optimal policy, with probability $1-\delta$. Notably, the bound is independent of the size of the state-action space;
(iii) Beyond the linear MDP setting, it can be implemented in a model-free manner, for both online and offline setups, with general function approximation without losing its theoretical specifications. This is justified by providing an error propagation analysis (Theorems~\ref{thm:error_propagation} and~\ref{thm:offline_error_propagation}), guaranteeing that small optimization errors lead to high-quality output policy; (iv) It enjoys not only strong theoretical guarantees but also favorable empirical performance. At the same time, our newly introduced methods bring challenges and open questions. One interesting question is whether one can accelerate the PPM updates and improve the convergence rate. 
Another  direction for future work is to provide rigorous arguments for the near-optimality of the recovered cost function. On the practical side, we plan to conduct experiments
in more challenging environments than MuJoCo and Atari. We hope our new techniques will be useful to future algorithm designers and lay the foundations for overcoming current limitations and challenges. In Appendix~\ref{app:future}, we point out in detail a few interesting future directions.

\section*{Acknowledgements}
The authors would like to thank the anonymous reviewer for their suggestions to improve the presentation and for motivating us to inspect the recovered cost function. This work has received funding from the Enterprise for Society Center (E4S), the European Research Council (ERC) under the European Union’s Horizon 2020 research and innovation programme grant agreement OCAL, No.~787845, the European Research Council (ERC) under the European Union's Horizon 2020 research and innovation programme (grant agreement n° 725594 - time-data), the Swiss National Science Foundation (SNSF) under  grant number 200021\_205011. Gergely Neu was supported by the European Research Council (ERC) under the
European Union’s Horizon 2020 research and innovation programme (Grant agreement No.~950180).
Luca Viano acknowledges travel support from ELISE (GA no 951847).
\bibliographystyle{abbrvnat}
\bibliography{neurips}

\newpage
\appendix
\section{Related Literature (Extended)}\label{app:related-literature}
In order to state our research questions and situate them among prior related theoretical and practical works, we provide an extended literature review.

\textbf{Theoretical Imitation Learning.} Our work is related to recent actor-critic IL schemes with theoretical guarantees for different MDP models, and different policy evaluation objectives (e.g., minimizing the squared Bellman error)~\cite{Cai:2019,Zhang:2020,Chang:2021,Liu:2022,Shani:2021}. Contrary to these actor-critic schemes, in our proximal-point imitation learning algorithm, the policy evaluation step involves optimization of a single objective over both cost and $Q$-functions. In this way, we avoid instability or poor convergence due to nested policy evaluation and cost update steps~\cite{Garg:2021} as well as the undesirable properties of the widely used squared Bellman error~\cite{Mnih:2015}. Moreover, for the context of linear MDPs~\cite{Bas-Serrano:2021, Yang:2019, Jin:2020, Cai:2020,  Wang:2019, Agarwal:2020b, Neu:2020}, we provide guarantees and convergence rates for the suboptimality of the learned policy, under mild assumptions, significantly weaker than those found in the literature
 until now. To our knowledge, such guarantees in this setting are provided for the first time. It is worth noting that in the case of continuous MDPs, despite being linear, the transition law can still have infinite degrees of freedom. This is a substantial difference from the recent theoretical works on IL~\cite{Cai:2019,Zhang:2020,Chang:2021,Liu:2022,Shani:2021} which consider either tabular MDPs~\cite{Shani:2021}, or a linear quadratic regulator~\cite{Cai:2019}, or a linear transition law that can be completely specified by a finite-dimensional matrix~\cite{Liu:2022}. In the last case, the degrees of freedom are bounded, and thus mitigate the challenges in estimating the transition model. Indeed, the linear MDP setting studied in~\cite{Liu:2022} reduces the unknown dynamics problem to estimating an unknown finite-dimensional matrix, which differs from our nonparametric approach. We also note that~\cite{Zhang:2020, Xu:2020} require the restrictive assumption of bounded concentrability coefficients, while this is not the case for the analysis in this paper. The convergence and generalization of actor-critic IL schemes for general MDPs has been studied in~\cite{Chen:2020a}. However, the authors in~\cite{Chen:2020a} only provide local optimality convergence guarantees, i.e., convergence to a stationary point. On the contrary, our algorithm provides global convergence guarantees for the linear MDP setting. Moreover, we account for potential policy evaluation errors , presenting an error
propagation analysis that leads to rigorous guarantees for both online and offline setting, beyond the linear MDP assumption. Indeed, it is worth noting that the provided error propagation analysis justifies using our derived actor-critic scheme with general function approximation. A scalable deep reinforcement learning implementation is possible, without losing the theoretical guarantees of Theorem~\ref{thm:error_propagation}. The work~\cite{Chang:2021} studies offline IL for the continuous kernelized nonlinear regulator and Gaussian process setting~\cite{Kakade:2020}. We notice that  this setting is different from the linear MDP model studied in this paper, and each one does not imply the other. Finally, a recent theoretical IL work that is rooted in the LP approach to MDPs is~\cite{Kamoutsi:2021}. The authors consider a Lagrangian reformulation of the problem and design a stochastic primal-dual algorithm with explicit performance bounds  on the quality of the extracted policy. The most important limitations of the primal-dual algorithm~\cite{Kamoutsi:2021} are (i) the need of a generative oracle, (ii) restricted coherence assumptions on the choice of features, as well as (iii) the problematic occupancy measure approximation. These limitations lead to poor practical performance for challenging high-dimensional and model-free IL setups.
On the other hand, our algorithm overcomes these difficulties by applying a proximal point update to an alternative $Q$-LP formulation~\cite{Mehta:2020,Neu:2021}. This results to a model-free actor-critic scheme with explicit tractable softmax policy updates. Compared with the setting in~\cite{Kamoutsi:2021}, 
where access to a generative-model oracle is assumed, we only have the ability to execute learned policies in the underlying MDP to generate trajectories. This assumption is considerably weaker that having a simulator-based MDP, however it is stronger than having ''irreversible experience'', where the learner must follow a single trajectory without having access to a \emph{reset action}, that obtains a new trajectory from the initial state distribution. Most importantly our algorithm enjoys not only strong theoretical guarantees, but also favorable practical performance.  

\textbf{Approximate Linear Programming.} There is an emerging body of literature~\cite{DeFarias:2003,Abbasi-Yadkori:2014,Chen:2018,Sutter:2017,Laksh:2018,MohajerinEsfahani:2018,Wang:2019,Lee:2019a,Banjac:2019,Beuchaut:2020,Martinelli:2020,Cheng:2020,Jin:2020,Shariff:2020,Bas-Serrano:2021} that studies ALP for the forward RL. While this approach dates back to 1960s~\cite{Manne:1960}, it has recently witnessed an interesting renaissance for its potential to provide a solid formal framework for newly derived methods, as well as a deeper understanding of existing empirically successful algorithms. In this paper, we present scalable imitation learning algorithms with theoretical guarantees rooted in the LP approach, highlighting how historical key limitations have been eliminated. Prior approximate linear programming (ALP) approaches developed algorithms for solving large-scale and/or continuous MDPs on a low-dimensional subspace by reducing the number of constraints (e.g., by constraint sampling)~\cite{DeFarias:2003,DeFarias:2004}. However, these prior works either scale badly with the size of the state-action spaces or require access to samples from a distribution that depends on the optimal policy. Moreover, they focus mainly on the approximation of the optimal value but not so much on extracting a near optimal policy. On the other hand, a recent line of works~\cite{Chen:2018,Lee:2019a,Wang:2019,Jin:2020,Shariff:2020} solve the problem for large-scale MDPs by employing stochastic primal-dual methods, in light of Lagrangian duality. Although this approach achieves state-of-the-art sample complexity guarantees, it shows poor performance in practice. First, current primal-dual algorithms need access to a simulator, mitigating implicitly the problem of exploration, Second, when dealing with linear relaxations of MDPs \cite{Bas-Serrano:2021, Kamoutsi:2021} one needs to impose a restrictive coherence assumptions to ensure that small duality gap for the linearly relaxed LP implies small suboptimality gap for the extracted policy. Finally, while their is enough intuition behind the use of linear function approximation for value functions, this is not the case for occupancy measure approximation. A new breed of algorithms that seem to overcome these difficulties is based on an alternative $Q$-LP formulation of RL. This approach has been first introduced by Mehta and Meyn~\cite{Mehta:2009} and has been recently revisited by~\cite{Bas-Serrano:2021,Neu:2020,Lee:2019a,Neu:2021,Mehta:2020,Lu:2021}. A salient feature of this equivalent formulation is that it introduces a $Q$-function as slack variables, and so lends itself to data-driven algorithms. Our work is inspired by these line of works. The most related works are the analysis of  REPS/Q-REPS~\cite{Peters:2010, Bas-Serrano:2021, Pacchiano:2021} and O-REPS~\cite{Zimin:2013} that first pointed out the connection between REPS and PPM. We build on their techniques with some important differences. In particular, while in the LP formulation of RL, PPM and mirror descent \cite{Beck:2003, Hazan:2016} are equivalent, recognizing that they are \textit{not equivalent} in IL is critical for stronger empirical performance. Moreover, our techniques can be used to improve upon the best rate for REPS in the tabular setting \cite{Pacchiano:2021} and to extend their guarantees to Linear MDPs.

\textbf{State-of-the-art Imitation Learning.} 
 Generative adversarial imitation learning (GAIL) \cite{Ho:2016} and other follow-up works \cite{Fu:2018, Ke:2020, Kostrikov:2019, Kostrikov:2020} formulate the IL as a minimax adversarial problem similar to a GAN~\cite{Goodfellow:2014} and leverage primal-dual optimization tools. In particular, GAIL solves IL with alternating updates of both policy and cost functions.
 On the other hand, a recent line of work~\cite{Garg:2021,Barde:2020,Reddy:2020}  bypasses the need of optimizing over cost functions and thus avoids instability due to adversarial training. Although these algorithms achieve impressive empirical performance in challenging high dimensional benchmark tasks, they are hampered by limited theoretical understanding. This is the fundamental difference from our work, which enjoys both favorable practical performance and strong theoretical guarantees. Moreover, a unique algorithmic feature of our proposed methodology is a convex and smooth logistic policy evaluation objective that optimizes jointly cost and $Q$-functions. As a result, our algorithm has the additional practical benefit that can also recover an explicit cost along with the Q-function without requiring knowledge or further interaction with the environment (as  in~\cite{Garg:2021,Barde:2020,Reddy:2020}). Therefore, the recovered cost functions show promising transfer capability to new dynamics. In addition, unlike IQ-Learn~\cite{Garg:2021}, in our online IL algorithm, instead of regularizing the IL objective, the key idea is to penalize the divergence between the current policy and the policy obtained at the previous iteration. We do so by employing a Bregman proximal point update. Most importantly, as we have already highlighted, the convergence properties of~\cite{Garg:2021,Barde:2020,Reddy:2020}) remain largely elusive in the function approximation and model-free regime.  It is unclear whether the sampling-based variants of their  algorithms converge to a global optimum or if they converge at all, even for the simple tabular setting.

\section{Future directions}\label{app:future}
In this work, we studied a proximal point imitation learning algorithm with both theoretical guarantees and convincing empirical performance in challenging benchmark tasks. Our methodology is rooted in classical stochastic optimization tools and in the LP approach to MDPs. We hope that our new techniques will be useful for future algorithm designers and lay foundations for overcoming current limitations and challenges. We point out a few interesting directions.

\textbf{Accelerated proximal point.}  An appealing possibility is to study an accelerated proximal point scheme with inexact updates and achieve faster convergence rates. While there has been an effort in this direction \cite{Yan:2020,Yang:2021}, the acceleration relies on the triangle/quadrangle scaling property assumption \cite{Hanzely:2021} that does not hold for KL divergence over the simplex. Understanding if it is possible to accelerate PPM without such an assumption is an open question, whose solution has direct application to the LP formulation of RL and imitation learning. 

\textbf{Primal-dual methods with conditional relative entropy.} Recent primal-dual RL algorithms rooted in the LP approach to MDPs achieve state-of-the-art sample complexity guarantees. See, for example, \cite{Bas-Serrano:2020} for exact gradients, \cite{Carmon:2019, Jin:2020} for stochastic gradients, and \cite{Kamoutsi:2021} for the imitation learning problem. The most important disadvantages of primal-dual RL algorithms are (i) the need of a generative oracle, (ii) restricted coherence assumptions on the choice of features, as well as (iii) the problematic occupancy measure approximation. Unfortunately, these limitations lead to poor practical performance for challenging high-dimensional and model-free RL and IL setups.
On the other hand, our algorithm overcomes these difficulties but requires to approximately solve a small-dimensional convex program repetitively. It is also challenging to account for the biased gradient estimates beyond the linear MDP setting.
It is promising to investigate if by combining the alternative $Q$-LP formulation and the conditional relative entropy as Bregman divergence in a primal-dual mirror descent scheme, one can avoid the current practical limitations of primal-dual RL methods. It is also interesting that in this case, the action-value parameters will be updated by taking one gradient step each time, instead of solving a small-dimensional convex program.

\textbf{Inexact policy improvement update.}
The error propagation presented in this work accommodates for errors only in the policy evaluation phase, while it assumes that the policy improvement step can be implemented exactly. This happens in other related works like \cite{Bas-Serrano:2021, Vieillard:2020}. In contrast, the error propagation analysis in \cite{Geist:2019} takes into account an error in the policy improvement step but unfortunately it does not provide a way to ensure that such an error is small. 
Future research effort will aim to include in our error propagation analysis a term given by inexact policy improvement steps, ensure that such errors are small and characterizing the deterioration in the sample complexity under policy improvement errors. This kind of analysis would be important for continuous actions environment where the softmax policy update can not be computed in closed-form.
\section{Dual Program Interpretation}~\label{app:strong-duality}
To motivate further the~\ref{eq:primal} formulation, we shed light to its dual and provide an interpretation of the dual optimizers. For brevity, we focus on the case $\mathcal{W}=B_1^m$. The proof can be found in Appendix~\ref{app:strong-duality-proof} and is based on strong duality between the two convex programs. 
	
	\begin{proposition}\label{prop:primal-dual-q-function}
	The dual convex program is given by
			\begin{multline}\label{eq:dual}
\zeta^\star=\max_{(\weight,\val,\qval)}\Big\{(1-\gamma)\innerprod{\initial}{\val}-\innerprod{\mv_{\expert}}{\cost_{\weight}}\mid \qval\geq\bmat\val,\;
\qval=\cost_{\weight}+\gamma\pmat\val,\\
\;\val\in\ar^{|\sspace|}, \tag{\color{blue}Dual}
\;\qval\in\ar^{|\sspace||\aspace|},\;\weight\in\mathcal{W}\Big\}. 
\end{multline}
Moreover, for $\mathcal{W}=B_1^m$, a triple $(\val_{\textup{A}},\qval_{\textup{A}},\weight_{\textup{A}})$ is dual optimal if and only if (i) $\expert$ is optimal for the RL problem with cost $\cost=\cost_{\wa}$, (ii) $\val_{\textup{A}}=\mbf{V}_{\weight_{\textup{A}}}^\star$, (iii) $\qval_{\textup{A}}=\mbf{Q}_{\weight_{\textup{A}}}^\star$, and (iv) $\weight_{\textup{A}}\in\mathcal{W}$. In particular, $(\mbf{V}^\star_{\mbf{w_{\textup{true}}}},\mbf{Q}^\star_{\mbf{w_{\textup{true}}}},\mbf{w_{\textup{true}}})$ is a dual optimizer.
			\end{proposition}
						Proposition~\ref{prop:primal-dual-q-function} states that the set of dual optimal costs $\cost_{\weight_{\textup{A}}}$ is the set of costs in $\mathcal{C}$  for which the expert is optimal. In this case, the optimal $\val_{\textup{A}}$ coincides with the corresponding optimal value function\footnote{To be precise, this is the case if $\initial\in\Re^{|\sspace|}_{++}$, otherwise they coincide $\initial$-almost surely.}, while the optimal $\qval_{\textup{A}}$ coincides with the corresponding optimal state-action value function. In particular, the true weights $\mbf{w_{\textup{true}}}$, the true optimal value function $\mbf{V}^\star_{\mbf{w_{\textup{true}}}}$ and the true optimal state-action value function $\mbf{Q}^\star_{\mbf{w_{\textup{true}}}}$ are dual optimizers. Therefore, the presented $Q$-convex approach allows to recover an optimal solution to the original problem~(\ref{eq:IL}) from both the~(\ref{eq:primal}) and~(\ref{eq:dual}) formulations:
						it can be obtained either as the induced policy of a primal optimal occupancy measure or as a greedy policy associated to a dual optimal $Q$-function. In \Cref{sec:PPM}, we  generalize the later observation to implement PPM using softmin updates in terms of $Q$-functions.
\subsection{Proof of Proposition~\ref{prop:primal-dual-q-function}} \label{app:strong-duality-proof}

	We recall the alternative $Q$-LP approach to MDPs~\cite{Mehta:2009,Mehta:2020,Neu:2020,Neu:2021,Bas-Serrano:2021}. Let $\cost\in\Re^{|\sspace||\aspace|}$ be a cost function. The forward RL problem is equivalent to the following linear programs\footnote{Note that usually in the literature the primal LP is~(\ref{MDP-dual}).}
\begin{align}
\rho_{\cost}^\star=&\min_{(\mv,\dv)\in\ar^{2|\sspace||\aspace|}}\{\innerprod{\mv}{\cost}\mid\bmat^\trans\dv=\gamma\pmat^\trans\mv+(1-\gamma)\initial,\;\dv=\mv,\;\dv\geq\boldsymbol{0}\}\tag{\color{blue}Primal $Q$-LP}\label{MDP-primal}\\
=& \max_{(\val,\qval)\in\ar^{|\sspace|+|\sspace||\aspace|}}\{(1-\gamma)\innerprod{\initial}{\uv}\mid \qval\geq\bmat\val,\;
\qval=\cost+\gamma\pmat\val,\;\val\in\ar^{|\sspace|}, \tag{\color{blue}Dual $Q$-LP}
\Big\},  \label{MDP-dual}
\end{align}
We have that if $\pi^\star$ is an optimal policy for the forward RL problem with cost $\cost$, then $(\mv_{\pi^\star},\mv_{\pi^\star})$ is optimal  for~(\ref{MDP-primal}) and conversely if $(\mv^\star,\dv^\star)$ is optimal for~(\ref{MDP-primal}), then $\pi_{\mv^\star}$ is an optimal policy for the forward RL problem with cost $\cost$. Moreover, $(\val_{\cost}^\star,\qval_{\cost}^\star)$ is an optimal solution to~(\ref{MDP-dual}) and it is the unique optimizer when $\initial\in\Re^{|\sspace|}_{++}$. For the following results, we will assume without loss of generality that $\initial\in\Re^{|\sspace|}_{++}$.

\begin{proofof}{Proof of Proposition~\ref{prop:primal-dual-q-function}} 

	We first derive the dual convex program. We have,
	\begin{align*}
	\zeta^\star&=\min_{(\mv,\dv)\in\mathfrak{M}}\max_{\weight\in\mathcal{W}}\innerprod{\mv-\mv_{\expert}}{\cost_{\weight}}\\
	&=\max_{\weight\in\mathcal{W}}\min_{(\mv,\dv)\in\mathfrak{M}}\innerprod{\mv-\mv_{\expert}}{\cost_{\weight}}\\
	&=\max_{\weight\in\mathcal{W}}\min_{\mv,\dv\geq\mbf{0}}\max_{\val,\qval}\{\innerprod{\mv-\mv_{\expert}}{\cost_{\weight}}+\innerprod{\gamma\pmat^\trans\mv+(1-\gamma)\initial-\bmat^\trans\dv}{\val}+\innerprod{\dv-\mv}{\qval}\}\\
&=	\max_{\weight\in\mathcal{W}}\min_{\mv,\dv\geq\mbf{0}}\max_{\val,\qval}\{(1-\gamma)\innerprod{\initial}{\val}-\innerprod{\mv_{\expert}}{\cost_{\weight}}+\innerprod{\mv}{\cost_{\weight}+\gamma\pmat\val-\qval}+\innerprod{\dv}{\qval-\bmat\val}\}\\
&=	\max_{\weight\in\mathcal{W}}\max_{\val,\qval}\min_{\mv,\dv\geq\mbf{0}}\{(1-\gamma)\innerprod{\initial}{\val}-\innerprod{\mv_{\expert}}{\cost_{\weight}}+\innerprod{\mv}{\cost_{\weight}+\gamma\pmat\val-\qval}+\innerprod{\dv}{\qval-\bmat\val}\}\\
	&=\max_{(\weight,\val,\qval)}\Big\{(1-\gamma)\innerprod{\initial}{\val}-\innerprod{\mv_{\expert}}{\cost_{\weight}}\mid \qval\geq\bmat\val,\;
\qval=\cost_{\weight}+\gamma\pmat\val,\\&\phantom{{}====================}\;\val\in\ar^{|\sspace|}, \tag{\color{blue}Dual}
\;\qval\in\ar^{|\sspace||\aspace|},\;\weight\in\mathcal{W}\Big\},
	\end{align*}
	where the second equality follows by Sion's minimax theorem~\cite{Sion:1958}, since $\mathcal{M}$ is convex and compact, $\mathcal{W}$ is convex and the objective is bilinear, the third equality follows by introducing Lagrange multipliers $\val$ and $\qval$, and the fifth equality follows by linear duality. Note that the derivations hold for any convex set $\mathcal{W}$.

From now on we consider the case $\mathcal{W}=B_1^m=\{\weight\in\ar^m\mid\norm{\weight}_2\le1\}$. Then, the~(\ref{eq:primal}) program can be written in the form
\begin{align*}
	\zeta^\star&=\min_{(\mv,\dv)}\{\bar{d}_{\mathcal{C}}(\mv,\mv_{\expert})\mid(\mv,\dv)\in\mathfrak{M}\}\\
	&=\min_{(\mv,\dv)}\{\max_{\weight\in\mathcal{W}}\innerprod{\mv-\mv_{\expert}}{\cost_{\weight}}\mid(\mv,\dv)\in\mathfrak{M}\}\\
	&=\min_{(\mv,\dv)}\{\max_{\weight\in\mathcal{W}}\innerprod{\phim^\trans\mv-\phim^\trans\mv_{\expert}}{\weight}\mid(\mv,\dv)\in\mathfrak{M}\}\\
	&=\min_{(\mv,\dv)}\{\norm{\phim^\trans\mv-\phim^\trans\mv_{\expert}}_2\mid(\mv,\dv)\in\mathfrak{M}\}, \tag{{\color{blue}Primal}}
	\end{align*}
	where in the last equality we used that the $\ell_2$-norm is self-dual, that is, the dual norm of the $\ell_2$-norm is still the $\ell_2$-norm. Therefore, when $\mathcal{W}=B_1^m$,  we get a quadratic objective with linear constraints~\cite{Abbeel:2004}.

Assume first that $(\val_{\textup{A}}, \qval_{\textup{A}},\wa)$ is optimal for~(\ref{eq:dual}). Then,
\begin{align}
&\qval_{\textup{A}}\geq\bmat\val_{\textup{A}},\quad
\qval_{\textup{A}}=\cost_{\wa}+\gamma\pmat\val_{\textup{A}},\quad\wa\in\,\mathcal{W}, \label{eq:feasibility}\\
&(1-\gamma)\innerprod{\initial}{\val_{\textup{A}}}-\innerprod{\mv_{\expert}}{\cost_{\wa}}=\zeta^\star=0, \label{eq:optimality}
\end{align}	
where~(\ref{eq:feasibility}) holds because $(\val_{\textup{A}}, \qval_{\textup{A}},\wa)$ is feasible to~(\ref{eq:dual}), and~(\ref{eq:optimality}) holds by optimality.
Therefore, $(\val_{\textup{A}}, \qval_{\textup{A}})$ is feasible for~(\ref{MDP-dual}) with cost $\cost=\cost_{\wa}$. Moreover, $(\mexp,\mexp)$ is feasible for~(\ref{MDP-primal}) with cost $\cost=\cost_{\wa}$. Therefore, 
\begin{equation}\label{eq:complementary}
(1-\gamma)\innerprod{\initial}{\val_{\textup{A}}}\le\rho^\star_{\wa}\le\innerprod{\mexp}{\cost_{\wa}}.
\end{equation}
However,  by~(\ref{eq:optimality}) we get that $(1-\gamma)\innerprod{\initial}{\val_{\textup{A}}}=\innerprod{\mv_{\expert}}{\cost_{\wa}}$.
Thus, $(\mexp,\mexp)$ is optimal for~(\ref{MDP-primal}) with cost $\cost=\cost_{\wa}$ and $(\val_{\textup{A}}, \qval_{\textup{A}})$ is optimal for~(\ref{MDP-dual}) with cost $\cost=\cost_{\wa}$. Thus $\expert$ is optimal for the forward RL problem with cost ${\cost_{\wa}}$, $\val_{\textup{A}}=\val^\star_{\cost_{\wa}}$, and $\qval_{\textup{A}}=\qval^\star_{\cost_{\wa}}$

Conversely, assume that $\wa\in\mathcal{W}$, $\expert$ is optimal for $\cost_{\weight_{\textup{A}}}$, $\val_{\textup{A}}=\val^\star_{\wa}$, and $\qval_{\textup{A}}=\qval^\star_{\wa}$. Then, we have that $(\mexp,\mexp)$ is optimal for~(\ref{MDP-primal}) with cost $\cost_{\wa}$, and $(\val_{\textup{A}},\qval_{\textup{A}})$ is optimal for~(\ref{MDP-dual}) with cost $\cost_{\wa}$. By dual feasibility, we get 
\begin{align}
\qval_{\textup{A}}\geq\bmat\val_{\textup{A}},\quad
\qval_{\textup{A}}=\cost_{\wa}+\gamma\pmat\val_{\textup{A}}. \label{eq:feasibility2}
\end{align}
Moreover, by primal-dual optimality, we have
\begin{align}
(1-\gamma)\innerprod{\initial}{\val_{\textup{A}}}=\innerprod{\mv_{\expert}}{\cost_{\wa}}.\label{eq:optimality2}
\end{align}
 From~(\ref{eq:feasibility2}), we get that $(\val_{\textup{A}}, \qval_{\textup{A}},\wa)$ is feasible to~(\ref{eq:dual}). Since $\zeta^\star=0$, by~(\ref{eq:optimality2}), we conclude that $(\val_{\textup{A}}, \qval_{\textup{A}},\wa)$ is optimal for~(\ref{eq:dual}).
\end{proofof}
\section{Saddle-Point Formulation}\label{app:SPP}

By using a compact notation, we have that~\ref{eq:primal'} is equivalent to the following bilinear saddle-point problem 
\begin{equation}
    \min_{\xv\in \mathcal{X}} \max_{\yv \in \mathcal{Y}} \innerprod{\yv}{\mbf{A}\xv+\mbf{b}}
    \label{eq:lagrangian}, \tag{\color{blue}SPP}
\end{equation}
where
\begin{align*}
  \mbf{A} &\triangleq
  \left[ {\begin{array}{ccc}
    \mbf{I}_{m} & 0 \\
    -\gamma \mmat^\trans  & \bmat^\trans  \\
    \mbf{I}_{m} & - \phim^\trans 
  \end{array} } \right], &  \mbf{b} &\triangleq
  \left[ {\begin{array}{ccc}
    - \FEV{\expert} \\
    (1 - \gamma) \initial \\\mbf{0}
  \end{array} } \right],
\end{align*}
$\xv\triangleq[\lv^\trans , \dv^\trans ]^\trans $, $\yv\triangleq[\weight^\trans , \val^\trans , \thv^\trans ]^\trans $, $\mathcal{X} \triangleq \Delta_{[m]} \times \Delta_{\sspace\times\aspace}$, and $\mathcal{Y}\triangleq\mathcal{W}\times\ar^{|\sspace|}\times\ar^{m}$.	

\section{Proof of Proposition~\ref{prop:q-update}}
\label{app:proof-of-upadates-proposition}
\begin{proofof}{Proof of Proposition~\ref{prop:q-update}} 
We break the proof in three parts. In the first two parts, we introduce and compute the explicit forms of the oracles, while in the third part we derive the proximal point updates.

\textbf{Analytical oracle.} We characterize the \texttt{analytical-oracle} by employing the first-order optimality conditions for $\lv$ and $\dv$. In particular, at each iteration step $k$, for any $[\weight^\trans, \val^\trans, \thv^\trans]^\trans$,
we have that the Lagrangian of the optimization problem in the definition of the \texttt{analytical-oracle} has the form
\begin{multline*} 
\innerprod{\lv}{\weight} - \innerprod{\efevphi}{  \weight}  + \innerprod{\val}{\gamma \mmat^{\mathsf{T}} \lv + (1 - \gamma) \initial - \bmat^{\mathsf{T}} \dv }\\
+\innerprod{\thv}{ \phim^{\mathsf{T}} \dv - \lv}+ \frac{1}{\eta}D(\lv || \lv_k) +\frac{1}{\alpha}H(\dv || \dv_{k})+\innerprod{\lv}{\tau\mbf{1}} - \tau,
\label{eq:appendix_lagrangian}
\end{multline*}
where we considered a Lagrangian multiplier $\tau$ for the simplex constraint $\sum_i\lambda(i)=1$.
Now taking the derivatives with respect to to $\lv$ and $\dv$, we obtain the following first order optimality conditions:
\begin{align*}
    \Big(\weight + \gamma \mmat \val - \thv \Big)(i)+ \tau + \frac{1}{\eta}\log\frac{\lv(i)}{\lv_k(i)} + \frac{1}{\eta}&= 0,\;\mbox{for all}\;i\in[m],
    \\
    \Big(\bmat \val + \phim \thv \Big)(s,a)+ \frac{1}{\alpha}\log \frac{\pi_{\dv}(a|s)}{\pi_{\dv_k}(a|s)} &= 0,\;\mbox{for all} (s,a)\in\sspace\times\aspace.
\end{align*}
Therefore, we obtain
\begin{equation}
    \lambda(i) = \lambda_k(i)\,e^{-\eta{\boldsymbol{\delta}}_{\weight,\thv}^{k}(i) + 1 - \eta\tau},
    \label{eq:lambda_almost}
\end{equation} where $\boldsymbol{\delta}_{\weight,\thv}^k\triangleq\weight+\gamma\mmat\val_{\thv}^k-\thv$.
In addition, the simplex constraint $\sum_i\lambda(i)=1$ is satisfied by choosing $\tau =\tau^k_{\weight,\thv}$, where 
\begin{equation}
    \tau^k_{\weight,\thv} \triangleq \frac{1}{\eta}\log \left( \sum_{i=1}^m (\phim^\trans\dv_k)(i) e^{-\eta {\boldsymbol{\delta}}^{k}_{\weight,\thv}(i)} \right).\label{eq:rho}
\end{equation}
Moreover, by setting $\qval_{\thv} = \phim \thv$, we get
\begin{equation}
    \pi_{\dv}(a|s)=\pi_{\dv_{k}}(a|s)\,e^{-\alpha \br{Q_{\thv}(s,a) - V(s)}}\label{d_via_pi}.
\end{equation}
 \Cref{eq:update2} follows by noting that the constraint $\sum_a\pi_{\dv}(a|x)$ implies that $\val$ equals the logistic value function $\val^k_{\thv}$ given in \Cref{prop:q-update}. Finally, since $(\lv_k,\dv_k)$ are ideal updates, they are primal feasible. Hence, we can use the constraint $\lv_k=\phim^\trans\dv_k$ in \Cref{eq:lambda_almost} to obtain \Cref{eq:update1}.

All in all, for any $\yv =[\weight^\trans , \val^\trans , \thv^\trans ]^\trans$ the \texttt{analytical-oracle} outputs $\gv(\yv;\xv_k) = \bs{\lv^\trans,\dv^\trans}^\trans$ with
\begin{align}
\lambda(i) &\propto (\phim^\trans\dv_k)(i)\,e^{-\eta{\boldsymbol{\delta}}_{\weight,\thv}^{k}(i)}\label{eq:lambda},\\
\pi_{\dv}(a|s)&=\pi_{\dv_{k}}(a|s)\,e^{-\alpha \br{Q_{\thv}(s,a) - V^k_{\thv}(s)}}.\label{eq:d}
\end{align}
Note that the derivatives with respect to $\lv$ and $\dv$ differ from the ones in Logistic Q-Learning \cite{Bas-Serrano:2021}. In our case,  $\boldsymbol{{\boldsymbol{\delta}}}_{\weight,\thv}^{k}$ depends on both cost weights $\weight$ and logistic action-value parameters $\thv$. In addition, $\boldsymbol{{\boldsymbol{\delta}}}_{\weight,\thv}^{k}$ is the reduced Bellman error in the feature space rather than in the high dimensional state-action space.

\textbf{Max oracle.} 
Since the objective in~(\ref{eq:q-update}) is convex in $\xv$ and linear in $\yv$, $\mathcal{X}$ is convex and compact, and $\mathcal{Y}$ is convex, by virtue of Sion's minimax theorem~\cite{Sion:1958}, we can exchange the $\mathrm{min}$ and $\mathrm{max}$ in \Cref{eq:q-update}. We then have 
 
\begin{equation*}
    \min_{\xv\in\mathcal{X}} \max_{\yv\in\mathcal{Y}} \innerprod{\yv}{ \mbf{A}\xv +  \widehat{\mbf{b}}} + \frac{1}{\tau}D_{\Omega}(\xv||\xv_k) =  \max_{\yv\in\mathcal{Y}}\min_{\xv\in\mathcal{X}} \innerprod{\yv}{ \mbf{A}\xv +  \widehat{\mbf{b}}} + \frac{1}{\tau}D_{\Omega}(\xv||\xv_k).
\end{equation*}
Therefore, we get
\begin{align*}
\yv^\star&=\argmax_{\yv\in\mathcal{Y}}
   \min_{\xv\in\mathcal{X}}\innerprod{\yv}{ \mbf{A}\xv +  \widehat{\mbf{b}}} + \frac{1}{\tau}D_{\Omega}(\xv||\xv_k) \\
    &=\argmax_{\yv\in\mathcal{Y}} \innerprod{\yv}{ \mbf{A}\gv(\yv;\xv_k) +  \widehat{\mbf{b}}} + \frac{1}{\tau}D_{\Omega}(\gv(\yv;\xv_k)||\xv_k) \\
    &= \mbf{h}(\xv_k).
\end{align*}

\textbf{Proximal point updates via max and analytical oracles.}
It remains to prove the closed-form expressions for $\pi_{\dv^\star}$ and $\lv^\star$ given in \Cref{eq:update1} and \Cref{eq:update2}, respectively. We start rewriting the objective of the \texttt{max-oracle} as a function of $\lv$ and $\dv$. In particular, we have
\begin{align*}
    &\innerprod{\yv}{\mbf{A}\gv(\yv;\xv_k) + \widehat{\mbf{b}}} + \frac{1}{\tau}D_{\Omega}(\gv(\yv;\xv_k)||\xv_k) \\
    &= \min_{\dv\in\Delta_{\sspace\times\aspace},\lv\in\Delta_{[m]}} \innerprod{\yv}{ \mbf{A}\left[ {\begin{array}{ccc}
    \lv \\ \dv
  \end{array} } \right]
     +  \widehat{\mbf{b}}}+ \frac{1}{\alpha}H(\dv||\dv_k) + \frac{1}{\eta}D(\lv||\lv_k).
\end{align*}

The minimizers of the previous expression are characterized via the \texttt{analytical-oracle}. In particular, plugging in the analytical forms for $\lv, \dv$ and $\val$, we obtain
\begin{align*} &\min_{\dv\in\Delta_{\sspace\times\aspace},\lv\in\Delta_{[m]}}  \innerprod{\yv}{ \mbf{A}\left[ {\begin{array}{ccc}
    \lv \\ \dv
  \end{array} } \right]
     +  \widehat{\mbf{b}}}+ \frac{1}{\alpha}H(\dv||\dv_k) + \frac{1}{\eta}D(\lv||\lv_k) \\ &= \innerprod{\lv}{\weight} - \innerprod{\efevphi}{\weight} + \frac{1}{\eta}\innerprod{\lv}{ -\eta {\boldsymbol{\delta}}^{k}_{\weight,\thv} - \eta \tau^k_{\weight,\thv}} \\
     &\phantom{{}=}
     +\frac{1}{\alpha}\innerprod{\dv}{-\alpha(\phim \thv - \bmat \val^k_{\thv})} + \innerprod{\lv}{\gamma \mmat^{\trans} \val^k_{\thv}} \\&\phantom{{}=}-\innerprod{\dv}{ \bmat\val^k_{\thv}}+ (1 - \gamma) \innerprod{\initial}{\val^k_{\thv}}
+\innerprod{\dv}{\phim \thv} - \innerprod{\lv}{\thv}  \\
&= - \innerprod{\efevphi}{ \weight} + (1 - \gamma) \innerprod{ \initial}{\val^k_{\thv}} - \tau^k_{\weight,\thv} \\
&= - \innerprod{\efevphi}{ \weight} + (1 - \gamma) \innerprod{ \initial}{\val^k_{\thv}} - \frac{1}{\eta} \log \left( \sum_{i=1}^m (\phim^\trans\dv_k)(i) e^{-\eta {\boldsymbol{\delta}}^{k}_{\weight, \thv}(i)} \right) \\ &\triangleq \mathcal{G}_k(\weight, \thv).
\end{align*}
This is the objective of the \texttt{max-oracle} in \Cref{prop:q-update}.
Given that the \texttt{max-oracle} returns $(\weight_k^\star, \thv_k^\star)$, the corresponding primal variables $(\dv_k^\star, \lv_k^\star)$ satisfy $(\dv_k^\star, \lv_k^\star)=\gv([\weight_k^\star, \val_{\thv^\star_k}, \thv_k^\star]; \xv_{k-1})$. This completes the proof of the first part of Proposition~\ref{prop:q-update}.

It remains to show the dual form of the max-oracle objective $\mathcal{G}_k(\weight, \thv)$. In particular, we will show that
\begin{equation}\label{eq:holds}
    \max_{\weight,\thv}\mathcal{G}_k(\weight, \thv) = \max_{\weight}\innerprod{\lv_{k+1}}{\weight} - \innerprod{\FEV{\expert}}{\weight} + \frac{1}{\eta}D(\lv_{k+1} || \lv_k) + \frac{1}{\alpha}H(\dv_{k+1}||\dv_k).
\end{equation}
We first recall that

\begin{equation*}
    \mathcal{G}_k(\weight, \thv) = \min_{\dv\in\Delta_{\sspace\times\aspace},\lv\in\Delta_{[m]}}  \innerprod{\yv}{ \mbf{A}\left[ {\begin{array}{ccc}
    \lv \\ \dv
  \end{array} } \right]
     +  \widehat{\mbf{b}}}+ \frac{1}{\alpha}H(\dv||\dv_k) + \frac{1}{\eta}D(\lv||\lv_k).
\end{equation*}
Then, by taking the maximum over $\yv=[\weight,\val,\thv]$ on both sides and using Sion's minimax theorem, we get
\begin{align*}
\max_{\weight,\thv}\mathcal{G}_k(\weight, \thv) &= \max_{\yv\in\mathcal{Y}} \min_{\dv\in\Delta_{\sspace\times\aspace},\lv\in\Delta_{[m]}}  \innerprod{\yv}{ \mbf{A}\left[ {\begin{array}{ccc}
    \lv \\ \dv
  \end{array} } \right]
     +  \widehat{\mbf{b}}}+ \frac{1}{\alpha}H(\dv||\dv_k) + \frac{1}{\eta}D(\lv||\lv_k) \\
    &= \min_{\dv\in\Delta_{\sspace\times\aspace},\lv\in\Delta_{[m]}} \max_{\yv\in\mathcal{Y}} \innerprod{\yv}{ \mbf{A}\left[ {\begin{array}{ccc}
    \lv \\ \dv
  \end{array} } \right]
     +  \widehat{\mbf{b}}}+ \frac{1}{\alpha}H(\dv||\dv_k) + \frac{1}{\eta}D(\lv||\lv_k) \\
     &=\max_{\yv\in\mathcal{Y}} \innerprod{\yv}{ \mbf{A}\left[ {\begin{array}{ccc}
    \lv_{k+1}\\ \dv_{k+1}
  \end{array} } \right]
     +  \widehat{\mbf{b}}}+ \frac{1}{\alpha}H(\dv_{k+1}||\dv_k) + \frac{1}{\eta}D(\lv_{k+1}||\lv_k),
\label{eq:dual}
\end{align*}
where in the last equality we used the definition of proximal point update in \Cref{eq:q-update}.
Finally, by LP strong duality, we have that
$
    \max_{\weight}\innerprod{\lv_{k+1}}{\weight} - \innerprod{\efevphi}{\weight} = \max_{\yv\in\mathcal{Y}} \innerprod{\yv}{ \mbf{A}\bs{\lv_{k+1}^\trans, \dv_{k+1}^\trans}^\trans
     +  \widehat{\mbf{b}}}.
$
Hence, we conclude that~(\ref{eq:holds}) holds.

\end{proofof}

\section{Proof of Proposition~\ref{prop:optimal_theta_bound}}
\begin{proofof}{Proof of Proposition~\ref{prop:optimal_theta_bound}} From first order optimality conditions for $\lv_k^\star$, we get
\begin{equation}
\label{1}
    \Big(\weight^\star_k + \gamma \mmat \val^k_{\thv^\star_k} - \thv^\star_k \Big)(i)+ \tau^k_{\weight^\star_{k},\thv^\star_{k}}+ \frac{1}{\eta}\log\frac{\lv_k^\star(i)}{\lv_{k-1}(i)} -\frac{1}{\eta}= 0,\;\mbox{for all}\;i\in[m].
\end{equation}
We define the regularized cost weights by $\widetilde\weight^\star_k \triangleq \weight^\star_k + \frac{1}{\eta}\log\frac{\lv_k^\star(i)}{\lv_{k-1}(i)}$, and the costant (wrt the vector index $i$) $c \triangleq - \tau^k_{\weight^\star_{k},\thv^\star_{k}} + \frac{1}{\eta}$. This gives for all $i \in [m]$
\begin{equation*}
    \Big(\widetilde{\weight}^\star_k + \gamma \mmat \val^k_{\thv^\star_k} - \thv^\star_k \Big)(i)  = c.
\end{equation*}
We define the \emph{span norm} as $\norm{\xv}_{\mathrm{sp}} = \inf_{c\in\mathbb{R}}\norm{\xv - c \mbf{1}}_{\infty}$. 
Then multiplying by $\phim$ from the left, we have that $\phim\widetilde{\weight}^\star_k + \gamma \pmat \val^k_{\thv^\star_k} - \phim\thv^\star_k  = c\mbf{1}$.
Moreover, we can write
\begin{align*}
\val_{\thv^\star_k}^k(s)=&-\frac{1}{\alpha}\log\left(\sum_a \pi_{\dv_{k-1}}(a|s)e^{-\alpha (\thv^\star_k)^\trans\phiv(s,a)}\right) \\ = & -\frac{1}{\alpha}\log\left(\sum_a \pi_{\dv_{k-1}}(a|s)e^{-\alpha (\phim\widetilde{\weight}^\star_k + \gamma \pmat \val^k_{\thv^\star_k})(s,a) + \alpha c}\right) \\
= & -\frac{1}{\alpha}\log\left(\sum_a \pi_{\dv_{k-1}}(a|s)e^{-\alpha (\phim\widetilde{\weight}^\star_k + \gamma \pmat \val^k_{\thv^\star_k})(s,a)}\right) + c
\end{align*}
We set $(\mathcal{T}\val^k_{\thv^\star_k})(s) \triangleq -\frac{1}{\alpha}\log\left(\sum_a \pi_{\dv_{k-1}}(a|s)e^{-\alpha (\phim\widetilde{\weight}^\star_k+ \gamma \pmat \val^k_{\thv^\star_k})(s,a)}\right)$. Note that $\mathcal{T}$ is the soft-Bellman operator \cite{Neu:2017, Geist:2019} that is a $\gamma$-contraction with respect to $\norm{\cdot}_\infty$-norm. It follows that
\begin{align*}
\norm{\val_{\thv^\star_k}^k}_{\mathrm{sp}}= \norm{\mathcal{T}\val_{\thv^\star_k}^k+ c}_{\mathrm{sp}} = \norm{\mathcal{T}\val_{\thv^\star_k}^k}_{\mathrm{sp}} \leq \norm{\mathcal{T}\val_{\thv^\star_k}^k - \mathcal{T}\mbf{0}}_{\mathrm{sp}} + \norm{\mathcal{T}\mbf{0}}_{\mathrm{sp}} \leq \gamma \norm{\val_{\thv^\star_k}^k}_{\mathrm{sp}} + \norm{\phim \widetilde{\weight}^\star_k}_{\mathrm{sp}}.
\end{align*}
Therefore, $\norm{\val_{\thv^\star_k}^k}_{\mathrm{sp}} \leq \frac{\norm{\phim \widetilde{\weight}^\star_k}_{\mathrm{sp}}}{1-\gamma} \leq \frac{1 + \log \frac{1}{\beta}}{1 - \gamma}$.
Moreover, using the relation $\widetilde{\weight}^\star_k + \gamma \mmat \val^k_{\thv^\star_k} - \thv^\star_k  = c\mbf{1}$, we have that
\begin{align*}
    \norm{\thv^\star_k}_{\mathrm{sp}} &\leq \norm{\widetilde{\weight}^\star_k + \gamma \mmat \val^k_{\thv^\star_k} - c\mbf{1}}_{\mathrm{sp}}\\
    &= \norm{\widetilde{\weight}^\star_k + \gamma \mmat \val^k_{\thv^\star_k}}_{\mathrm{sp}} \\ 
    &\leq \norm{\widetilde{\weight}^\star_k}_{\mathrm{sp}} + \gamma \norm{\mmat \val^k_{\thv^\star_k}}_{\mathrm{sp}}\\
    &\leq 1 + \log \br{\frac{1}{\beta}} + \gamma\frac{1 + \log \br{\frac{1}{\beta}}}{1 - \gamma} \\
    &= \frac{1 + \log \br{\frac{1}{\beta}}}{1 - \gamma}.
\end{align*}
This proves that for every maximizer the span norm is bounded.
Finally, for showing that there exists a maximizer with bounded infinity norm, we want to prove that the negative logistic Bellman error is shift invariant in $\thv$. That is, $\mathcal{G}_k(\weight, \thv + c \mathbf{1}) = \mathcal{G}_k(\weight, \thv)$. Towards this goal, we start proving that $\val^k_{\thv + c \mathbf{1}} = \val^k_{\thv} + c$ for any constant $c\in\ar$. Indeed,
\begin{align*}
    \val_{\thv+c \mathbf{1}}^k(s)=&-\frac{1}{\alpha}\log\left(\sum_a \pi_{\dv_{k-1}}(a|s)e^{-\alpha \thv^\trans\phiv(s,a) + -\alpha c \mathbf{1}^\trans\phiv(s,a)}\right) \\=&-\frac{1}{\alpha}\log\left(\sum_a \pi_{\dv_{k-1}}(a|s)e^{-\alpha \thv^\trans\phiv(s,a) -\alpha c}\right)\\=&-\frac{1}{\alpha}\log\left(\sum_a \pi_{\dv_{k-1}}(a|s)e^{-\alpha \thv^\trans\phiv(s,a)}\right) -\frac{1}{\alpha}\log\left(e^{-\alpha c}\right)\\=& \val_{\thv}^k(s) + c
\end{align*}
At this point, we can show the shift invariance of $\mathcal{G}_k$.
\begin{align*}
    \mathcal{G}_k(\weight, \thv + c \mathbf{1}) &= -\frac{1}{\eta}\log\sum_{i=1}^m (\phim^\trans \dv_{k-1})(i)   e^{-\eta\br{\weight(i) + \gamma (\mmat \val^k_{\thv + c \mathbf{1}})(i) -\thv(i) - c}}\\&\phantom{=}+(1-\gamma)\innerprod{\initial}{\val_{\thv+ c \mathbf{1}}^k}- \innerprod{\EFEV{\expert}}{\weight} \\&= -\frac{1}{\eta}\log\sum_{i=1}^m (\phim^\trans \dv_{k-1})(i)   e^{-\eta\br{\weight(i) + \gamma (\mmat \val^k_{\thv})(i) + \gamma c -\thv(i) - c}}\\&\phantom{=}+(1-\gamma)\innerprod{\initial}{\val_{\thv}^k} + (1 - \gamma) c - \innerprod{\EFEV{\expert}}{\weight} \\&= -\frac{1}{\eta}\log\underbrace{\sum_{i=1}^m (\phim^\trans \dv_{k-1})(i)}_{=1}   e^{-\eta \gamma c+ \eta c} + (1 - \gamma) c  + \mathcal{G}_k(\weight, \thv) \\&= - (1 - \gamma)c + (1 - \gamma) c  + \mathcal{G}_k(\weight, \thv) \\&= \mathcal{G}_k(\weight, \thv)
\end{align*}
It follows that there exists a maximizer $\thv^\star_k$ for which $\norm{\thv^\star_k}_{\mathrm{sp}}=\norm{\thv^\star_k}_{\infty}$. To see this, we show that we can find a value of $c$ for which the span seminorm equals the $\ell_{\infty}$-norm, that is $\norm{\thv^\star_k + c\mathbf{1}}_{\infty} = \norm{\thv^\star_k}_{\mathrm{sp}}$. By definition of the span norm (and assuming that the infimum is attained), the equality is attained for $c = \argmin_c \norm{\thv^\star_k + c\mathbf{1}}_{\infty} = \frac{\max_{i\in[m]}\thv^\star_k(i) + \min_{i\in[m]}\thv^\star_k(i)}{2}$. Then, choosing the shift for which $\max_{i\in[m]}\thv^\star_k(i) = - \min_{i\in[m]}\thv^\star_k(i)$, gives the maximizer for which $\norm{\thv^\star_k}_{\mathrm{sp}}=\norm{\thv^\star_k}_{\infty}$.  This concludes the proof for the bound on the $\ell_{\infty}$-norm.
\end{proofof}
\section{Proof of Theorem~\ref{thm:error_propagation} }

    We will analyze the proximal point method applied to \textcolor{blue}{SPP}. 
    We use a similar error propagation analysis as in~\cite{Bas-Serrano:2021}. 
        
        
       
        By Proposition~\ref{prop:q-update}, the ideal updates $(\thv_k^\star,\weight_k^\star,\pi_k^\star,\lv_k^\star,\dv_k^\star)$ are given by
        \begin{align*}
        (\weight^\star_k, \thv^\star_k) &= \arg\max_{\weight, \thv}\mathcal{G}_k(\weight, \thv), &
        \lambda_k^\star(i) &= (\phim^\trans \dv^\star_{k-1})(i)\,e^{-\eta({\boldsymbol{\delta}}_{\thv_k^\star,\weight^\star_{k}}^k(i)+\tau^k_{\thv^\star_k,\weight^\star_k})},\\
        \dv_k^\star&=\mv_{\pi_{k}^\star}, & \pi_k^\star(a|s) &= \pi_{\dv^\star_{k-1}}(a|s)\,e^{-\alpha (Q_{\thv_k^\star}(s,a)-V_{\thv_k^\star}^k(s))}, 
        \end{align*}
        where $\tau^k_{\thv^\star_k,\weight^\star_k}$ is a normalization constant. By feasibility of the ideal updates we also have $\lv_k^\star = \phim^\trans\dv_k^\star$
        On the other hand, the realized updates $(\thv_k,\weight_k, \pi_k,\lv_k,\dv_k)$ are given by
        \begin{align*}
        (\weight_k, \thv_k) &= \arg\max_{\weight, \thv}\mathcal{G}_k^{\epsilon_k}(\weight, \thv), & \lambda_k(i)&=(\phim^\trans \dv_{k-1})(i)\,e^{-\eta({\boldsymbol{\delta}}_{\thv_k,\weight_{k}}^k(i)+\tau^k_{\thv_k,\weight_k})},\\
        \dv_k&=\mv_{\pi_{k}}, & \pi_k(a|s) &= \pi_{\dv_{k-1}}(a|s)\,e^{-\alpha (Q_{\thv_k}(s,a)-V_{\thv_k}^k(s))} ,
        \end{align*}
        where $\tau^k_{\thv_k,\weight_k}$ is a normalization constant, and the notation $(\weight_k, \thv_k)=\arg\max_{\weight, \thv}\mathcal{G}_k^{\epsilon_k}(\weight, \thv)$ means that $\mathcal{G}_k(\weight_k^\star, \thv_k^\star)-\mathcal{G}_k(\weight_k, \thv_k)=\epsilon_k$. 
         We start by introducing some auxiliary results

\begin{lemma}
\label{lemma:bound_c_class}
For any occupancy measures $\dv_1, \dv_2\in\mathfrak{F}$, and for any cost vectors $\cost, \cost^\prime \in\mathcal{C}$, we have:
\begin{equation*} \inner{\mv_{\pi_E} - \dv_1, \cost} - 
\min_{\cost^\prime\in\mathcal{C}} \inner{\mv_{\pi_E} - \dv_2, \cost^\prime} \geq d_{\mathcal{C}}(\pi_E, \pi_{\dv_2}) - d_{\mathcal{C}}(\pi_E, \pi_{\dv_1}).
\end{equation*}
\end{lemma}
\begin{proof}
We have that
\begin{align*} \inner{\mv_{\pi_E} - \dv_1, \cost} - 
\min_{\cost^\prime\in\mathcal{C}} \inner{\mv_{\pi_E} - \dv_2, \cost^\prime} &\geq
\min_{\cost\in\mathcal{C}} \inner{\mv_{\pi_E} - \dv_1, \cost} - 
\min_{\cost^\prime\in\mathcal{C}} \inner{\mv_{\pi_E} - \dv_2, \cost^\prime} \\ &=
\max_{\cost^\prime} \inner{\dv_2 -\mv_{\pi_E}, \cost^\prime} - \max_{\cost} \inner{ \dv_1 - \mv_{\pi_E} , \cost}\\ &= d_{\mathcal{C}}(\pi_E, \pi_{\dv_2}) - d_{\mathcal{C}}(\pi_E, \pi_{\dv_1}).
\end{align*}
\end{proof}

\begin{corollary}
\label{corollary:lower_bound_c_distance}
Let $\dv^\star = \mathrm{argmin}_{\dv\in{\mathfrak{F}}}\max_{\cost\in\mathcal{C}} \inner{\dv, \cost} - \inner{\mv_{\pi_E}, \cost}$. Setting $\cost = \phim \weight_k$, $\dv_1 = \dv^\star$, $\dv_2= \dv_k$, we ge that
$
\inner{\mv_{\pi_E} - \dv^\star, \phim \weight_k} -\min_{\cost\in\mathcal{C}} \inner{\mv_{\pi_E} - \dv_k, \cost} \geq  d_{\mathcal{C}}(\pi_E, \pi_{\dv_k}) - d_{\mathcal{C}}(\pi_E, \pi_{\dv^\star}).
$
\end{corollary}

\begin{lemma}
\label{lem:H+D}
It holds that
$
    \frac{D(\lv^\star_k || \lv_k)}{\eta} + \frac{H(\dv_k^\star||\dv_k)}{\alpha} = \inner{\efevphi - \lv^\star_k, (\weight^\star_k - \weight_k)} + \epsilon_k.
$
\end{lemma}
\begin{proof}
The proof is analogous to Lemma 1 in \cite{Bas-Serrano:2021}.
\end{proof}

\begin{lemma}[First order optimality conditions for $\mathcal{G}_k$]
\label{lem:first_order}
For all $k\in[K]$, it holds that
\begin{equation*}
    \inner{ \efevphi - \lv_k^\star, \weight_k^\star -  \weight_k} \leq 0.
\end{equation*}
\end{lemma}
\begin{proof}
We start by taking the gradient of $\mathcal{G}_k(\weight_k, \thv_k)$ with respect to $\weight$. In particular, the partial derivative with respect to the $i^{th}$ component is given by
\begin{align*}
    \frac{\partial \mathcal{G}_k(\weight^\star_k, \thv^\star_k)}{\partial \weight(i)} =& - (\efevphi)(i) +\frac{(\phim^\trans\dv_{k-1})(i) e^{-\eta {\boldsymbol{\delta}}_{\thv^\star_k,\weight^\star_{k}}^k(i) }}{\sum^m_{i=1} (\phim^\trans\dv_{k-1})(i) e^{-\eta {\boldsymbol{\delta}}_{\thv^\star_k,\weight^\star_{k}}^k(i) }}
    \\=& - (\efevphi)(i) +  \lv^\star_k(i).
\end{align*}
Therefore,
\begin{equation*}
    \nabla_\weight \mathcal{G}_k(\weight^\star_k, \thv^\star_k) = - \efevphi + \lv^\star_k.
\end{equation*}
Then, by using the first-order optimality conditions for a concave function, we have
\begin{equation*}
    \inner{\nabla_{\weight}\mathcal{G}_k(\weight^\star_k, \thv^\star_k), \weight_k - \weight^\star_k} \leq 0, \quad \forall k.
\end{equation*}
By replacing the expression for $\nabla_{\weight}\mathcal{G}_k(\weight^\star_k, \thv^\star_k)$, we obtain
\begin{align}
    \inner{- \efevphi + \lv^\star_k,  \weight_k - \weight^\star_k} \leq 0 \quad \forall k
    \iff \inner{\efevphi - \lv^\star_k, \weight^\star_k - \weight_k}, \leq 0 \quad \forall k.
\end{align}

\end{proof}
We also need the following auxiliary result.
\begin{lemma}
\label{eq:bound_min_w}
For all $k\in[K]$, it holds that
\begin{equation}
\inner{\efevphi - \phim^\trans\dv_k, \weight^\star_k} \leq \min_{\weight\in\wspace} \inner{\efevphi - \phim^\trans\dv_k, \weight} + 2 \norm{\dv_k - \dv^\star_k}_1.
\end{equation}
\end{lemma}
\begin{proof}
By introducing $\bar{\weight}_k^\star = \argmin_{\weight\in\wspace} \inner{\efevphi - \phim^\trans\dv_k,\weight}$,
and applying triangular inequality, we obtain
\begin{align*}
\inner{\efevphi - \phim^\trans\dv_k,\weight_k^\star} &= \inner{\efevphi - \phim^\trans\dv_k,\bar{\weight}_k^\star} +  \inner{\efevphi - \phim^\trans\dv_k,\weight_k^\star-\bar{\weight}_k^\star} \\
&= \min_{\weight\in\wspace}\inner{\efevphi - \phim^\trans\dv_k,\weight} +  \inner{\efevphi - \phim^\trans\dv_k,\weight_k^\star-\bar{\weight}_k^\star}.
\end{align*}
Moreover, we have
\begin{align*}
    \inner{\efevphi - \phim^\trans\dv_k,\weight_k^\star-\bar{\weight}_k^\star} = &
    \max_{\weight\in\wspace} \inner{\efevphi - \phim^\trans\dv_k,\weight_k^\star - \weight} \\ =
    &\max_{\weight\in\wspace} \inner{\efevphi + \phim^\trans\dv^\star_k - \phim^\trans\dv^\star_k - \phim^\trans\dv_k,\weight_k^\star - \weight} \\ =
    &\max_{\weight\in\wspace} \inner{\efevphi - \phim^\trans\dv^\star_k,\weight_k^\star - \weight}
    + \inner{\dv^\star_k - \dv_k, \phim (\weight^\star_k - \bar{\weight}_k^\star)}  \\ \leq
    &\underbrace{\max_{\weight\in\wspace} \inner{\efevphi - \lv^\star_k,\weight_k^\star - \weight} }_{:= (A)}
    + \norm{\dv^\star_k - \dv_k}_1 \norm{\phim (\weight^\star_k - \bar{\weight}_k^\star)}_\infty  \\\leq
    & 2 \norm{\dv^\star_k - \dv_k}_1. 
\end{align*}
The first equality holds because the term in $\weight^\star_k$ is a constant wrt $\weight$, the variable of the max. In the last inequality follows from (A) being zero as we show next:
\begin{align*}
    \max_{\weight\in\wspace} \inner{\efevphi - \lv^\star_k,  \weight^\star_k - \weight} &=
    \max_{\weight\in\wspace} \inner{\efevphi - \lv^\star_k,  \weight^\star_k - \weight} + \frac{1}{\eta} D(\lv^\star_k||\phim^\trans\dv_{k-1}) \\&\phantom{{}=}  - \frac{1}{\eta} D(\lv^\star_k||\phim^\trans\dv_{k-1})  +\frac{1}{\alpha}H(\dv^\star_k||\dv_{k-1})
    -\frac{1}{\alpha}H(\dv^\star_k||\dv_{k-1}) \\& =
    \max_{\weight\in\wspace}  \bigg( \inner{\lv^\star_k - \efevphi, \weight} + \frac{1}{\eta} D(\lv^\star_k||\phim^\trans\dv_{k-1}) \\&\phantom{{}=}+\frac{1}{\alpha}H(\dv_k^\star||\dv_{k-1})\bigg)  - \inner{\lv^\star_k - \efevphi,  \weight^\star_k} \\
    &\phantom{{}=}-  \frac{1}{\eta} D(\lv^\star_k||\phim^\trans\dv_{k-1}) - \frac{1}{\alpha}H(\dv_k^\star||\dv_{k-1}) \\ &=
    \max_{\weight\in\wspace}   \bigg( \inner{\lv^\star_k - \efevphi, \weight} + \frac{1}{\eta} D(\lv^\star_k||\phim^\trans \dv_{k-1}) \\
    &\phantom{{}=}+ \frac{1}{\alpha}H(\dv_k^\star||\dv_{k-1}) \bigg) 
    - \max_{\weight\in\wspace}  \min_{\lv, \dv \in \mathfrak{M}_{\phim}} \bigg( \inner{ \lv - \efevphi,  \weight} \\ &\phantom{{}=} + \frac{1}{\eta} D(\lv||\phim^\trans \dv_{k-1}) + \frac{1}{\alpha}H(\dv||\dv_{k-1})\bigg) \\&= 0.
\end{align*}
\end{proof}
\begin{lemma}[Lower Bound on feature expectation vectors]\label{lem:lower_bound_fev}
Let Assumption~\ref{ass:eigenvalue} hold. We then have $(\phim^\trans\dv_k)(j) \geq \beta$ for all $j \in [m]$.
\end{lemma}
\begin{proof}
Let $\mbf{e}_j\in\mathbb{R}^m$ the vector with zeros everywhere but in position $j$ where it takes the value of $1$. Then, we observe that
\begin{align*}
    \beta \leq & \lambda_{\mathrm{min}}(\Ee{s,a\sim \dv_k}{\phiv(s,a)\phiv(s,a)^\trans}) \\ &=  \min_{\{\xv\in\mathbb{R}^m:\norm{\xv}_2 = 1\}} \xv^\trans \Ee{s,a\sim \dv_k}{\phiv(s,a)\phiv(s,a)^\trans} \xv \\ &\leq  \mbf{e}_j^\trans \Ee{s,a\sim \dv_k}{\phiv(s,a)\phiv(s,a)^\trans} \mbf{e}_j \\& =  \Ee{s,a\sim \dv_k}{\phiv_j^2(s,a)} \leq \Ee{s,a\sim \dv_k}{\phiv_j(s,a)} = (\phim^\trans\dv_k)(j).
\end{align*}
\end{proof}
\begin{theorem}[Error propagation with empirical expert feature expectation vector]
\label{thm:error_propagation_appendix}
Let $\dv^\star = \mathrm{argmin}_{\dv\in{\mathfrak{F}}}\max_{\cost\in\mathcal{C}} \inner{\dv, \cost} - \inner{\mv_{\pi_E}, \cost}$,  and let $\lv^\star$ be any feature vector such that $(\lv^\star, \dv^\star) \in \mathfrak{M}_{\phim}$. Moreover, let $C\triangleq \frac{1}{\beta\eta}\big(\sqrt{\frac{2 \alpha}{1 - \gamma}} + \sqrt{8 \eta}\big) + \sqrt{\frac{18 \alpha}{1 - \gamma}}$. Then, we have that
\begin{align*}
   & \frac{1}{K} \sum_k  \inner{\efevphi - \phim^\trans\dv^\star, \weight_k} -\min_{\weight\in\wspace} \inner{\efevphi - \phim^\trans\dv_k, \weight}\\
   &\leq \frac{D(\lv^\star||\phim^\trans\dv_0)}{K\eta} + \frac{H(\dv^\star||\dv_0)}{K\alpha}+ \frac{C}{K}\sum_k \sqrt{\epsilon_k} + \frac{\sum_k \epsilon_k}{K}.
\end{align*}
\end{theorem}
\begin{proof}
We have that
\begin{align*}
D(\lv^\star || \lv_k) = &D(\lv^\star || \phim^\trans\dv_{k-1}) + \eta \inner{\lv^\star, \weight_k + \gamma \mmat \val^k_{\thv_k} - \thv_k} + \eta \tau^k_{\thv_k,\weight_k} \\ =
&D(\lv^\star || \phim^\trans\dv_{k-1}) + \eta \inner{\lv^\star, \weight_k - \thv_k} + \eta \inner{\gamma \mmat^T \lv^\star, \val^k_{\thv_k}} + \eta \tau^k_{\thv_k,\weight_k} \\ =
&D(\lv^\star || \phim^\trans\dv_{k-1}) + \eta \inner{\lv^\star, \weight_k - \thv_k} + \eta \inner{\mbf{B}^\trans \dv^\star, \val^k_{\thv_k}} - \eta(1 - \gamma)\inner{\initial, \val^k_{\thv_k}} \\\phantom{{}=}&+ \eta \tau^k_{\thv_k,\weight_k} \\ = &D(\lv^\star || \phim^\trans\dv_{k-1}) + \eta \inner{\lv^\star, \weight_k - \thv_k} + \eta \inner{\mbf{B}^\trans \dv^\star, \val^k_{\thv_k}} - \eta (1 - \gamma)\inner{\initial, \val^k_{\thv_k}} \\\phantom{{}=}&+ \eta \tau^k_{\thv_k,\weight_k} \\
= &D(\lv^\star || \phim^\trans\dv_{k-1}) + \eta \inner{\lv^\star, \weight_k - \thv_k} + \eta \inner{\mbf{B}^\trans \dv^\star, \val^k_{\thv_k}} - \eta \mathcal{G}_k(\thv_k, \weight_k) \\\phantom{{}=}& - \eta \inner{\efevphi, \weight_k}\\
\leq &D(\lv^\star || \phim^\trans\dv_{k-1}) + \eta \inner{\lv^\star, \weight_k - \thv_k} + \eta \inner{\mbf{B}^\trans \dv^\star, \val^k_{\thv_k}} - \eta \mathcal{G}_k(\thv^\star_k, \weight^\star_k)\\\phantom{{}=}& + \eta \epsilon_k - \eta \inner{\efevphi, \weight_k}\\
\leq &D(\lv^\star || \phim^\trans\dv_{k-1}) + \eta \inner{\lv^\star, \weight_k - \thv_k} + \eta \inner{\mbf{B}^\trans \dv^\star, \val^k_{\thv_k}} + \eta \inner{\efevphi - \lv^\star_k, \weight_k^\star} \\\phantom{{}=}& - D(\lv_k^\star|| \phim^\trans\dv_{k-1}) - \eta \frac{H(\dv_k^\star|| \dv_{k-1})}{\alpha} + \eta \epsilon_k - \eta \inner{\efevphi, \weight_k} \\
\leq &D(\lv^\star || \phim^\trans\dv_{k-1}) + \eta \inner{\lv^\star, \weight_k}  + \eta \inner{\dv^\star, \mbf{B} \val^k_{\thv_k} - \phim \thv_k} + \eta \inner{\efevphi - \lv^\star_k, \weight^\star_k}\\\phantom{{}=}& + \eta \epsilon_k - \eta \inner{\efevphi, \weight_k} \\
\leq &D(\lv^\star || \phim^\trans\dv_{k-1}) + \eta \inner{\dv^\star, \phim \weight_k} + \eta \inner{\dv^\star, \mbf{B} \val^k_{\thv_k} - \phim \thv_k} + \eta \inner{\efevphi - \phim^\trans\dv_k,  \weight^\star_k}  \\\phantom{{}=}&+ \eta\inner{\dv_k - \dv^\star_k, \phim \weight^\star_k} + \eta \epsilon_k - \eta \inner{\efevphi, \weight_k} \\
\leq &D(\lv^\star || \phim^\trans\dv_{k-1}) + \eta \inner{\dv^\star, \phim \weight_k}  + \eta \inner{\dv^\star, \mbf{B} \val^k_{\thv_k} - \phim \thv_k} + \eta \inner{\efevphi - \phim^\trans\dv_k,  \weight^\star_k} \\\phantom{{}=}& + \eta \norm{\dv_k - \dv^\star_k}_1 + \eta \epsilon_k - \eta \inner{\efevphi, \weight_k} \quad \text{Using Lemma~\ref{eq:bound_min_w}}\\
\leq &D(\lv^\star || \phim^\trans\dv_{k-1}) + \eta \inner{\dv^\star, \phim \weight_k}  + \eta \inner{\dv^\star, \mbf{B} \val^k_{\thv_k} - \phim \thv_k}  \\\phantom{{}=}&  + \min_{\weight\in\wspace} \eta \inner{\efevphi -  \phim^\trans \dv_k,\weight} + 3\eta \norm{\dv_k - \dv^\star_k}_1 + \eta \epsilon_k - \eta \inner{\efevphi, \weight_k}.
\end{align*}

Therefore, it follows that
\begin{multline}
     \inner{\efevphi - \phim^\trans\dv^\star,  \weight_k} - \min_{\weight\in\wspace} \inner{\efevphi - \phim ^\trans\dv_k, \weight} \leq \frac{D(\lv^\star || \phim^\trans\dv_{k-1}) - D(\lv^\star || \lv_k)}{\eta} \\ + \inner{\dv^\star, \mbf{B} \val^k_{\thv_k} - \phim \thv_k} + 3\norm{\dv_k - \dv^\star_k}_1 + \epsilon_k.
\end{multline}
Then, by using $H(\dv^\star||\dv_k) = H(\dv^\star||\dv_{k-1}) - \alpha \inner{\dv^\star, \phim\thv_k - \mbf{B} \val_{\thv_k}^k}$, we obtain

\begin{align*}
    \inner{\efevphi - \phim ^\trans\dv^\star, \weight_k} - \min_{\weight\in\wspace} \inner{\efevphi -\phim ^\trans \dv_k, \weight} &\leq \frac{D(\lv^\star || \phim^\trans\dv_{k-1}) - D(\lv^\star || \lv_k)}{\eta} \\& \phantom{{}=}+ \frac{H(\dv^\star||\dv_{k-1}) - H(\dv^\star||\dv_{k})}{\alpha} \\& \phantom{{}=}+ 3\norm{\dv_k - \dv^\star_k}_1 + \epsilon_k.
\end{align*}
Summing over iteration indices $k$ and dividing by the total number of iterations $K$, we obtain
\begin{multline}
    \frac{1}{K} \sum_k \inner{\efevphi - \phim ^\trans\dv^\star, \weight_k}  - \min_{\weight\in\wspace} \inner{\efevphi - \phim^\trans \dv_k,  \weight}  \leq \frac{1}{K} \sum_k \Bigg( \frac{D(\lv^\star || \phim^\trans\dv_{k-1}) }{\eta} \\ - \frac{D(\lv^\star || \lv_k)}{\eta}+ \frac{H(\dv^\star||\dv_{k-1}) - H(\dv^\star||\dv_{k})}{\alpha} + 3 \norm{\dv_k - \dv^\star_k}_1 \Bigg) + \frac{\sum_k \epsilon_k}{K}.
    \label{eq:saddle_bound_1}
\end{multline}
Moreover, by a telescoping sum, we get
\begin{align*}
    \sum_k & \left( \frac{D(\lv^\star || \phim^\trans\dv_{k-1}) - D(\lv^\star || \lv_k)}{\eta} + \frac{H(\dv^\star||\dv_{k-1}) - H(\dv^\star||\dv_{k})}{\alpha} \right) \\&=
    \sum_k \Bigg( \frac{D(\lv^\star || \phim^\trans\dv_{k-1}) - D(\lv^\star || \phim^\trans\dv_k)}{\eta} + \frac{D(\lv^\star || \phim^\trans\dv_{k}) - D(\lv^\star || \lv_k)}{\eta} \\& \phantom{{}=} + \frac{H(\dv^\star||\dv_{k-1}) - H(\dv^\star||\dv_{k})}{\alpha} \Bigg) \\ &=
    \frac{D(\lv^\star||\phim^\trans\dv_0) - D(\lv^\star||\phim^\trans\dv_K)}{\eta} + \frac{H(\dv^\star||\dv_0) - H(\dv^\star||\dv_K)}{\alpha} \\& \phantom{{}=} + \sum_k \frac{D(\lv^\star || \phim^\trans \dv_k) - D(\lv^\star || \lv_k)}{\eta} \\  &\leq
    \frac{D(\lv^\star||\phim^\trans\dv_0)}{\eta} + \frac{H(\dv^\star||\dv_0)}{\alpha} + \sum_k \frac{D(\lv^\star || \phim^\trans \dv_k) - D(\lv^\star || \lv_k)}{\eta}
\end{align*}
Combining this derivation with \eqref{eq:saddle_bound_1}, we get
\begin{align}
    & \frac{1}{K}\sum_k \inner{\efevphi - \phim ^\trans\dv^\star, \weight_k}  - \min_{\weight\in\wspace} \inner{\efevphi - \phim ^\trans\dv_k, \weight} \leq \frac{D(\lv^\star||\lv_0)}{K\eta} + \frac{H(\dv^\star||\dv_0)}{K\alpha} \nonumber \\& \phantom{{}=} + \frac{1}{K}\sum_k \left( \frac{D(\lv^\star ||\phim^\trans \dv_k) - D(\lv^\star || \lv_k)}{\eta} + 3 \norm{\dv_k - \dv^\star_k}_1 \right) +  \frac{\sum_k \epsilon_k}{K}.
    \label{eq:saddle_bound_2}
\end{align}
In order to bound the term $D(\lv^\star||\phim^\trans\dv_k) - D(\lv^\star||\lv_k)$, we introduce the Bregman projection to the space of feature expectation vectors induced by valid occupancy measures $\tilde \lv_k = \arg\min_{\{\lambda = \phim \dv | \dv \in \mathfrak{F}\}} D(\lv || \lv_k)$. We then have
\begin{align*}
    D(\lv^\star||\phim^\trans\dv_k) - D(\lv^\star||\lv_k) =& D(\lv^\star||\phim^\trans\dv_k) - D(\lv^\star||\lv_k) + D(\lv^\star||\tilde \lv_k) - D(\lv^\star||\tilde \lv_k)
    \\\leq&  D(\lv^\star||\phim^\trans\dv_k) - D(\lv^\star||\tilde \lv_k) - D(\tilde \lv_k||\lv_k) 
    \\\leq& D(\lv^\star||\phim^\trans\dv_k) - D(\lv^\star||\tilde \lv_k),
\end{align*}
where in the second inequality, we used Lemma 11.3 in~\cite{Cesa-Bianchi:2006}. Furthermore,
\begin{align*}
    D(\lv^\star||\phim^\trans\dv_k) - D(\lv^\star||\tilde \lv_k) =& \sum^m_{i=1} \lv^\star(i) \log \frac{\tilde \lv_k(i)}{\phim^\trans\dv_k(i)} 
    \\\leq& \sum^m_{i=1} \lv^\star(i) \left(\frac{\tilde \lv_k(i)}{\phim^\trans\dv_k(i)} - 1\right)
    \\\leq& \sum^m_{i=1} \frac{\lv^\star(i)}{\phim^\trans\dv_k(i)} \abs{\tilde \lv_k(i) - \phim^\trans\dv_k(i)}
    \\\leq& \max_{i}\frac{\lv^\star(i)}{\phim^\trans\dv_k(i)} \norm{\tilde \lv_k - \phim^\trans\dv_k}_1
    \\\leq& \frac{1}{\beta} \norm{\tilde \lv_k - \phim^\trans\dv_k}_1
    \\\leq& \frac{1}{\beta} (\norm{\tilde \lv_k - \lv_k}_1 + \norm{\lv_k^\star - \lv_k}_1 + \norm{\lv^\star_k - \phim^\trans\dv_k}_1)
    \\\leq& \frac{1}{\beta} (\sqrt{2 D(\tilde \lv_k||\lv_k)} + \sqrt{2 D(\lv_k^\star||\lv_k)}  + \norm{\lv^\star_k - \phim^\trans\dv_k}_1)
    \\\leq& \frac{1}{\beta} (2\sqrt{2 D(\lv_k^\star||\lv_k)} + \norm{\phim^\trans \dv^\star_k - \phim^\trans \dv_k}_1)
    \\\leq& \frac{1}{\beta} \bigg(\sqrt{8 \eta (\epsilon_k + \inner{\efevphi - \phim^\trans\dv_k^\star, \weight^\star_k - \weight_k}) } \\& \phantom{{}=} + \norm{\phim}_{\infty}\norm{\dv^\star_k - \dv_k}_1\bigg),
\end{align*}
where we used $\max_{i}\frac{\lv^\star(i)}{\phim^\trans\dv_k(i)} \leq \frac{1}{\beta}$ thanks to \Cref{lem:lower_bound_fev} while in the last line we use the fact that $H(\dv^\star_k||\dv_k)$ is positive and the equality in \Cref{lem:H+D}. To bound the $\ell_1$-norm, we apply Pinkser's inequality and Lemma~2 in~\cite{Bas-Serrano:2021} to get that 
$$\norm{\dv_k - \dv^\star_k} \leq \sqrt{2 D(\dv_k || \dv^\star_k)} \leq \sqrt{2 \frac{H(\dv_k || \dv^\star_k)}{1-\gamma}} \leq \sqrt{\frac{2\alpha}{1-\gamma}(\epsilon_k + \inner{\efevphi - \phim^\trans\dv_k^\star, \weight^\star_k - \weight_k})}. $$
Plugging the last derivation in~\Cref{eq:saddle_bound_2} gives
\begin{align}
    \frac{1}{K} \sum_k  \inner{\efevphi - \dv^\star, \phim \weight_k} &- \min_{\weight\in\wspace} \inner{\efevphi - \dv_k, \phim \weight} \leq \frac{D(\lv^\star||\lv_0)}{K\eta} + \frac{H(\dv^\star||\dv_0)}{K\alpha} \nonumber \\& \phantom{{}=} + \frac{C}{K}\sum_k \left( \sqrt{\epsilon_k + \inner{\efevphi - \lv_k^\star, \weight^\star_k - \weight_k}} \right) + \frac{\sum_k \epsilon_k}{K} .
    \label{eq:saddle_bound_3}
\end{align}

Finally, using Lemma~\ref{lem:first_order} we have that the term $\inner{\efevphi - \lv_k^\star, \weight^\star_k - \weight_k}$ is non positive. Therefore,
\begin{align*}
    \frac{1}{K} \sum_k  \inner{\efevphi - \phim ^\trans\dv^\star, \weight_k} - \min_{\weight\in\wspace} \inner{\efevphi - \phim ^\trans\dv_k, \weight} &\leq \frac{D(\lv^\star||\lv_0)}{K\eta} + \frac{H(\dv^\star||\dv_0)}{K\alpha} \\ &\phantom{{}=}+ \frac{C}{K}\sum_k  \sqrt{\epsilon_k} + \frac{\sum_k \epsilon_k}{K},
\end{align*}
where $C= \frac{1}{\beta\eta}(\sqrt{\frac{2 \alpha}{1 - \gamma}} + \sqrt{8 \eta}) + 3 \sqrt{\frac{2 \alpha}{1 - \gamma}}$.
\end{proof}

Finally, we need a Lemma that provides a concentration for the estimated expert feature expectation vector. 
\begin{lemma}[\cite{Syed:2007}]
\label{lemma:expert_concentration}Let $\mathcal{D}_{\pi_{E}}\triangleq\{(s_0^\ell,a_0^\ell,s_1^\ell,a_1^\ell,\ldots,s_H^\ell,a_H^\ell)\}_{\ell=1}^{n_{\textup{E}}}\sim\pi_{E}$ be a finite set of \textrm{i.i.d.} truncated sample trajectories. We consider the empirical expert feature expectation vector $\EFEV{\pi_{E}}$ by taking sample averages, i.e., 
		$$
		\EFEV{\pi_{E}}\triangleq (1-\gamma)\frac{1}{n_{\textup{E}}} \sum_{t=0}^H \sum_{\ell=1}^{N} \gamma^t \phi_i(s_t^\ell,a_t^\ell), \; \forall \; i\in[m].
		$$
		Suppose the trajectory length is $H\geq\frac{1}{1-\gamma}\log(\frac{1}{\varepsilon})$, and the number of of expert trajectories is $n_{\textup{E}}\geq\frac{2\log(\frac{2m}{\delta})}{\varepsilon^2}$. Then, with probability at least $1-\delta$, it holds that 
		$\normm{\mbs{\rho}_{\phim}(\pi_{E})-\EFEV{\pi_{E}}}_\infty\le\varepsilon.$
\end{lemma}

At this point, \Cref{thm:error_propagation} is proven from the results of \Cref{thm:error_propagation_appendix},Lemma~\ref{lemma:expert_concentration} and \Cref{lemma:bound_c_class}.

 \section{Biased Stochastic Gradients and their Properties}\label{app:stochastic gradients}

In order to estimate the gradient $\nabla_{\thv}\, G(\weight,\thv)$, we define the policy $\pi_{k,\thv}(a|s)\propto\pi_k(a|s)e^{-\alpha Q_{\thv}(s,a)}$, for all $k\in\mathds{N}$, and for all $\thv\in\ar^m$. Then, by standard computations we get that for all $(\weight,\thv)$, and for all $j\in[m]$, 
\begin{align*}
&\nabla_{\thv,j}\, G(\weight,\thv)\\
&=\sum_{i=1}^m (\phim^\trans \dv_{k-1})(i) \bmat_{\weight,\thv}^k(i)\big[\gamma \mbs{\Gamma}_k(i,j)-\mathds{1}\{i=j\}\big]
+(1-\gamma)\sum_{s}\initial(s)\sum_a \pi_{k-1,\thv}(a|s)\phiv_i(s,a)\\
&=\Exp_{(s,a)\sim\dv_{k-1}, i\sim\phiv(s,a)}\Big[\bmat_{\weight,\thv}^k(i)\big[\gamma \mbs{\Gamma}_k(i,j)-\mathds{1}\{i=j\}\big]\Big]+(1-\gamma)\Exp_{s_0\sim\initial,a_0\sim\pi_{k-1,\thv}(\cdot|s_0)}\Big[\phiv_i(s_0,a_0)\Big],
\end{align*}
where $
\bmat_{\weight,\thv}^k(i)\triangleq \frac{\exp\big(- \eta {{\mbs{\delta}}}\,_{\weight,\thv}^{k}(i)\big)}{{Z}_k}$, ${Z}_k\triangleq\sum_{i=1}^m \exp\big(- \eta {{\mbs{\delta}}}\,_{\weight,\thv}^{k}(i)\big)\FEV{\pi_{k-1}}(i)$, and
$\mbs{\Gamma}_k(i, j)\triangleq\sum_{s',a'}\mmat_{i,s'}\pi_{k-1,\thv}(a'|s')\phiv_j(s',a')$ . Similarly, for the gradient $\nabla_{\weight}\, G(\weight,\thv)$, we can write
\begin{align*}
\nabla_{\weight,j}\, G(\weight,\thv)&=-\efevphi(j)+\sum_{i=1}^m (\phim^\trans \dv_{k-1})(i) \bmat_{\weight,\thv}^k(i)\mathds{1}\{i=j\}\\
&=-\efevphi(j)+\Exp_{(s,a)\sim\dv_{k-1}, i\sim\phiv(s,a)}\Big[\bmat_{\weight,\thv}^k(i)\mathds{1}\{i=j\}\Big]
\end{align*}
Note that the following estimators of $\nabla_{\thv}\, G_k(\weight,\thv)$ and $\nabla_{\weight}\, G_k(\weight,\thv)$ are unbiased:
Sample $(s',a')\sim\dv_{k-1}$, $i'\sim\phiv(s',a')$, $s_0\sim\initial$, and $a_0\sim\pi_{k-1,\thv}(\cdot|s_0)$, then define
\begin{align}
				    \widetilde{\nabla}_{\weight,j}\mathcal{G}_k(\weight,\thv)&=-\efevphi(j)+{\bmat}_{\weight,\thv}^k(i')\mathds{1}\{i'=j\} \label{eq:westimator2},\\
				    \widetilde{\nabla}_{\thv,j}\mathcal{G}_k(\weight,\thv)&={\bmat}_{\weight,\thv}^k(i')\Big[\gamma{\mbs{\Gamma}}_k(i',j)-\mathds{1}\{i'=j\}\Big]+(1-\gamma)\phiv_j(s_0,a_0). \label{eq:thetaestimator2}
				\end{align}
These expressions give rise to the Biased Stochastic Gradient Estimator subroutine (\textrm{BSGE}) given in Algorithm~\ref{alg:BSGE}, where we plug-in  estimators $\widehat{\bmat}_{\weight,\thv}^k\in\ar^m$ and $\widehat{\mbs{\Gamma}}_k\in\ar^{m \times m}$ to Equations~(\ref{eq:westimator2}) and~(\ref{eq:thetaestimator2}). It remains to show how to maintain good estimators $\widehat{\bmat}_{\weight,\thv}^k$ and $\widehat{\mbs{\Gamma}}_k$ by using the linear MDP Assumption~\ref{ass:linear-MDP}. While the estimator $\widehat{\mbs{\Gamma}}_k\in\ar^m$ is a standard ridge regression estimator, the construction of $\widehat{\bmat}_{\weight,\thv}^k$ is more involved. In particular, we first need to build an estimator for the product $\mmat \val^k_{\theta}$ via ridge regression. 
Then, the estimator for  $\widehat{\bmat}_{\weight,\thv}^k$ is derived by plugging-in the estimator of $\mmat \val^k_{\theta}$, and the estimator for the feature expectation vector $\FEV{\pi_{k-1}}$ to equation $
\bmat_{\weight,\thv}^k(i)\triangleq \frac{\exp\big(- \eta {{\mbs{\delta}}}\,_{\weight,\thv}^{k}(i)\big)}{{Z}_k}$ . The reasoning and analysis is inspired by~\cite{Howard:2021, Pacchiano:2021}.


		\begin{algorithm}[!t]
			\caption{Biased Stochastic Gradient Estimator: \textrm{BSGE}$(k,\weight,\thv, N)$ \label{alg:BSGE}}
			
			\begin{algorithmic}
				\STATE {\bfseries Input:} Policy evaluation step $k$, reference points $(\weight,\thv)$, number of samples $N$
				\STATE Compute empirical estimators $\widehat{{\mbs{\delta}}}\,_{\weight,\thv}^{k}\in\ar^m$, $\widehat{\mbs{\Gamma}}_k\in\ar^{m\times m}$, $\EFEV{\pi_{k-1}}\in\ar^m$  \textcolor{black}{using the first $N$ samples $\{(s_{{k-1}}^{(n)},a_{{k-1}}^{(n)},s_{{k-1}}^{\prime (n)})\}_{n=1}^{{N}}$ from the buffer $\mathcal{B}_k$}
				\FOR{$i=1,\ldots, m$} 
				\STATE Compute $\widehat{\bmat}_{\weight,\thv}^k(i)=\frac{\exp\big(- \eta \widehat{{\mbs{\delta}}}\,_{\weight,\thv}^{k}(i)\big)}{\widehat{Z}_k}$,
				Where $\widehat{Z}_k=\sum_{i=1}^m \exp\big(- \eta \widehat{{\mbs{\delta}}}\,_{\weight,\thv}^{k}(i)\big)\EFEV{\pi_{k-1}}(i)$
				\ENDFOR
				\STATE Sample $(s_{{k-1}}^{(N+1)}$, $a_{{k-1}}^{(N+1)})\sim \mv_{\pi_{k-1}}$, $i_{k-1}^{(N+1)}\sim\phiv(s_{{k-1}}^{(N+1)},a_{{k-1}}^{(N+1)})$
				\STATE Sample $s_{k-1}^{(0)}\sim\initial$, and $a_{k-1}^{(0)}\sim\pi_{k-1,\thv}(\cdot|s_0)$
				\STATE Compute
				\begin{align*}
				    \widehat{\nabla}_{\weight,j}\mathcal{G}_k(\weight,\thv)&=-\efevphi(j)+\widehat{\bmat}_{\weight,\thv}^k(i_{k-1}^{(N+1)})\mathds{1}\{i_{k-1}^{(N+1)}=j\}\\
				    \widehat{\nabla}_{\thv,j}\mathcal{G}_k(\weight,\thv)&=\widehat{\bmat}_{\weight,\thv}^k(i_{k-1}^{(N+1)})\Big[\gamma\widehat{\mbs{\Gamma}}_k(i_{k-1}^{(N+1)}, j)-\mathds{1}\{i_{k-1}^{(N+1)}=j\}\Big]+(1-\gamma)\phiv_j(s_{k-1}^{(0)},a_{k-1}^{0})
				\end{align*}
				\STATE {\bfseries Output:} $(\widehat{\nabla}_{\weight}\mathcal{G}_k(\weight,\thv),\widehat{\nabla}_{\thv}\mathcal{G}_k(\weight,\thv))$
			\end{algorithmic}
		\end{algorithm}


\subsection{Ridge estimators}
This section leverages ridge regression \cite{Hsu:2012} to build estimators $\widehat{\bmat}_{\weight,\thv}^k$ and $\widehat{\mbs{\Gamma}}_k\in\ar^{m\times m}$.
We work under the Assumption~\ref{ass:eigenvalue} which ensures that every iterate covers the features space. We recall that by Lemma~\ref{lem:lower_bound_fev}, Assumption~\ref{ass:eigenvalue} implies that $\phim^\trans\dv_k(s,a) \geq \beta$, for all $k\in[K]$.

\subsubsection{Estimator for $\mmat \val^k_{\theta}$}

We first construct an estimator for $\mmat_k \val^k_\theta$.
We can start noticing that we can rewrite $\mmat_k \val^k_\theta$ using the feature covariance matrix $\bar{\mbs{\Lambda}}_k \triangleq \Ee{(s,a)\sim \dv_{k-1}}{\phiv(s,a)\phiv(s,a)^{\mathsf{T}}}$ as showed by the next lemma.

\begin{lemma}
\label{exact_ridge_solution_mV}
It holds that
$
    \mmat \val^k_{\theta} = \bar{\mbs{\Lambda}}_k ^{-1} \Ee{(s,a)\sim \dv_{k-1}, s^\prime\sim P(\cdot|s,a)}{\phiv(s,a) V^k_{\theta}(s^\prime)}.
$
\end{lemma}
\begin{proof}
\begin{align*}
    \mmat \val^k_{\theta} =& \bar{\mbs{\Lambda}}_k ^{-1}\bar{\mbs{\Lambda}}_k  \mmat \val^k_{\theta} \\=&
    \bar{\mbs{\Lambda}}_k ^{-1} \Ee{(s,a)\sim \dv_{k-1}}{\phiv(s,a)\phiv(s,a)^{\mathsf{T}} \mmat \val^k_{\theta}} \\=&
    \bar{\mbs{\Lambda}}_k ^{-1} \Ee{(s,a)\sim \dv_{k-1}}{\phiv(s,a)\phiv(s,a)^{\mathsf{T}} \sum_{s^\prime}\mmat_{:s'} V^k_{\theta}(s^\prime)} \\=&
    \bar{\mbs{\Lambda}}_k ^{-1} \Ee{(s,a)\sim \dv_{k-1}}{\phiv(s,a) \sum_{s^\prime}\phiv(s,a)^{\mathsf{T}}\mmat_{:s'} V^k_{\theta}(s^\prime)}\\=&
    \bar{\mbs{\Lambda}}_k ^{-1} \Ee{(s,a)\sim \dv_{k-1}}{\phiv(s,a) \sum_{s^\prime}P(s^\prime|s,a) V^k_{\theta}(s^\prime)}\\=&
    \bar{\mbs{\Lambda}}_k ^{-1} \Ee{(s,a)\sim \dv_{k-1}, s^\prime\sim P(\cdot|s,a)}{\phiv(s,a) V^k_{\theta}(s^\prime)}.
\end{align*}
\end{proof}

It follows that 
$
    \mmat \val^k_{\theta} = \arg\min\limits_{\zv} \Ee{(s,a)\sim \dv_{k-1}, s^\prime \sim P(\cdot|s,a)}{\br{\phiv(s,a)^{\mathsf{T}}\zv - V^k_{\theta}(s^\prime)}^2 }.
$

Now, we move to the problem of estimating $\widehat{\mmat\val^k_{\theta}}$ with a finite amount of environment interactions sampled \textrm{i.i.d} from $\dv_{k-1}$. We define
\begin{equation*}
    \widehat{\mmat\val^k_{\theta}} \triangleq \arg\min_{\zv} \frac{1}{N}\sum^N_{n=1}\br{\phiv(s_k^{(n)},a_k^{(n)})^{\mathsf{T}}\zv - V^k_{\theta}(s_k^{\prime(n)})}^2 + \chi \norm{\zv}^2_2.
\end{equation*}
By optimality conditions, we can obtain a closed-form expression for $\widehat{\mmat\val^k_{\theta}}$.
\begin{lemma}
 It holds that
\begin{equation*}
    \widehat{\mmat\val^k_{\theta}} = \frac{1}{N}\br{ \mbs{\Lambda}_{k,N} + \chi \mbf{I}}^{-1}\sum^N_{n=1} \phiv(s_{k-1}^{(n)},a_{k-1}^{(n)})V^k_{\theta}(s_{k-1}^{\prime(n)}),
\end{equation*}
where  $\mbs{\Lambda}_{k,N} \triangleq \frac{1}{N}\sum^N_{n=1}\phiv(s_{k-1}^{(n)},a_{k-1}^{(n)})\phiv(s_{k-1}^{(n)},a_{k-1}^{(n)})^{\mathsf{T}}$ is the empirical covariance matrix.
\end{lemma}
\begin{proof}
Let $\mathcal{L}(\zv) \triangleq  \frac{1}{N}\sum^N_{n=1}\br{\phiv(s_{k-1}^{(n)},a_{k-1}^{(n)})^{\mathsf{T}}\zv - V^k_{\theta}(s_{k-1}^{\prime(n)})}^2 + \chi \norm{\zv}^2_2$. The first derivative is given by
\begin{equation}
   \frac{1}{2} \nabla_\zv \mathcal{L}(\zv) = \frac{1}{N}\sum^N_{n=1}\phiv(s_{k-1}^{(n)},a_{k-1}^{(n)})\br{\phiv(s_{k-1}^{(n)},a_{k-1}^{(n)})^{\mathsf{T}}\zv - V^k_{\theta}(s_{k-1}^{\prime(n)})} + \chi \zv.
\end{equation}
Since $\mathcal{L}(\cdot)$ is convex in $\zv$, by first-order optimality conditions, we get
\begin{equation*}
    \frac{1}{N}\sum^N_{n=1}\phiv(s_{k-1}^{(n)},a_{k-1}^{(n)})\br{\phiv(s_{k-1}^{(n)},a_{k-1}^{(n)})^{\mathsf{T}}\widehat{\mmat\val^k_{\theta}} - V^k_{\theta}(s_{k-1}^{\prime(n)})} + \chi\widehat{\mmat\val^k_{\theta}} = 0
\end{equation*}
The statement follows from rearranging the terms.
\end{proof}
\begin{remark}
Note that when $\chi = 0$, and $\phiv(s,a)$ is one-hot vector for every $(s,a)$, then we obtain the tabular estimators $\mbf{W}_v$ proposed in~\cite{Pacchiano:2021}.
\end{remark}
We invoke Theorem 2 in~\cite{Hsu:2012} to derive an upper bound for $\norm{\mmat \val^k_{\theta}-\widehat{\mmat\val^k_{\theta}}}^2_{\bar{\mbs{\Lambda}}_k}$. 
\begin{lemma}
Fix some $\chi > 0$ and take $N \geq \mathcal{O}(\frac{\log(\frac{m}{\delta})}{\chi \beta})$. Then, with probability at least $1-\delta$, we have
\begin{align*}
    \norm{\mmat \val^k_{\theta}-\widehat{\mmat\val^k_{\theta}}}^2_{\bar{\mbs{\Lambda}}_k} \leq \mathcal{O}\br{\frac{m \chi^2}{\beta^3} D^2 + \frac{1}{N}\frac{m \chi}{\beta^4} D^2\log\br{\frac{1}{\delta}} + \frac{D^2m}{N}\log\br{\frac{1}{\delta}}},
\end{align*}
where $D\triangleq \frac{1 + \log \br{\frac{1}{\beta}}}{1 - \gamma}\geq 1$ is the upper bound of $\norm{\val^k_{\theta}}_{\infty}$ derived in Proposition~\ref{prop:optimal_theta_bound}. 
\end{lemma}
\begin{proof}
We introduce the following auxiliary quantities:
\begin{align*}
    \mmat_{\chi} \val^k_{\theta} &= \arg\min_{\zv} \Ee{(s,a)\sim \dv_{k-1}, s^\prime \sim P(\cdot|s,a)}{\phiv(s,a)^{\mathsf{T}}\zv - V^k_{\theta}(s^\prime) } + \chi\norm{\zv}^2_2 \\ &= \br{\bar{\mbs{\Lambda}}_k + \chi \mbf{I}}^{-1} \Ee{(s,a)\sim \dv_{k-1}, s^\prime\sim P(\cdot|s,a)}{\phiv(s,a) V^k_{\theta}(s^\prime)},
\end{align*}

and the conditional expectation
\begin{equation*}
    \bar{\mmat}\val^k_{\theta} = \E{\widehat{\mmat\val^k_{\theta}} | \mathcal{F}_n} = \frac{1}{N}\br{\mbs{\Lambda}_{k,N} + \chi \mbf{I}}^{-1}\sum^N_{n=1} \phiv(s_{k-1}^{(n)},a_{k-1}^{(n)})\Ee{s^{\prime}\sim P(\cdot|s_{k-1}^{(n)},a_{k-1}^{(n)})}{V^k_{\theta}(s^{\prime})}
\end{equation*}
with $\mathcal{F}_n$ being the filtration $\mathcal{F}_n = \{ s_{k-1}^{(i)}, a_{k-1}^{(i)} \}^{n}_{i=0}$.
Then applying the general random design decomposition in (\cite{Hsu:2012}, Proposition 3) we obtain:
\begin{equation}\label{eq:randomdesign}
    \norm{\mmat \val^k_{\theta}-\widehat{\mmat\val^k_{\theta}}}^2_{\bar{\mbs{\Lambda}}_k} \leq 3\underbrace{\norm{\mmat \val^k_{\theta} - \mmat_{\chi} \val^k_{\theta}}^2_{\bar{\mbs{\Lambda}}_k}}_{\triangleq\epsilon_{\mathrm{rg}}} + 3\underbrace{\norm{\mmat_{\chi} \val^k_{\theta} - \bar{\mmat}\val^k_{\theta}}^2_{\bar{\mbs{\Lambda}}_k}}_{\triangleq\epsilon_{\mathrm{bs}}} + 3\underbrace{\norm{\bar{\mmat}\val^k_{\theta} - \widehat{\mmat\val^k_{\theta}}}^2_{\bar{\mbs{\Lambda}}_k}}_{\triangleq\epsilon_{\mathrm{vr}}},
\end{equation}
where similarly to~\cite{Hsu:2012}, we define $\epsilon_{\mathrm{rg}}$ as the ridge error, $\epsilon_{\mathrm{bs}}$ the ridge estimator bias and with $\epsilon_{\mathrm{vr}}$ the ridge estimator variance.
By choosing $N \geq \mathcal{O}(6 \rho^2_{\chi} d_{1,\chi}(\log \max \br{1,d_{1,\chi}} + \log \frac{1}{\delta})) = \mathcal{O}(\frac{1}{\beta \chi}\log \frac{m}{\delta})$, we ensure that the conditions in Theorem~2 in~\cite{Hsu:2012} are satisfied. We next bound each term separately.  

\paragraph{Ridge error.} In \cite{Hsu:2012}, the bound derived for the ridge error is a function of the regularization parameter $\chi$, the eigenvalues of the covariance matrix $\bar{\mbs{\Lambda}}_k$ denoted as $\{\sigma_j\}^m_{j=1}$ and the corresponding eigenvectors $\{\mbf{v}_j\}^m_{j=1}$. In particular, we have
\begin{align*}
    \epsilon_{\mathrm{rg}} \leq& \sum^m_{j=1} \frac{\sigma_j}{(\frac{\sigma_j}{\chi} + 1)^2} (\mathbf{v}_j^\trans\mmat \val^k_\theta )^2\\=&\sum^m_{j=1} \frac{\sigma_j}{(\frac{\sigma_j}{\chi} + 1)^2} \br{\Ee{(s,a)\sim \dv_{k-1}, s^\prime\sim P(\cdot|s,a)}{\phiv(s,a) V^k_{\theta}(s^\prime)}^{\mathsf{T}}\bar{\mbs{\Lambda}}_k ^{-1}\mathbf{v}_j}^2\\=&\sum^m_{j=1} \frac{1}{(\frac{\sigma_j}{\chi} + 1)^2\sigma_j} \br{\Ee{(s,a)\sim \dv_{k-1}, s^\prime\sim P(\cdot|s,a)}{\phiv(s,a) V^k_{\theta}(s^\prime)}^{\mathsf{T}}\mathbf{v}_j}^2
    \\\leq&\sum^m_{j=1} \frac{1}{(\frac{\sigma_j}{\chi} + 1)^2\sigma_j} \norm{\val^k_{\theta}}_{\infty}^2
    \\\leq&\sum^m_{j=1} \frac{1}{(\frac{\beta}{\chi} + 1)^2\beta} D^2
    \\=&\frac{m \chi^2}{(\beta + \chi)^2\beta} D^2
    \\\leq&\frac{m \chi^2}{\beta^3} D^2,
\end{align*}
where in the first inequality we used bullet (3) of Theorem~2 in~\cite{Hsu:2012}.

\paragraph{Bias.} It holds that
\begin{align*}
    \epsilon_{\mathrm{bs}} \leq \mathcal{O}\br{\frac{\rho_\chi^2d_{1,\chi}\Ee{(s,a)\sim\dv_{k-1}}{\mathrm{approx}(s,a)} + (1 + \rho_\chi^2d_{1,\chi})\epsilon_{\mathrm{rg}}}{N}\log\br{\frac{1}{\delta}}},
\end{align*}
where we used the notation
\begin{align*}
    \Ee{(s,a)\sim\dv_{k-1}}{\mathrm{approx}(s,a)} \triangleq& \Ee{(s,a)\sim\dv_{k-1}}{\Ee{s^\prime\sim P(\cdot|s,a)}{V^k_\theta(s^\prime)} - \phi(s,a)^{\mathsf{T}}\mmat \val^k_\theta} \quad 
    \\=&\Ee{(s,a)\sim\dv_{k-1}}{\Ee{s^\prime\sim P(\cdot|s,a)}{V^k_\theta(s^\prime)} - \pmat \val^k_\theta (s,a)}
    \\=&\Ee{(s,a)\sim\dv_{k-1}}{\Ee{s^\prime\sim P(\cdot|s,a)}{V^k_\theta(s^\prime)} - \Ee{s^\prime\sim P(\cdot|s,a)}{V^k_\theta(s^\prime)}} = 0.
\end{align*}
Moreover,
\begin{equation*}
    d_{1,\chi}\triangleq\sum^m_{j=1}\frac{\sigma_j}{\sigma_j + \chi} \leq m
\end{equation*}
Finally, according to Remark~2 in~\cite{Hsu:2012}, we have that $\rho_\chi$ is bounded as follows
\begin{equation*}
    \rho^2_{\chi} \leq \frac{\norm{\phiv(s,a)}^2_2}{\chi d_{1,\chi}} \leq \frac{1+\chi}{\chi\beta m} \leq \frac{2}{\chi\beta m},
\end{equation*}
where the last inequality follows from noticing that $d_{1,\chi} \geq \frac{\beta m }{1 +  \chi}$.
Therefore, we can conclude that:
\begin{align*}
    \epsilon_{\mathrm{bs}} \leq& \mathcal{O}\br{\frac{(1 + \frac{2}{\chi \beta})\epsilon_{\mathrm{rg}}}{N}\log\br{\frac{1}{\delta}}}
    \\=&\mathcal{O}\br{\frac{\epsilon_{\mathrm{rg}}}{\chi \beta N}\log\br{\frac{1}{\delta}}}
    \\=&\mathcal{O}\br{\frac{1}{N}\frac{m \chi}{\beta^4} D^2\log\br{\frac{1}{\delta}}},
\end{align*}

\paragraph{Variance.}
From the bullet (5) in \cite{Hsu:2012} it follows that
\begin{equation*}
    \epsilon_{\mathrm{vr}} = \mathcal{O}\br{\frac{\mathrm{Var}\bs{\val^k_\theta(s')\mid s,a} d_{2,\chi}}{N}\log\br{\frac{1}{\delta}}}.
\end{equation*}
We have $\mathrm{Var}\bs{\val^k_\theta(s')\mid s,a} \leq \norm{\val^k_\theta}^2_{\infty}\le D^2$. Finally, bounding $d_{2,\chi}$ we obtain that
\begin{equation*}
    d_{2,\chi} = \sum^m_{j=1}\br{\frac{\sigma_j}{\sigma_j + \chi}}^2 \leq m.
\end{equation*}
Hence we can conclude
\begin{equation*}
    \epsilon_{\mathrm{vr}} = \mathcal{O}\br{\frac{D^2m}{N}\log\br{\frac{1}{\delta}}}.
\end{equation*}
\paragraph{Final bound.}By combining the above bounds with Equation~(\ref{eq:randomdesign}), we get the final bound
\begin{equation*}
    \norm{\mmat \val^k_{\theta}-\widehat{\mmat\val^k_{\theta}}}^2_{\bar{\mbs{\Lambda}}_k} \leq \mathcal{O}\br{\frac{m \chi^2}{\beta^3} D^2 + \frac{1}{N}\frac{m \chi}{\beta^4} D^2\log\br{\frac{1}{\delta}} + \frac{D^2m}{N}\log\br{\frac{1}{\delta}}}.
\end{equation*}
\end{proof}
The bound above is minimized by choosing $\chi$ as small as allowed. This is made precise in the next corollary.
\begin{corollary}
Let $\chi = \mathcal{O}(\frac{\log \frac{m}{\delta}}{\beta N})$. With probability at least $1 - \delta$, it holds that
\begin{align*}
    \norm{\mmat \val^k_{\theta}-\widehat{\mmat\val^k_{\theta}}}^2_{\bar{\mbs{\Lambda}}_k} \leq& \mathcal{O}\br{\frac{D^2 m}{\beta^5N^2}\br{\log\br{ \frac{m}{\delta}}}^2 + \frac{mD^2}{N}\log\br{\frac{1}{\delta}}}.
\end{align*}
\end{corollary}

In order to upper bound $\norm{\mmat \val^k_{\theta}-\widehat{\mmat\val^k_{\theta}}}^2_2$ we need the next lemma.
Hence, to bound $\norm{\mmat \val^k_{\theta}-\widehat{\mmat\val^k_{\theta}}}^2_2$, we can directly apply  Theorem 2 in \cite{Hsu:2012} that leads to the following lemma.
\begin{lemma}
        Given a matrix $\mbf{A}\in\mathbb{R}^{m\times m}$ and a vector $\xv \in \mathbb{R}^m$, we have that $\norm{\xv}_{\mbf{A}} \geq \lambda_{\mathrm{min}}(\mbf{A})\norm{\xv}_{2}$.
\end{lemma}
\begin{proof}
We have that $\mbf{A} - \lambda_{\mathrm{min}}(\mbf{A})\mbf{I} \geq 0$ that implies $\xv^\trans\mbf{A}\xv \geq \lambda_{\mathrm{min}}(\mbf{A})\xv^\trans\xv$.
\end{proof}

\begin{corollary}
Let $\chi = \mathcal{O}(\frac{\log \frac{m}{\delta}}{\beta N})$. With probability at least $1 - \delta$, it holds that
\begin{equation}
    \norm{\mmat \val^k_{\theta}-\widehat{\mmat\val^k_{\theta}}}_2 \leq \mathcal{O}\br{\frac{ D \sqrt{m}}{\beta^3N}\log\br{\frac{m}{\delta}} + \frac{ D \sqrt{m}}{ \sqrt{N\beta}}\sqrt{\log\br{\frac{1}{\delta}} }} .
\end{equation}
\end{corollary}
\begin{corollary} \label{cor:MV_bound}
Let $\chi=\mathcal{O}\br{\frac{\log\frac{m}{\delta}}{\beta N}}$, and $N \geq \max \br{\frac{\gamma^2m D^2 }{\beta \epsilon^2}\log(1/\delta), \frac{\gamma \sqrt{m} D }{\beta^3 \epsilon}\log(m/\delta)}$. Then, with probability at least $1-\delta$, it holds that $\norm{\mmat \val^k_{\theta}-\widehat{\mmat\val^k_{\theta}}}_2 \leq \frac{\epsilon}{\gamma}$.
\end{corollary}


\subsubsection{Estimators for $\mbs{\Gamma}_k$} 
Recall that we introduced $\mbs{\Gamma}_k(i, j)\triangleq\sum_{s',a'}\mmat_{i,s'}\pi_{k-1,\thv}(a'|s')\phiv_j(s',a')$. We can equivalently rewrite it as
\begin{align*}
    \mbs{\Gamma}_k(\cdot, j)=&\mmat \underbrace{\sum_{a'}\pi_{k-1,\thv}(a'|s')\phiv_j(s',a')}_{h_{k, j}(s^\prime)} 
\\=&
    \bar{\mbs{\Lambda}}_k ^{-1} \Ee{(s,a)\sim \dv_{k-1}, s^\prime\sim P(\cdot|s,a)}{\phiv(s,a) h_{k, j}(s^\prime)},
\end{align*}
where the last equality is obtained with manipulations analogous to \Cref{exact_ridge_solution_mV}.

Similarly, We can estimate $\mbs{\Gamma}_k(i, j)$ with a finite amount of environment interactions sampled \textrm{i.i.d.} from $\dv_{k-1}$, by solving the following ridge regression problem:
\begin{equation*}
     \widehat{\mbs{\Gamma}}_k(\cdot, j) =\arg \min_{\zv} \frac{1}{N}\sum^N_{n=1}\br{\phiv(s_{k-1}^{(n)},a_{k-1}^{(n)})^{\mathsf{T}}\zv - h_{k,j}(s_{k-1}^{\prime(n)})}^2 + \chi \norm{\zv}^2_2
\end{equation*}
\begin{lemma}
By optimality conditions, we can obtain a closed form for $\widehat{\mbf{\Gamma}}_k$ as
\begin{equation*}
    \widehat{\mbs{\Gamma}}_k(\cdot, j) = \frac{1}{N}\br{ \mbs{\Lambda}_{k,N} + \chi \mbf{I}}^{-1}\sum^N_{n=1} \phiv(s_{k-1}^{(n)},a_{k-1}^{(n)}) h_{k,j}(s_{k-1}^{\prime(n)}).
\end{equation*}
\end{lemma}
By noting that $\norm{\mbf{h}_{k-1,j}}_{\infty} \leq 1$ for any $k$, it follows that
\begin{corollary}
For $\chi = \mathcal{O}(\frac{\log \frac{m}{\delta}}{\beta N})$, with probability at least $1-\delta$, it holds that
\begin{equation}
    \norm{\mbs{\Gamma}_k(\cdot, j)-\widehat{\mbs{\Gamma}}_k(\cdot, j)}_2 \leq \mathcal{O}\br{ \frac{\sqrt{m}}{ \sqrt{N \beta}}\sqrt{\log\br{\frac{1}{\delta}}} + \frac{ \sqrt{m}}{\beta^3N}\log\br{\frac{m}{\delta}}}.
\end{equation}
\end{corollary}
\begin{corollary}
\label{corollary:bound_on_gamma}
For $\chi = \mathcal{O}(\frac{\log \frac{m}{\delta}}{\beta N})$, and $N \geq \max \br{\mathcal{O}\left(\frac{m }{\beta \epsilon^2}\log(1/\delta)\right),\mathcal{O}\left(\frac{\sqrt{m} }{\beta^3 \epsilon}\log(m/\delta)\right)}$, with probability at least $1-\delta$, it holds that
$
    \norm{\mbs{\Gamma}_k(\cdot, j)-\widehat{\mbs{\Gamma}}_k(\cdot, j)}_2 \leq \epsilon .
$
\end{corollary}
\subsubsection{Estimator for feature expectation vector$\FEV{\pi_{k-1}}$} 
The goal is to estimate $\FEV{\pi_{k-1}}$. Consider the sample transitions $\{s^{(n)}_{k-1}, a^{(n)}_{k-1}\}^N_{n=1}\sim\dv_{k-1}^N$. Then we estimate $\FEV{\pi_{k-1}}=\phim^\trans\dv_{k-1}$ by
$
  \EFEV{\pi_{k-1}}\triangleq \frac{1}{N}\sum^N_{n=1} \phiv(s^{(n)}_{k-1}, a^{(n)}_{k-1}).
$

In the next lemma, we provide a useful concentration result.
\begin{lemma} \label{lemma:rho_concentration}
        With probability at least $1 - \delta$, for all $N \geq  \frac{1.4 \log \log (2N) + \log \frac{10.4 m}{\delta}}{\beta \epsilon^2}$, and for all $i\in[m]$ simultaneously, it holds that
        \begin{equation} \abs{\EFEV{\pi_{k-1}}(i) - \boldsymbol{\rho}_{\phiv}(\pi_{k-1})(i)} \leq 2.26 \epsilon \boldsymbol{\rho}_{\phiv}(\pi_{k-1})(i)
        \end{equation}
\end{lemma}
\begin{proof}
Consider the martingale difference sequence $Z_i(n) = \phiv_i(s^{(n)}_{k-1}, a^{(n)}_{k-1}) - \FEV{\pi_{k-1}}(i)$ with the variance process $V_i{(n)} = \sum^n_{j=1}\E{Z^2_i(j)| \mathcal{F}_{j-1}}$, where $\mathcal{F}_{j-1}$ being the filtration up to the state action pair $(s_{k-1}^{(j)}, a_{k-1}^{(j)})$.
We have,
\begin{align*}
    V_i{(n)}& = \sum^n_{j=1}\E{Z^2_i(j)| \mathcal{F}_{j-1}} \\ & = \sum^n_{j=1}\Ee{(s,a)\sim\dv_{k-1}}{\br{\phiv_i(s,a) - \FEV{\pi_{k-1}}(i)}^2| \mathcal{F}_{j-1}} \\ & =
    \sum^n_{j=1}\Ee{(s,a)\sim\dv_{k-1}}{\phiv^2_i(s,a) - 2 \phiv_i(s,a) \FEV{\pi_{k-1}}(i) + \FEV{\pi_{k-1}}(i)^2| \mathcal{F}_{j-1}} \\
    & \leq
    \sum^n_{j=1}\Ee{(s,a)\sim\dv_{k-1}}{\phiv_i(s,a)| \mathcal{F}_{j-1}} -  n \FEV{\pi_{k-1}}(i)^2
    \\ & = n \br{\FEV{\pi_{k-1}}(i) - \FEV{\pi_{k-1}}(i)^2} \leq n \FEV{\pi_{k-1}}(i).
\end{align*}
The martingale difference sequence $Z_i(j)$ satisfies the sub-$\psi_P$ condition of \cite{Howard:2021} (see Bennet case in their Table 3) with constant $c=2$. Therefore, by Lemma~13 in \cite{Pacchiano:2021} with $m=\FEV{\pi_{k-1}}(i)$, with probability at least $1 - \frac{\delta}{2m}$, for all $N \geq  \frac{1.4 \log \log (2N) + \log \frac{10.4 m}{\delta}}{\beta \epsilon^2}$ simultaneously, it holds that
\begin{align*}
    N \EFEV{\pi_{k-1}}(i) &\geq  N \FEV{\pi_{k-1}}(i) - 1.44 \sqrt{\FEV{\pi_{k-1}}(i)N \br{\log\log 2N + \frac{10.4 m}{\delta}}} \\ 
     &\phantom{{}\geq}- 0.82 \br{1.4\log\log 2N + \frac{10.4 m}{\delta}}
    \\ &\geq  N \FEV{\pi_{k-1}}(i) - 1.44 \sqrt{\FEV{\pi_{k-1}}(i)^2N^2 \epsilon^2} -0.82 N \beta \epsilon^2 
    \\ &\geq  N \FEV{\pi_{k-1}}(i) - 2.26 \FEV{\pi_{k-1}}(i)N\epsilon.
\end{align*}
Similarly, with probability at least $1 - \frac{\delta}{2m}$, for all $N \geq  \frac{1.4 \log \log (2N) + \log \frac{10.4 m}{\delta}}{\beta \epsilon^2}$ simultaneously, it holds that $\EFEV{\pi_{k-1}}(i) \leq \FEV{\pi_{k-1}}(i) + 2.26 \FEV{\pi_{k-1}}(i)N\epsilon$. A union bound concludes the proof.
\end{proof}
\subsubsection{Estimators for $\widehat{\bmat}_{\weight, \thv}^k$}
We can directly invoke Lemma 17 in \cite{Pacchiano:2021} to get guarantees for the estimator $\widehat{\bmat}_{\weight,\thv}^k(i)$. In particular, we obtain the following result.
\begin{lemma} \label{lemma:bound_B_k}
Let $\norm{\mmat \val^k_{\theta} - \widehat{\mmat}\val^k_{\theta}}_{\infty}\leq \frac{\epsilon}{\gamma}$ and $\abs{\EFEV{\pi_{k-1}}(i) - \FEV{\pi_{k-1}}(i)}\leq 2.26 \epsilon \EFEV{\pi_{k-1}}(i)$. Then, it holds that
$
    \abs{\widehat{\bmat}_{\weight,\thv}^k(i) - \bmat_{\weight,\thv}^k(i)} \leq 38 \eta \epsilon \bmat_{\weight,\thv}^k(i) \leq 38 \frac{\eta \epsilon}{\beta} .
$
\end{lemma}
\begin{proof}
First, we notice that $\norm{\mmat \val^k_{\theta} - \widehat{\mmat}\val^k_{\theta}}_{\infty}\leq \frac{\epsilon}{\gamma}$ implies that $\widehat{{\mbs{\delta}}}\,_{\weight,\thv}^{k}(i) - {{\mbs{\delta}}}\,_{\weight,\thv}^{k}(i) \leq \epsilon$.
Therefore, by Lemma 17 in \cite{Pacchiano:2021} we get
$
    \abs{\widehat{\bmat}_{\weight,\thv}^k(i) - \bmat_{\weight,\thv}^k(i)} \leq 38 \eta \epsilon \bmat_{\weight,\thv}^k(i).
$
Moreover, it holds that
\begin{equation*}
    \bmat_{\weight,\thv}^k(i) = \frac{e^{-\eta{\mbs{\delta}}\,_{\weight,\thv}^{k}(i)}}{\sum^m_i \rho_{\phiv_i}(\pi_{k-1}) e^{-\eta {{\mbs{\delta}}}\,_{\weight,\thv}^{k}(i)}} \leq \frac{e^{-\eta{\mbs{\delta}}\,_{\weight,\thv}^{k}(i)}}{\beta \sum^m_i e^{-\eta {{\mbs{\delta}}}\,_{\weight,\thv}^{k}(i)}} \leq \frac{1}{\beta}.
\end{equation*}
Therefore,
\begin{equation*}
    \abs{\widehat{\bmat}_{\weight,\thv}^k(i) - \bmat_{\weight,\thv}^k(i)} \leq\frac{ 38 \eta \epsilon}{\beta}, \quad \textup{and} \quad \widehat{\bmat}_{\weight,\thv}^k(i) \leq \bmat_{\weight,\thv}^k(i) \br{1 + 38\eta\epsilon} \leq \frac{1}{\beta} \br{1 + 38\eta\epsilon}.
\end{equation*}
\end{proof}

\begin{corollary}\label{corollary:beta_bound}
Let $N_1 \geq \max \br{\mathcal{O}\left(\frac{\gamma^2m D^2 }{\beta \epsilon^2}\log(2/\delta)\right), \mathcal{O}\left(\frac{\gamma \sqrt{m} D }{\beta^3 \epsilon}\log(2m/\delta)\right)}$ and $N_2 \geq  \frac{1.4 \log \log (2N_2) + \log \frac{20.8 m}{\delta}}{\beta \epsilon^2}$. Then, for $\chi = \mathcal{O}\left(\frac{\log \frac{2m}{\delta}}{\beta N}\right)$, and for $N \geq \max\br{N_1, N_2}$, with probability at least $1-\delta$, it holds that $\abs{\widehat{\bmat}_{\weight,\thv}^k(i) - \bmat_{\weight,\thv}^k(i)} \leq 38 \frac{\eta \epsilon}{\beta}$,  for all  $i\in[m]$.
\end{corollary}
\begin{proof}
By Corollary~\ref{cor:MV_bound}, we have that with $N \geq N_1$ it holds that $\norm{\mmat \val^k_{\theta} - \widehat{\mmat}\val^k_{\theta}}_{\infty}\leq \frac{\epsilon}{\gamma}$,  with probability $1 - \delta/2$. Furthermore,  Lemma~\ref{lemma:rho_concentration} gives that for $N \geq  \frac{1.4 \log \log (2N) + \log \frac{20.8 m}{\delta}}{\beta \epsilon^2}$, it holds with probability $1 - \delta/2$ that $\abs{\EFEV{\pi_{k-1}}(i) - \FEV{\pi_{k-1}}(i)}\leq 2.26 \epsilon \EFEV{\pi_{k-1}}(i)$,
for all $i\in[m]$ simultaneously.

Therefore, a union bound gives that for $N \geq \max\br{N_1, N_2}$, with probability $1 -\delta$, we have that $\norm{\mmat \val^k_{\theta} - \widehat{\mmat}\val^k_{\theta}}_{\infty}\leq \frac{\epsilon}{\gamma}$, and $\abs{\EFEV{\pi_{k-1}}(i) - \FEV{\pi_{k-1}}(i)}\leq 2.26 \epsilon \EFEV{\pi_{k-1}}(i)$, for all $i \in [m]$. An application of Lemma~\ref{lemma:bound_B_k} concludes the proof.
\end{proof}
\subsubsection{Estimators for $\bmat_{\weight,\thv}^k(i)\mathbf{\Gamma}_k(i,j)$}

We obtain an estimator for $\bmat_{\weight,\thv}^k(i)\mbs{\Gamma}_k(i,j)$ simply as $\widehat{\bmat}_{\weight,\thv}^k(i)\widehat{\mbs{\Gamma}}_k(i,j)$.
The next lemma gives guarantees for such an estimator.
\begin{lemma}
\label{lemma:epsilon_bound_gamma_B}
Assume that for any $(i,j) \in[m]^2$, it holds that $\abs{\widehat{\bmat}_{\weight,\thv}^k(i) - \bmat_{\weight,\thv}^k(i)}\leq \frac{38\eta\epsilon}{\beta}$ and $\abs{\widehat{\mbs{\Gamma}}_k(i,j) - \mbs{\Gamma}_k(i,j)} \leq \epsilon$. Then,
$
\abs{\bmat_{\weight,\thv}^k(i)\mbs{\Gamma}_k(i,j) - \widehat{\bmat}_{\weight,\thv}^k(i)\widehat{\mbs{\Gamma}}_k(i,j)} \leq
\frac{\epsilon}{\beta}(1 + (1 + \epsilon)38\eta)
$, for all $(i,j)\in[m]^2$.\end{lemma}
\begin{proof} We have that
\begin{align*}
    \abs{\bmat_{\weight,\thv}^k(i)\mbs{\Gamma}_k(i,j) - \widehat{\bmat}_{\weight,\thv}^k(i)\widehat{\mbs{\Gamma}}_k(i,j)} \leq & \bmat_{\weight,\thv}^k(i) \abs{\widehat{\mbs{\Gamma}}_k(i,j) - \mbs{\Gamma}_k(i,j)}  + \widehat{\mbs{\Gamma}}_k(i,j)\abs{\widehat{\bmat}_{\weight,\thv}^k(i) - \bmat_{\weight,\thv}^k(i)} \\ \leq &\frac{1}{\beta}\abs{\widehat{\mbs{\Gamma}}_k(i,j) - \mbs{\Gamma}_k(i,j)} + (1 + \epsilon)\abs{\widehat{\bmat}_{\weight,\thv}^k(i) - \bmat_{\weight,\thv}^k(i)} \\ \leq & \frac{\epsilon}{\beta} + (1 + \epsilon)\frac{38\eta\epsilon}{\beta}= \frac{\epsilon}{\beta}\big(1 + (1 + \epsilon)38\eta\big),
\end{align*}
where we used the bound $\hat{\mbs{\Gamma}}_k(i,j) \leq \mbs{\Gamma}_k(i,j) + \epsilon \leq 1 + \epsilon$.
\end{proof}
\begin{lemma}
\label{lemma:bound_Gamma_B}
        For $\chi= \mathcal{O}\br{\frac{\log\frac{m}{\delta}}{\beta N}}$, choose $N \geq \max\Bigg(N_1, N_2, \mathcal{O}\left(\frac{m }{\beta \epsilon^2}\log(m/\delta)\right),\mathcal{O}\left(\frac{\sqrt{m} }{\beta^3 \epsilon}\log(m^2/\delta)\right)\Bigg)$ with $N_1$ and $N_2$ as defined in Corollary~\ref{corollary:beta_bound}, then with probability $1-\delta$, for all $(i,j)\in[m]^2$ simultaneously:
        \begin{equation*}
         \abs{\bmat_{\weight,\thv}^k(i)\mbs{\Gamma}_k(i,j) - \widehat{\bmat}_{\weight,\thv}^k(i)\widehat{\mbs{\Gamma}}_k(i,j)} \leq \frac{\epsilon}{\beta}(1 + (1 + \epsilon)38\eta).
        \end{equation*}
\end{lemma}
\begin{proof}
By Corollary~\ref{corollary:beta_bound}, when $\chi=\mathcal{O}\br{\frac{\log\frac{m}{\delta}}{\beta N}}$, and $N\geq \max\br{N_1, N_2}$, it holds with probability at least $1-\delta$ that
\begin{equation*}
        \abs{\widehat{\bmat}_{\weight,\thv}^k(i) - \bmat_{\weight,\thv}^k(i)} \leq \frac{38 \eta \epsilon}{\beta},\;\mbox{for all}\;i\in[m].
\end{equation*}
Moreover, when $N\geq\max\br{\mathcal{O}\left(\frac{m }{\beta \epsilon^2}\log(m/\delta)\right),\mathcal{O}\left(\frac{\sqrt{m} }{\beta^3 \epsilon}\log(m^2/\delta)\right)}$, by Corollary~\ref{corollary:bound_on_gamma}, with probability at least $1-\delta$, it holds that $\norm{\mbs{\Gamma}_k(\cdot, j)-\widehat{\mbs{\Gamma}}_k(\cdot, j)}_2 \leq \epsilon$, for all $j\in[m]$ simultaneously.

Finally, a union bound and Lemma~\ref{lemma:epsilon_bound_gamma_B} give that with probability at least $1-\delta$, it holds that $\abs{\bmat_{\weight,\thv}^k(i)\mbs{\Gamma}_k(i,j) - \widehat{\bmat}_{\weight,\thv}^k(i)\widehat{\mbs{\Gamma}}_k(i,j)} \leq \frac{\epsilon}{\beta}(1 + (1 + \epsilon)38\eta)$.
\end{proof}
\subsection{Properties of Stochastic gradients}
\begin{lemma}\label{lemma:BSG_bounds}
Let $N \geq \max\br{N_1, N_2, \mathcal{O}\left(\frac{m }{\beta \epsilon^2}\log(m/\delta)\right),\mathcal{O}\left(\frac{\sqrt{m} }{\beta^3 \epsilon}\log(m^2/\delta)\right)}$ with $N_1$ and $N_2$ with $N_1$ and $N_2$ as defined in Corollary~\ref{corollary:beta_bound}. Then, with probability $1 - \delta$, the following bounds 
on the stochastic gradient variance 
hold simultaneously:
\begin{align*}
    \norm{\widehat{\nabla}_{\thv}\mathcal{G}_k(\weight,\thv) - \Ee{i_{k-1}^{(N+1)}}{\widehat{\nabla}_{\thv}\mathcal{G}_k(\weight,\thv)| \mathcal{F}_N}}_{\infty} &\leq  2\frac{(1 + 38 \epsilon \eta)}{\beta}(2 + \epsilon) + 2(1 - \gamma), \\
     \norm{\widehat{\nabla}_{\weight}\mathcal{G}_k(\weight,\thv) - \Ee{i_{k-1}^{(N+1)}}{\widehat{\nabla}_{\weight}\mathcal{G}_k(\weight,\thv)|\mathcal{F}_N}}_{\infty} &\leq  2\br{1 + \frac{1 + 38\eta\epsilon}{\beta}} .
\end{align*}
Furthermore, with probability at least $1 - \delta$, the following bounds on the stochastic gradient bias hold simultaneously:
\begin{align*}
\norm{\E{\widehat{\nabla}_{\thv, j}\mathcal{G}_k(\weight,\thv)| \mathcal{F}_N} - \nabla_{\thv, j}\mathcal{G}_k(\weight,\thv)}_1 &\leq  m\frac{\epsilon}{\beta}\br{\gamma + 38\eta\br{1 + \gamma(1 + \epsilon)}}, \\
    \norm{\nabla_{\weight, j}\mathcal{G}_k(\weight,\thv) - \E{\widehat{\nabla}_{\weight, j}\mathcal{G}_k(\weight,\thv)|\mathcal{F}_N}}_1 &\leq  \frac{38\eta\epsilon}{\beta}.
\end{align*}
\end{lemma}
\begin{proof}
\textbf{Variance for gradient wrt $\thv$.} Recall that by definition of the stochastic gradient we have that
\begin{equation*}
    \widehat{\nabla}_{\thv,j}\mathcal{G}_k(\weight,\thv)-(1-\gamma)\phiv_j(s_{k-1}^{(0)},a_{k-1}^{(0)}) =\widehat{\bmat}_{\weight,\thv}^k(i_{k-1}^{(N+1)})\Big[\gamma\widehat{\mbs{\Gamma}}_k(i_{k-1}^{(N+1)}, j)-\mathds{1}\{i_{k-1}^{(N+1)}=j\}\Big].
\end{equation*}
It then follows that
\begin{align*}
    \abs{\widehat{\nabla}_{\thv}\mathcal{G}_k(\weight,\thv)-(1-\gamma)\phiv_j(s_{k-1}^{(0)},a_{k-1}^{(0)})} \leq & \gamma \abs{\widehat{\bmat}_{\weight,\thv}^k(i_{k-1}^{(N+1)})\widehat{\mbs{\Gamma}}_k(i_{k-1}^{(N+1)}, j)} + \norm{\widehat{\bmat}_{\weight,\thv}^k(j)}_{\infty} 
\end{align*}
Invoking Lemma~\ref{lemma:bound_Gamma_B}, we have that if $N \geq \max\br{N_1, N_2, \mathcal{O}\left(\frac{m }{\beta \epsilon^2}\log(m/\delta)\right),\mathcal{O}\left(\frac{\sqrt{m} }{\beta^3 \epsilon}\log(m^2/\delta)\right)}$ with $N_1$ and $N_2$ as defined in Corollary~\ref{corollary:beta_bound}, then with probability $1-\delta$, $$\abs{\widehat{\bmat}_{\weight,\thv}^k(i_{k-1}^{(N+1)})\widehat{\mbs{\Gamma}}_k(i_{k-1}^{(N+1)}, j)} \leq \frac{1}{\beta} + \frac{\epsilon}{\beta}(1 + 38(1+\epsilon)\eta) = \frac{1}{\beta}(1 + \epsilon(1 + 38(1+\epsilon)\eta)).$$

Similarly, by Corollary~\ref{corollary:beta_bound}, for $N \geq \max\br{N_1, N_2}$, we have that with probability $1-\delta$,
\begin{equation*}
 \widehat{\bmat}_{\weight,\thv}^k(i^{(N+1)}_{k-1}) \leq \bmat_{\weight,\thv}^k(i^{(N+1)}_{k-1}) \br{1 + 38\eta\epsilon} \leq \frac{1}{\beta} \br{1 + 38\eta\epsilon}.
\end{equation*}
Hence, a union bound gives that with probability $1-\delta$,
\begin{equation*}
    \abs{\widehat{\nabla}_{\thv, j}\mathcal{G}_k(\weight,\thv)-(1-\gamma)\phiv_j(s_{k-1}^{(0)},a_{k-1}^{0})} \leq \gamma \frac{(1 + 38 \epsilon \eta)}{\beta}(1 + \epsilon) + \frac{(1 + 38 \epsilon \eta)}{\beta}.
\end{equation*}
This implies that
\begin{equation*}
    \abs{\widehat{\nabla}_{\thv, j}\mathcal{G}_k(\weight,\thv)} \leq \gamma \frac{(1 + 38 \epsilon \eta)}{\beta}(1 + \epsilon) + \frac{(1 + 38 \epsilon \eta)}{\beta} + (1 - \gamma).
\end{equation*}
Therefore, by introducing a filtration $\mathcal{F}_N = \sigma\br{\{(s_{{k-1}}^{(n)},a_{{k-1}}^{(n)},s_{{k-1}}^{\prime (n)})\}_{n=1}^N}$, and noticing that $\widehat{\bmat}_{\weight,\thv}^k$ and $\widehat{\mbs{\Gamma}}_k$ are $\mathcal{F}_N$-measurable, we get
\begin{align*}
    \abs{\Ee{i_{k-1}^{(N+1)}}{\widehat{\nabla}_{\thv,j}\mathcal{G}_k(\weight,\thv) | \mathcal{F}_N}} &\leq  \Ee{i_{k-1}^{(N+1)}}{\abs{\widehat{\nabla}_{\thv,j}\mathcal{G}_k(\weight,\thv)}| \mathcal{F}_N} \\ &\leq  \gamma \frac{(1 + 38 \epsilon \eta)}{\beta}(1 + \epsilon) + \frac{(1 + 38 \epsilon \eta)}{\beta} + (1 - \gamma) 
\end{align*}

At this point, we can simply notice that
\begin{align*}
    \abs{\widehat{\nabla}_{\thv}\mathcal{G}_k(\weight,\thv) - \Ee{i_{k-1}^{(N+1)}}{\widehat{\nabla}_{\thv}\mathcal{G}_k(\weight,\thv)| \mathcal{F}_N}} &\leq  2 \bs{\gamma \frac{(1 + 38 \epsilon \eta)}{\beta}(1 + \epsilon) + \frac{(1 + 38 \epsilon \eta)}{\beta} + 2(1 - \gamma)} \\ &\leq  2\frac{(1 + 38 \epsilon \eta)}{\beta}(2 + \epsilon) + 2(1 - \gamma).
\end{align*}
Therefore, with probability $1 -\delta$, it holds that
\begin{equation*}
    \norm{\widehat{\nabla}_{\thv}\mathcal{G}_k(\weight,\thv) - \Ee{i_{k-1}^{(N+1)}}{\widehat{\nabla}_{\thv}\mathcal{G}_k(\weight,\thv)| \mathcal{F}_N}}_{\infty} \leq 2\frac{(1 + 38 \epsilon \eta)}{\beta}(2 + \epsilon) + 2(1 - \gamma).
\end{equation*} 

\textbf{Variance for gradient wrt $\weight$.} 
 Similarly with Corollary~\ref{corollary:beta_bound}, we obtain that if $N \geq \max\br{N_1, N_2}$, then with probability at least $1-\delta$,
 \begin{equation*}
     \norm{\widehat{\nabla}_{\weight}\mathcal{G}_k(\weight,\thv)}_{\infty} \leq 1 + \frac{1 + 38\eta\epsilon}{\beta}.
 \end{equation*}
 This implies that
  \begin{equation*}
     \norm{\widehat{\nabla}_{\weight}\mathcal{G}_k(\weight,\thv) - \Ee{i_{k-1}^{(N+1)}}{\widehat{\nabla}_{\weight}\mathcal{G}_k(\weight,\thv)|\mathcal{F}_N}}_{\infty} \leq 2\br{1 + \frac{1 + 38\eta\epsilon}{\beta}}.
 \end{equation*}
\textbf{Bias  for gradient wrt $\thv$.} By using the unbiased estimator $\widetilde{\nabla}_{\thv, j}\mathcal{G}_k(\weight,\thv)$ in \Cref{eq:thetaestimator2}, we get
\begin{align*}
    \abs{\widetilde{\nabla}_{\thv, j}\mathcal{G}_k(\weight,\thv) - \widehat{\nabla}_{\thv, j}\mathcal{G}_k(\weight,\thv)} &\leq  \abs{\gamma \br{\widehat{\bmat}_{\weight,\thv}^k(i_{k-1}^{(N+1)})\widehat{\mbs{\Gamma}}_k(i_{k-1}^{(N+1)}, j) - \bmat_{\weight,\thv}^k(i_{k-1}^{(N+1)})\mbs{\Gamma}_k(i_{k-1}^{(N+1)}, j)}} \\  &\phantom{{}=}+ \abs{\mathds{1}\{i_{k-1}^{(N+1)}=j\}\br{\widehat{\bmat}_{\weight,\thv}^k(i_{k-1}^{(N+1)}) - \bmat_{\weight,\thv}^k(i_{k-1}^{(N+1)})}} \\&\leq  \gamma \abs{ \widehat{\bmat}_{\weight,\thv}^k(i_{k-1}^{(N+1)})\widehat{\mbs{\Gamma}}_k(i_{k-1}^{(N+1)}, j) - \bmat_{\weight,\thv}^k(i^{(N+1)})\mbs{\Gamma}_k(i_{k-1}^{(N+1)}, j)}. 
\end{align*}
By choosing $\chi$ and $N$ as in Lemma~\ref{lemma:bound_Gamma_B} and Corollary~\ref{corollary:beta_bound}, and by a union bound, we have that with probability $1-\delta$,
\begin{align*}
    \abs{\widetilde{\nabla}_{\thv, j}\mathcal{G}_k(\weight,\thv) - \widehat{\nabla}_{\thv, j}\mathcal{G}_k(\weight,\thv)} &\leq \gamma\frac{\epsilon}{\beta}(1 + (1+\epsilon)38\eta) +\frac{38\epsilon\eta}{\beta} \\ &=  \frac{\epsilon}{\beta}\br{\gamma + 38\eta\br{1 + \gamma(1 + \epsilon)}}.
\end{align*}
Using that $\widetilde{\nabla}_{\thv, j}\mathcal{G}_k(\weight,\thv)$ is an unbiased estimator of ${\nabla}_{\thv, j}\mathcal{G}_k(\weight,\thv)$, we get
\begin{align*}
\abs{\E{\widehat{\nabla}_{\thv, j}\mathcal{G}_k(\weight,\thv)| \mathcal{F}_N} - \nabla_{\thv, j}\mathcal{G}_k(\weight,\thv)} &=  \abs{\E{\widehat{\nabla}_{\thv, j}\mathcal{G}_k(\weight,\thv) - \widetilde{\nabla}_{\thv, j}\mathcal{G}_k(\weight,\thv)}} \\ &\leq   \E{\abs{\widehat{\nabla}_{\thv, j}\mathcal{G}_k(\weight,\thv) - \widetilde{\nabla}_{\thv, j}\mathcal{G}_k(\weight,\thv)}} \\ &\leq  \frac{\epsilon}{\beta}\br{\gamma + 38\eta\br{1 + \gamma(1 + \epsilon)}}.
\end{align*}
Hence, we have that $\norm{\E{\widehat{\nabla}_{\thv, j}\mathcal{G}_k(\weight,\thv)| \mathcal{F}_N} - \nabla_{\thv, j}\mathcal{G}_k(\weight,\thv)}_1 \leq m\frac{\epsilon}{\beta}\br{\gamma + 38\eta\br{1 + \gamma(1 + \epsilon)}}$.

\paragraph{Bias bound for the gradient wrt $\weight$.} 
Similarly, we can notice that with probability at least $1-\delta$, it holds that
\begin{align*}
    \abs{\widetilde{\nabla}_{\weight, j}\mathcal{G}_k(\weight,\thv) - \widehat{\nabla}_{\weight, j}\mathcal{G}_k(\weight,\thv)}& = \abs{\mathds{1}\{i_{k-1}^{(N+1)}=j\}\br{\widehat{\bmat}_{\weight,\thv}^k(i_{k-1}^{(N+1)}) - \bmat_{\weight,\thv}^k(i_{k-1}^{(N+1)})}} \\ &\leq  \frac{38\eta\epsilon}{\beta}.
\end{align*}
Since we have only one non-zero element, and by the unbiasedness of $\widetilde{\nabla}_{\weight, j}\mathcal{G}_k(\weight,\thv)$, we get
\begin{equation*}
    \norm{\nabla_{\weight, j}\mathcal{G}_k(\weight,\thv) - \E{\widehat{\nabla}_{\weight, j}\mathcal{G}_k(\weight,\thv)|\mathcal{F}_N}}_1 \leq \frac{38\eta\epsilon}{\beta}.
\end{equation*}
\end{proof}

\section{Proof of Theorem~\ref{thm:biased_sgd}}
We first prove a generalization of the Azuma-Hoeffding inequality~(Theorem 3.14 in~\cite{McDiarmid:1998}) that holds when the martingale difference sequence is bounded with high probability but not almost surely.
\begin{lemma}[Modified Azuma-Hoeffding] 
Let $\{ Y_i \}^n_i$ be a martingale difference sequence adapted to $\mathcal{F}_i$, such that for each $i$, $\abs{Y_i} \leq c_i$ with probability at least $1-\delta_2$. Then, it holds that
\begin{equation}
    \P{\sum^n_{i=1} Y_i \geq \epsilon} \leq \exp\br{-\frac{2 \epsilon^2}{\sum^n_{i=1} c_i^2}} + n \delta_2.
\end{equation}
\label{lemma:modified_azuma}
\end{lemma}
\begin{proof}
Define the events $E_i =  \bc{Y_i \leq c_i} $ and the intersection $E = \cap^n_{i=1}\{ E_i \}$, and notice that $\P{E^c} = \P{\cup^n_{i=1}\{ E^c_i \}} \leq \sum^n_{i=1} \P{E^c_i} = n \delta_2 $.
We then have the following decomposition:
\begin{align*}
    \P{\sum^n_{i=1} Y_i \geq \epsilon} &= \P{\{\sum^n_{i=1} Y_i \geq \epsilon\}\cap E} + \P{\{\sum^n_{i=1} Y_i \geq \epsilon\}\cap E^c}
    \\ &\leq \P{\{\sum^n_{i=1} Y_i \geq \epsilon\}\cap E} + \P{E^c}
    \\ &\leq \P{\sum^n_{i=1} Y_i \geq \epsilon \big | E} \underbrace{\P{E}}_{\leq 1}+ n\delta_2 \leq \exp\br{-\frac{2 \epsilon^2}{\sum^n_{i=1} c_i^2}} + n \delta_2,
\end{align*}
where in the last step we noticed that under the event $E$, the martingale difference sequence is bounded almost surely, therefore we can apply the standard Azuma-Hoeffding inequality. 
\end{proof}
\begin{corollary}\label{cor:Hoeffding}
Let $\{ Y_i \}^n_i$ be a martingale difference sequence adapted to $\mathcal{F}_i$, such that for each $i$, $\abs{Y_i} \leq c_i$ with probability at least $1-\delta_2$. Then, with probability $1-\delta_1$ (with $\delta_1 > n\delta_2$), it holds that
\begin{equation*}
    \sum^n_{i=1} Y_i \geq \sqrt{\frac{\br{\sum^n_{i=1}c^2_i}\log\br{1/\br{\delta_1 - n\delta_2}}}{2}}.
\end{equation*}
\end{corollary}

\begin{proofof}{Proof of Theorem~\ref{thm:biased_sgd}}

We fix a policy evaluation step $k\in[K]$, i.e., we study the $k$-th iteration of the outer loop of Algorithm~\ref{alg:PPIQL}. Similarly to the proof of Lemma~19 in~\cite{Pacchiano:2021}, the biased SGD subroutine can be seen as an inexact gradient ascent scheme with updates
\begin{align}
\weight_{t+1}^k&=\Pi_{\mathcal{W}}\Big(\weight_{t}^k+\beta_{t}(\nabla_{\weight}f(\weight_t^k,\thv_t^k)+b_{\weight,t}^k+\epsilon_{\weight,t}^k\Big),\\
\thv_{t+1}^k&=\Pi_{{\Theta}}\Big(\thv_{t}^k+\beta_{t}(\nabla_{\thv}f(\weight_t^k,\thv_t^k)+b_{\thv,t}^k+\epsilon_{\thv,t}^k\Big),
\end{align}
with
\begin{align}
    \epsilon_{\thv,t}^k&\triangleq\widehat{\nabla}_{\thv}\mathcal{G}_k(\weight_t^k,\thv_t^k)-\Exp\Big[\widehat{\nabla}_{\thv}\mathcal{G}_k(\weight_t^k,\thv_t^k)\mid\mathcal{F}_{t-1}\Big],\\
    \epsilon_{\weight,t}^k&\triangleq\widehat{\nabla}_{\weight}\mathcal{G}_k(\weight_t^k,\thv_t^k)-\Exp\Big[\widehat{\nabla}_{\weight}\mathcal{G}_k(\weight_t^k,\thv_t^k)\mid\mathcal{F}_{t-1}\Big],\\
    b_{\thv,t}^k&\triangleq\Exp\Big[\widehat{\nabla}_{\thv}\mathcal{G}_k(\weight_t^k,\thv_t^k)\mid\mathcal{F}_{t-1}\Big]-{\nabla}_{\thv}\mathcal{G}_k(\weight_t^k,\thv_t^k),\\
     b_{\weight,t}^k&\triangleq\Exp\Big[\widehat{\nabla}_{\weight}\mathcal{G}_k(\weight_t^k,\thv_t^k)\mid\mathcal{F}_{t-1}\Big]-{\nabla}_{\weight}\mathcal{G}_k(\weight_t^k,\thv_t^k).
   \end{align}
   By Lemma~\ref{lemma:BSG_bounds}, and a union bound, we get that for $n(t)\geq\max\left\{\mathcal{O}\left(\frac{\gamma^2 m D^2}{\beta\xi_t^2}\log(\frac{Tm}{\delta}), \right), \mathcal{O}\left(\frac{ m }{\beta\xi_t^2}\log(\frac{Tm}{\delta}), \right)\right\}$, with probability at least $1-\delta/2$, for all $t=1,\ldots,T$ simultaneously, it holds that
   \begin{align}
    \normm{\epsilon_{\thv,t}^k}_{1} &\leq  2m\frac{(1 + 38 \xi_t \eta)}{\beta}(2 + \xi_t) + 2(1 - \gamma)\le\frac{6m}{\beta}(1+38\eta)+2,\label{eq:firstone} \\
      \normm{\epsilon_{\weight,t}^k}_{1} &\leq  2m\br{1 + \frac{1 + 38\eta\xi_t}{\beta}}\le 2m(1+\frac{1+38\eta}{\beta}), \\
 \normm{b_{\thv,t}^k}_{1} &\leq  m\frac{\xi_t}{\beta}\br{\gamma + 38\eta\br{1 + \gamma(1 + \xi_t)}}\le\frac{m}{\beta}(1+114\beta), \\
     \normm{b_{\weight,t}^k}_{1} &\leq  \frac{38\eta\xi_t}{\beta}\le \frac{38\eta}{\beta},
\end{align}
where we used that $\{\xi_t\}_{t=1}^{T}\cup\{\gamma\}\subset(0,1)$.

Moreover, by H\"{o}lder's inequality, we get
\begin{align}
    |\innerprod{\epsilon_{\thv,t}^k}{\thv_t^k-\thv^\star_k}|&\le \normm{\epsilon_{\thv,t}^k}_{1}\normm{\thv_t^k-\thv_k^\star}_\infty\le \frac{12 D m}{\beta}(1+38\eta)+2\triangleq M_1,\\
    |\innerprod{\epsilon_{\weight,t}^k}{\weight_t^k-\weight_k^\star}|&\le \normm{\epsilon_{\weight,t}^k}_{1}\normm{\weight_t^k-\weight_k^\star}_\infty\le 2m(1+\frac{1+38\eta}{\beta})\triangleq M_2,
\end{align}
where we used that by the triangle inequality and Proposition~\ref{prop:optimal_theta_bound}, it holds that $\normm{\thv_t^k-\thv_k^\star}_\infty\le 2\frac{1 + \abs{\log\beta}}{1 - \gamma}\triangleq 2D$. We recall that $D\triangleq \frac{1 + \log \br{\frac{1}{\beta}}}{1 - \gamma}\geq 1$.

Since $\Big\{X_{\thv,t}^k\triangleq\innerprod{\epsilon_{\thv,t}^k}{\thv_t^k-\thv_k^\star}\Big\}_{t=1}^\infty$ and $\Big\{X_{\weight,t}^k\triangleq\innerprod{\epsilon_{\weight,t}^k}{\weight_t^k-\weight_k^\star}\Big\}_{t=1}^\infty$ are martingale differences, by using Corollary~\ref{cor:Hoeffding} and a simple union bound, we get that with probability at least $1-\delta/2$,
\begin{align}
   -\sum_{t=1}^T \innerprod{\epsilon_{\thv,t}^k}{\thv_t^k-\thv_k^\star}&\le 2 M_1\sqrt{T\log(\frac{16T}{\delta})},\\
   -\sum_{t=1}^T \innerprod{\epsilon_{\weight,t}^k}{\weight_t^k-\weight_k^\star}&\le 2 M_2\sqrt{T\log(\frac{16T}{\delta})}.
\end{align}
Furthermore, note that $\mathcal{G}_k$ is $\eta+\alpha$-smooth with respect to the $\normm{\cdot}_\infty$-norm, and so by Lemma~12 in~\cite{Pacchiano:2021}, we can bound the $\normm{\cdot}_1$-norm of its gradients. In particular, we have
\begin{equation}
    \normm{\nabla_{\thv}\mathcal{G}(\weight_t^k,\thv_t^k)}_1+ \normm{\nabla_{\weight}\mathcal{G}(\weight_t^k,\thv_t^k)}_1\le 2(\eta+\alpha)(D+1).
\end{equation}
This in turn implies that
\begin{equation}
    \normm{\nabla_{\thv}\mathcal{G}(\weight_t^k,\thv_t^k)}_2^2+ \normm{\nabla_{\weight}\mathcal{G}(\weight_t^k,\thv_t^k)}_2^2\le 4(\eta+\alpha)^2(D+1)^2.\label{eq:secondone}
\end{equation}
By smoothness and concavity of the objective $\mathcal{G}_k$, we can apply  Lemma~9 in~\cite{Pacchiano:2021}. In particular, by Equations~(\ref{eq:firstone})--(\ref{eq:secondone}), a union bound, and by summing over $t$ in the bound of Lemma~9 in~\cite{Pacchiano:2021}, we have the following guarantee for our inexact gradient scheme: 

If $n(t)\geq\max\left\{\mathcal{O}\left(\frac{\gamma^2 m D^2}{\beta\xi_t^2}\log(\frac{Tm}{\delta}), \right), \mathcal{O}\left(\frac{ m }{\beta\xi_t^2}\log(\frac{Tm}{\delta}), \right)\right\}$, and $\beta_t\le\frac{2}{\alpha+\eta}$, then with probability at least $1-\delta$, it holds that
\begin{align}
&\sum_{t=1}^T\left(\mathcal{G}_k(\weight_k^\star,\thv_k^\star)-\mathcal{G}_k(\weight_t^k,\thv_t^k)\right)\\
&\le\sum_{t=1}^T\frac{\normm{\weight_t^k-\weight_k^\star}_2^2+\normm{\weight_{t+1}^k-\weight_k^\star}_2^2}{2\beta_t}+\sum_{t=1}^T\frac{\normm{\thv_t^k-\thv_k^\star}_2^2+\normm{\thv_{t+1}^k-\thv_k^\star}_2^2}{2\beta_t}\\
&\phantom{{}\le}+ 2\sum_{t=1}^T\beta_t\left(\normm{\nabla_{\thv}\mathcal{G}(\weight_t^k,\thv_t^k)}_2^2+ \normm{\nabla_{\weight}\mathcal{G}(\weight_t^k,\thv_t^k)}_2^2\right)\\
&\phantom{{}\le}+5\sum_{t=1}^T\beta_t\left(\normm{b_{\weight,t}^k}_2^2+\normm{b_{\thv,t}^k}_2^2+\normm{\epsilon_{\weight,t}^k}_2^2+\normm{\epsilon_{\thv,t}^k}_2^2\right)\\
&\phantom{{}\le}+\sum_{t=1}^T\left(\normm{b_{\weight,t}^k}_1+\normm{b_{\thv,t}^k}_1\right)\max\{\normm{\weight_t^k-\weight_k^\star}_\infty,\normm{\thv_t^k-\thv_k^\star}_\infty\}\\
&\phantom{{}\le}-\sum_{t=1}^T \innerprod{\epsilon_{\thv,t}^k}{\thv_t^k-\thv_k^\star}-\sum_{t=1}^T \innerprod{\epsilon_{\weight,t}^k}{\weight_t^k-\weight_k^\star}\\
&\le\sum_{t=1}^T\frac{\normm{\weight_t^k-\weight_k^\star}_2^2+\normm{\weight_{t+1}^k-\weight_k^\star}_2^2}{2\beta_t}+\sum_{t=1}^T\frac{\normm{\thv_t^k-\thv_k^\star}_2^2+\normm{\thv_{t+1}^k-\thv_k^\star}_2^2}{2\beta_t}\\
&\phantom{{}\le}+\sum_{t=1}^T\Big(\beta_t L_1 +2 D L_2\xi_t \Big)+2 (M_1+M_2)\sqrt{T\log(\frac{16T}{\delta})}, \label{eq:661}
\end{align}
where
\begin{align}
    L_1&=\mathcal{O}\Big((\eta+\alpha)^2D^2+\frac{\max\{\eta,1\}^2m^2}{\beta^2}\Big),\label{eq:771}\\
    L_2&=\mathcal{O}\Big(\frac{\eta+m}{\beta}\Big),\\
    M_1&=\mathcal{O}\Big(\frac{\max\{\eta,1\} m}{\beta}\Big),\\
    M_2&=\mathcal{O}\Big(\frac{\max\{\eta,1\} D m}{\beta}\Big),\label{eq:772}\\
\end{align}
We choose $\beta_t=\frac{L}{\sqrt{t}}$, for some constant $L$. Then a telescoping sum gives 
\begin{align}\label{eq:662}
    \sum_{t=1}^T\Big(\frac{\normm{\weight_t^k-\weight_k^\star}_2^2+\normm{\weight_{t+1}^k-\weight_k^\star}_2^2}{2\beta_t}+\frac{\normm{\thv_t^k-\thv_k^\star}_2^2+\normm{\thv_{t+1}^k-\thv_k^\star}_2^2}{2\beta_t}\Big)
   \le\frac{1}{2L}(D^2+1)\sqrt{T}.
\end{align}
Moreover, $\sum_{t=1}^T\beta_t L_1\le2L_1 L\sqrt{T}$. By combining this inequality with Equations~(\ref{eq:661}) and~(\ref{eq:662}), we get that
\begin{align}
    \sum_{t=1}^T\left(\mathcal{G}_k(\weight_k^\star,\thv_k^\star)-\mathcal{G}_k(\weight_t^k,\thv_t^k)\right)&\le \frac{1}{2L}(D^2+1)\sqrt{T}+2L_1 L\sqrt{T}+2DL_2\sum_{t=1}^T\xi_t\\
    &\phantom{{}\le}+2 (M_1+M_2)\sqrt{T\log(\frac{16T}{\delta})}
\end{align}
The optimal choice for $L$ is $L=\frac{\sqrt{1+D^2}}{2\sqrt{L_1}}$. In addition, by setting $\xi_t=\sqrt{\frac{L_1}{t}}$, we conclude that
\begin{align}\label{eq:773}
    \sum_{t=1}^T\left(\mathcal{G}_k(\weight_k^\star,\thv_k^\star)-\mathcal{G}_k(\weight_t^k,\thv_t^k)\right)&\le 4\max\Big\{\sqrt{1+D^2},\,2DL_2\Big\}\sqrt{L_1}\sqrt{T}\\
    &\phantom{{}\le}+2 (M_1+M_2)\sqrt{T\log(\frac{16T}{\delta})}.
\end{align}
Therefore, by combining Equations~(\ref{eq:771})--(\ref{eq:772}) and Equation~(\ref{eq:773}), and by Jensen's inequality, we get that if $n(t) \geq \max\br{\mathcal{O}\br{\frac{\gamma^2 m D t }{(\eta+\alpha)^2\beta}\log\frac{Tm}{\delta}}, \mathcal{O}\br{\frac{m t}{\beta}\log\frac{Tm}{\delta}}}$, and $\beta_t=\mathcal{O}(\frac{1}{\sqrt{t}})$, then, with probability at least $1-\delta$, it holds that 
$
    \mathcal{G}_k(\weight_k^\star,\thv_k^\star)-\mathcal{G}_k(\weight_k,\thv_k)\le\mathcal{O}\Big(\frac{\max\{\eta,1\} m D}{\beta \sqrt{T}}\Big).
$

\end{proofof}

\subsection{Proof of Corollary~\ref{cor:sample_complexity}}
\begin{proofof}{Proof of Corollary~\ref{cor:sample_complexity}}
We plug the upper bound for $\epsilon_k$ given by \Cref{thm:biased_sgd} in the error propagation analysis of \Cref{thm:error_propagation}.
In particular, from \Cref{thm:error_propagation}, with probability at least $1-\delta_1$, it holds that 
\begin{equation*}
    d_{\mathcal{C}}(\widehat{\pi}, \expert) 
    \leq \frac{1}{K}\Big( \frac{D(\lv^*||\phim^{\mathsf{T}}\dv_0)}{\eta} + \frac{H(\dv^*||\dv_0)}{\alpha}
    + C(\eta,\alpha)\sum_k \sqrt{\epsilon_k} + \sum_k \epsilon_k\Big)+\varepsilon.
\end{equation*}
where we replaced we made explicit the fact the constant (wrt to $K$ and $T$) $C$ depends on $\alpha$ and $\eta$ (See \Cref{thm:error_propagation} for the exact expression).
By plugging in the bound for $\epsilon_k$ given by \Cref{thm:biased_sgd}, and a union bound, we get that and if we use $n(t) \geq \max\br{\mathcal{O}\br{\frac{\gamma^2 m D t }{\beta}\log\frac{Tm}{\delta_2}}, \mathcal{O}\br{\frac{m t}{\beta}\log\frac{Tm}{\delta_2}}}$ samples per iteration, then with probability at least $1 - \delta_1 - \delta_2$, it holds that
\begin{equation*}
    d_{\mathcal{C}}(\widehat{\pi}, \expert) 
    \leq \frac{1}{K}\br{ \frac{D(\lv^*||\phim^{\mathsf{T}}\dv_0)}{\eta} + \frac{H(\dv^*||\dv_0)}{\alpha}
    + C(\eta,\alpha)\sum_k \mathcal{O}\br{\sqrt{\frac{\eta m D}{\beta \sqrt{T}}}} + \sum_k \mathcal{O}\br{\frac{\eta m D}{\beta \sqrt{T}}}}+\varepsilon.
\end{equation*}
Setting $\eta = \alpha = 1$, letting $C_1 \triangleq C(1,1)$ and keeping only the dominant terms, we obtain
\begin{align*}
    d_{\mathcal{C}}(\widehat{\pi}, \expert) 
    \leq \frac{D(\lv^*||\phim^{\mathsf{T}}\dv_0)+H(\dv^*||\dv_0)}{K} + \mathcal{O}\br{C_1\frac{m D}{\beta \sqrt[4]{T}}}+\varepsilon 
\end{align*}
Then, choosing $K=\frac{D(\lv^*||\phim^{\mathsf{T}}\dv_0)+H(\dv^*||\dv_0)}{\epsilon}$ and $T=\Omega\br{\frac{m^4 D^4}{\beta^4 C_1^4 \epsilon^{4}}}$, we can ensure that $d_{\mathcal{C}}(\widehat{\pi}, \expert) 
    \leq \epsilon +\varepsilon $.
The overall sample complexity is $K n(T) = \Omega\br{KT}=\Omega\br{\epsilon^{-5}}$.
Notice that the corollary improves upon the sample complexity bound of $\Omega\br{\epsilon^{-8}}$ derived in \cite{Pacchiano:2021}.
\end{proofof}
\section{Offline imitation learning version}\label{app:offline}
Inspecting  \Cref{eq:PEobjective}, one can notice that estimating the empirical logistic Bellman evaluation objective $\mathcal{G}_k$ or its gradients requires sampling from $\dv_{k-1}$. Hence, the algorithm needs interactions with the environment at every iteration $k$.
It is possible to alleviate this requirement, changing the center point for the relative entropy. This is akin to smoothing \cite{Nesterov:2005} choosing a convenient center point. In particular, we replace \Cref{eq:q-update} with the following update:
\begin{equation}
    (\lv_{1},\dv_{1})=\argmin_{\lv\in\Delta_{[m]},\dv \in\Delta_{\sspace\times\aspace}}\innerprod{\yv^\star}{ \mbf{A}\left[ {\begin{array}{ccc}
    \lv \\ \dv
  \end{array} } \right]+  \widehat{\mbf{b}}} + \frac{1}{\eta}D(\lv||\phim^\trans\mv_{\expert}) + \frac{1}{\alpha}H(\dv||\dv_0).
\end{equation}

Note that we have removed the iteration index $k$, since the offline version does not require to iteratively collect new samples from the environment.
Changing the reference distribution from $\phim^\trans\dv_k$ to $\phim^\trans\mv_{\expert}$ gives \Cref{algo:offline}. In this case, the logistic Bellman evaluation objective takes the form
\begin{equation}
\mathcal{G}(\weight,\thv)\triangleq-\frac{1}{\eta}\log\sum_{i=1}^m (\phim^\trans \mv_{\pi_E})(i)   e^{-\eta\delta_{\weight,\thv}(i)}+(1-\gamma)\innerprod{\initial}{\val_{\thv}}- \innerprod{\EFEV{\expert}}{\weight},
\end{equation}
The difference with the online variant is that in the first term we have the expert occupancy measure instead of the occupancy measure induced by the current policy. We describe the corresponding empirical estimate in Algorithm~\ref{algo:offline}. Furthermore, we suppress the index $k$, since the offline algorithm does not require multiple iterations.
\begin{algorithm}[!t]
			\caption{Offline Proximal Point Imitation Learning (\texttt{OP$^2$IL})
		}\label{algo:offline}
			\begin{algorithmic}
				\STATE {\bfseries Input:} Feature matrix $\mbs{\Phi}$, number of iterations $K$, step sizes $\eta$, $\alpha$, and $\beta$
				\STATE {\bfseries Input:} Expert demonstrations $\mathcal{D}_{\textup{E}}^{n_{\textup{E}},H}$
				\STATE $\circ$ Initialize $\pi_0$ as uniform distribution over $\aspace$, and set $\weight_0=\frac{1}{m}\mbf{1}$
				\STATE $\circ$ Compute the empirical FEV $\widehat{\mbs{\rho}_{\phim}}({\expert})$ using expert demonstrations
				\STATE $\circ$ Sample $\{(s^{(n)},a^{(n)},s^{\prime (n)})\}_{n=1}^N$ with $s^{(n)},a^{(n)}$ sampled i.i.d. from $\mv_{\pi_E}$ and $s^{\prime (n)}\sim P(\cdot|s^{(n)},a^{(n)})$ and compute the  empirical offline logistic Bellman error by
				
				\[
				\widehat{\mathcal{G}}^{}(\weight,\thv) = - \innerprod{\efevphi}{\weight}- \frac{1}{\eta}\log \left(\frac{1}{N}\sum^N_{n=1} e^{- \eta \widehat{{\boldsymbol{\delta}}}_{\weight,\thv}(s^{(n)},a^{(n)},s^{\prime (n)})}\right)
+ (1 - \gamma) \innerprod{\initial}{V_{\thv}}
				\]
				
			    \STATE \texttt{// policy evaluation \& cost update}
				
				\STATE $\circ$ Find an approximate maximizer of the negative empirical logistic Bellman error
				\[
				(\weight_1,\thv_1)\approx\mathrm{argmax}_{\weight,\thv}\widehat{\mathcal{G}}^{}(\weight,\thv)
				\]
				\STATE \texttt{// policy improvement}
				\STATE Policy update:
				$$
				\pi_{\dv_1}(a|s)\propto\pi_{\dv_{0}}(a|s)\,e^{-\alpha Q_{\thv_1}(s,a)}
				$$
				
				\STATE {\bfseries Output:} Policy $\pi_{\dv_1}$
				\label{alg:offline}
			\end{algorithmic}
		\end{algorithm}
\subsection{Theoretical guarantees for the offline case}
With minor modifications of the error propagation analysis given in \Cref{thm:error_propagation}, one can prove the following result.
\begin{theorem}
\label{thm:offline_error_propagation}
Under the same assumptions as in \Cref{thm:error_propagation}, and by choosing $\alpha = \br{\frac{2H(\dv^\star||\dv_0)}{3w_{\mathrm{max}}}\sqrt{\frac{1-\gamma}{2\epsilon}}}^{2/3}$, we obtain 
\begin{equation} 
d_{\mathcal{C}}(\expert, \pi_{\dv_1}) - d_{\mathcal{C}}(\expert, \pi_{\dv^\star}) \leq \frac{D(\lv^\star || \phim^\trans\mv_{\expert})}{\eta} + \br{\frac{243 H(\dv^\star||\dv_0) w^2_{\mathrm{max}}}{2(1 - \gamma)}}^{1/3} \epsilon_1^{1/3} + \epsilon_1 + \varepsilon.
\end{equation}
where $\epsilon_1$ is the error in the maximization of the logistic Bellman error, i.e. $\epsilon=\max_{\weight\in\wspace,\thv}\mathcal{G}(\weight, \thv) - \mathcal{G}(\weight_1, \thv_1)$ and $\varepsilon$ is the error in estimating the expert feature expectation vector as in Lemma~\ref{lemma:expert_concentration}.
\end{theorem}
\begin{proof}
Following exactly the same steps in the proof of \Cref{thm:error_propagation} for the special case of $K=1$, we get
\begin{equation}
    d_{\mathcal{C}}(\pi_E,\pi_{\dv_1}) - d_{\mathcal{C}}(\pi_E,\pi_{\dv^\star}) \leq \frac{D(\lv^\star || \phim^\trans \mv_{\pi_E})}{\eta} + \frac{H(\dv^\star||\dv_0)}{\alpha} + 3 w_{\max}\norm{\dv_1 - \dv_1^{\star}}_1 + \epsilon_1 + \varepsilon.
\end{equation}
By using the bound $\norm{\dv_1 - \dv_1^{\star}}_1 \leq \sqrt{\frac{2  \alpha \epsilon_1}{1 - \gamma}}$, we have
\begin{equation}
    d_{\mathcal{C}}(\pi_E,\pi_{\dv_1}) - d_{\mathcal{C}}(\pi_E,\pi_{\dv^\star}) \leq \frac{D(\lv^\star || \phim^\trans \mv_{\pi_E})}{\eta} + \frac{H(\dv^\star||\dv_0)}{\alpha} + 3 w_{\max}\sqrt{\frac{2  \alpha \epsilon_1}{1 - \gamma}} + \epsilon_1 + \varepsilon.
\end{equation}
Therefore, by choosing $\alpha$ as stated in the theorem we conclude the proof.
\end{proof}
Notice that if the expert is nearly optimal, the step size $\eta$ can be taken small, ensuring low bias in the gradients.
This allows to use the original empirical logistic Bellman error analysis, proposed in \cite{Bas-Serrano:2021}, where one can control the bias by choosing $\eta$ appropriately small. To this end, we need to relate the logistic bellman error in the feature space to the one in the state-action space. As we will show, this introduces an additional bias of order $\mathcal{O}(\eta)$.
The statement is made precise in \Cref{thm:Delta_sa}.
Thanks to this result and Theorem~2 in~\cite{Bas-Serrano:2021}, we have that $\epsilon \leq (8+e)\eta B^2 + 56\sqrt{\frac{m\log{(1+4BN)}{\delta}}{N}}$ where $N$ is the number of expert transitions in the dataset.
We have the following result.
\begin{corollary}
Let $C_1 = \br{\frac{243 H(\dv^\star||\dv_0) w^2_{\mathrm{max}}}{2(1 - \gamma)}}^{1/3}$,  $\eta = \mathcal{O}\br{\frac{D(\lv^\star || \phim^\trans\mv_{\expert})^{3/4}}{(C_1 B)^{1/4}}}$ and $N=\tilde{\mathcal{O}}(m\epsilon^{-6}\log\br{1/\delta})$. Then, with probability $1-\delta$, it holds that
\begin{equation} 
d_{\mathcal{C}}(\expert, \pi_{\dv_1}) - d_{\mathcal{C}}(\expert, \pi_{\dv^\star}) \leq \mathcal{O}\br{C^{1/4}_1 B^{1/4} D(\lv^\star || \phim^\trans\mv_{\expert})^{1/4}} + \mathcal{O}(\epsilon).
\end{equation}
\end{corollary}
\begin{remark} 
We notice that the optimal choice of $\eta$ is smaller as the expert is closely optimal, i.e. $D(\lv^\star || \phim^\trans\mv_{\expert})$ is small. In this condition, we can use the empirical objective estimator proposed in\cite{Bas-Serrano:2021} ensuring small bias. This means that estimating the objective from sample is feasible in the offline setting. It is still an open question if this is viable for the online setting improving the error propagation analysis.
\end{remark}
Next, we present an important result showing that it is possible to replace the minimization of $\mathcal{G}$, with its counterpart in the state-action space defined as
\begin{equation*}
    \mathcal{G}^{\sspace, \aspace}(\thv, \weight)  = -\frac{1}{\eta}\log\sum_{s,a} \mv_{\pi_E}(s,a)   e^{-\eta\delta^{\sspace\aspace}_{\weight,\thv}(s,a)}+(1-\gamma)\innerprod{\initial}{\val_{\thv}}- \innerprod{\EFEV{\expert}}{\weight},
\end{equation*}
where we introduced ${\boldsymbol{\delta}}^{\sspace,\aspace}_{\weight,\thv} = \phim {\boldsymbol{\delta}}_{\weight,\thv}$.
\begin{theorem}
\label{thm:Delta_sa}
Let $B\triangleq 1 + 2 \frac{1 + \abs{\log \beta}}{1-\gamma} $. Suppose $\eta$ is chosen such that $\eta B \leq 1$. Then, it holds that
\begin{equation*}
    \abs{\mathcal{G}^{}(\thv,\weight) - \mathcal{G}^{\sspace, \aspace}(\thv, \weight)} \leq e \eta B^2.
\end{equation*}
\end{theorem}
\begin{proof}
From Proposition~\ref{prop:optimal_theta_bound}, we have that $\norm{\thv}_{\infty} \leq \frac{1 + \abs{\log \beta}}{1-\gamma}$ and $\norm{\val_{\thv}}_{\infty} \leq \frac{1 + \abs{\log \beta}}{1-\gamma}$, for all $\theta\in\ar^m$. It follows that for any $(\weight,\thv)\in\wspace\times\ar^m$, it holds that $\norm{\boldsymbol{\delta}_{\thv, \weight}}_{\infty} = \norm{\weight + \gamma \mmat \val_{\thv} - \thv}_{\infty} \leq 1 + 2 \frac{1 + \abs{\log \beta}}{1-\gamma} = B$.
Hence, it holds that $\eta \norm{\boldsymbol{\delta}_{\thv, \weight}}_{\infty} \leq \eta B \leq 1$.
First, we recall the assumption that the rows of $\phim$ are probability distributions, i.e., $\phiv(s,a)\in\Delta_{[m]}$, for all $(s,a)$. We then have
\begin{equation}
    {\boldsymbol{\delta}}^{\sspace,\aspace}_{\weight,\thv}(s,a) = (\phim {\boldsymbol{\delta}}_{\weight,\thv})(s,a) = \sum_{i=1}^m \phiv_i(s,a) {\boldsymbol{\delta}}_{\weight,\thv}(i) = \Ee{i \sim \phiv(s,a)}{{\boldsymbol{\delta}}_{\weight,\thv}(i)}.
    \label{eq:delta_sa_vs_melta_m}
\end{equation}
Moreover, we have 
\begin{align*}
\mathcal{G}(\weight,\thv) - \mathcal{G}^{\sspace,\aspace}(\weight,\thv) = - \underbrace{\frac{1}{\eta} \log \left( \sum_{i=1}^m (\phim^\trans \mv_{\pi_E})(i) e^{-\eta {\boldsymbol{\delta}}_{\weight,\thv}(i)} \right)}_{\triangleq W} + \underbrace{ \frac{1}{\eta} \log \left( \sum_{s,a}\mv_{\pi_E}(s,a) e^{-\eta {\boldsymbol{\delta}}^{\sspace\aspace}_{\weight,\thv}(s,a)} \right)}_{\triangleq W^{\sspace,\aspace}}
\end{align*}
We can then lower bound $W$ as
\begin{align*}
W & =
\frac{1}{\eta} \log \left( \sum_{i=1}^m \sum_{s,a} \phiv_i(s,a) \mv_{\pi_E}(s,a) e^{-\eta {\boldsymbol{\delta}}_{\weight,\thv}(i)} \right) \\ &=
\frac{1}{\eta} \log \left(\Ee{(s,a) \sim \mv_{\pi_E}}{\Ee{i \sim \phiv(s,a)}{e^{-\eta {\boldsymbol{\delta}}_{\weight,\thv}(i)}}} \right) \\ &\geq \frac{1}{\eta} \log \left(\Ee{(s,a) \sim \mv_{\pi_E}}{e^{-\eta \Ee{i \sim \phiv(s,a)}{{\boldsymbol{\delta}}_{\weight,\thv}(i)}}} \right), \\
&= W^{\sspace,\aspace},
\end{align*}
where the inequality follows by Jensen's inequality for expectations.

We will now upper bound the term $W^{}$.
Thanks to the choice of $\eta$ such that $\eta B \leq 1$, we have that $\eta \leq \frac{1}{B} \leq \frac{1}{\abs{{\boldsymbol{\delta}}_{\weight, \thv}(i)}}$, for all $i$. Therefore, we can apply the inequality $e^{x} \leq 1 + x + x^2$ for $x=-\eta{\boldsymbol{\delta}}^{\sspace\aspace}_{\weight,\thv}(i)\leq 1$ and obtain
\begin{align*}
    \Ee{i \sim \phiv(s,a)}{e^{-\eta{\boldsymbol{\delta}}_{\weight, \thv}(i)}} &\leq \Ee{i \sim \phiv(s,a)}{1 -\eta{\boldsymbol{\delta}}_{\weight, \thv}(i) +(\eta{\boldsymbol{\delta}}_{\weight, \thv}(i))^{2}} \\ 
    &\leq 1 -\Ee{i \sim \phiv(s,a)}{\eta{\boldsymbol{\delta}}_{\weight,\thv}(i)} +(\eta B)^{2} \\
    &= 1 - \eta{\boldsymbol{\delta}}^{\sspace\aspace}_{\weight,\thv}(s,a) +(\eta B)^{2} \\
    &\leq e^{-\eta{\boldsymbol{\delta}}^{\sspace\aspace}_{\weight,\thv}(s,a)} +(\eta B)^{2},
\end{align*}
where in the third line we used \Cref{eq:delta_sa_vs_melta_m}, and in the last line we used the inequality $1 - x \leq e^{-x}$ for $x={\boldsymbol{\delta}}^{\sspace\aspace}_{\weight,\thv}(s,a)$.
By taking expectations with respect to $\mv_{\pi_E}$ and logarithms on both sides, we get
\begin{align*}
    W^{\sspace\aspace} \leq W \leq \frac{1}{\eta}\log \Ee{(s,a) \sim \mv_{\pi_E}}{e^{-\eta{\boldsymbol{\delta}}^{\sspace\aspace}_{\weight,\thv}(s,a)} +(\eta B)^{2}}.
\end{align*}
Subtracting $W^{}$ yields
\begin{align*}
0 \leq W - W^{\sspace\aspace} &\leq \frac{1}{\eta} \log \Ee{(s,a) \sim \mv_{\pi_E}}{e^{-\eta{\boldsymbol{\delta}}^{\sspace\aspace}_{\weight,\thv}(s,a)} +(\eta B)^{2}} - W^{\sspace\aspace} \\ 
&= \frac{1}{\eta}\log \left(1 + \frac{(\eta B)^{2}}{\Ee{(s,a) \sim \mv_{\pi_E}}{e^{-\eta{\boldsymbol{\delta}}^{\sspace\aspace}_{\weight,\thv}(s,a)}}} \right)
\\
&\leq \frac{\eta B^2}{\Ee{(s,a) \sim \mv_{\pi_E}}{e^{-\eta{\boldsymbol{\delta}}^{\sspace\aspace}_{\weight,\thv}(s,a)}}}\\
&\leq \frac{\eta B^{2}}{\Ee{(s,a) \sim \mv_{\pi_E}}{e^{-\eta B}}}\\
&\leq e \eta B^{2},
\end{align*}
where in the third line we used the inequality $\log(1 + x) \leq x$ for $x=\frac{(\eta B)^{2}}{\Ee{(s,a)\sim \mv_{\pi_E}}{e^{-\eta{\boldsymbol{\delta}}^{\sspace\aspace}_{\weight,\thv}(s,a)}}}$, while in the last line we used that $\eta B \leq 1$.
This concludes the proof.
\end{proof}
 After having established with \Cref{thm:Delta_sa} that $\mathcal{G}^{\mathcal{S},\mathcal{A}}$ can be used as biased estimate of $\mathcal{G}$, we can proceed as in \cite{Bas-Serrano:2021}. In particular, we maximize the empirical objective $\widehat{\mathcal{G}}$ (see \Cref{algo:offline}) that is a biased estimate of $\mathcal{G}^{\mathcal{S},\mathcal{A}}$ (\cite[Theorem 2]{Bas-Serrano:2021}). Then, we compute unbiased gradients of $\widehat{\mathcal{G}}$, recurring to the Donsker-Varadhan formula \cite[Corollary 4.15]{Boucheron:2013} that implies the following result.
\begin{theorem}\label{thm:Donsker-Varadhan}
Given a batch of expert data $\{ \widetilde{S}_n, \widetilde{A}_n, \widetilde{S}^\prime_n \}^N_{n=1} \sim \mv_{\expert}\times \pmat$, the following is true:
\begin{equation}
    \max_{\theta} \max_{w} \widehat {\mathcal{G}}(\theta, w) = \max_{\theta} \max_{w} \min_z \mathcal{S}(\theta, w, z)
\end{equation}
with:
\begin{align}
    \mathcal{S}(\theta, w,z) =& -\frac{1}{N} \sum^N_{n=1} \mv_{\expert}(\tilde S_{n},\tilde A_{n})\sum^m_{i=1} \weight_i\phiv_i(\tilde S_{n},\tilde A_{n}) \\&+
    \frac{1}{N} \sum^N_{n=1} z(n) \left( \widehat{\delta}_{\weight,\thv}(\widetilde{S}_{n},\widetilde{A}_{n}, \widetilde{S}^\prime_{n}) + \frac{1}{\eta} \log (N z(n))\right) \\&+
    (1 - \gamma) \inner{\initial, \val_{\thv}}
\end{align}
and the minimum attained at $z^\star \propto \frac{1}{N}e^{- \eta \widehat{\delta}_{\weight,\thv}(\widetilde{S}_{n},\widetilde{A}_{n}, \widetilde{S}^\prime_{n})}$
\end{theorem}
Hence, in the deep learning implementation we update the cost and the value networks backpropagating through $\mathcal{S}(\thv, \weight,z^\star)$.

\subsection{Practical implementation}
We test a practical relaxation of \Cref{algo:offline} that uses two separate neural networks for cost and value function approximation. We use a two layers neural network with $128$ units per layer with \texttt{ReLu} activation for the \texttt{CartPole-v1} environment.
Whereas, for \texttt{Acrobot-v1} and \texttt{LunarLander-v2} we used a 3 layers architecture with $64$ units per layer.

\section{Mirror Descent versus Proximal Point}\label{sec:mirror-descent}
To highlight an important message of our work, in this section, we briefly discuss a mirror descent scheme with alternating updates, and we compare it to our proximal point algorithm in~\Cref{fig:ours_comparison}. Note that in contrast to the classical RL setting, where proximal point and mirror descent coincide because of the linear objective, in imitation learning this is not the case.

The updates for the mirror descent scheme involve alternation between updating the occupancy measure $\dv_k$ and the feature expectation vector $\lv_k$ in one stage and the cost weights in a second stage. That is,
\begin{align}
	(\lv_{k},\dv_{k})&=\argmin_{(\lv,\dv)\in\mathcal{M}_{\phim}}\langle \mv,\cost_{\weight_{k}}\rangle +\tfrac{1}{\eta}D(\lv||\phim^\trans\dv_{k-1})+\tfrac{1}{\alpha}H(\dv||\dv_{k-1}), \label{eq:RL_update}\\
		\weight_{k+1}&=\argmin_{\weight\in\Delta_{[m]}}\innerprod{{\mv_{\expert}-\dv_{k}}}{\cost_\weight}+\tfrac{1}{\beta}D(\weight||\weight_{k}).
\end{align}
One can notice that the update in \Cref{eq:RL_update} corresponds to one update of Logistic $Q$-Learning \cite{Bas-Serrano:2021}. Therefore, it can be implemented by maximizing the negative logistic Bellman error that is now a function only of the variable $\thv$ and not of both $(\thv,\weight)$ as in PPM.
The next proposition is the counterpart of Proposition~\ref{eq:q-update} for the mirror descent scheme.

\begin{proposition}\label{prop:KKT-conditions}
For a parameter $\thv\in\ar^m$, we define the state-action logistic value function $\qval_{\thv}\in\ar^{|\sspace||\aspace|}$ by $\qval_{\thv}\triangleq\phim\thv$, and the $k$-step state logistic value function $\val_{\thv}^k\in\ar^{|\sspace|}$ by
\[
V_{\thv}^k(s)\triangleq-\frac{1}{\alpha}\log\left(\sum_a \pi_{\dv_{k-1}}(a|s)e^{-\alpha Q_{\thv}(s,a)}\right).
\]
Moreover, for a fixed cost $\cost=\cost_\weight$, we define the $k$-step Bellman error function $\boldsymbol{\delta}_{\thv,\weight}^k$ by
$
\boldsymbol{\delta}_{\thv,\weight}^k\triangleq\weight+\gamma\mmat\val_{\thv}^k-\thv.
$
Then, the unique solution of the aforementioned problem is given by 
\begin{align}
\lambda_k(i) &\propto (\phim^\trans\dv_{k-1})(i)\,e^{-\eta\delta_{\thv_k,\weight_{k}}^k(i)},\\
\pi_{\dv_k}(a|s)&\propto\pi_{\dv_{k-1}}(a|s)\,e^{-\alpha Q_{\thv_k}(s,a)},\\
w_{k+1,i}&\propto w_{k,i}\,e^{-\beta\langle \phiv_i\,,\,\mv_\expert-\dv_{k}\rangle},
\end{align}
where $\thv_k$ is the maximizer of the negative $k$-step logistic Bellman error function 
\[
\mathcal{G}_k(\thv)\triangleq-\frac{1}{\eta}\log\sum^m_{i=1}(\phim^\trans\dv_{k-1})(i)e^{-\eta\delta^k_{\thv,\weight_{k}}(i)}+(1-\gamma)\innerprod{\initial}{\val_{\thv}^k}.
\]
\end{proposition}
Proposition~\ref{prop:KKT-conditions} leads to an actor critic scheme that has three separate and alternating updates: (i) policy update stage, (ii) policy evaluation update, and (iii) cost weights update. Similar actor critic-schemes for different MDP models, and different policy evaluation objectives (e.g., minimizing the squared Bellman error) have been also proposed in~\cite{Zhang:2020,Liu:2022,Shani:2021}. Contrary to these schemes, in our proximal imitation learning algorithm, the policy evaluation step involves optimization of a single objective over both cost and $Q$-functions. In this way, we avoid instability or poor convergence in optimization due to nested policy evaluation and cost update steps. In section~\ref{sec:comparison}, we verify numerically that PPM outperforms Mirror Descent in simple tabular environments (see \Cref{fig:ours_comparison}).

\section{Experimental Details}\label{app:experiments}
\begin{figure*}[t] 
\centering
\begin{tabular}{cccc}
\subfloat[WideTree]{%
       \includegraphics[width=0.2\linewidth]{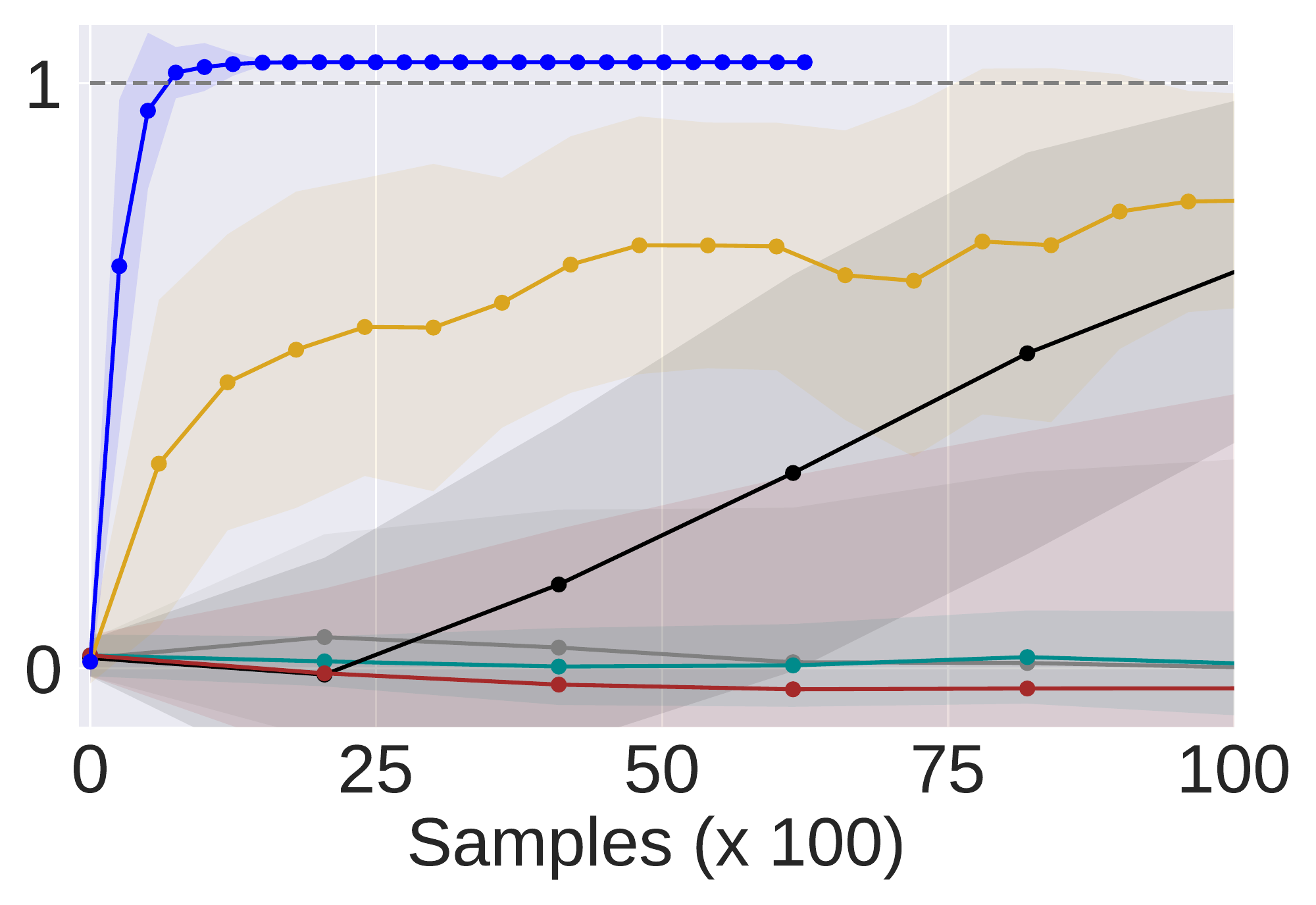}
     } &
\subfloat[RiverSwim]{%
       \includegraphics[width=0.2\linewidth]{plot/only_proximal_point/RiverSwim-v0_normalized.pdf}
     } &
\subfloat[SingleChain]{%
       \includegraphics[width=0.2\linewidth]{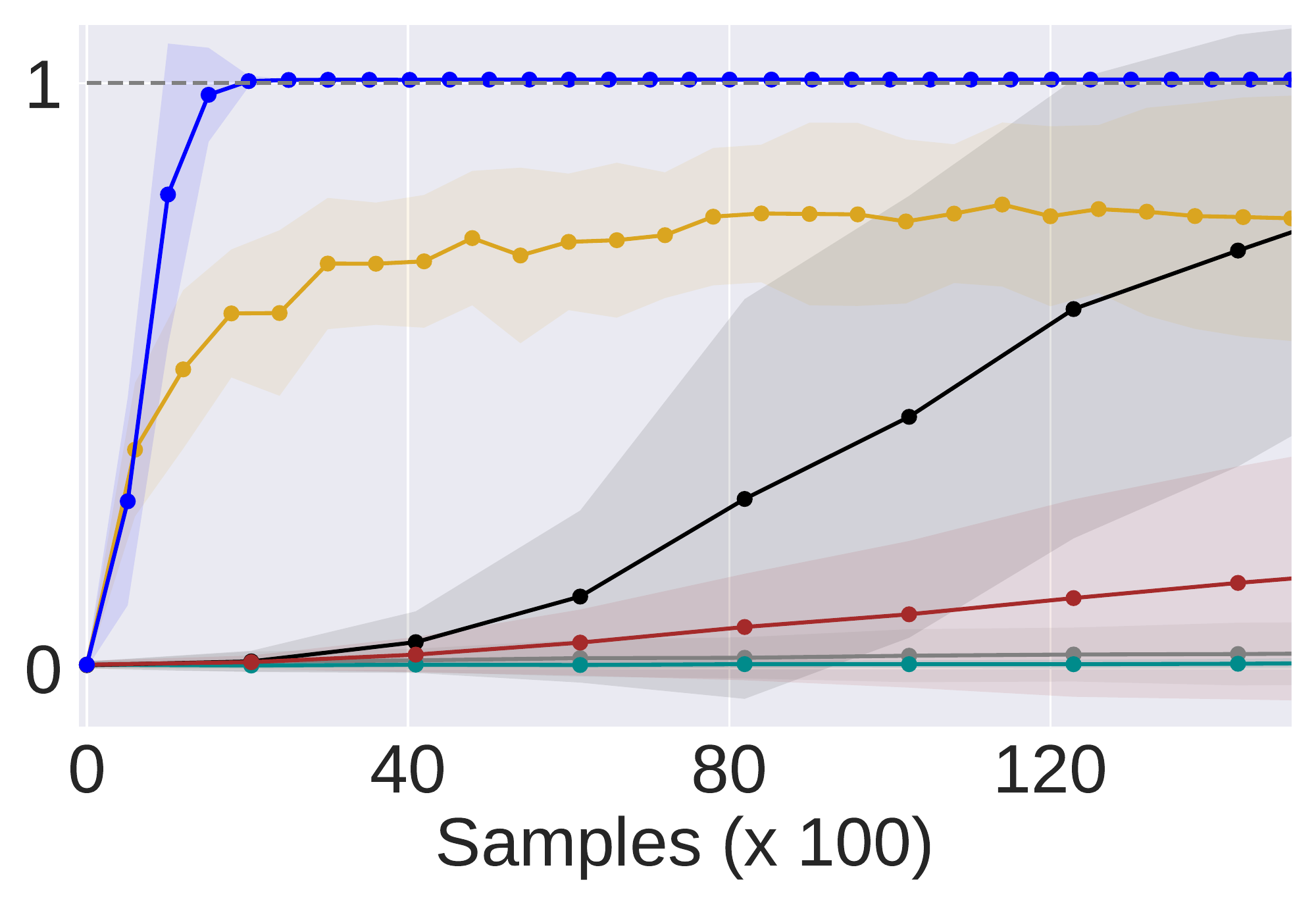}
     } & 
     \subfloat[CartPole]{%
       \includegraphics[width=0.2\linewidth]{plot/only_proximal_point/CartPole-v1_normalized.pdf}
     } \\
\subfloat[DoubleChain]{%
       \includegraphics[width=0.2\linewidth]{plot/only_proximal_point/DoubleChainProblem-v0_normalized.pdf}
     } &
\subfloat[TwoStateStochastic]{%
       \includegraphics[width=0.2\linewidth]{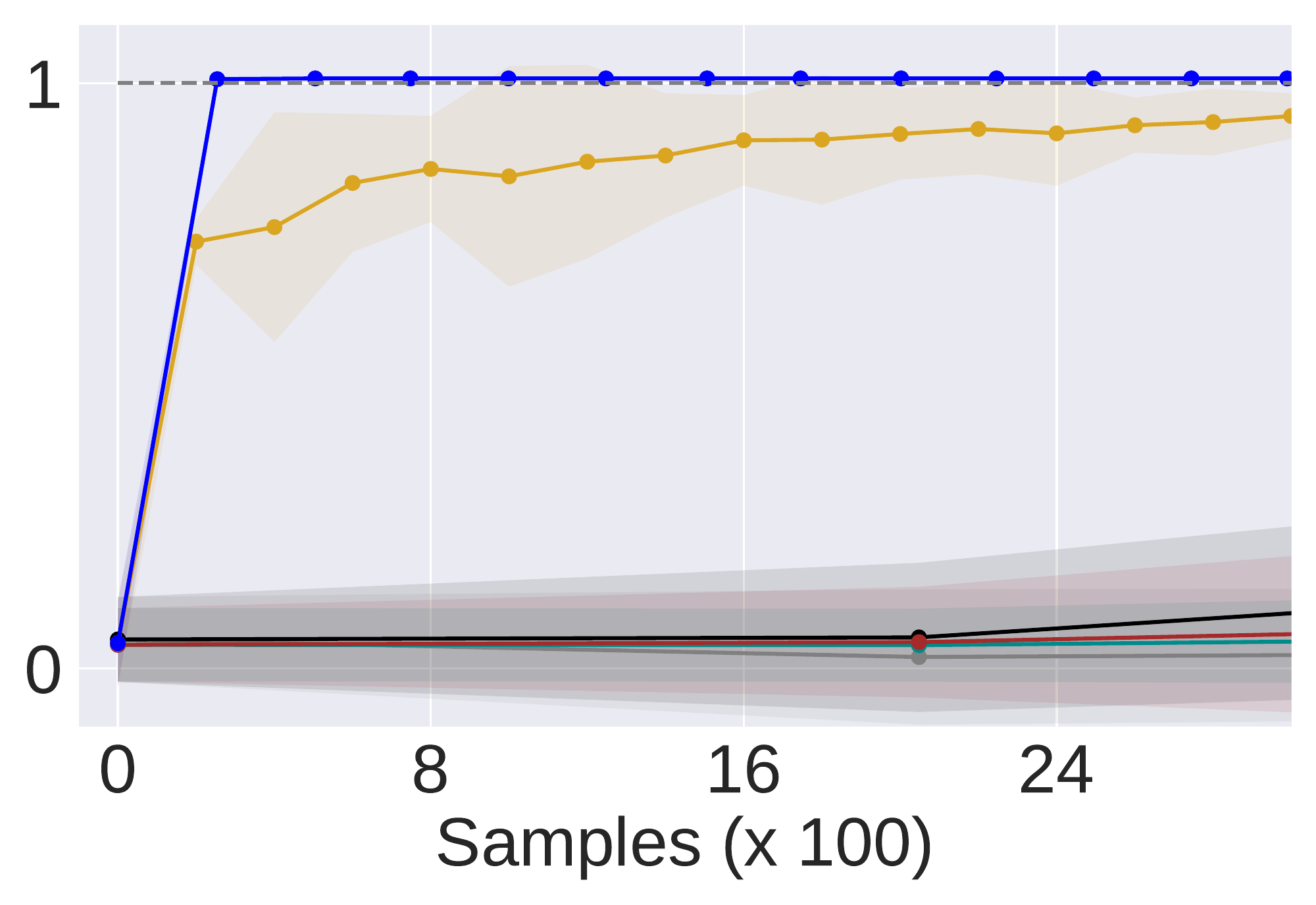}
     } &
\subfloat[Gridworld]{%
       \includegraphics[width=0.2\linewidth]{plot/only_proximal_point/WindyGrid-v0_normalized.pdf}
     } &
     \subfloat[Acrobot]{%
       \includegraphics[width=0.2\linewidth]{plot/only_proximal_point/Acrobot-v1_normalized.pdf}
     } \\ 
    \multicolumn{4}{c}{
       \includegraphics[scale=0.5]{plot/final_paper_legend_horizontal.pdf}
     }
\end{tabular}
\caption{\textbf{Extended Online IL Experiments}. We show the total returns vs the number of env steps. We report the results of some environments omitted in the main text.}
\label{fig:appendix_simple_env_results}
\end{figure*}
\subsection{Refereences for environments description}
In the tabular case we used the environments
 (\texttt{DoubleChain} \cite{Furmston:2010}, \texttt{SingleChain} \cite{Furmston:2010}, \texttt{RiverSwim} \cite{Strehl:2008}, \texttt{WideTree} \cite{Ayoub:2020}, \texttt{Two States Deterministic} \cite{Bagnell:2003}, \texttt{Two States Stochastic} \cite{Bas-Serrano:2021} and \texttt{WindyGrid} \cite{Sutton:2018}). While for the offline setting, we used the environments \texttt{CartPole} \cite{Barto:1983}, \texttt{Acrobot} \cite{Geramifard:2015} and \texttt{LunarLander} \cite{Brockman:2016}. 
The curves are averaged over 50 seeds. For the environments $\texttt{Cartpole}$ and $\texttt{Acrobot}$, we used a three layer neural network to approximate the value function. In these cases we averaged 5 seeds.
\subsection{Hyperparameters}
We report the hyperparameters for the tabular online experiments in \Cref{tab:online_experiments} and for the offline experiments in \Cref{tab:offline_experiments}
\begin{table}[h]
    \centering
    \begin{tabular}{c|c c c c c c}
    \textbf{Environment}& \texttt{n-trajs} & \texttt{lr $\weight$} & \texttt{lr $\thv$} & $\eta$ & $\alpha$ & \texttt{optimizer}
    \\
    \hline \\
    \textbf{TwoStateStochastic-v0} & 25 & 0.5 & 0.5 & 10 & 1 & \texttt{FoRB} \\
    \textbf{TwoStateStochastic-v0} & 25 & 0.5 & 0.5 & 10 & 1 & \texttt{Adam} \\
    \textbf{WideTree-v0} & 25 & 0.5 & 0.5 & 10 & 1 & \texttt{FoRB} \\
    \textbf{RiverSwim-v0} & 50 & 0.2 & 0.2 & 10 & 1 & \texttt{FoRB} \\
    \textbf{WindyGrid-v0} & 50 & 0.5 & 0.01 & 10 & 1 & \texttt{FoRB}
    \\
    \textbf{SingleChainProblem-v0} & 50 & 0.3 & 0.005 & 10 & 1 & \texttt{Adam}
    \\
    \textbf{DoubleChainProblem-v0} & 50 & 0.5 & 0.005 & 10 & 1 & \texttt{Adam} \\
    \\\hline
    \end{tabular}
    \caption{Hyperparameters for proximal point imitation learning in tabular experiments. \texttt{FoRB} stands for Forward Reflected Backward \cite{malitsky2020forward}.} 
    \label{tab:online_experiments}
\end{table}

\begin{table}[h]
    \centering
    \begin{tabular}{c|c c c c c}
    \textbf{Environment}&\texttt{lr $\weight$} & \texttt{lr $\thv$} & $\eta$ & $\alpha$ & \texttt{optimizer}
    \\
    \hline \\
    \textbf{CartPole-v1} & $5e-3$ & $5e-3$ & 10 & 1 & \texttt{Adam} \\
    \textbf{Acrobot-v1} & $5e-3$ & $5e-3$ & 10 & 1 & \texttt{Adam} \\
    \textbf{LunarLander-v2} & $1e-4$ & $1e-4$ & 10 & 0.01 & \texttt{Adam} \\
    \\\hline
    \end{tabular}
    \caption{Hyperparameters for offline experiments}
    \label{tab:offline_experiments}
\end{table}

\begin{table}[h]
    \centering
    \begin{tabular}{c|c c c c c }
    \textbf{Environment}& \texttt{n-trajs} & \texttt{lr $\weight$} & \texttt{lr $\thv$} & $\eta$ & $\alpha$ \\ 
    \hline \\
    \textbf{TwoStateStochastic-v0} & 25 & 0.5 & 0.5 & 10 & 1 \\ 
    \textbf{TwoStateProblem-v0} & 25 & 0.5 & 0.5 & 10 & 1 \\ 
    \textbf{WideTree-v0} & 25 & 0.5 & 0.5 & 10 & 1 \\ 
    \textbf{RiverSwim-v0} & 25 & 0.5 & 0.01 & 10 & 1 \\ 
    \textbf{WindyGrid-v0} & 50 & 0.5 & 0.0006 & 10 & 1 \\ 
    \textbf{SingleChainProblem-v0} & 50 & 0.03 & 0.05 & 10 & 1 \\ 
    \textbf{DoubleChainProblem-v0} & 50 & 0.03 & 0.025 & 10 & 1 \\ 
    \\\hline
    \end{tabular}
    \caption{Hyperparameters for primal dual mirror descent imitation learning in tabular experiments. As optimizer, we used OGD in all cases.}
    \label{tab:mirror_mescent_experiments}
\end{table}

\subsection{On the data sampling}
In all the experiments, we perform a relaxation of our theoretical scheme. In particular, to increase the sample efficiency we sample state action pairs from the Markovian stream of experience.
Analyzing this setting is an open problem.

\subsection{Offline experiments setting}
We consider a training environment and a test environment with different random seeds. We train both IQLearn and Proximal Point for $2e5$ environment steps and we evaluate the policy running $10$ episodes on the evaluation environment every $1e3$ steps. We report the maximum evaluation result achieved at the end of training.
We average the seeds from $0$ to $10$ for the results shown in \Cref{fig:offline_experiments}. We use two separate instances of the same architecture as function approximation for the $Q$-values and cost respectively.
Finally, since the algorithm operates offline it has no access to the distribution $\initial$. In order to approximate the term $\innerprod{\initial}{\val}$, we use the Bellman flow constraints and the fact that the expert occupancy measure is feasible, i.e. $(1 - \gamma)\innerprod{\initial}{\val} = \innerprod{\mv_{\pi_E}}{-\gamma\pmat\val + \bmat \val}$ where the last term can be estimated from the expert samples.

\subsection{Comparison with mirror descent}\label{sec:comparison}
We designed also a mirror descent scheme with alternating updates for imitation learning, briefly described in Appendix~\ref{sec:mirror-descent}. The best hyperparameters are given in \Cref{tab:mirror_mescent_experiments}. Furthermore, we show a comparison with our proximal point scheme in \Cref{fig:ours_comparison}. It is interesting to notice that mirror descent and proximal point have been used interchangeably in the RL literature. Indeed, in that case the objective is linear therefore the two algorithms coincide. However, when considering the max-form objective in imitation learning the equivalence between mirror descent and proximal point does not longer hold true. We verify numerically that PPM outperforms mirror descent in simple tabular environments (see \Cref{fig:ours_comparison}).
\begin{figure*}[t] 
\centering
\begin{tabular}{cccc}
\subfloat[Two State Deterministic]{%
       \includegraphics[width=0.2\linewidth]{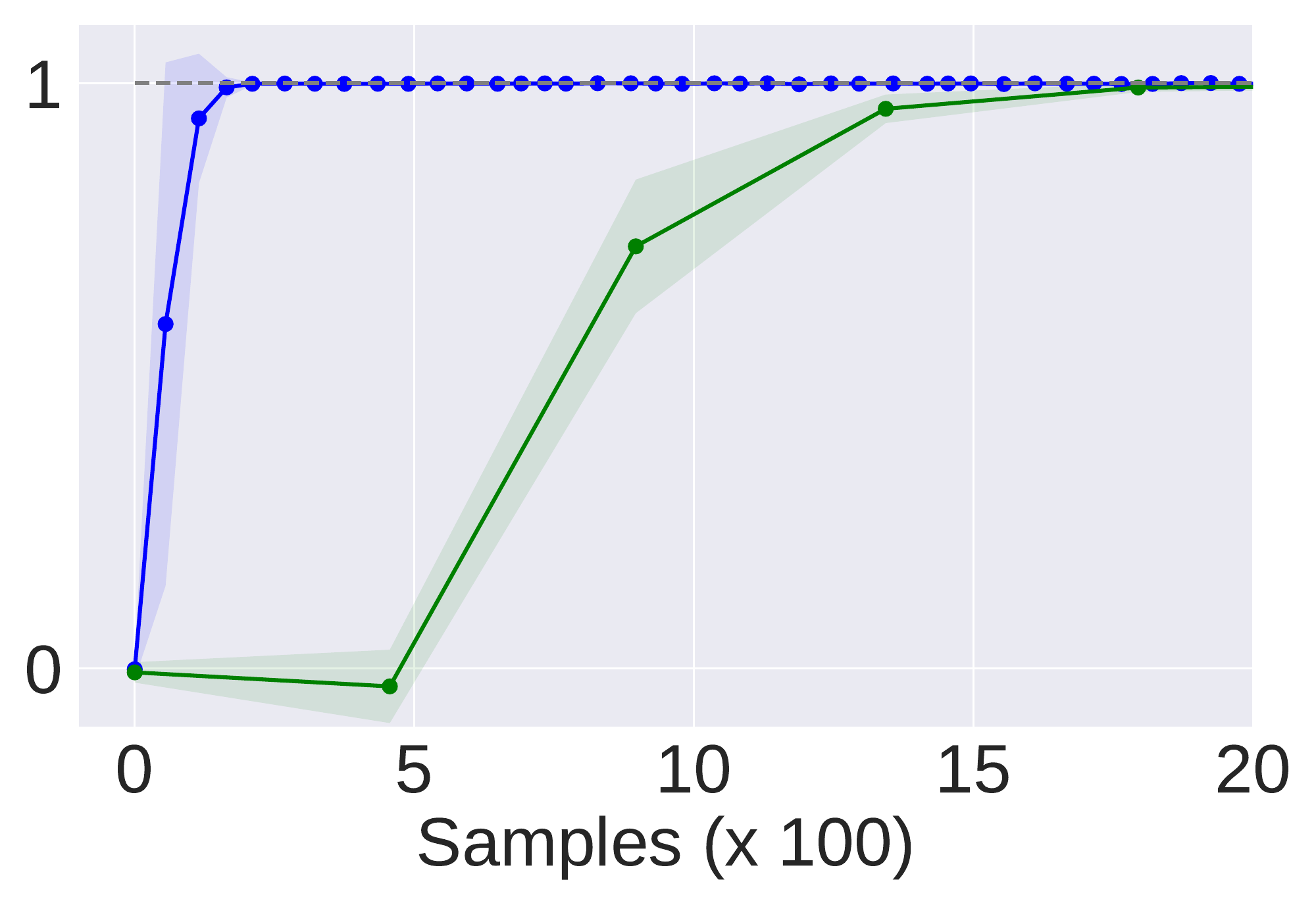}
     } &
\subfloat[WideTree]{%
       \includegraphics[width=0.2\linewidth]{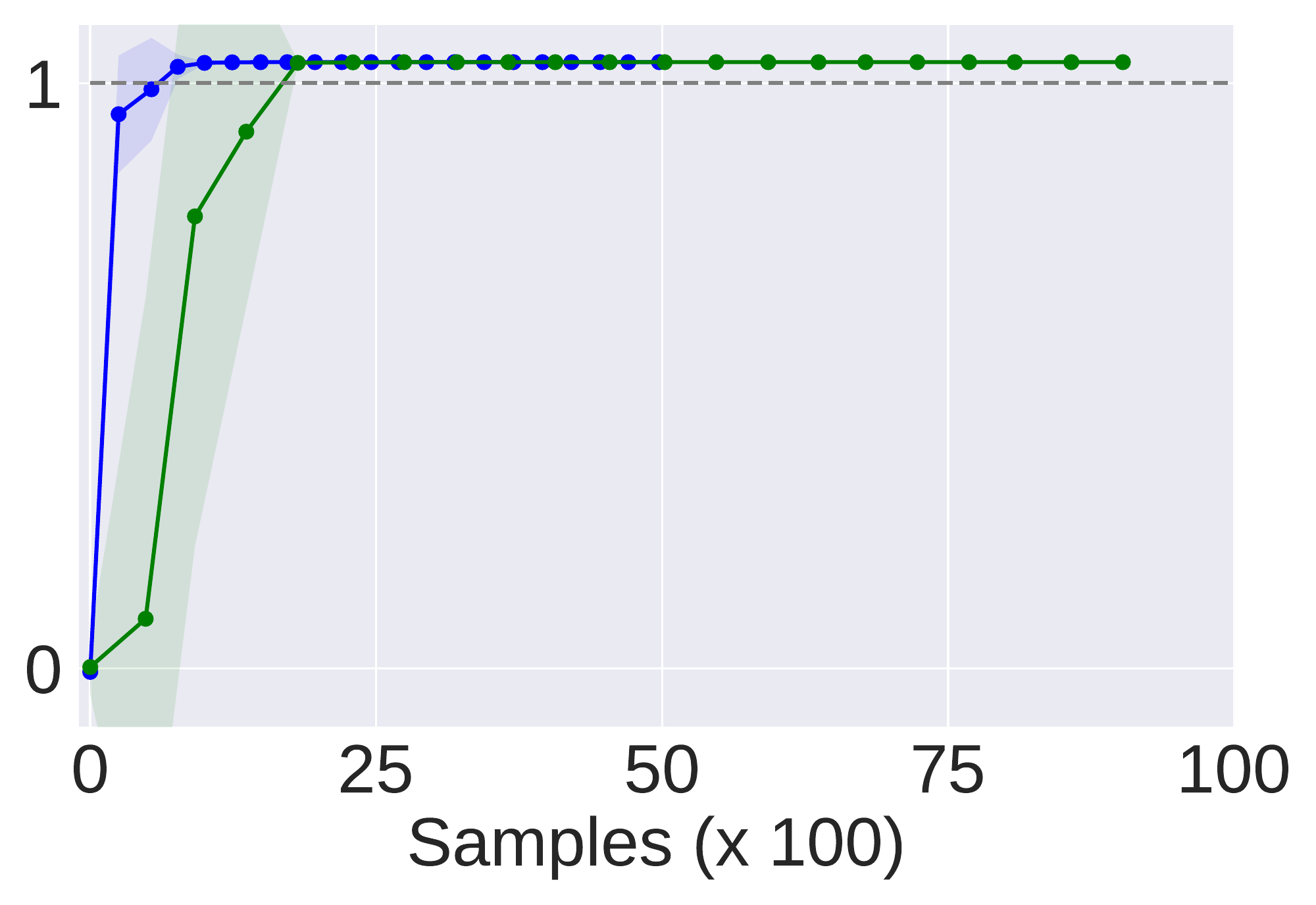}
     } &
\subfloat[RiverSwim]{%
       \includegraphics[width=0.2\linewidth]{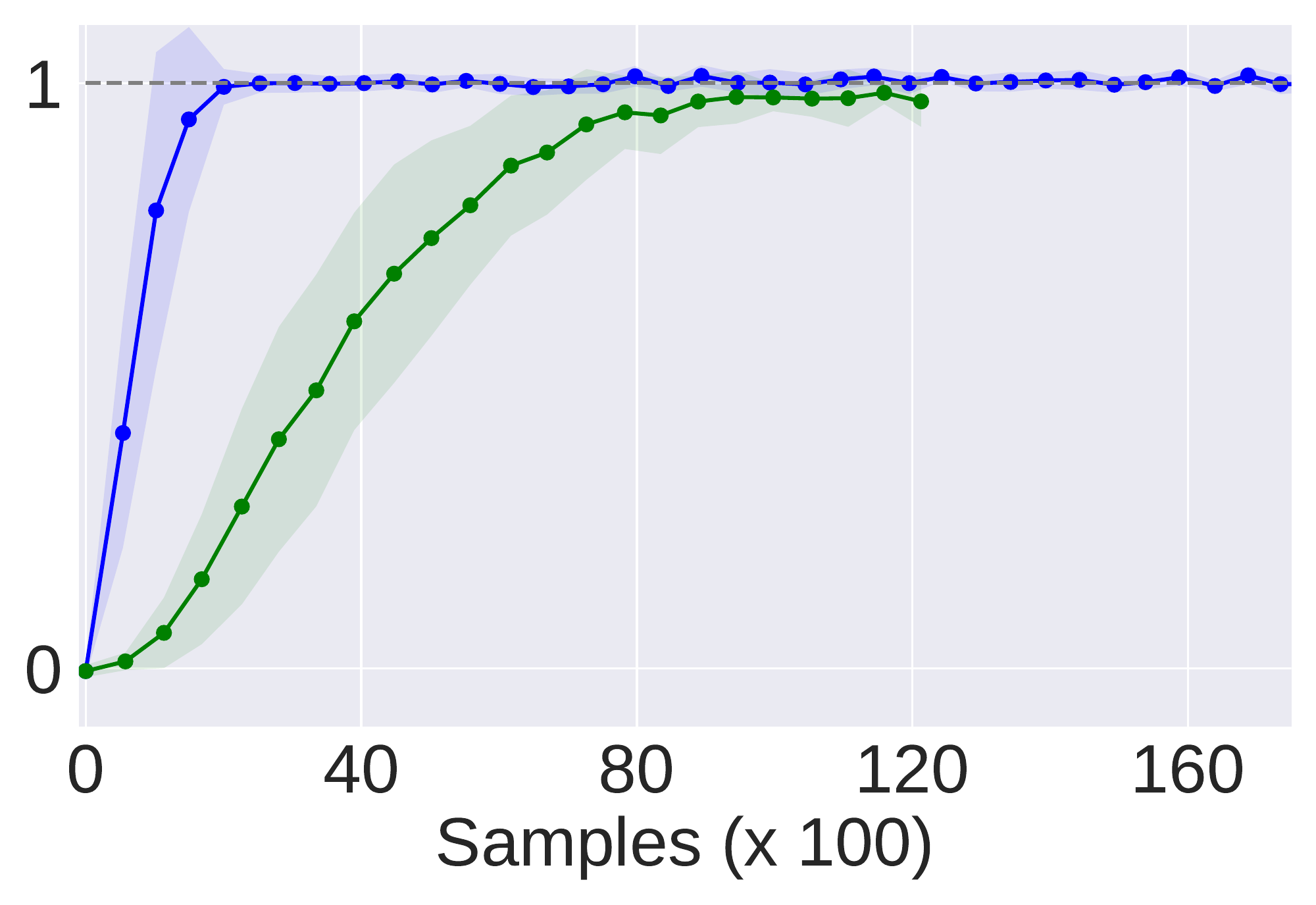}
     } &
\subfloat[SingleChain]{%
       \includegraphics[width=0.2\linewidth]{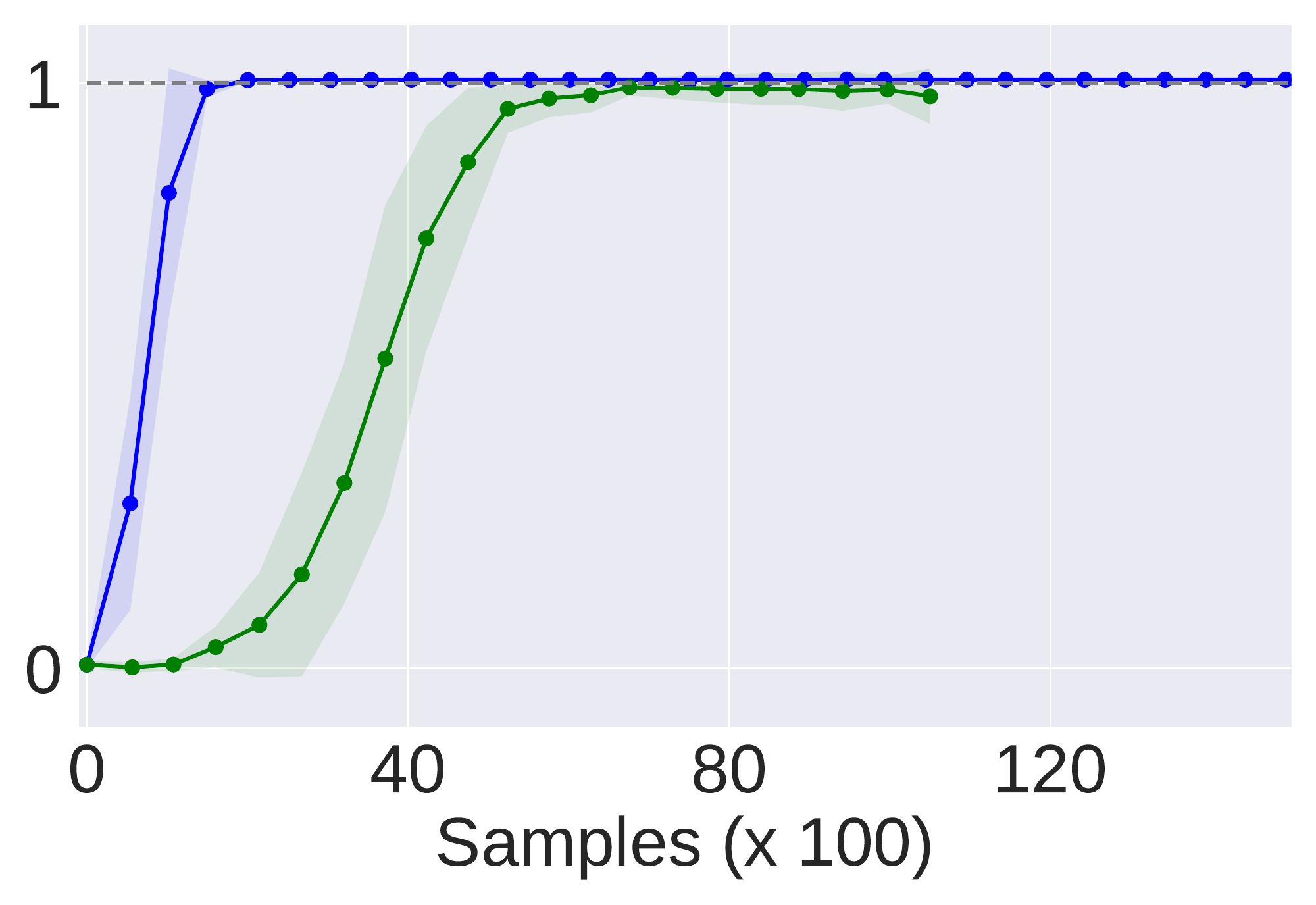}
     } \\
\subfloat[DoubleChain]{%
       \includegraphics[width=0.2\linewidth]{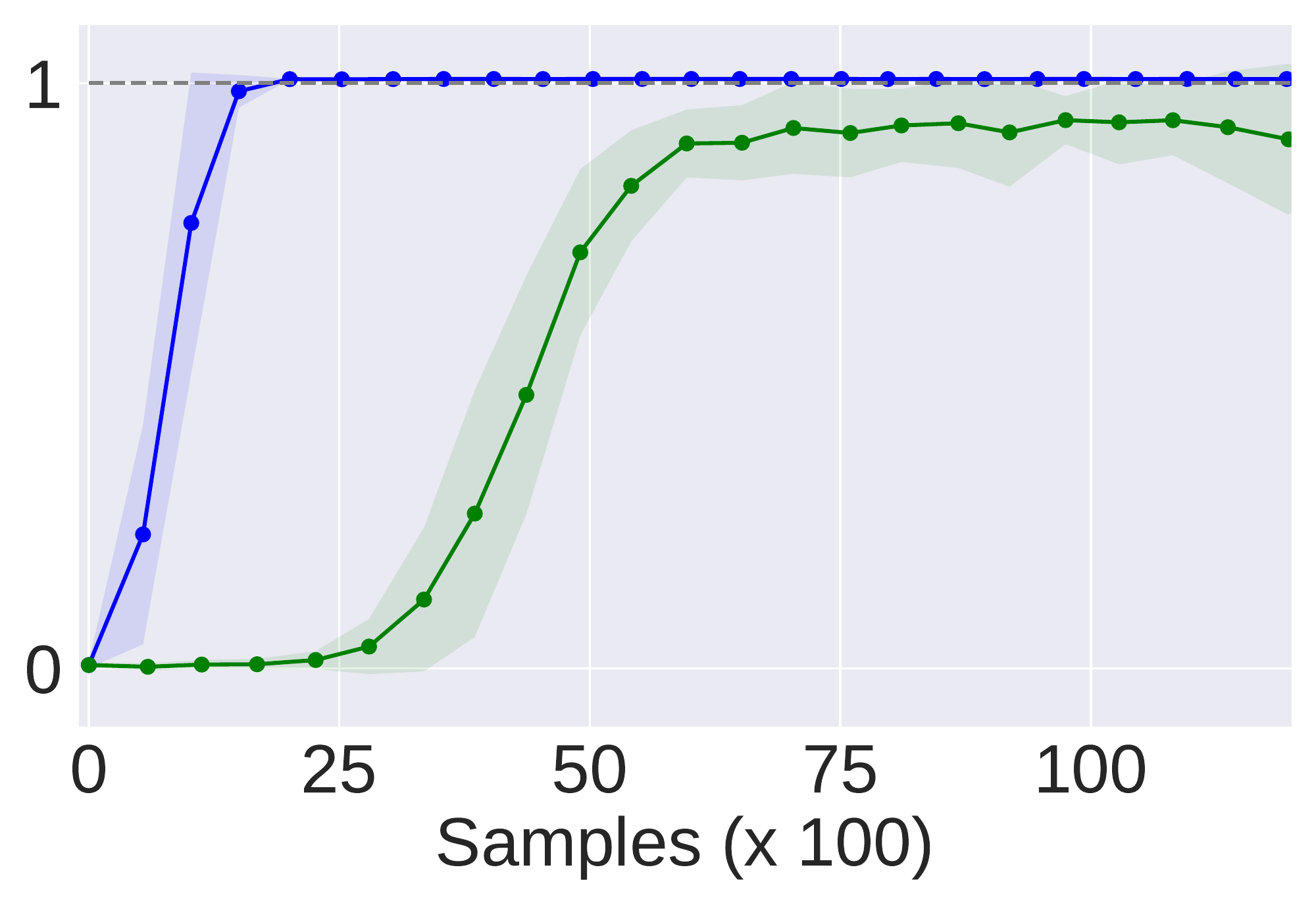}
     } &
\subfloat[Two State Stochastic]{%
       \includegraphics[width=0.2\linewidth]{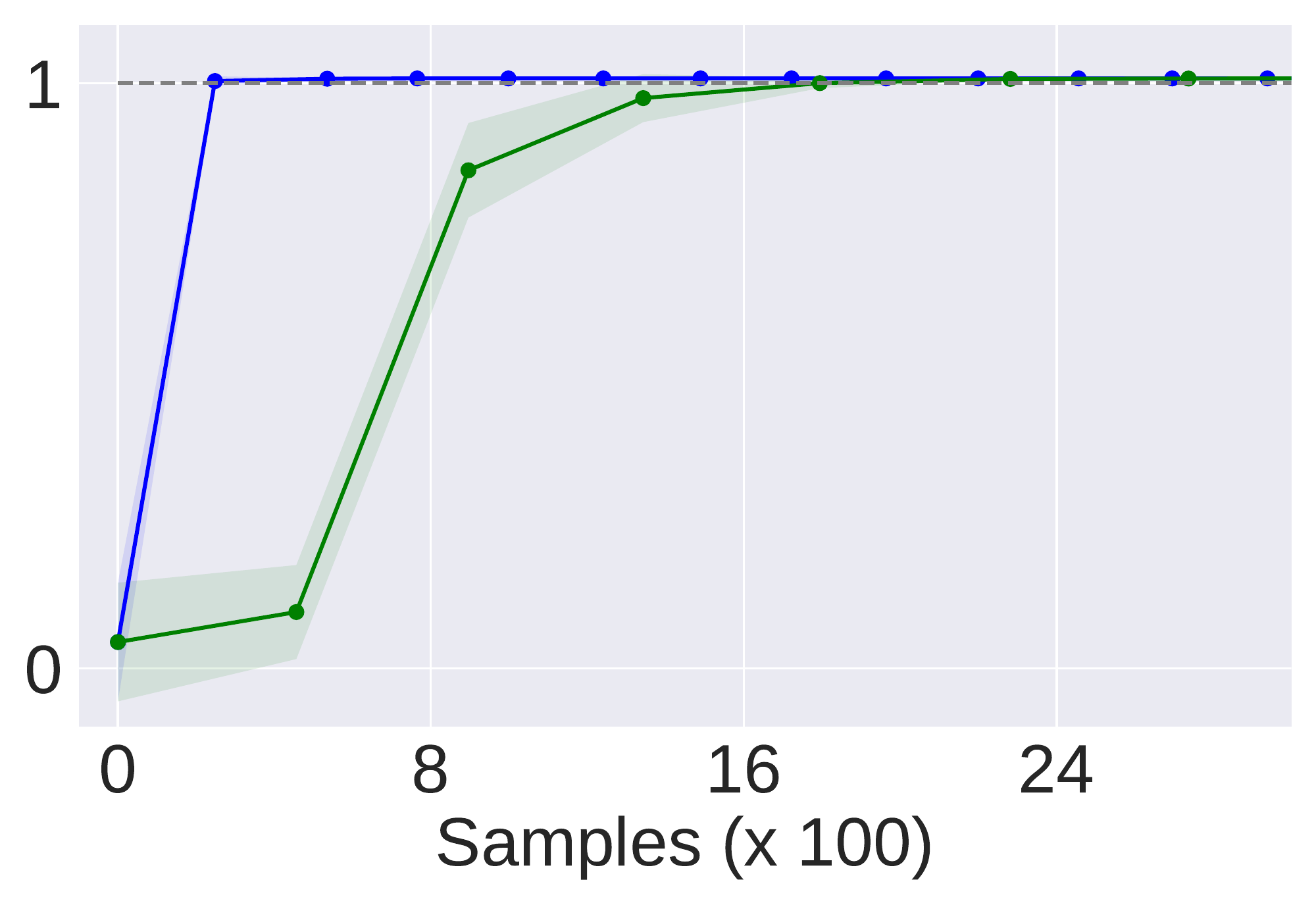}
     } &
\subfloat[Gridworld]{%
       \includegraphics[width=0.2\linewidth]{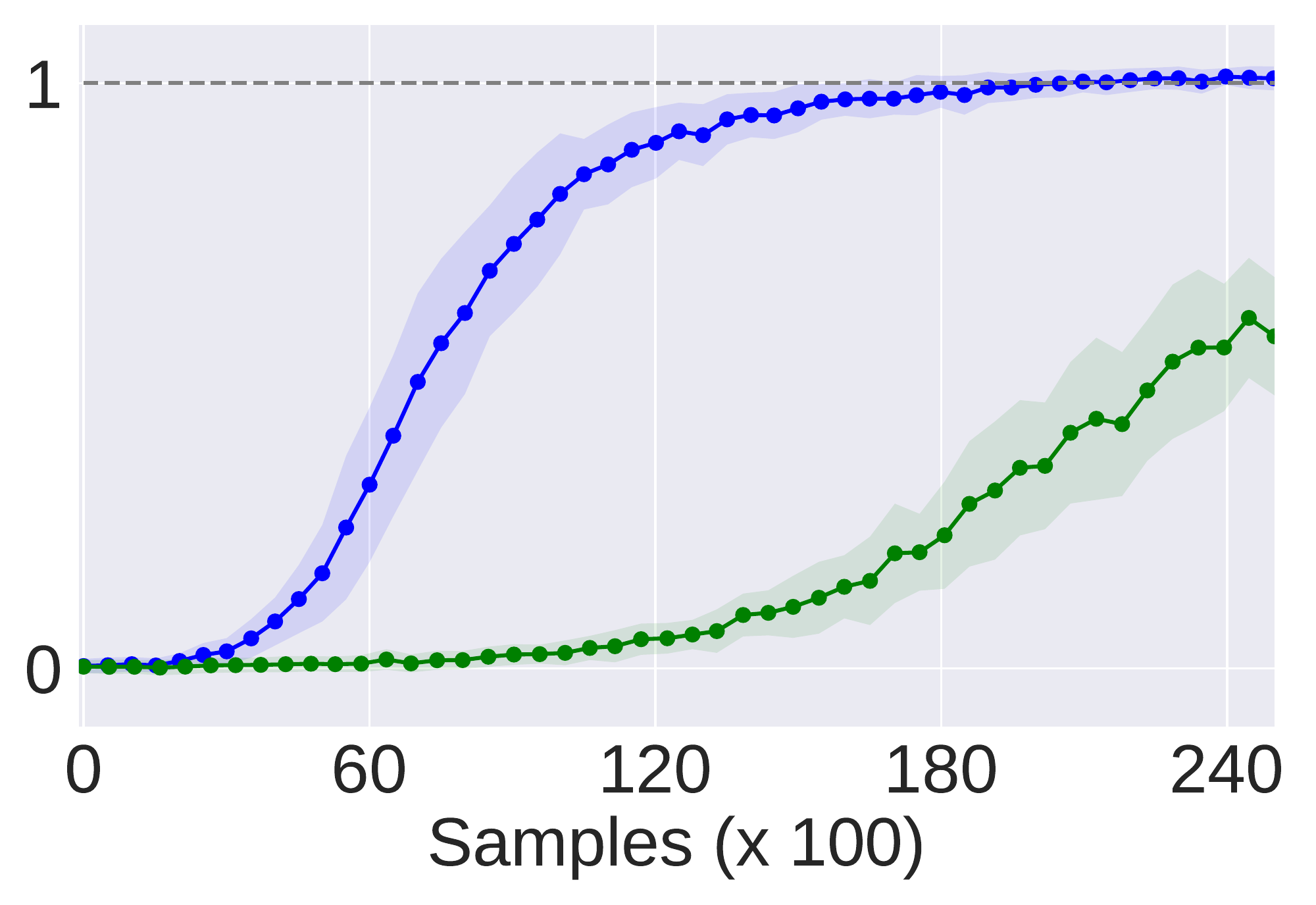}
     } &
\subfloat[Legend]{%
       \includegraphics[scale=0.5]{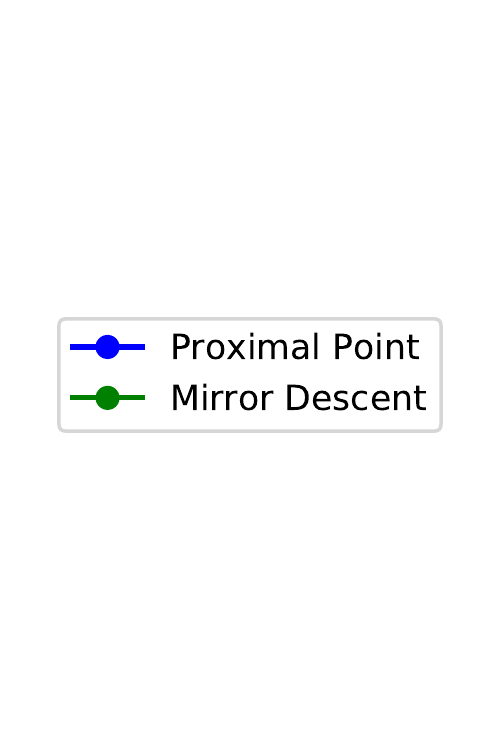}
     } \\
\end{tabular}
\caption{\textbf{Proximal Point vs Mirror Descent.} Comparison of proximal point and mirror descent in tabular domains. Averages of 10 seeds.}
\label{fig:ours_comparison}
\end{figure*}

%

\subsection{Hyperparameters for Pong (Atari)}
\label{sec:atari_hyper}
We use a convolutional neural network to learn the $Q$ values instead of the linear function approximation class we considered in the theoretical analysis. We set the parameter $\alpha$ to $1e-3$ and $\eta$ to $8e-2$, we used expert samples to approximate expectation with respect to the initial distribution. For optimizing the network we used Adam \cite{Kingma:2015} with learning rate $1e-4$ and defaults value for $\beta_1, \beta_2$ Instead of hard constraints on the euclidean norm of the elements of $\wspace$ we consider a $\ell_2$ penalty to the loss function. As expert trajectories we used the dataset released by \cite{Garg:2021}. This is the only hyperparameters configuration we tried using a single seed (using the seed $0$) on our method because of the high computation requirements of this environment.

\subsection{Hyperparameters for MuJoCo (continuous control)}
\label{sec:mujoco_hyper}
 The policy network outputs a distribution over continuous action and is parametrized by independent gaussian distributions for every component of the continuous action vector. We use a three layer neural network to estimate their means and variances. We used as center point in the divergence $D$ the expert feature expectation vector.
With further modifications our method can extend also to continuous control tasks in MuJoCo \cite{Todorov:2012}. The main challenge is that the policy improvement step can not be computed in closed form. We therefore approximate it with a SAC architecture as proposed in \cite{Garg:2021}. We set $\alpha$ to $1e-3$, $\eta$ to $8e-2$, the SAC actor learning rate to $3e-5$ using Adam as optimizer using default values of $\beta_1, \beta_2$,for the critic we used  again Adam with learning rate $3e-4$ and default values for $\beta_1, \beta_2$. The actor training of SAC is performed using a transition buffer containing expert and learner data in equal proportion. We used samples from the expert policy to estimate expectations wrt the initial distribution.
We avoid using target networks. We tested our algorithm on both the environment \texttt{Ant} and \texttt{HalfCheetah} using either the data provided in \cite{Garg:2021} or fresh expert data that we generated training experts with PPO \cite{Schulman:2017}. The results are averaged across 5 seeds.
For \texttt{Hopper}, we used a larger SAC actor learning rate equal to $2e-4$ and $\alpha=1e-2$. In addition, we notice that for this environment having a large $\beta_1$ in Adam was harmful. Hence, we used $\beta_1 = 0$.

For \texttt{Walker}, we set the actor learning to $1e-4$.
\subsection{Acknowledging existing assets and license.}
We built on the code and expert data provided in \cite{Garg:2021}. They are open sourced for academic scope according to their GitHub page \url{https://github.com/Div99/IQ-Learn/blob/main/LICENSE.md}.
\subsection{On the importance of the dataset} 
We observed that the performance of our imitation learning algorithm and IQ-Learn can be affected by the choice of the expert data. In particular, in \Cref{fig:mujoco_iq}, we show that IQ-Learn works better with the expert data provided in \cite{Garg:2021}.
\begin{figure}[t]
\label{fig:mujoco_iq}
    \centering
\begin{tabular}{cc}
\subfloat[HalfCheetah-v2]{%
       \includegraphics[width=0.4\linewidth]{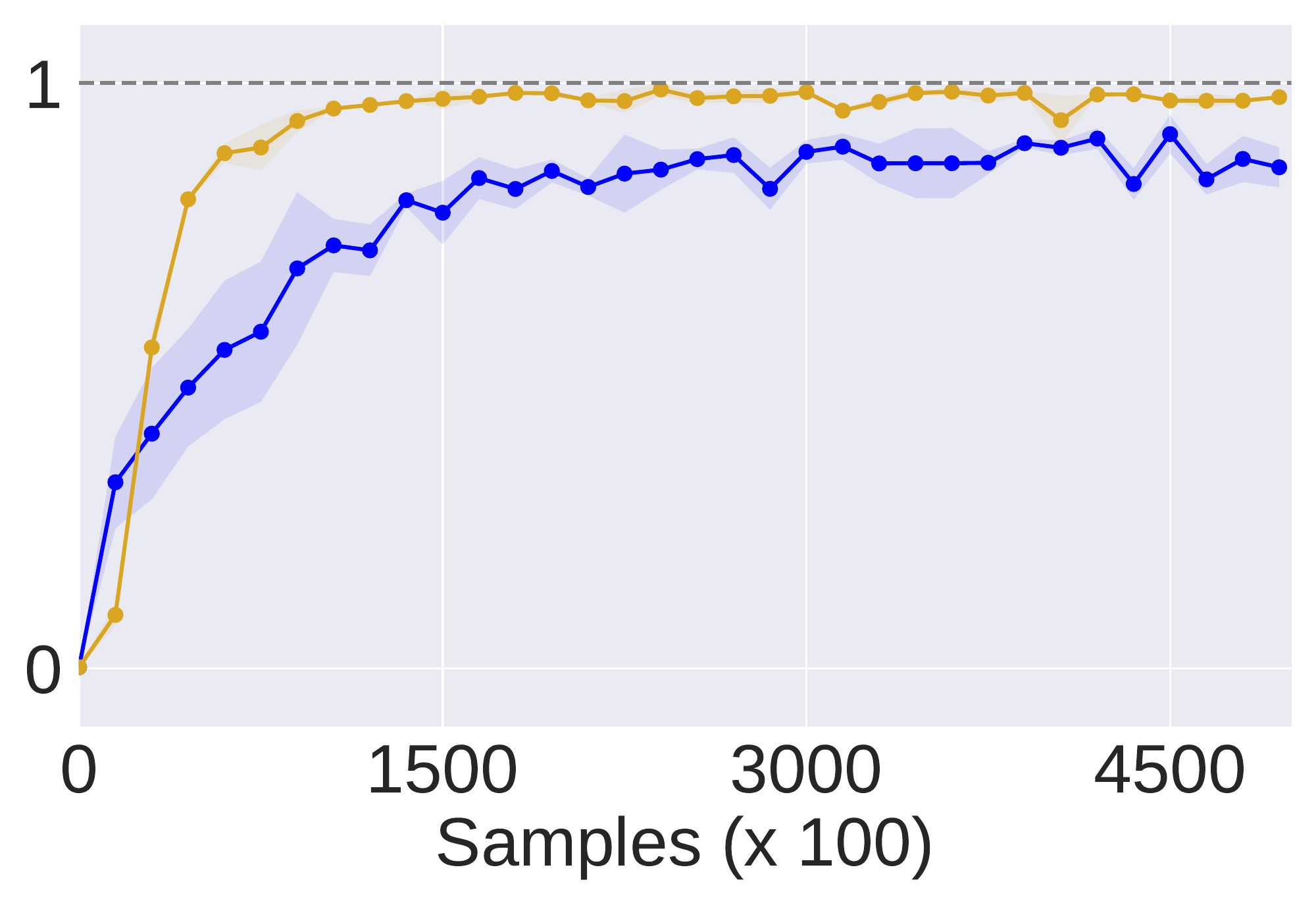}
     } &
     \subfloat[Ant-v2]{%
       \includegraphics[width=0.4\linewidth]{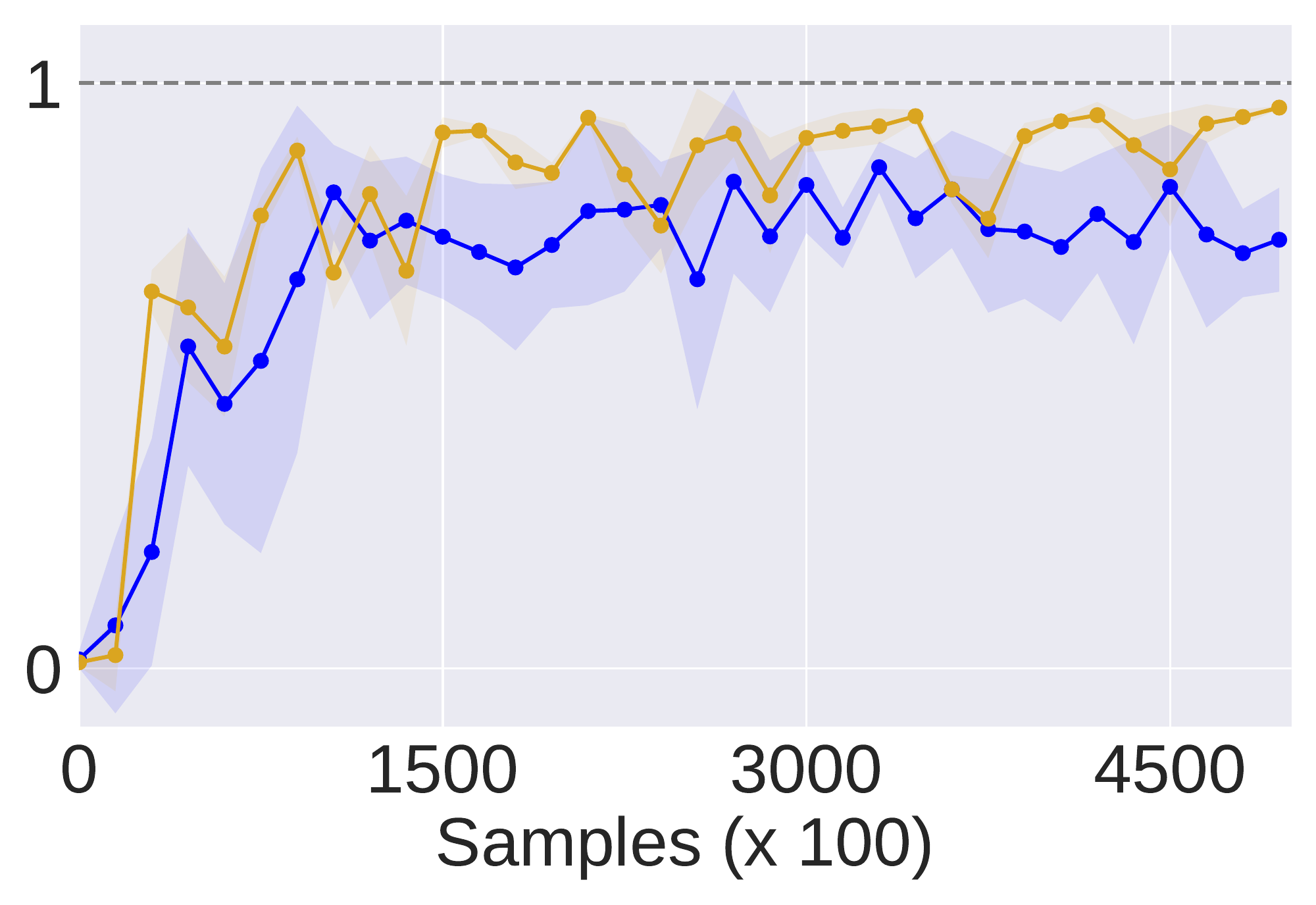}
     }
\end{tabular}
\caption{Experiments in the MuJoCo environments with the expert data provided by~\cite{Garg:2021}. The blue line is proximal point while the yellow line is IQLearn.}
\end{figure}

\subsection{Hardware} We ran the experiments on our internal cluster.
\section{Recovered Costs}\label{app:recovered-rewards}


A unique algorithmic feature of the proposed methodology is that we can explicitly recover a cost along with the $Q$-function without requiring adversarial training. In Figures~\ref{fig:cost} and~\ref{fig:gridworld_cost}, we visualize our recovered costs in several simple tabular environments (\texttt{River Swim}, \texttt{Single Chain}, \texttt{Double Chain}, and \texttt{Gridworld}, respectively). Most importantly, we verify that the recovered costs induce nearly optimal policies w.r.t. the unknown true cost function. Compared to IQ-Learn, the  we do not require knowledge or further interaction with the environment. Therefore, the recovered cost functions show promising transfer capability to new dynamics.

We experimented with a transfer reward setting on a \texttt{Gridworld} (Figure~\ref{fig:transfer_cost}).
We consider two different Gridworld MDP environments, say $M$ and $\widetilde{M}$, with opposite action effects. This means that action \texttt{Down} in $\widetilde{M}$ corresponds to action \texttt{Left} in $M$ and vice versa. Similarly, the effects of \texttt{Up} and \texttt{Right} are swapped between $\widetilde{M}$ and $M$. We denote by $\val^\pi_{\widetilde{M},\cost_{\mathrm{true}}}$ (resp. $\val^\star_{\widetilde{M},\cost_{\mathrm{true}}})$ the value function of policy $\pi$ (resp. optimal value function) in the MDP environment $\widetilde{M}$ with cost function $\cost_{\mathrm{true}}$. Moreover, we denote by $\pi^\star_{{M},\cost}$ the optimal policy in the MDP environment $M$ under cost function $\cost$.  We notice that the recovered cost induces an optimal policy for the new dynamics while the imitating policy fails.  Albeit, cost transfer is successful in this experiment we do not expect this fact to be true in general because we do not tackle the issue of cost shaping \cite{Ng:2000}.

\begin{figure}[t]
    \centering
\begin{tabular}{cccc}
\multicolumn{4}{c}{River Swim}\\
\subfloat[$- \mbf{c}_{\mathrm{true}}$]{%
       \includegraphics[height=4cm]{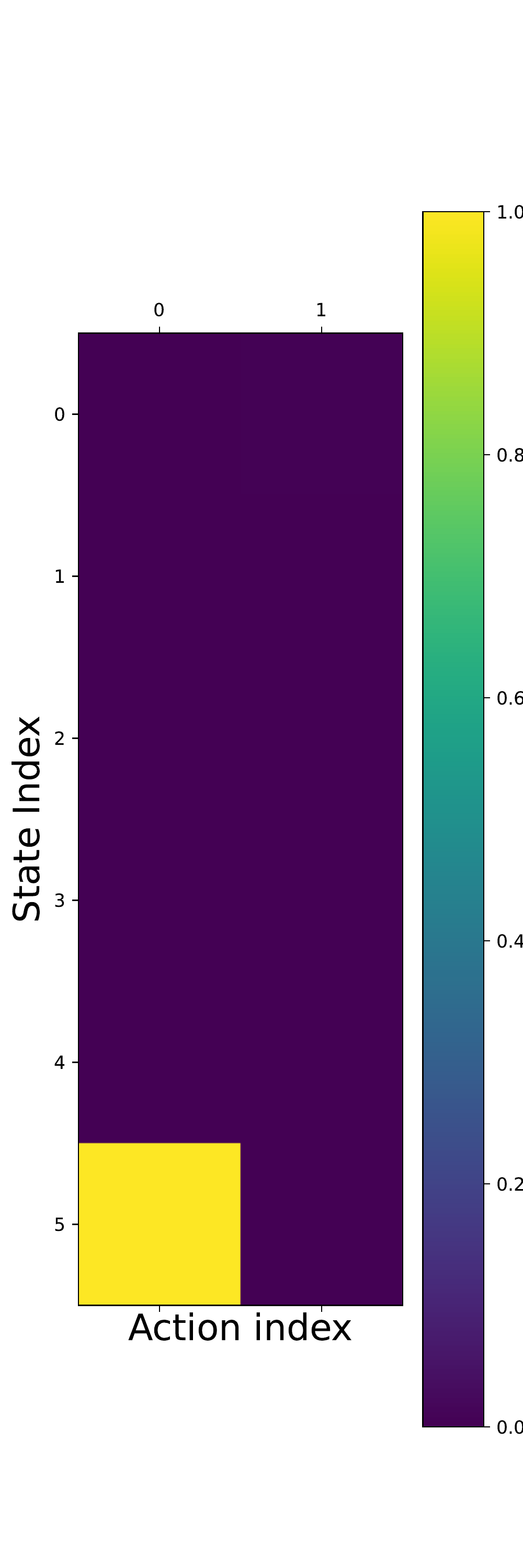}
     } &
     \subfloat[$- \mbf{c}_{K}$]{%
       \includegraphics[height=4cm]{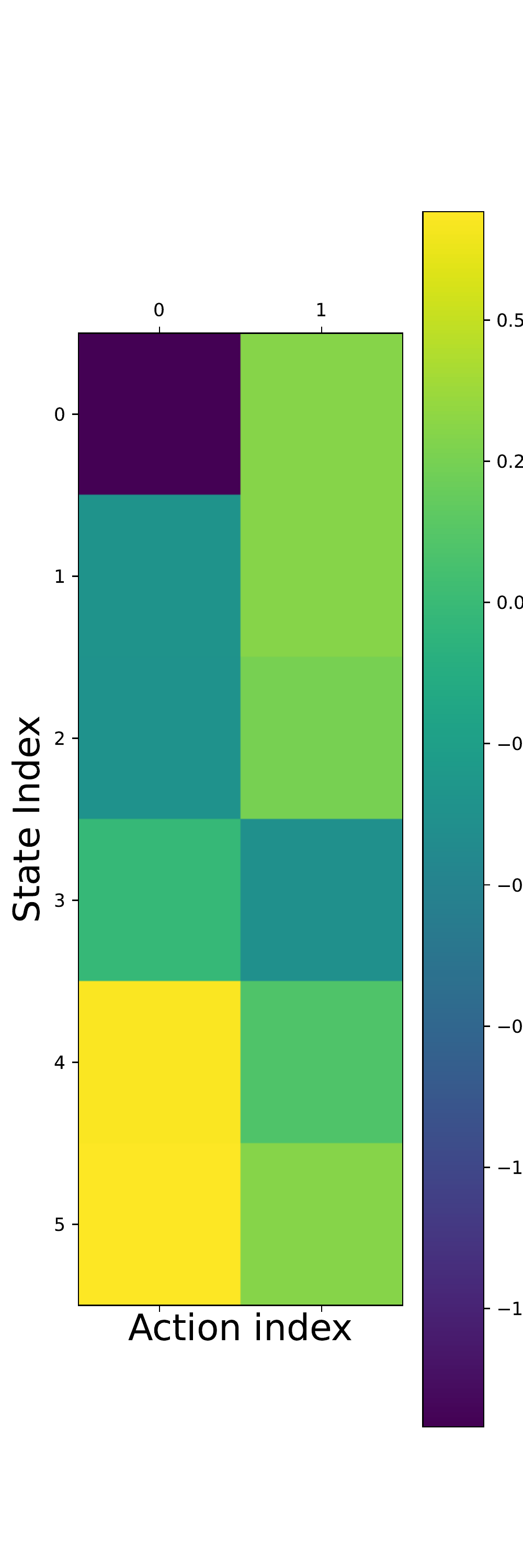}
     } &
     \subfloat[- $V^\star_{\mbf{c}_{\mathrm{true}}}$]{%
       \includegraphics[height=4cm]{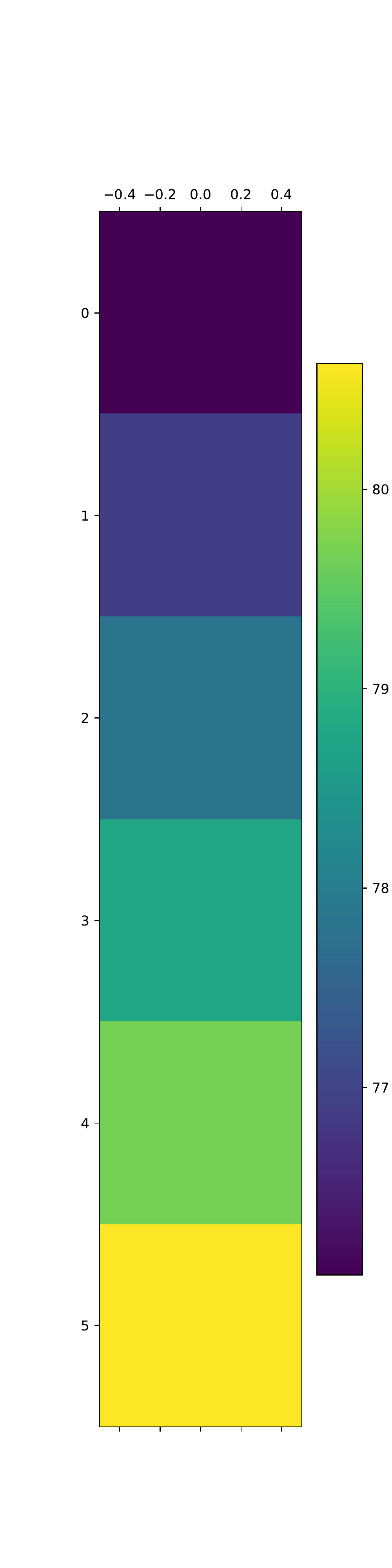}
     } &
     \subfloat[-$V^\star_{\mbf{c}_{K}}$]{%
       \includegraphics[height=4cm]{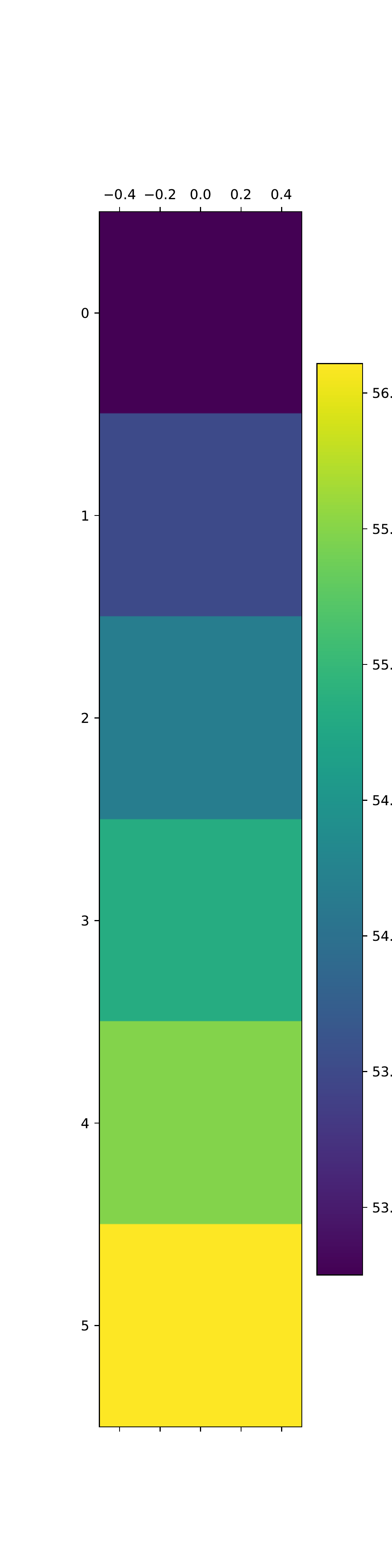}
     } \\ \\
     \multicolumn{4}{c}{Single Chain}\\
\subfloat[$- \mbf{c}_{\mathrm{true}}$]{%
       \includegraphics[height=4cm]{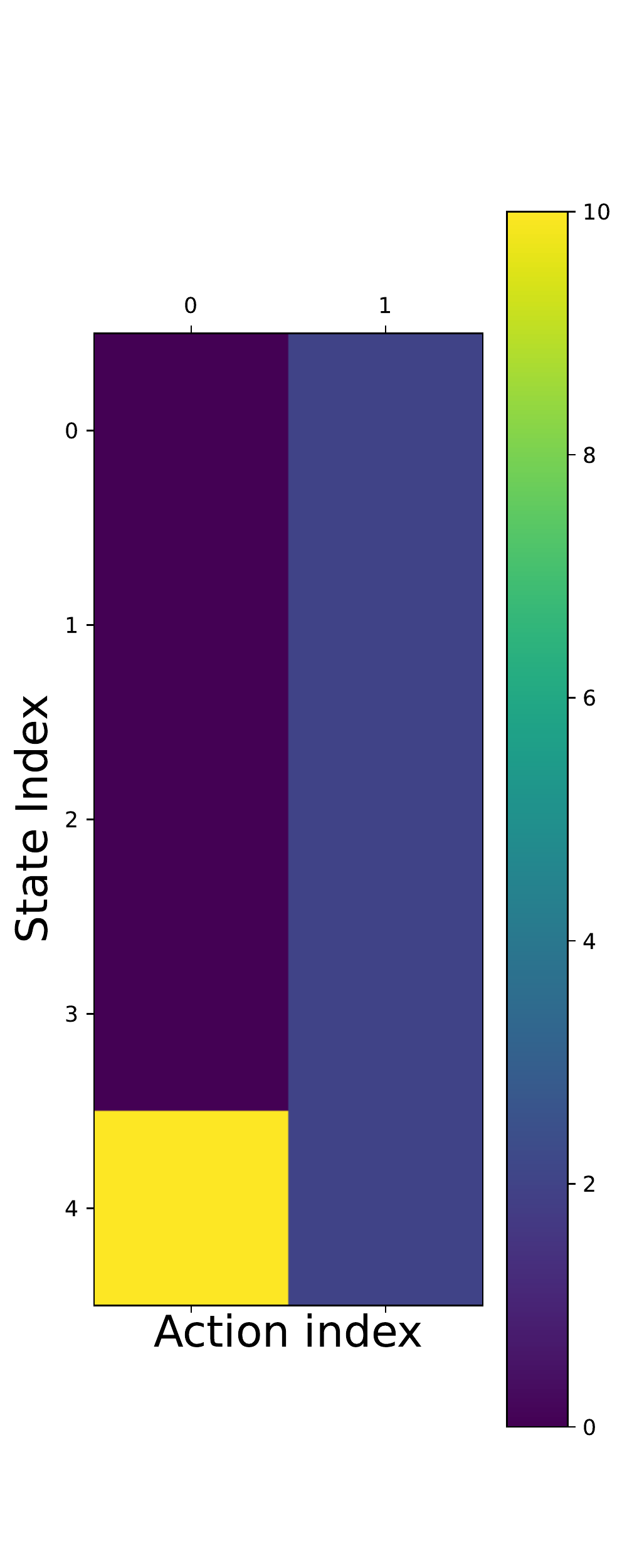}
     } &
     \subfloat[$- \mbf{c}_{K}$]{%
       \includegraphics[height=4cm]{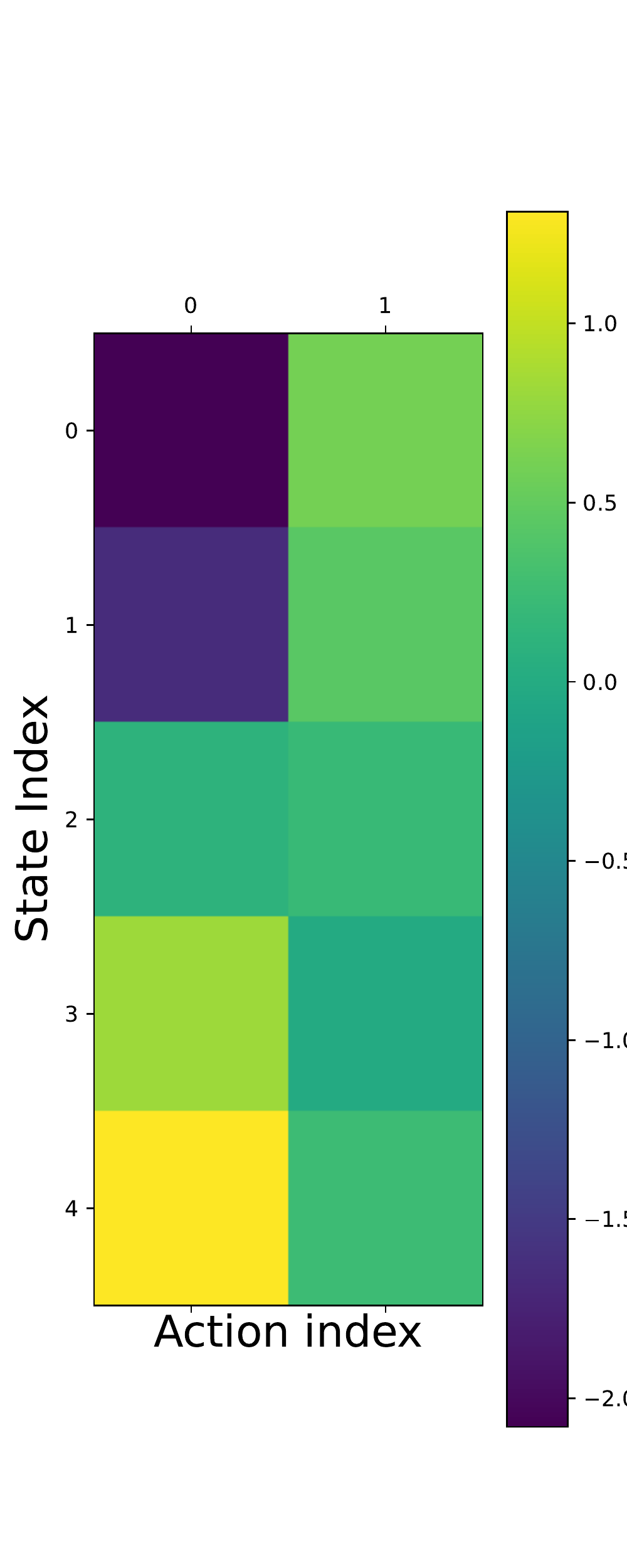}} &
\subfloat[- $V^\star_{\mbf{c}_{\mathrm{true}}}$]{%
       \includegraphics[height=4cm]{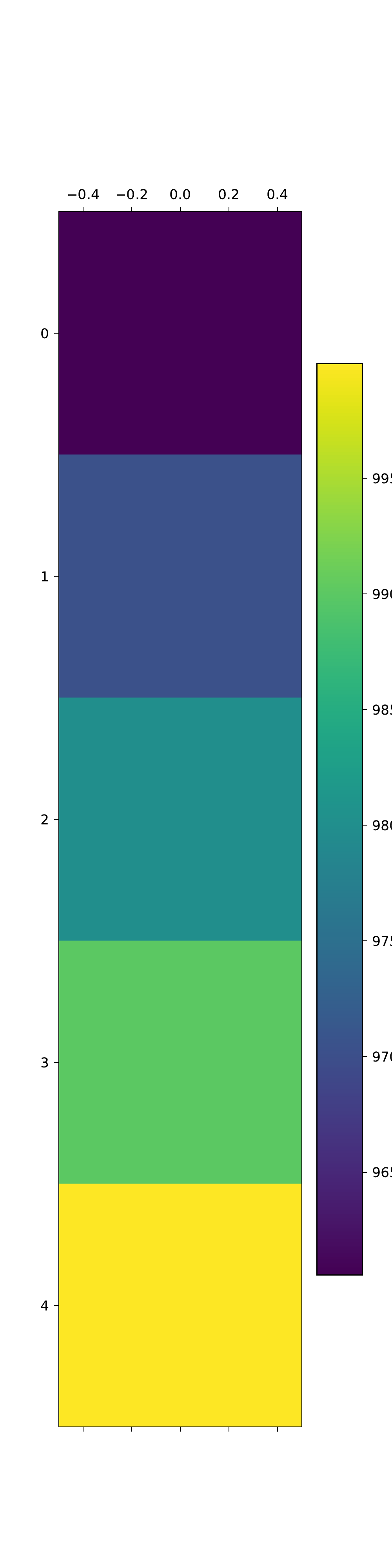}
     } &
     \subfloat[-$V^\star_{\mbf{c}_{K}}$]{%
       \includegraphics[height=4cm]{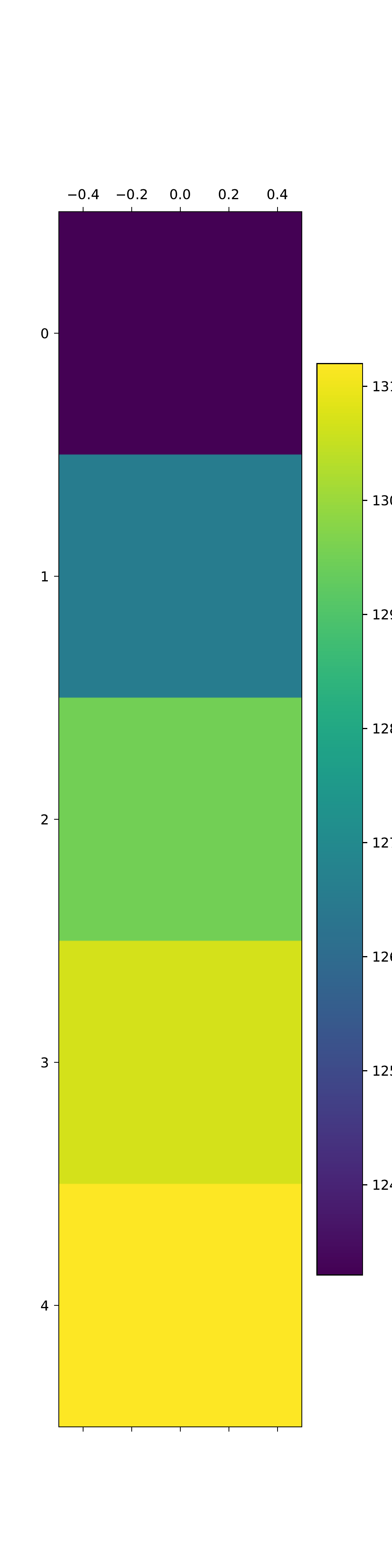}}
      \\ \\
     \multicolumn{4}{c}{Double Chain}\\
      \subfloat[$- \mbf{c}_{\mathrm{true}}$]{%
       \includegraphics[height=4cm]{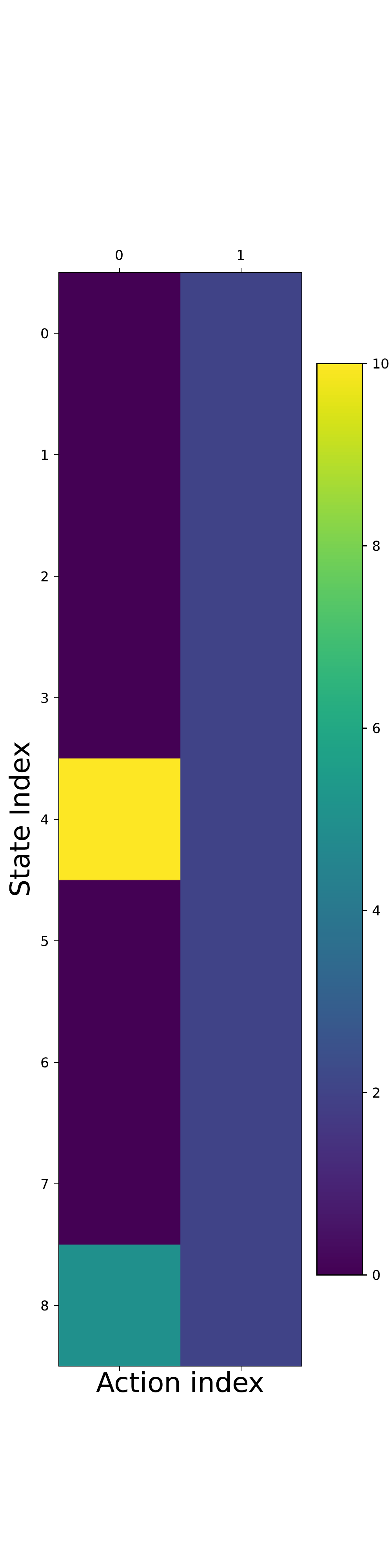}
     } &
     \subfloat[$- \mbf{c}_{K}$]{%
       \includegraphics[height=4cm]{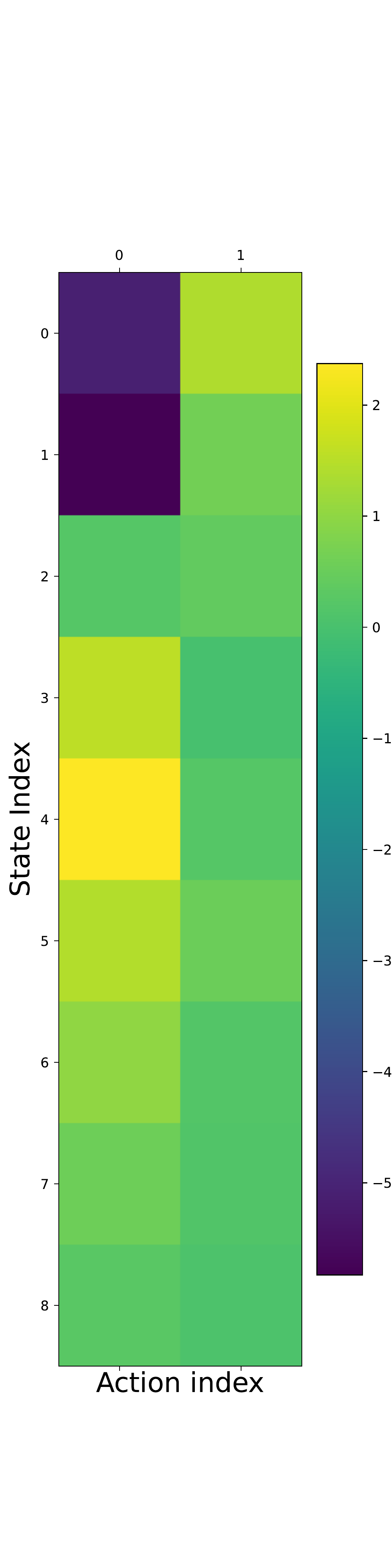}
     } &
     \subfloat[- $V^\star_{\mbf{c}_{\mathrm{true}}}$]{%
       \includegraphics[height=4cm]{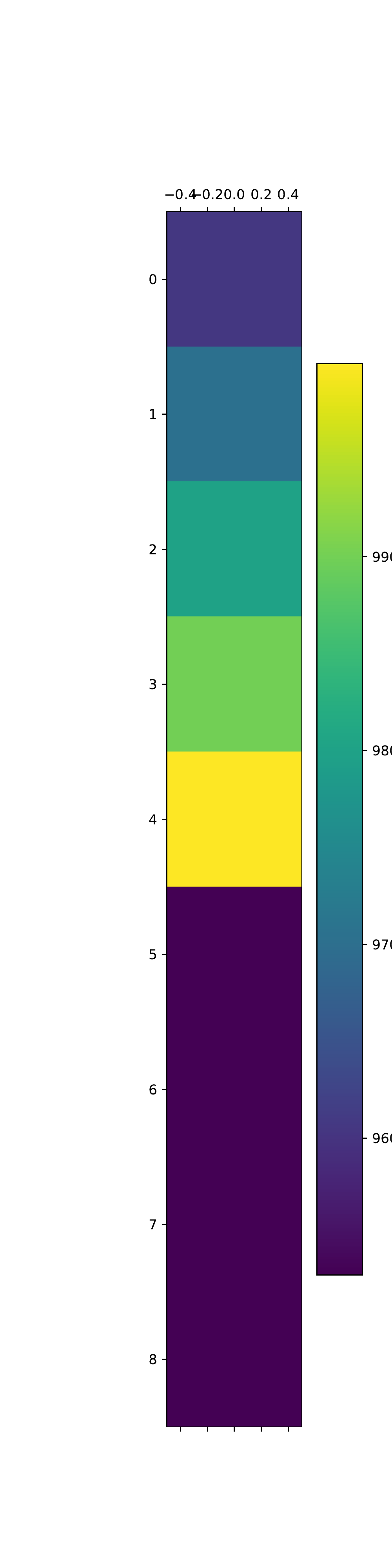}
     } &
     \subfloat[-$V^\star_{\mbf{c}_{K}}$]{%
       \includegraphics[height=4cm]{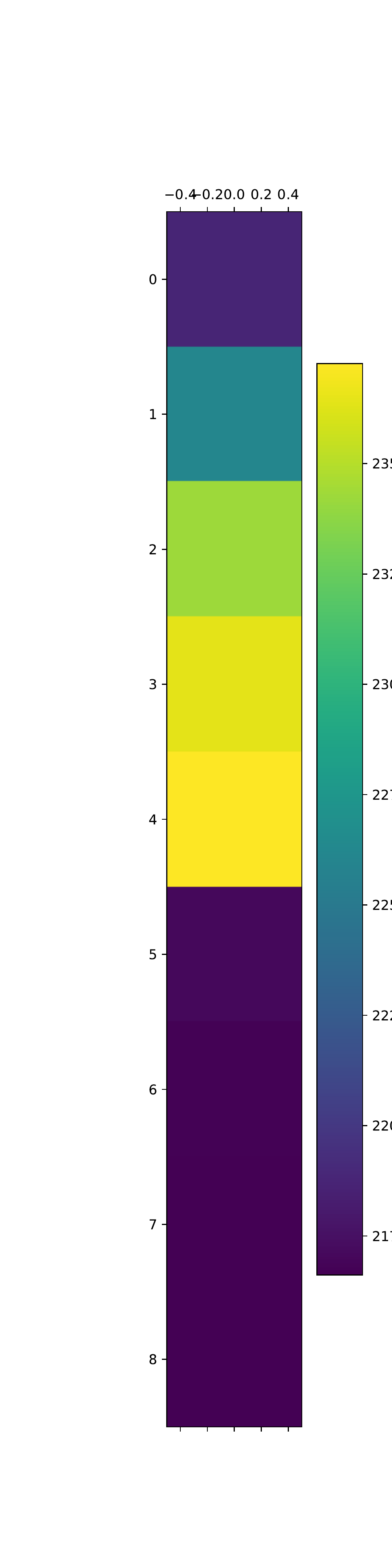}
     }
\end{tabular}
\caption{\textbf{Recovered Costs.} Comparison between the true cost $\cost_{\mathrm{true}}$ and the cost $\cost_K$ recovered by \texttt{P$^2$IL}. We notice that the optimal value functions $V^\star_{\mbf{c}_{\mathrm{true}}}$ and $V^\star_{\mbf{c}_{K}}$ present the same pattern. Hence, the optimal policy with respect to $\mbf{c}_{K}$ is nearly optimal with respect to $\mbf{c}_{\mathrm{true}}$.
\label{fig:cost}}
\end{figure}
\subsection{Preliminary theoretical arguments}
We have some preliminary theoretical arguments justifying the near optimality of the recovered costs/rewards. We present briefly the reasoning.

For brevity, we consider the case $\mathcal{W}=B_1^m$. Then $\pi_{\textup{E}}$ is optimal for the IL problem. Moreover, for simplicity, we consider the case $\boldsymbol{\Phi}=\mathbf{I}$. Otherwise, in the following derivations, we replace $\mathbf{Q}$-values by parameterized $\mathbf{Q}_{\boldsymbol{\theta}}$.

Let $(\widehat{\mathbf{w}}_K,\widehat{\mathbf{Q}}_K)$ be the output (average iterate) of \texttt{P$^2$IL} after $K$ outer loop iterations. We give a sketch of proof that  $\widehat{\mathbf{w}}_K$ converges to an optimal solution to the inverse problem as $K\rightarrow\infty$, i.e., $\widehat{\mathbf{w}}_K$ converges to some $\mathbf{w}_{\textup{A}}\in\mathcal{W}$ such that $\pi_{\textup{E}}$ is optimal for $\mathbf{c}_{\mathbf{w}_{\textup{A}}}$. To this end, we first introduce the following definition.

\begin{defn}
We say that $\mathbf{w}\in\mathcal{W}$ is $\varepsilon_1$-optimal and $\varepsilon_2$-feasible for the~(\ref{eq:dual}) program if-f there exists $\mathbf{V}\in\mathds{R}^{|\mathcal{S}|}$, such that
\begin{eqnarray}
\langle\boldsymbol{\mu}_{\pi_{\textup{E}}},\mathbf{c}_{\mathbf{w}}\rangle-(1-\gamma)\langle\boldsymbol{\nu}_{0},\mathbf{V}\rangle&\le& \varepsilon_1,\\
\mathbf{c}_{\mathbf{w}}-(\mathbf{B}-\gamma\mathbf{P})\mathbf{V}&\geq& -\varepsilon_2\mathbf{1}.
\end{eqnarray}
In this case, $\mathbf{V}\in\mathds{R}^{|\mathcal{S}|}$ is called a certificate.
\end{defn}
Note that the definition of $\varepsilon_1$-optimality for the ~(\ref{eq:dual}) program follows from the fact that the dual optimal value is $\zeta^\star=0$. Moreover, in the definition of $\varepsilon_2$-feasibility we have relaxed the nonnegativity constraint in the dual program~(\ref{eq:dual}). We make the following conjecture.

\textbf{Conjecture:} For a sufficiently large number of samples $N=\mathcal{O}\Big(\textup{poly}\big(\frac{1}{\varepsilon},\log(\frac{1}{\delta}\big),m)\Big)$, with probability at least $1-\delta$, the output cost weight $\widehat{\mathbf{w}}_K$ is $\varepsilon$-optimal and $\varepsilon$-feasible for the~(\ref{eq:dual}) program, with certificate the corresponding logistic value function $\mathbf{V}_{\widehat{\mathbf{Q}}_K}$.

This is easy to show for the exact PPM updates, since $(\mathbf{d}_{\widehat{\pi}_K},\mathbf{d}_{\widehat{\pi}_K},\widehat{\mathbf{w}}_K,\mathbf{V}_{\widehat{\mathbf{Q}}_K},\widehat{\mathbf{Q}}_K)$ is a saddle-point of the (SPP). The proof needs much more effort for the inexact updates used in the sampling-based algorithm.

\begin{lemma}
Assume that $\widetilde{\mathbf{w}}$ is $\varepsilon_1$-optimal and $\varepsilon_2$-feasible for the~(\ref{eq:dual}) program. Then, $\pi_{\textup{E}}$ is $(\varepsilon_1+\varepsilon_2)$-optimal for $\mathbf{c}_{\widetilde{\mathbf{w}}}$.
\end{lemma}
\begin{proof}
There exists $\mathbf{\widetilde{V}}\in\mathds{R}^{|\mathcal{S}|}$, such that
\begin{eqnarray}
\langle\boldsymbol{\mu}_{\pi_{\textup{E}}},\mathbf{{c}}_{\mathbf{\widetilde{w}}}\rangle-(1-\gamma)\langle\boldsymbol{\nu}_{0},\mathbf{\widetilde{V}}\rangle&\le& \varepsilon_1,\\
\mathbf{c}_{\mathbf{\widetilde{w}}}-(\mathbf{B}-\gamma\mathbf{P})\mathbf{\widetilde{V}}&\geq& -\varepsilon_2\mathbf{1}.
\end{eqnarray}
Let $\widetilde{\pi}$ be an optimal policy for $\mathbf{c}_{\widetilde{\mathbf{w}}}$. 
Then, we have that
$$\big\langle\boldsymbol{\mu}_{\widetilde{\pi}},\mathbf{c}_{\mathbf{\widetilde{w}}}-(\mathbf{B}-\gamma\mathbf{P})\mathbf{\widetilde{V}}\big\rangle\geq -\varepsilon_2 \langle\boldsymbol{\mu}_{\widetilde{\pi}},\mathbf{1}\rangle=-\varepsilon_2.$$
By using that $(\mathbf{B}-\gamma\mathbf{P})^{\mathsf{T}}\boldsymbol{\mu}_{\widetilde{\pi}}=(1-\gamma)\boldsymbol{\nu}_0$, we equivalently that
$$
\langle\boldsymbol{\mu}_{\widetilde{\pi}},\mathbf{c}_{\mathbf{\widetilde{w}}}\rangle-(1-\gamma)\langle\boldsymbol{\nu}_{0},\mathbf{\widetilde{V}}\rangle\geq-\varepsilon_2.
$$
Therefore,
$$\langle\boldsymbol{\mu}_{\textup{E}},\mathbf{c}_{\widetilde{\mathbf{w}}}\rangle\le(1-\gamma)\langle\boldsymbol{\nu}_{0},\mathbf{\widetilde{V}}\rangle+\varepsilon_1\le\langle\boldsymbol{\mu}_{\widetilde{\pi}},\mathbf{c}_{\mathbf{\widetilde{w}}}\rangle+\varepsilon_1+\varepsilon_2.$$
Thus, $\pi_{\textup{E}}$ is $(\varepsilon_1+\varepsilon_2)$-optimal for $\mathbf{c}_{\widetilde{\mathbf{w}}}$.
\end{proof}

\textbf{Claim:} As $K\rightarrow\infty$ one may approach as closely as desired an optimal solution to the inverse problem. 

\begin{proofof}{Proof for the ideal PPM updates}
We recall that by Proposition~\ref{prop:primal-dual-q-function}, the set of such solutions is characterized as the set of $\weight$-optimizers to~(\ref{eq:dual}).

Let $\widehat{\mathbf{V}}_K=\mathbf{V}_{\widehat{\mathbf{Q}}_K}$.
By the conjecture, for all $K$, we have 
\begin{eqnarray}
\langle\boldsymbol{\mu}_{\pi_{\textup{E}}},\mathbf{c}_{\widehat{\mathbf{w}}_K}\rangle-(1-\gamma)\langle\boldsymbol{\nu}_{0},\widehat{\mathbf{V}}_K\rangle&\le& \varepsilon_K,\label{limit1}\\
\mathbf{c}_{\widehat{\mathbf{w}}_K}-(\mathbf{B}-\gamma\mathbf{P})\widehat{\mathbf{V}}_K&\geq& -\varepsilon_K\mathbf{1} \label{limit2},
\end{eqnarray}
for some sequence $\{\varepsilon_K\}_{K=1}^\infty$ such that $\lim_{K\rightarrow\infty}\varepsilon_K=0$.
The sequence $\{\widehat{\mathbf{w}}_K\}_{K=1}^\infty\subset\mathcal{W}$ is bounded and so there exists a subsequence $\{\widehat{\mathbf{w}}_{K_l}\}_{l=1}^\infty$, such that $\lim_{l\rightarrow\infty}\widehat{\mathbf{w}}_{K_l}=\mathbf{w}_{\textup{A}}$, for some $\mathbf{w}_{\textup{A}}\in\mathcal{W}$.
Similarly, by Proposition~\ref{prop:optimal_theta_bound} the sequence $\{\widehat{\mathbf{V}}_{K_l}\}_{l=1}^\infty$ is bounded and so there exists a subsequence $\{\widehat{\mathbf{V}}_{K_{l_n}}\}_{n=1}^\infty$, such that $\lim_{n\rightarrow\infty}\widehat{\mathbf{V}}_{K_{l_n}}=\mathbf{V}_{\textup{A}}$, for some $\mathbf{V}_{\textup{A}}$. By Equations(\ref{limit1})--(\ref{limit2}), we have that for all $n\in\mathds{N}$,
\begin{eqnarray}
\langle\boldsymbol{\mu}_{\pi_{\textup{E}}},\mathbf{c}_{\widehat{\mathbf{w}}_{K_{l_n}}}\rangle-(1-\gamma)\langle\boldsymbol{\nu}_{0},\widehat{\mathbf{V}}_{K_{l_n}}\rangle&\le& \varepsilon_{K_{l_n}},\\ 
\mathbf{c}_{\widehat{\mathbf{w}}_{K_{l_n}}}-(\mathbf{B}-\gamma\mathbf{P})\widehat{\mathbf{V}}_{K_{l_n}}&\geq& -\varepsilon_{K_{l_n}}\mathbf{1} . 
\end{eqnarray}

Taking $n\rightarrow\infty$, we end up that
\begin{eqnarray}
\langle\boldsymbol{\mu}_{\pi_{\textup{E}}},\mathbf{c}_{{\mathbf{w}}_{\textup{A}}}\rangle-(1-\gamma)\langle\boldsymbol{\nu}_{0},{\mathbf{V}}_{\textup{A}}\rangle&\le& 0,\\
\mathbf{c}_{{\mathbf{w}}_{\textup{A}}}-(\mathbf{B}-\gamma\mathbf{P}){\mathbf{V}}_{\textup{A}}&\geq& 0 .
\end{eqnarray}
Equivalently,
\begin{eqnarray}
\langle\boldsymbol{\mu}_{\pi_{\textup{E}}},\mathbf{c}_{{\mathbf{w}}_{\textup{A}}}\rangle-(1-\gamma)\langle\boldsymbol{\nu}_{0},{\mathbf{V}}_{\textup{A}}\rangle&=& 0,\\
\mathbf{c}_{{\mathbf{w}}_{\textup{A}}}-(\mathbf{B}-\gamma\mathbf{P}){\mathbf{V}}_{\textup{A}}&\geq& 0.
\end{eqnarray}
Therefore, by Proposition~\ref{prop:primal-dual-q-function}, $\pi_{\textup{E}}$ is optimal for $\mathbf{c}_{\mathbf{w}_{\textup{A}}}$.

\end{proofof}

\end{document}